%% file: DRO_OfflineRL.tex
\documentclass[10pt,english]{article}

\usepackage{geometry}
\usepackage[T1]{fontenc}
\usepackage[latin9]{inputenc}
\usepackage{bm}
\usepackage{amsmath}
\usepackage{amssymb} 
\usepackage[unicode=true,
 bookmarks=false,
 breaklinks=false,pdfborder={0 0 1},colorlinks=false]
 {hyperref}
\hypersetup{
 colorlinks,citecolor=blue,filecolor=blue,linkcolor=blue,urlcolor=blue}
\usepackage{xcolor,colortbl}
\definecolor{Gray}{gray}{1}
\usepackage{enumitem}
\makeatletter
\usepackage{amsthm}
\usepackage{cite}  
\usepackage{comment}
\usepackage[sort,compress]{natbib}
\usepackage{booktabs,mathtools}
\usepackage{graphicx}
\usepackage[linesnumbered,ruled,vlined]{algorithm2e}
\usepackage{algorithmic}

\usepackage{multirow}
\usepackage{dsfont}
\usepackage{color}
\usepackage{float}

\definecolor{yxc}{RGB}{255,0,0}
\definecolor{yjc}{RGB}{225,0,100}
\definecolor{ytw}{RGB}{255,69,0}
\definecolor{gen}{RGB}{0,0,200}
\definecolor{lxs}{RGB}{138,43,226}

\definecolor{own_pink}{RGB}{217,25,169}
\definecolor{own_blue}{RGB}{0,100,223}

\newcommand{\Ntrim}{N^{\mathsf{trim}}}
\newcommand{\Nmain}{N^{\mathsf{main}}}
\newcommand{\Naux}{N^{\mathsf{aux}}}
\newcommand{\Dtrim}{\mathcal{D}^{\mathsf{trim}}}

\newcommand{\Dmain}{\mathcal{D}^{\mathsf{main}}}
\newcommand{\Daux}{\mathcal{D}^{\mathsf{aux}}}

\definecolor{own_pink}{RGB}{217,25,169}
\definecolor{own_blue}{RGB}{0,100,223}

\allowdisplaybreaks

\input{macro.tex}

\newcommand{\tpe}{\widehat{\mathcal{T}}^{\sigma}_{\mathsf{pe}}}

\title{Distributionally Robust  Model-Based Offline Reinforcement Learning with Near-Optimal Sample Complexity}

\author{
	Laixi Shi\thanks{Department of Electrical and Computer Engineering, Carnegie Mellon University, Pittsburgh, PA 15213, USA.}\\
	Carnegie Mellon University\\
	\texttt{laixis@andrew.cmu.edu}
	\and
	Yuejie Chi\footnotemark[1] \\ 	 
 	Carnegie Mellon University  \\
	\texttt{yuejiechi@cmu.edu}
	}

\date{August 2022; revised December 2023}

\begin{document}

\theoremstyle{plain} \newtheorem{lemma}{\textbf{Lemma}}
\newtheorem{proposition}{\textbf{Proposition}}\newtheorem{theorem}{\textbf{Theorem}} \newtheorem{assumption}{Assumption}

\theoremstyle{remark}\newtheorem{remark}{\textbf{Remark}}

\maketitle

\input{abstract.tex}


\allowdisplaybreaks

\setcounter{tocdepth}{2}
\tableofcontents

\input{introduction.tex}

\input{formulation.tex}

\input{main-result}

\input{results_infty.tex}

\input{experiment}
\input{conclusion}

\section*{Acknowledgements}
 
This work is supported in part by the grants ONR N00014-19-1-2404, NSF CCF-2106778, DMS-2134080, and CNS-2148212. 
L. Shi is also gratefully supported by the Leo Finzi Memorial Fellowship, Wei Shen and Xuehong Zhang Presidential Fellowship, and
Liang Ji-Dian Graduate Fellowship at Carnegie Mellon University. The authors thank Gen Li, Zhengyuan Zhou, and Nian Si for helpful discussions.

%%%%%%%%%%%%%%%%%%%%%%%%%%%%%%%%%%%%%%%%%%%%%%%%%%%%%%%%%%%%

\appendix

\bibliography{bibfileRL,bibfileDRO}
\bibliographystyle{apalike}

\input{auxiliary-lemmas.tex}

\input{upper-bound-analysis.tex}
\input{auxiliary-upper-bound.tex}

\input{lower-bound-analysis}

\input{upper-bound-analysis_infty.tex}

\input{lower-bound-analysis-infty.tex}

\end{document}

%% file: macro.tex
\DeclareMathOperator{\ind}{\mathds{1}}  % Indicator

\newcommand{\defn}{\coloneqq}

\newcommand{\rhob}{\rho^{\mathsf{b}}}

\newcommand{\myrho}{d^{\mathsf{b}}}

\newcommand{\Cstar}{C^{\star}_{\mathsf{rob}}}

\newcommand{\pib}{\pi^{\mathsf{b}}}

\newcommand{\no}{0} % nominal transition kernel superscript
\newcommand{\ror}{\sigma} % uncertainty set radius
\newcommand{\unb}{\cU} % uncertainty set
\newcommand{\piopt}{\pi^\star} % optimal robust policy
\newcommand{\minp}{P_{\mathsf{min}}} % the minimum of transition kernel P_h(. |s,a}
\newcommand{\minpall}{P_{\mathsf{min}}^\star} % the minimum of all transition kernel P_h(. |s,a}
\newcommand{\DRLCB}{{\sf DRVI-LCB}\xspace} 
\newcommand{\DRVI}{{\sf DRVI}\xspace}

\newcommand{\cb}{c_{\mathrm{b}}}

\newcommand{\cA}{\mathcal{A}}

\newcommand{\cC}{\mathcal{C}}
\newcommand{\cD}{\mathcal{D}}

\newcommand{\cM}{\mathcal{M}}

\newcommand{\cP}{\mathcal{P}}

\newcommand{\cS}{{\mathcal{S}}}
\newcommand{\cT}{{\mathcal{T}}}
\newcommand{\cU}{\mathcal{U}}

\newcommand{\mymid}{\,|\,} 

\usepackage{scalerel,stackengine}
\stackMath
\newcommand\reallywidehat[1]{%
\savestack{\tmpbox}{\stretchto{%
  \scaleto{%
    \scalerel*[\widthof{\ensuremath{#1}}]{\kern-.6pt\bigwedge\kern-.6pt}%
    {\rule[-\textheight/2]{1ex}{\textheight}}%WIDTH-LIMITED BIG WEDGE
  }{\textheight}% 
}{0.5ex}}%
\stackon[1pt]{#1}{\tmpbox}%
}
\newcommand\reallywidecheck[1]{%
\savestack{\tmpbox}{\stretchto{%
  \scaleto{%zh
    \scalerel*[\widthof{\ensuremath{#1}}]{\kern-.6pt\bigwedge\kern-.6pt}%
    {\rule[-\textheight/2]{1ex}{\textheight}}%WIDTH-LIMITED BIG WEDGE
  }{\textheight}% 
}{0.5ex}}%
\stackon[1pt]{#1}{\scalebox{-1}{\tmpbox}}%
}

%% file: abstract.tex
\begin{abstract}

This paper concerns the central issues of model robustness and sample efficiency in offline reinforcement learning (RL), which aims to learn to perform decision making from history data without active exploration. Due to uncertainties and variabilities of the environment, it is critical to learn a robust policy---with as few samples as possible---that performs well even when the deployed environment deviates from the nominal one used to collect the history dataset. 
We consider a distributionally robust formulation of offline RL, focusing on tabular robust Markov decision processes with an uncertainty set specified by the Kullback-Leibler divergence in both finite-horizon and infinite-horizon settings. To combat with sample scarcity, a model-based algorithm that combines distributionally robust value iteration with the principle of pessimism in the face of uncertainty is proposed, by penalizing the robust value estimates with a carefully designed data-driven penalty term. Under a mild and tailored assumption of the history dataset that measures distribution shift without requiring full coverage of the state-action space, we establish the finite-sample complexity of the proposed algorithms. We further develop an information-theoretic lower bound, which suggests that learning RMDPs is at least as hard as the standard MDPs when the uncertainty level is sufficient small, and corroborates the tightness of our upper bound up to polynomial factors of the (effective) horizon length for a range of uncertainty levels. To the best our knowledge, this provides the first provably near-optimal robust offline RL algorithm that learns under model uncertainty and partial coverage.

\end{abstract}

%% file: introduction.tex
%!TEX root = ./DRO_OfflineRL.tex
\section{Introduction}

Reinforcement learning (RL) concerns about finding an optimal policy that maximizes an agent's expected total reward in an unknown environment. A fundamental challenge of deploying RL to real-world applications is the limited ability to explore or interact with the environment, due to resources, time, or safety constraints. Offline RL, or batch RL, seeks to circumvent this challenge by resorting to history data---which are often collected by executing some possibly unknown behavior policy in the past---with the hope that the history data might already provide significant insights about the targeted optimal policy without further exploration \citep{levine2020offline}. 

Besides maximizing the expected total reward, perhaps an equally important goal---to say the least---for an RL agent is safety and robustness \citep{garcia2015comprehensive}, especially in high-stake applications such as robotics, autonomous driving, clinical trials, financial investments, and so on \citep{choi2009reinforcement,schulman2013finding}. It has been observed that a standard RL agent trained in an ideal environment might be extremely sensitive and fail catastrophically when the deployed environment is subject to small adversarial perturbations \citep{zhang2020robust}. Consequently, robust RL has attracted a surge of attentions with the goal to learn an optimal policy that is robust to environment perturbations. In fact, providing robustness guarantees becomes even more relevant in the offline setting, which can be formulated as {\em robust offline RL}, since the history data is often inevitably collected from a timeframe where it is no longer reasonable to assume model stillness, due to the highly non-stationary and time-varying dynamics of many real-world applications. Altogether, this naturally leads to a question:
\begin{itemize}
\item[]{\em Can we learn a near-optimal policy which is robust with respect to uncertainties and variabilities of the environments using as few history samples as possible?}
\end{itemize}

\subsection{Challenges and premises in robust offline RL}

Despite significant amount of recent activities in robust RL and offline RL, addressing model uncertainty and sample efficiency simultaneously remains challenging due to several key issues that we single out below. 

\begin{itemize}
\item {\em Distribution shift.} The history data is generated by following some behavior policy in an outdated environment, which can result in a data distribution that is heavily deviated from the desired one, i.e., induced by the target policy in the deployed environment.  
\item {\em Partial and limited coverage.} The history data might only provide partial and limited coverage over the entire state-action space, where the limited sample size leads to a poor estimate of the associated model parameters, and consequently, unreliable policy learning outcomes.  

\end{itemize}  

Understanding the implications of---and designing algorithms that work around---these challenges play a major role in advancing the state-of-the-art of robust offline RL. In particular, two prevalent algorithmic ideas, distributional robustness and pessimism, are called out as our guiding principles.   
\begin{itemize}
\item {\em Distributional robustness.} Instead of finding an optimal policy in a fixed environment, motivated by the literature in distributionally robust optimization \citep{delage2010distributionally}, one might seek to find a policy that achieves the best worst-case performance for all the environments in some uncertainty set around the offline environment, as formalized in the framework of robust RL \citep{iyengar2005robust,nilim2005robust}.  
\item {\em Pessimism.} When the samples are scarce, it is wise to act with caution based on the principle of pessimism, where one subtracts a penalty term---representing the confidence of the corresponding estimate---from the value functions to avoid  excessive risk.  Encouragingly, pessimism has been recently shown as an indispensable ingredient to achieve sample efficiency in offline RL without requiring full coverage \citep{jin2021pessimism,rashidinejad2021bridging,li2022settling}, as long as the trajectory of the behavior policy provides sufficient overlap with that of the target policy.

\end{itemize}

While these two ideas have been proven useful for robust RL and offline RL {\em separately}, tackling robust offline RL needs novel ingredients that go significantly beyond a na\"ive combination of existing techniques. This is because, in robust offline RL, one needs to handle the distribution shift induced not only by the behavior policy, but also by model perturbations, thus the penalty term derived from the pessimism principle in standard offline RL is no longer applicable. Indeed, while the value function of standard RL depends linearly with respect to the transition kernel, the dependency between the nominal transition kernel and the robust value function unfortunately becomes highly nonlinear---even without a closed-form expression---making the control of statistical uncertainty extremely challenging in robust offline RL.

\subsection{Main contributions}\label{sec:main-contri}

In this work, we provide an affirmative answer to the question raised earlier, by developing a provably efficient model-based algorithm that learns a near-optimal {\em distributionally-robust} policy from a minimal number of offline samples. Specifically, we consider a Robust Markov Decision Process (RMDP) with $S$ states, $A$ actions in both the nonstationary finite-horizon setting (with horizon length $H$) and the discounted infinite-horizon setting (with discount factor $\gamma$). Different from standard MDPs, RMDPs specify a family transition kernels, which lie within an uncertainty set taken as a small ball of size $\ror$ around a nominal transition kernel with respect to the Kullback-Leibler (KL) divergence. Given $K$ episodes (resp. $N$ transitions) of history data drawn by following some behavior policy $\pib$ under the nominal transition kernel in the finite-horizon (resp. infinite-horizon) setting, our goal is to learn the optimal robust policy $\piopt$ in the maximin sense, which has the best worst-case value for all the transition kernels within the uncertainty set \citep{iyengar2005robust,nilim2005robust}. Our main results are summarized below.  

 \begin{itemize}
\item We introduce a notion called {\em robust single-policy clipped concentrability coefficient} $\Cstar \in [1/S , \infty ]$ to quantify the quality of history data, which measures the distribution shift between the behavior policy $\pib$ and the optimal robust policy $\piopt$ in the presence of model perturbations, without requiring full coverage of the entire state-action space by the behavior policy. In contrast, prior algorithms \citep{yang2021towards,zhou2021finite,panaganti2021sample}---using simulator or offline data---all require full coverage of the entire state-action space. 

\item We propose a novel pessimistic variant of distributionally robust value iteration with a plug-in estimate of the nominal transition kernel \citep{iyengar2005robust,nilim2005robust}, called \DRLCB, by penalizing the robust value estimates with a carefully designed data-driven penalty term. We demonstrate that  \DRLCB finds an $\varepsilon$-optimal robust policy as soon as the sample size is above 
$  \widetilde{O} \left( \frac{S \Cstar H^5 }{ \minpall \ror^2\varepsilon^2 } \right) $ for the finite-horizon setting and $  \widetilde{O} \left( \frac{S \Cstar  }{ \minpall \ror^2(1-\gamma)^4\varepsilon^2 } \right)  $ for the infinite-horizon setting,
up to some logarithmic factor after a burn-in cost independent of $\varepsilon$. Here,
 $\minpall$ is the smallest positive state transition probability of the optimal robust policy $\piopt$ under the nominal kernel. 

%\footnote{Throughout the paper, we use the standard notation $f(n)=O(g(n))$ to indicate that $f(n)/g(n)$ is bounded above by a constant as $n$ grows. The notation $\widetilde{O}(\cdot)$ resembles  $O(\cdot)$  except that it hides any logarithmic scaling. The notation $f(n)=o(g(n))$ means that $\lim_{n\rightarrow\infty} {f(n)}/{g(n)}=0$. } 

\item To complement the upper bound, we further develop information-theoretic lower bounds for a range of uncertainty levels, showing there exists some transition kernel such that at least $\Omega \left( \frac{ S\Cstar H^4 }{     \varepsilon^2} \right) $
samples (resp. $\Omega \left( \frac{ S\Cstar }{     (1-\gamma)^3\varepsilon^2} \right) $ samples) are needed to find an $\varepsilon$-optimal robust policy when the uncertainty level $\ror\lesssim 1/H$ (resp. $\ror\lesssim (1-\gamma)$), and at least $\Omega \left( \frac{ S\Cstar H^3 }{  \minpall \ror^2\varepsilon^2} \right) $
samples (resp. $\Omega \left( \frac{ S\Cstar }{  \minpall \ror^2(1-\gamma)^2\varepsilon^2} \right) $ samples) are needed to find an $\varepsilon$-optimal robust policy when the uncertainty level $\ror \asymp  \log(1/\minpall)$, regardless of the choice of algorithms in the finite-horizon (resp. infinite-horizon) setting. Hence, this suggests that learning RMDPs is at least as hard as the standard MDP \citep{li2022settling} when the uncertainty level is sufficiently small, and corroborates the near-optimality of \DRLCB with respect to all key parameters up to a polynomial factor of the horizon length $H$ (resp. the effective horizon length $\frac{1}{1-\gamma}$) for a range of uncertainty levels ($\ror\asymp  \log(1/\minpall)$). 

\end{itemize}

%In addition, 

To the best of our knowledge, our paper is the first work to execute the principle of pessimism in a data-driven manner for robust offline RL, leading to
the first provably efficient algorithm that learns under simultaneous model uncertainty and partial coverage of the history dataset. See Table~\ref{tab:our-work}  for a summary.

\paragraph{Comparison with prior art under full coverage.} Prior works  \citep{yang2021towards,zhou2021finite,panaganti2021sample} have only addressed the infinite-horizon setting under full coverage of the history data. Fortunately, our results also seamlessly cover this easier scenario, by replacing $\Cstar$ with $A$. 
Specializing our result to this setting to facilitate comparison, \DRLCB finds an $\varepsilon$-optimal robust policy with at most $\widetilde{O} \left(\frac{SA  }{  \minpall (1-\gamma)^4 \ror^2 \varepsilon^2}  \right)$ samples, which depends linearly with respect to the size of the state space $S$ (ignoring other parameters). In contrast, all prior works \citep{yang2021towards,zhou2021finite,panaganti2021sample} incur sample complexities that scale at least quadratically with respect to the size of the state space $S$. In addition, our bound improves the {\em exponential} dependency on $\frac{1}{1-\gamma}$ of \citet{zhou2021finite,panaganti2021sample} to a {\em polynomial} dependency, as well as the {\em quadratic} dependency on $1/\minp$ (which satisfies $\minp \leq \minpall$) of \citet{yang2021towards} to a {\em linear} one on $1/\minpall$. These improvements further corroborate the benefit of the proposed \DRLCB even under full coverage.  See Table~\ref{tab:prior-work} for detailed comparisons.

\begin{table}[t]
	\begin{center}
% \resizebox{\textwidth}{!}{
\begin{tabular}{c|c|c|c|c}
\toprule
	Horizon & Algorithm &  Coverage  & Sample complexity   & Uncertainty level  \tabularnewline
 \hline
 \multirow{6}{*}{infinite-horizon} 
 &    {   DRVI-LCB}  \vphantom{$\frac{1^{7}}{1^{7^{7}}}$}   & &   &    \tabularnewline
	&{{\bf (this work)}}  &    \multirow{-2}{*}{partial} & \multirow{-2}{*}{ $\frac{S\Cstar  }{  \minpall (1-\gamma)^4 \ror^2 \varepsilon^2} $}  &   \multirow{-2}{*}{full range} \tabularnewline 
\cline{2-5}  
 &    {   Lower bound}  \vphantom{$\frac{1^{7}}{1^{7^{7}}}$}   & &   &    \tabularnewline
	&{{\bf (this work)}}  &    \multirow{-2}{*}{partial} & \multirow{-2}{*}{ $\frac{S\Cstar  }{   (1-\gamma)^3   \varepsilon^2} $}  &   \multirow{-2}{*}{ $\ror\lesssim (1-\gamma)$} \tabularnewline 
\cline{2-5} 
& { Lower bound} \vphantom{$\frac{1^{7}}{1^{7^{7}}}$}    &  & &   \tabularnewline
& {{\bf (this work)}}  &   \multirow{-2}{*}{partial}  & \multirow{-2}{*}{ $ \frac{ S\Cstar  }{ \minpall (1-\gamma)^2 \ror^2 \varepsilon^2}  
$} &  \multirow{-2}{*}{ $\ror\asymp \log(1/\minpall)$}   \tabularnewline
\hline
  \multirow{6}{*}{finite-horizon}  
&   {  DRVI-LCB} \vphantom{$\frac{1^{7}}{1^{7^{7}}}$}    &  &    &    \tabularnewline
	&{ {\bf (this work)}} &     \multirow{-2}{*}{partial} & \multirow{-2}{*}{ $\frac{S\Cstar  H^5}{  \minpall \ror^2 \varepsilon^2} $}   &   \multirow{-2}{*}{full range}  \tabularnewline  
\cline{2-5} 
& {  Lower bound} \vphantom{$\frac{1^{7}}{1^{7^{7}}}$}     &   &   &     \tabularnewline
& {{\bf (this work)}} &     \multirow{-2}{*}{partial}  & \multirow{-2}{*}{ $ \frac{ S\Cstar H^4 }{   \varepsilon^2}  
$}    &   \multirow{-2}{*}{ $\sigma\lesssim 1/H$} \tabularnewline
\cline{2-5} 
& {  Lower bound} \vphantom{$\frac{1^{7}}{1^{7^{7}}}$}     &   &   &     \tabularnewline
& {{\bf (this work)}} &     \multirow{-2}{*}{partial}  & \multirow{-2}{*}{ $ \frac{ S\Cstar H^3 }{ \minpall  \ror^2 \varepsilon^2}  
$}    &   \multirow{-2}{*}{ $\ror\asymp \log(1/\minpall)$} \tabularnewline
\hline
\toprule
\end{tabular}
% }

	\end{center}
	\caption{Our results for finding an $\varepsilon$-optimal robust policy in the infinite/finite-horizon robust MDPs with an uncertainty set measured with respect to the KL divergence using history data under partial coverage. The sample complexities included in the table are valid for sufficiently small $\varepsilon$, with all logarithmic factors omitted. Here, $\ror$ is the uncertainty level, $S$ is the size of the state space, $H$ is the horizon length for the finite-horizon setting, $\gamma$ is the discount factor for the infinite-horizon setting, $\Cstar$ is the robust single-policy clipped concentrability coefficient, and $\minpall  $ is the smallest positive state transition probability of the nominal kernel {\em visited by the optimal robust policy $\piopt$}. 
	\label{tab:our-work}   } 
 
\end{table}

\begin{table}[ht]
	\begin{center}
% \resizebox{\textwidth}{!}{
\begin{tabular}{c|c|c|c}
\toprule

	Problem type & Algorithm &  Coverage  & Sample complexity   \tabularnewline
\toprule
% \hline 
 \multirow{8}{*}{infinite-horizon} &DRVI \vphantom{$\frac{1^{7}}{1^{7^{7}}}$} &  \multirow{2}{*}{full}  & \multirow{2}{*}{$\frac{S^2A \exp\left(O(\frac{1}{1-\gamma})\right)}{ (1-\gamma)^4 \ror^2 \varepsilon^2} $}   \tabularnewline
&\citep{zhou2021finite} &  &      \tabularnewline
% \hline
\cline{2-4} &REVI/DRVI \vphantom{$\frac{1^{7}}{1^{7^{7}}}$}   & \multirow{2}{*}{full} & \multirow{2}{*}{$\frac{S^2A \exp\left(O(\frac{1}{1-\gamma})\right)}{ (1-\gamma)^4 \ror^2 \varepsilon^2} $
}  \tabularnewline
&\citep{panaganti2021sample} &     &  \tabularnewline 
\cline{2-4}
&DRVI \vphantom{$\frac{1^{7}}{1^{7^{7}}}$}   &\multirow{2}{*}{full} & \multirow{2}{*}{$\frac{S^2A }{ \minp^2 (1-\gamma)^4 \ror^2 \varepsilon^2} $ }   \tabularnewline
&\citep{yang2021towards} &       &  \tabularnewline
\cline{2-4} 
 &   {  DRVI-LCB} \vphantom{$\frac{1^{7}}{1^{7^{7}}}$}   &  &    \tabularnewline
	&{  {\bf (this work)}} &\multirow{-2}{*}{ full} &    \multirow{-2}{*}{  $\frac{SA  }{  \minpall (1-\gamma)^4 \ror^2 \varepsilon^2} $}  \tabularnewline
\hline
  \multirow{2}{*}{finite-horizon}  
&   {  DRVI-LCB} \vphantom{$\frac{1^{7}}{1^{7^{7}}}$}    & &    \tabularnewline
	&{{\bf (this work)}} &     \multirow{-2}{*}{full} & \multirow{-2}{*}{ $\frac{SA  H^5}{  \minpall \ror^2 \varepsilon^2} $}    \tabularnewline  
\hline
\toprule
\end{tabular}
% }

	\end{center}
	\caption{Comparisons between our results and prior arts for finding an $\varepsilon$-optimal robust policy in the infinite/finite-horizon robust MDPs with an uncertainty set measured with respect to the KL divergence under full coverage of the history data. The sample complexities included in the table are valid for sufficiently small $\varepsilon$, with all logarithmic factors omitted. Here, $\ror$ is the uncertainty level, $S$ is the size of the state space, $A$ is the size of the action space, $H$ is the horizon length for the finite-horizon setting, $\gamma$ is the discount factor for the infinite-horizon setting, $\minpall  $ is the smallest positive state transition probability of the nominal kernel {\em visited by the optimal robust policy $\piopt$}, and $\minp$ is the smallest positive state transition probability of the nominal kernel; it holds $\minp \leq \minpall$. 
	\label{tab:prior-work}   } 
 
\end{table}

\input{related-work.tex}

\subsection{Notation and paper organization}
 
 Throughout this paper, we denote by $\Delta(\cS)$ the probability simplex over a set $\cS$, and introduce the notation $[H]\coloneqq \{1,\cdots,H\}$ for any positive integer $H>0$. In addition, for any vector $x = \big[x(s,a)\big]_{(s,a)\in\cS\times\cA}\in \mathbb{R}^{SA}$ (resp.~$x = \big[x(s)\big]_{s\in\cS}\in \mathbb{R}^{S}$) that constitutes certain values for each state-action pair (resp.~state), we overload the notation by letting $x^2 = \big[x(s,a)^2\big]_{(s,a)\in\cS\times\cA}$ (resp.~$x^2 = \big[x(s)^2\big]_{s\in\cS}$). Moreover, for any two vectors $x=[x_i]_{1\leq i\leq n}$ and $y=[y_i]_{1\leq i\leq n}$, the notation $ {x}\leq {y}$ (resp.~$ {x}\geq {y}$) means
$x_{i}\leq y_{i}$ (resp.~$x_{i}\geq y_{i}$) for all $1\leq i\leq n$. Finally, the Kullback-Leibler (KL) divergence for any two distributions $P$ and $Q$ is denoted as $\mathsf{KL}(P \parallel Q)$. 

The rest of this paper is organized as follows. Section~\ref{sec:problem-formulation} provides the backgrounds and introduces the distributionally robust formulation of finite-horizon MDPs in the offline setting under partial coverage. Section~\ref{sec:main-result} presents the proposed algorithm and provides sample complexity guarantees. Section~\ref{sec:problem-formulation-infty} develops the corresponding results for the infinite-horizon setting. Section~\ref{sec:experiments} demonstrate the performance of the proposed algorithm through numerical experiments. Finally, we conclude in Section~\ref{sec:discussions}. The detailed proofs are postponed to the appendix.

%% file: related-work.tex
\subsection{Related works}

We shall focus on the closely related works on offline RL and distributionally robust RL.

\paragraph{Offline RL.} Focusing on the task of learning an optimal policy from offline data, a significant amount of prior arts sets to understand the sample complexity and efficacy of offline RL under different assumptions of the history dataset. A bulk of prior results requires the history data to cover all the state-action pairs, under assumptions such as uniformly bounded concentrability coefficients \citep{chen2019information,munos2005error} and uniformly lower bounded data visitation distribution~\citep{yin2021optimal,yin2021near}, where the latter assumption is also related to studies of asynchronous Q-learning \citep{li2021sample}. More recently, the principle of pessimism has been investigated for offline RL in both model-based \citep{jin2021pessimism,xie2021policy,rashidinejad2021bridging,li2022settling} and model-free algorithms \citep{kumar2020conservative,shi2022pessimistic,yan2022efficacy}, without the stringent requirement of full coverage. In particular, \citet{li2022settling} established the near-minimax optimality of a pessimistic variant of value iteration under the single-policy clipped concentrability of history data, which inspired our algorithm design in the distributionally robust setting.

\paragraph{Distributionally robust RL.} While distributionally robust optimization has been mainly investigated in the context of supervised learning \citep{rahimian2019distributionally,gao2020finite,bertsimas2018data,duchi2018learning,blanchet2019quantifying,sinha2018certifying}, distributionally robust dynamic programming has also attracted considerable amount of attention, e.g. \citet{iyengar2005robust,nilim2003robustness,xu2012distributionally,nilim2005robust}, where natural robust extensions to the standard Bellman machineries are developed under mild assumptions. Targeting robust MDPs, empirical and theoretical works have been widely explored under different forms of uncertainty sets \citep{iyengar2005robust,xu2012distributionally,wolff2012robust,kaufman2013robust,ho2018fast,smirnova2019distributionally,ho2021partial,goyal2022robust,derman2020distributional,tamar2014scaling,badrinath2021robust,abdullah2019wasserstein,hou2020robust,song2020optimistic,yang2017convex,wang2022reliable,ding2023seeing}. Nonetheless, the majority of prior theoretical analyses focus on planning with an exact knowledge of the uncertainty set \citep{iyengar2005robust,xu2012distributionally,tamar2014scaling}, or are asymptotic in nature \citep{roy2017reinforcement}. 

A number of robust RL algorithms were proposed recently with an emphasis on finite-sample performance guarantees under different data generating mechanisms. \citet{wang2021online} proposed a robust Q-learning algorithm with an R-contamination uncertain set for the online setting, which achieves a similar bound as its non-robust counterpart. \citet{badrinath2021robust} proposed a model-free algorithm for the online setting with linear function approximation to cope with large state spaces. \citet{yang2021towards,panaganti2021sample} developed sample complexities for a model-based robust RL algorithm with a variety of uncertainty sets where the data are collected using a generative model. In addition, \citet{zhou2021finite} examined the uncertainty set defined by the KL divergence for offline data with uniformly lower bounded data visitation distribution. These works all require full coverage of the state-action space, whereas ours is the first one to leverage the principle of pessimism in robust offline RL. 

Since the first appearance of our paper on arXiv in August 2022, a few more papers have emerged that also tackle the sample complexity of robust RL algorithms. For example, 
\citet{wang2023finite,wang2023sample} developed finite-sample complexity bounds for robust variants of Q-learning with the generative model when the uncertainty set is measured by KL divergence; in particular, the improved bound of variance-reduced robust Q-learning \citep{wang2023sample} becomes independent of the size of the uncertainty set when it is sufficiently small with respect to the minimal support probability of the nominal kernel  at a price of worse dependency with $1/\minpall$.
\cite{shi2023curious} provided near-optimal sample complexity bounds for model-based robust RL algorithms with the generative model when the uncertainty set is measured by the total variation or chi-square distances, which highlighted that different uncertainty sets can lead to drastically different sample complexities, and hence, statistical consequences.

%% file: formulation.tex
%!TEX root = ./DRO-offline.tex
\section{Problem formulation: episodic finite-horizon RMDPs}
\label{sec:problem-formulation}

\subsection{Basics of finite-horizon episodic tabular MDPs}

Consider an episodic finite-horizon MDP, represented by 
$\mathcal{M}= \big(\mathcal{S},\mathcal{A},H, P:=\{P_h\}_{h=1}^H, \{r_h\}_{h=1}^H \big)$, where $\mathcal{S} =\{1,\cdots, S\}$ and $\mathcal{A}= \{1,\cdots,A\}$ are the finite state and action spaces, respectively, $H$ is the horizon length, $P_h : \cS \times \cA \rightarrow \Delta (\cS) $ (resp.~$r_h: \cS \times \cA \rightarrow [0,1]$) denotes the probability transition kernel (resp.~reward function) at  step $h$ $(1\leq h\leq H)$.\footnote{Without loss of generality, we assume the reward function is deterministic, fixed, and normalized to be within $[0,1]$; it is straightforward to generalize our framework to incorporate random rewards with uncertainties.} For any transition kernel $P$, we introduce the $S$-dimensional distribution vectors
\begin{equation} \label{eq:transition_vector}
P_{h,s,a} \coloneqq P_h(\cdot \mymid s,a ) \in [0,1]^{1\times S}, \qquad \forall (h,s,a)\in[H]\times\cS\times\cA
\end{equation}
to represent the probability transition vector in state $s$ when taking action $a$ at step $h$.

Denote by $\pi =\{\pi_h\}_{h=1}^H$ as the policy or action selection rule of an agent, where $\pi_h: \mathcal{S} \rightarrow \Delta(\mathcal{A})$ specifies the action selection probability over the action space; when the policy is deterministic, we slightly abuse the notation and refer to $\pi_h(s)$ as the action selected by policy $\pi$ in state $s$ at step $h$. The value function $V^{\pi,P} = \{V_h^{\pi,P}\}_{h=1}^H$ of policy $\pi$ with a transition kernel $P$ is defined by
\begin{align}
	\label{eq:def_Vh}
	\forall (h,s)\in[H]\times \cS:\qquad V^{\pi,P}_{h}(s ) &\defn  \mathbb{E}_{\pi, P} 
	\left[  \sum_{t=h}^{H} r_{t}\big(s_{t}, a_t \big) \,\Big|\, s_{h}=s \right]  , 
\end{align}
where the expectation is taken over the randomness of the trajectory $\{s_h, a_h, r_h\}_{h=1}^H$ generated by executing policy $\pi$, namely, $a_t\sim \pi_t(s_t)$, and $s_{t+1} \sim P_t(\cdot \mymid s_t, a_t )$. Similarly, the Q-function $Q^{\pi,P} = \{Q_h^{\pi,P}\}_{h=1}^H$  of policy $\pi$ is defined as
\begin{align} 
	\label{eq:def_Qh}
	\forall (h,s,a)\in [H]\times \cS \times \cA:\qquad Q^{\pi, P}_{h}(s,a ) & \defn r_{h}(s,a)+ \mathbb{E}_{\pi,P} \left[  \sum_{t=h +1}^{H} r_t (s_t, a_t ) \,\Big|\, s_{h}=s, a_h  = a\right] ,
	\end{align}
where the expectation is again taken over the randomness of the trajectory.

Moreover, when the initial state $s_1$ is drawn from a given distribution $\rho$, let $d_h^{\pi,P}(s \mymid \rho)$ and $d_h^{\pi,P}(s,a \mymid \rho)$ denote respectively the state occupancy distribution and the state-action occupancy distribution induced by $\pi$ at time step $h\in [H]$.
In particular, we often dropped the dependency with respect to $\rho$ whenever it is clear from the context, by simply writing $d_h^{\pi,P}(s ): = d_h^{\pi,P}(s \mymid \rho)$ and $d_h^{\pi,P}(s,a) := d_h^{\pi,P}(s, a \mymid \rho)$, i.e.,
\begin{subequations}\label{eq:visitation_dist}
\begin{align} 
\forall (h, s)\in[H]\times\cS :\quad\quad  	d_h^{\pi,P}(s )&   \defn \mathbb{P}(s_h = s \mymid s_1 \sim \rho, \pi,P),   \\
\forall (h, s,a)\in[H]\times\cS\times \cA :\quad  d_h^{\pi,P}(s,a) & \defn \mathbb{P}(s_h = s \mymid s_1 \sim \rho, \pi,P) \, \pi_h(a \mymid s),
\end{align}
\end{subequations}
which are conditioned on $s_1\sim \rho$ and  the event that all actions and states are drawn according to policy $\pi$ and transition kernel $P$.

 \subsection{Distributionally robust MDPs}
 
In this section, we focus on finite-horizon episodic distributionally robust MDPs (RMDPs), denoted by $\mathcal{M}_{\mathsf{rob}}= \big(\mathcal{S},\mathcal{A}, H, \unb^\ror(P^{\no}), \{r_h\}_{h=1}^H \big)$. Different from standard MDPs, we now consider an ensemble of probability transition kernels or models within an uncertainty set centered around a nominal one $P^{\no} = \{P^{\no}_h\}_{h=1}^H$, where the distance between the transition kernels is measured in terms of the  Kullback-Leibler (KL) divergence. Specifically, given an uncertainty level $\sigma>0$, the uncertainty set around $P^{\no}$, which satisfies the so-called $(s,a)$-rectangularity condition \citep{wiesemann2013robust}, is specified as
\begin{align}\label{eq:kl-ball-finite-P}
	\unb^\ror(P^{\no})\defn \otimes \; \unb^\ror (P^{\no}_{h,s,a}),\qquad \unb^\ror(P_{h,s,a}^0) \defn \left\{ P_{h, s,a} \in \Delta (\cS): \mathsf{KL}\left(P_{h,s,a} \parallel P^0_{h,s,a}\right) \leq \ror \right\},
\end{align}
where $\otimes$ denote the Cartesian product.
In words, the KL divergence between the true transition probability vector and the nominal one at each state-action pair is at most $\sigma$; moreover, the RMDP reduces to the standard MDP when $\sigma =0$. 
 
Instead of evaluating a policy in a fixed MDP, the performance of a policy in the RMDP is evaluated based on its worst-case---i.e., smallest---value function over all the instances in the uncertainty set. That is, we define the {\em robust value function} $V^{\pi, \ror} = \{V_h^{\pi,\ror}\}_{h=1}^H$ and the {\em robust Q-function} $Q^{\pi,\ror} = \{Q_h^{\pi,\ror}\}_{h=1}^H$ respectively as
\begin{align*}
	\forall (h, s,a)\in[H]\times\cS \times \cA:\quad  V^{\pi,\ror}_{h}(s) &\defn \inf_{P\in \unb^{\ror}(P^{\no})} V_h^{\pi,P} (s), \qquad Q^{\pi,\ror}_{h}(s,a) \defn \inf_{P\in \unb^{\ror}(P^{\no})} Q_h^{\pi,P}(s,a),
\end{align*}
where the infimum is taken over the uncertainty set of transition kernels. 

\paragraph{Optimal robust policy.}
For finite-horizon RMDPs, it has been established that there exists at least one deterministic policy that maximizes the robust value function and the robust Q-function simultaneously \citep{iyengar2005robust,nilim2005robust}. In view of this, we shall denote a deterministic policy $\piopt = \{\piopt_h\}_{h=1}^H$ as an optimal robust policy throughout this paper. The resulting {\em optimal robust  value function} $V^{\star,\ror} =\{V_h^{\star,\ror} \}_{h=1}^H$ and {\em optimal robust Q-function} $Q^{\star,\ror} =\{Q_h^{\star,\ror} \}_{h=1}^H$ are denoted by
\begin{subequations}
\begin{align}\label{eq:optimal-pi-rmdp}
	\forall (h, s)\in[H]\times\cS :\quad  	V_h^{\star,\ror}(s) & \defn V_h^{\piopt,\ror}(s) = \max_{\pi} V_h^{\pi,\ror}(s), \\
		\forall (h, s,a)\in[H]\times\cS \times \cA:\quad   Q_h^{\star,\ror}(s,a) &\defn Q_h^{\piopt,\ror}(s,a) = \max_{\pi} Q_h^{\pi,\ror}(s,a).
\end{align}
\end{subequations}

Similar to \eqref{eq:visitation_dist}, we adopt the following short-hand notation for the occupancy distributions associated with the optimal policy:
\begin{subequations}
	\begin{align}\label{eq:visitation_dist-optimal}
	\forall (h,s)\in [H]\times\cS :\qquad d_h^{\star, P}(s)  & \defn d_h^{\piopt, P}(s),   \\
\forall (h,s,a)\in [H]\times\cS\times\cA:\qquad 	d_h^{\star,P}(s,a )  & \defn d_h^{\piopt, P}(s,a)= d_h^{\star, P}(s ) \ind\{a = \piopt_h(s)\}.
\end{align}
\end{subequations}

\paragraph{Robust Bellman equations.}
It turns out the Bellman's principle of optimality can be extended naturally to its robust counterpart \citep{iyengar2005robust,nilim2005robust}, which plays a fundamental role in solving the RMDP. To begin with, for any policy $\pi$, the robust value function and robust Q-function satisfy the following {\em robust Bellman consistency equation}:
\begin{align} \label{eq:robust_bellman_consistency}
\forall (h, s,a)\in[H]\times\cS \times \cA:\qquad  Q^{\pi,\ror}_{h}(s,a) = r_h(s,a) + \inf_{ \cP \in \unb^\ror(P^{\no}_{h,s,a})} \cP V_{h+1}^{\pi,\ror}.
\end{align}
Additionally, the optimal robust Q-function obeys the {\em robust Bellman optimality equation}:
\begin{align} \label{eq:robust_bellman_optimality}
\forall (h, s,a)\in[H]\times\cS \times \cA:\qquad  Q^{\star,\ror}_{h}(s,a) = r_h(s,a) + \inf_{ \cP \in \unb^{\ror}(P^{\no}_{h,s,a})} \cP V_{h+1}^{\star,\ror},
\end{align} 
which can be solved efficiently via a robust variant of value iteration when the RMDP is known \citep{iyengar2005robust,nilim2005robust}.

\subsection{Distributionally robust offline RL}
	\label{sec:offline-concentrability}

Let $\cD$ be a history/batch dataset, which consists of a collection of $K$ {\em independent} episodes generated based on executing a behavior policy $\pib = \{\pib_h\}_{h=1}^H$ in some nominal MDP $\mathcal{M}^{\no}= \big(\mathcal{S},\mathcal{A},H,P^{\no}:= \{P^{\no}_h\}_{h=1}^H, \{r_h\}_{h=1}^H \big)$. More specifically, for $1\leq k \leq K$, the $k$-th episode $\big( s_1^k, a_1^k,  \ldots, s_H^k, a_H^k, s_{H+1}^k \big)$ is generated according to
\begin{align}\label{eq:finite-batch-size-def}
s_1^k \sim \rhob,
\qquad a_h^k \sim \pib_h(\cdot\mymid s_h^k) 
\qquad\text{and}\qquad 
	s_{h+1}^k \sim P^{\no}_h(\cdot\mymid s_h^k, a_h^k) ,
	\qquad 1 \le h \le H.
\end{align}
Throughout the paper, $\rhob$ represents for some initial distribution associated with the history dataset. Then, we introduce the following short-hand notation for the occupancy distribution w.r.t. $\pib$:
\begin{equation}
	\label{eq:visitation_dist-optimal}
	\forall (h,s,a)\in [H]\times\cS\times\cA:\qquad \myrho_h(s )  \defn d_h^{\pib, P^{\no}}(s ),  \qquad
	\myrho_h(s,a )  \defn d_h^{\pib,P^{\no}}(s,a   ).
\end{equation}

\paragraph{Goal.}
With the history dataset $\cD$ in hand, our goal is to find a near-optimal robust policy $\widehat{\pi}$, which satisfies
\begin{align}\label{eq:goal}
	V_1^{\widehat{\pi},\ror}(\rho) \geq V_1^{\star,\ror}(\rho) -  \varepsilon
\end{align} 
using as few samples as possible, where $\varepsilon$ is the target accuracy level, and 
\begin{align}\label{eq:defn-V-rho}
	V_1^{\pi,\ror}(\rho) \coloneqq {\mathbb{E}}_{s_1\sim \rho} \big[ V_1^{\pi,\ror} (s_1) \big]\qquad \text{and} \qquad V_1^{\star,\ror}(\rho) \coloneqq  {\mathbb{E}}_{s_1\sim \rho} \big[ V_1^{\star,\ror} (s_1) \big]
\end{align}
are evaluated when the initial state $s_1$ is drawn from a given distribution $\rho$.

\paragraph{Robust single-policy clipped concentrability.} 
To quantify the quality of the history dataset to achieve the set goal, it is desirable to capture the distribution mismatch between the history dataset and the desired ones, inspired by the {\em single-policy clipped concentrability} assumption recently proposed by \citet{li2022settling},  we introduce a tailored assumption for robust MDPs as follows.

\begin{assumption}[Robust single-policy clipped concentrability] 
\label{assumption:dro-finite}
The behavior policy of the history dataset $\mathcal{D}$ satisfies
\begin{align}
	\max_{(s, a, h, P) \in \mathcal{S} \times \cA \times [H] \times \unb^{\ror}(P^{\no})} \frac{\min\big\{d_h^{\star,P}(s, a ), \frac{1}{S}\big\}}{\myrho_h(s, a )} \le \Cstar
	\label{eq:concentrate-finite}
\end{align}
for some quantity $\Cstar \in\big[\frac{1}{S},\infty\big]$.
Here, we take $\Cstar$ to be the smallest quantity satisfying \eqref{eq:concentrate-finite}, and refer to it as the robust single-policy clipped concentrability coefficient. In addition, we follow the convention $0/0=0$.
\end{assumption}

In words, $\Cstar$ measures the worst-case discrepancy---between the optimal robust policy $\piopt$ (from initial state distribution $\rho$) in any model $P\in  \unb^{\ror}(P^{\no})$ within the uncertainty set and the behavior policy $\pib$ (from initial state distribution $\rhob$) in the nominal model $P^{\no}$---in terms of the clipped maximum density ratio of the state-action occupancy distributions. 
\begin{itemize}
\item {\em Distribution shift.} When the uncertainty level $\sigma =0$, Assumption~\ref{assumption:dro-finite} reduces back to the single-policy clipped concentrability in \citet{li2022settling} for standard offline RL, a weaker notion that can be $S$ times smaller than the single-policy concentrability adopted in, e.g.,  \citet{rashidinejad2021bridging,xie2021policy,shi2022pessimistic}. On the other end, whenever $\sigma >0$, the proposed robust single-policy clipped concentrability accounts for the distribution shift not only due to the policies in use ($\piopt$ versus $\pib$) with respect to the respective initial state distributions, but also the underlying environments ($P\in  \unb^{\ror}(P^{\no})$ versus $P^{\no}$), and therefore, is generally larger than that in the non-robust counterpart.

\item {\em Partial coverage.} As long as $\Cstar$ is finite, i.e., $\Cstar < \infty$, it admits the scenarios when the history dataset only provides {\em partial coverage} over the entire state-action space, as long as the behavior policy $\pib$ visits the state-action pairs that are visited by the optimal robust policy $\piopt$ under at least one  model in the uncertainty set.

\end{itemize}
\begin{remark}\label{remark:Cstar_full_finite}
To facilitate comparison with prior works assuming full coverage, we can bound $\Cstar$ when the batch dataset is generated using a simulator \citep{yang2021towards,panaganti2021sample}; namely, we can generate sample state transitions based on the transition kernel of the nominal MDP for all state-action pairs at all time steps. In this case, it amounts to that $\myrho_h(s, a) = \frac{1}{SA}$ for all $(s,a,h)\in \cS\times \cA\times [H]$, which directly leads to the bound $\Cstar = \max_{(s, a, h, P) \in \mathcal{S} \times \cA \times [H] \times \unb^{\ror}(P^{\no})} \frac{\min\big\{d_h^{\star,P}(s, a ), \frac{1}{S}\big\}}{\myrho_h(s, a )}  \leq \frac{1/S}{1/(SA)} = A$. 
\end{remark}

%% file: main-result.tex
%!TEX root = ./DRO_OfflineRL.tex
\section{Algorithm and theory: episodic finite-horizon RMDPs}\label{sec:main-result}
In this section, we present a model-based algorithm---namely \DRLCB ---for robust offline RL in the finite-horizon setting, along with its performance guarantees.

\subsection{Building an empirical nominal MDP}

For a moment, imagine we have access to $N$ {\em independent} sample transitions $\cD_0 := \{(h_i,s_i,a_i,s_i')\}_{i=1}^N$ drawn from the transition kernel $P^{\no}$ of the nominal MDP $\mathcal{M}^{\no}$, where each sample $(h_i,s_i,a_i,s_i')$ indicates the transition from state $s_i$ to state $s_i'$ when action $a_i$ is taken at step $h_i$, drawn according to $s_i' \sim P^{\no}_{h_i}(\cdot \mymid s_i, a_i)$. It is then natural to build an empirical estimate $\widehat{P}^{\no}=\{\widehat{P}^{\no}_h\}_{h=1}^H$ of $P^{\no}$ based on the empirical frequencies of state transitions, where
\begin{align}
	\widehat{P}^0_{h}(s'\mymid s,a) \defn 
	\begin{cases} \frac{1}{N_h(s,a)} \sum\limits_{i=1}^N \mathds{1} \big\{ (h_i, s_i, a_i, s_i') = (h,s,a,s') \big\}, & \text{if } N_h(s,a) > 0 \\
		0, & \text{else}
	\end{cases}  
	\label{eq:empirical-P-finite}
\end{align}
for any $(h,s,a,s')\in[H] \times \cS\times \cA\times  \cS$. Here, $N_h(s,a)$ denotes the total number of sample transitions from $(s,a)$ at step $h$ as
\begin{align}\label{eq:defn-Nh-sa-finite}
	N_h(s,a) &\coloneqq \sum_{i=1}^N \mathds{1} \big\{ (h_i,s_i, a_i) = (h,s,a) \big\}. 
\end{align}

While it is possible to directly break down the history dataset $\cD$ into sample transitions, unfortunately, the sample transitions from the same episode are not independent, significantly hurdling the analysis. To alleviate this, \citet[Algorithm 2]{li2022settling} proposed a simple two-fold subsampling scheme to preprocess the history dataset $\cD$ and decouple the statistical dependency, resulting into a distributionally equivalent dataset $\cD_0$ with independent samples; for completeness, we provide the procedure in Algorithm~\ref{alg:finite-split}.   We have the following lemma paraphrased from \citet{li2022settling} for the obtained  dataset $\cD_0$.

\begin{algorithm}[t]
\DontPrintSemicolon
	\textbf{input:} a dataset $\mathcal{D}$, probability $\delta$. \\

	\textbf{data splitting:} split $\mathcal{D}$ into two  $\Dmain$ and $\Daux$, where each contain $K/2$ trajectories. 
	
	\textbf{lower bounding the number of transitions in $\Dmain$:} denote the number of  transitions from state $s$ at step $h$ in $\Dmain$ (resp.~$\Daux$) as $\Nmain_h(s)$ (resp.~$\Naux_h(s)$), construct 
\begin{align}
	\label{eq:defn-Ntrim}
	\Ntrim_h(s) &\coloneqq \max\left\{\Naux_h(s) - 10\sqrt{\Naux_h(s)\log\frac{HS}{\delta}}, \, 0\right\} ;
	%\\
	%\Ntrim_h(s,a) &\coloneqq \max\left\{\Naux_h(s,a) - 10\sqrt{\Naux_h(s,a)\log\frac{HSA}{\delta}}, \, 0\right\}; 
\end{align}

\textbf{generate the subsampled dataset $\Dtrim$:} randomly sample the transitions (i.e., the quadruples taking the form $(s,a,h,s')$) from $\Dmain$ uniformly at random, such that for each $(s,h)\in \cS\times [H]$, $\Dtrim$ contains $\min \{ \Ntrim_h(s), \Nmain_h(s)\}$ sample transitions.  
 
	\textbf{output:} set $\mathcal{D}_0 = \Dtrim$.  \\

	\caption{Two-fold subsampling trick for the finite-horizon setting.}
 	\label{alg:finite-split}

\end{algorithm}

\begin{lemma}[\citep{li2022settling}]\label{lemma:D0-property}
	With probability at least $1-8\delta$, the output dataset from the two-fold subsampling scheme in \citet{li2022settling} is distributionally equivalent to $\cD_0$, where $\{N_h(s,a)\}$ are independent of the sample transitions in $\mathcal{D}^0$ and obey
	\begin{equation}
	N_h(s,a) \geq \frac{K \myrho_h(s,a)}{8} - 5 \sqrt{ K \myrho_h(s,a) \log \frac{KH}{\delta} }.
	\end{equation}
for all $(h,s,a)\in [H]\times \cS\times\cA$.
\end{lemma}

Therefore, by invoking the two-fold sampling trick from \citet{li2022settling}, it is sufficient to treat the dataset $\cD_0$ with independent samples onwards with Lemma~\ref{lemma:D0-property} in place, which greatly simplifies the analysis.

\subsection{\DRLCB: a pessimistic variant of robust value iteration}

Armed with the estimate $\widehat{P}^{\no}$ of the nominal transition kernel $P^{\no}$, we are positioned to introduce our algorithm \DRLCB, summarized in Algorithm~\ref{alg:vi-lcb-dro-finite}.

\paragraph{Distributionally robust value iteration.} Before proceeding, 
 let us recall the update rule of the classical distributionally robust value iteration (\DRVI), which serves as the basis of our algorithmic development.  Given an estimate of the nominal MDP $\widehat{P}^{\no}$ and the radius $\sigma$ of the uncertainty set, \DRVI updates the robust value functions according to
\begin{align}\label{eq:VI-primal}
	\widehat{Q}_h(s,a) = r_h(s,a) + \inf_{ \cP \in \unb^{\sigma}(\widehat{P}^{\no}_{h,s,a})} \cP \widehat{V}_{h+1}, \qquad \mbox{and}\qquad \widehat{V}_{h}(s) = \max_a \widehat{Q}_h(s,a), 
\end{align}
which works backwards from $h=H$ to $h=1$, with the terminal condition $\widehat{Q}_{H+1} =0$. Due to strong duality \citep{hu2013kullback}, the update rule of the robust Q-functions in \eqref{eq:VI-primal} can be equivalently reformulated  in its dual form as
\begin{align} \label{eq:VI-dual}
	\widehat{Q}_h(s,a) = r_h(s, a)  + \sup_{\lambda\geq 0}  \left\{ -\lambda \log\left(\widehat{P}^0_{h, s, a}   \exp \left(\frac{-\widehat{V}_{h+1}}{\lambda}\right) \right) - \lambda \ror \right\},
\end{align}
which can be solved efficiently \citep{iyengar2005robust,yang2021towards,panaganti2021sample}.

\begin{algorithm}[t]
\DontPrintSemicolon
	\textbf{input:} a dataset $\mathcal{D}_0$; reward function $r$; uncertainty level $\ror$. \\ 
	\textbf{initialization:} $\widehat{Q}_{H+1} = 0$, $\widehat{V}_{H+1}=0$. \\

   \For{$h=H,\cdots,1$}
	{
		Compute the empirical nominal transition kernel $\widehat{P}^{\no}_h$ according to \eqref{eq:empirical-P-finite}; \\
		\For{$s\in \cS, a\in \cA$}{
		Compute the penalty term $b_h\big(s,a\big)$ according to \eqref{def:bonus-dro}; \\
			Set $\widehat{Q}_h(s, a)$ according to \eqref{eq:algorithm-Q-update};  \label{line:finite-update-Q}\\ 
		}
		\For{$s\in \cS$}{
			Set $\widehat{V}_h(s) = \max_a \widehat{Q}_h(s, a)$ and $\widehat{\pi}_h(s) = \arg\max_{a} \widehat{Q}_h(s,a)$;
		}
	}

	\textbf{output:} $\widehat{\pi}=\{\widehat{\pi}_h\}_{1\leq h\leq H}$. 
	\caption{Robust value iteration with LCB (\DRLCB) for robust offline RL.}
 \label{alg:vi-lcb-dro-finite}
\end{algorithm}

\paragraph{Our algorithm \DRLCB.} Motivated by the principle of pessimism in standard offline RL \citep{jin2021pessimism,xie2021policy,rashidinejad2021bridging,li2022settling}, we propose to perform a pessimistic variant of \DRVI, where the update rule of \DRLCB at step $h$ is modified as
\begin{align}\label{eq:algorithm-Q-update}
	\widehat{Q}_h(s,a) = \max\left\{ r_h(s, a)  +\sup_{\lambda\geq 0}  \left\{ -\lambda \log\left(\widehat{P}^0_{h, s, a} \cdot \exp \left(\frac{-\widehat{V}_{h+1}}{\lambda}\right) \right) - \lambda \ror \right\} - b_h \big(s,a\big), \,  0 \right\}.
\end{align}
Here, the robust $Q$-function estimate is adjusted by subtracting a carefully designed data-driven penalty term $b_h(s,a) $ that measures the uncertainty of the value estimates. Specifically, for some $\delta \in (0,1)$ and any $(s,a,h) \in \cS\times \cA\times [H]$, the penalty term $b_h(s,a) $ is defined as
\begin{align}
	&b_h(s,a)= \begin{cases}
	\min\left\{ \cb \frac{H}{\ror} \sqrt{\frac{\log(\frac{KHS}{\delta})}{ \widehat{P}_{\mathsf{min},h}(s,a) N_h(s,a)}}~,~H\right\} & \text{if } N_h(s,a) > 0, \\
	H & \text{ otherwise},
	\end{cases}
	\label{def:bonus-dro}
\end{align}
where $\cb$ is some universal constant, and
\begin{align}\label{eq:P-min-hat-def}
	\widehat{P}_{\mathsf{min},h}(s,a) \defn  \min_{s'} \Big\{\widehat{P}_h^{\no}(s' \mymid s,a): \; \widehat{P}_h^{\no}(s' \mymid s,a)>0 \Big\}.
\end{align}
The penalty term is novel and different from the one used in standard (no-robust) offline RL \citep{jin2021pessimism,xie2021policy,rashidinejad2021bridging,li2022settling,shi2022pessimistic}, by taking into consideration the unique problem structure pertaining to robust MDPs. In particular, 
it tightly upper bounds the statistical uncertainty which carries a non-linear and implicit dependency \ w.r.t. the estimated nominal transition kernel induced by the uncertainty set $\unb(P^{\no})$, addressing unique challenges not present for the standard MDP case.

\subsection{Performance guarantees}\label{sec:main-theorem}

Before stating the main theorems, let us first introduce several important metrics. 
\begin{itemize}
\item $\minpall$, which only depends on the state-action pairs covered by the optimal robust policy $\pi^\star$ under the nominal model $P^{\no}$:
\begin{align}\label{eq:P-min-star-def}
	\minpall \defn \min_{h,s,s'} \Big\{P_h^{\no}\left(s' |s, \pi^\star_h(s)\right):\; P_h^{\no}\left(s' |s, \pi^\star_h(s)\right)>0 \Big\}.
\end{align}
In words, $\minpall$ is the smallest positive state transition probability of the optimal robust policy $\piopt$ under the nominal kernel $P^{\no}$. 
\item Similarly, we introduce $P_{\mathsf{min}}^{\mathsf{b}}$ which only depends on the state-action pairs covered by the behavior policy $\pi^{\mathsf{b}}$ under the nominal model $P^{\no}$:
\begin{align}
	P_{\mathsf{min}}^{\mathsf{b}} \defn \min_{h,s,a,s'} \Big\{P_h^{\no}\left(s' |s, a\right):\; \myrho_h(s, a )>0 ,\, P_h^{\no}\left(s' \mymid s, a\right)>0 \Big\}. \label{eq:def-P-min-b}
\end{align}
In words, $P_{\mathsf{min}}^{\mathsf{b}}$ is the smallest positive state transition probability of the behavior policy $\pib$ under the nominal kernel $P^{\no}$. 

\item Finally, let $d_{\mathsf{min}}^{\mathsf{b}}$ denote the smallest positive state-action occupancy distribution of the behavior policy $\pi^{\mathsf{b}}$ under the nominal model $P^{\no}$:
\begin{align}\label{eq:d-min-b}
d_{\mathsf{min}}^{\mathsf{b}} \defn  \min_{h,s,a} \left\{\myrho_h(s, a ): \; \myrho_h(s, a ) >0 \right\}.
\end{align}

\end{itemize}
We are now positioned to present the performance guarantees of \DRLCB for robust offline RL.

\begin{theorem}\label{thm:dro-upper-finite}
Given an uncertainty level $\ror>0$, suppose that the penalty terms in Algorithm~\ref{alg:vi-lcb-dro-finite} are chosen as \eqref{def:bonus-dro} for sufficiently large $\cb$. With probability at least $1-\delta$, the output $\widehat{\pi}$ of Algorithm~\ref{alg:vi-lcb-dro-finite} obeys
\begin{align}
	V_{1}^{\star, \ror}(\rho) - V_{1}^{\widehat{\pi}, \ror}(\rho) \leq c_0 \frac{H^2}{\ror} \sqrt{\frac{S \Cstar\log^2(KHS / \delta ) }{ \minpall  K} },
\end{align}
as long as the number of episodes $K$ satisfies
\begin{align} \label{eq:dro-b-bound-N-condition}
	K \geq \frac{c_1 \log( KHS/ \delta )}{  d_{\mathsf{min}}^{\mathsf{b}}  P_{\mathsf{min}}^{\mathsf{b}} },
\end{align}
where $c_0$ and $c_1$ are some sufficiently large universal constants. 
\end{theorem}

Our theorem is the first to characterize the sample complexities of robust offline RL  under {\em partial coverage}, to the best of our knowledge (cf.~Table~\ref{tab:prior-work}). 
Theorem~\ref{thm:dro-upper-finite} shows that \DRLCB finds an $\varepsilon$-optimal robust policy as soon as the sample size $T= KH$ is above  the order of
\begin{equation}
\underbrace{ \frac{S \Cstar H^5}{ \minpall \ror^2\varepsilon^2} }_{\epsilon\textsf{-dependent}} + \underbrace{ \frac{H}{  d_{\mathsf{min}}^{\mathsf{b}}  P_{\mathsf{min}}^{\mathsf{b}} }}_{\textsf{burn-in cost}} , 
\end{equation}
up to some logarithmic factor, where the burn-in cost is independent of the accuracy level $\varepsilon$. For sufficiently small accuracy level $\varepsilon$, this
results in a sample complexity of
\begin{equation} \label{eq:final-sample-complexity}
\widetilde{O}\left( \frac{S \Cstar H^5}{ \minpall \ror^2\varepsilon^2}  \right).  
\end{equation}
Our theorem suggests that the sample efficiency of robust offline RL critically depends on the problem structure of the given RMDP (i.e. coverage of the optimal robust policy $\pi^{\star}$ as measured by $\minpall$) as well as the quality of the history dataset (as measured by $\Cstar$). Given that $\Cstar$ can be as small as on the order of $1/S$, the sample complexity requirement can exhibit a much weaker dependency with the size of the state space $S$.

On the flip side, to assess the optimality of Theorem~\ref{thm:dro-upper-finite}, we develop an information-theoretic lower bound for robust offline RL as provided in the following theorem.
\begin{theorem}\label{thm:dro-lower-finite}
For any $(H, S, C, \minpall, \sigma, \varepsilon)$ obeying $H\geq 2e^8$, $C \geq 4/S$, $\minpall \in (0, \frac{1}{H}]$, and $\varepsilon \leq \frac{H}{384   e^6\log (1/\minpall)  }$, we can construct a collection of finite-horizon RMDPs $\{\cM_\theta \mymid \theta \in \Theta\}$, an initial state distribution $\rho$, and a batch dataset with $K$ independent sample trajectories each with length $H$ satisfying $2 C\leq \Cstar \leq 4C$, such that 
	\[
	\inf_{\widehat{\pi}}\max_{\theta\in\Theta} \mathbb{P}_{\theta}\left\{  V_1^{\star,\ror}(\rho)-V_1^{\widehat{\pi}, \ror}(\rho)>\varepsilon\right\} \geq\frac{1}{8},
\]
provided that  
\begin{align}
 T =KH \leq 
 \begin{cases}
 \frac{c_1 S\Cstar H^{4}}{\varepsilon^2} \quad \text{ if } \quad 0 < \ror \leq \frac{1}{20H},
  \\
 \frac{c_1 S\Cstar H^{3}}{\minpall \ror^2 \varepsilon^2}  \quad \text{ if } \quad \log (1/\minpall) - 6 \leq \ror \leq \log (1/\minpall) - 5
		\end{cases}.
		\end{align}
Here, $c_1>0$ is some universal constant, the infimum is taken over all estimators $\widehat{\pi}$, and  
$\mathbb{P}_{\theta}$ denotes the probability
when the RMDP is $\mathcal{M}_{\theta}$. 
\end{theorem} 

The messages of Theorem~\ref{thm:dro-lower-finite} are two-fold. 
\begin{itemize}
\item When the uncertainty level $\ror\lesssim 1/H$ is relatively small, Theorem~\ref{thm:dro-lower-finite} shows that no algorithm can succeed in finding an $\varepsilon$-optimal robust policy when the sample complexity falls below the order of 
$$ \Omega\left( \frac{S\Cstar H^{4}}{  \varepsilon^2}  \right), $$
which is at least as large as the sample complexity requirement of non-robust offline RL \citep{li2022settling}. Consequently, this leads to new insights regarding the statistical hardness of   learning robust RMDPs with the KL uncertainty set: it can be at least as hard as the standard MDPs (which corresponds to $\ror=0$), for sufficiently small uncertainty levels. 
\item When the uncertainty level $\ror\asymp  \log (1/\minpall) $, Theorem~\ref{thm:dro-lower-finite} shows that no algorithm can succeed in finding an $\varepsilon$-optimal robust policy when the sample complexity falls below the order of 
$$ \Omega\left( \frac{S\Cstar H^{3}}{\minpall \ror^2 \varepsilon^2}  \right), $$
which confirms the near-optimality of \DRLCB  up to a factor of $H^2$ ignoring logarithmic factors. Therefore, \DRLCB is the first provable algorithm for robust offline RL with a near-optimal sample complexity without requiring the stringent full coverage assumption. 
\end{itemize}

%% file: results_infty.tex
%!TEX root = ./DRO_OfflineRL.tex
\section{Robust offline RL for discounted infinite-horizon RMDPs}
\label{sec:problem-formulation-infty}
In this section, we turn to the studies of robust offline RL for discounted infinite-horizon MDPs. 

\subsection{Backgrounds on discounted infinite-horizon RMDPs}
Similar to the finite-horizon setting, we consider the discounted infinite-horizon robust MDPs (RMDPs) represented by $\cM_{\mathsf{rob}} = \{\cS,\cA, \gamma, \cU^{\ror}(P^\no), r\}$. Here, $\cS = \{1,2,\cdots, S\}$ is the state space, $\cA = \{1,2,\cdots, A\}$ is the action space, $\gamma\in[0,1)$ is the discounted factor, and $r: \cS\times \cA \rightarrow [0,1]$ is the intermediate reward function. Different from the standard MDPs, $\cU^{\ror}(P^\no)$ denote the set of possible transition kernels within an uncertainty set centered around a nominal kernel $P^\no: \cS\times\cA \rightarrow \Delta(\cS)$ using the distance measured in terms of the KL divergence. In particular, given an uncertainty level $\sigma>0$, the uncertainty set around $P^{\no}$ is specified as
\begin{align}\label{eq:kl-ball-infinite-P}
	\unb^\ror(P^{\no})\defn \otimes \; \unb^\ror (P^{\no}_{s,a}),\qquad \unb^\ror(P_{s,a}^0) \defn \left\{ P_{s,a} \in \Delta (\cS): \mathsf{KL}\left(P_{s,a} \parallel P^0_{s,a}\right) \leq \ror \right\},
\end{align}
where we denote a vector of the transition kernel $P$ or $P^{\no}$ at $(s,a)$ respectively as
\begin{align}
	P_{s,a} \defn P(\cdot \mymid s,a) \in \mathbb{R}^{1\times S}, \qquad P_{s,a}^\no \defn P^\no(\cdot \mymid s,a) \in \mathbb{R}^{1\times S}.
\end{align}  

Note that at any time step, the adversary of the nature chooses a history-independent component within the fixed uncertainty set $\cU^\ror(P^0_{s,a})$ defined in (30), conditioned only on the current state-action pair $(s,a)$. This is to ensure the computation tractability of finding such adversary.

\paragraph{Policy and robust value/Q functions.} 
A (possibly random) stationary policy $\pi: \cS \rightarrow \Delta(\cA)$ represents the selection rule of the agent, namely, $\pi(a\mymid s)$ denote the probability of choosing $a$ in state $s$. With some abuse of notation, let $\pi(s)$ represent the action chosen by $\pi$ when $\pi$ is a deterministic policy. We define the {\em robust value function} $V^{\pi, \ror} $ and {\em robust Q-function} $Q^{\pi,\ror}$ respectively as
\begin{align*}
	\forall (s,a)\in \cS \times \cA:\quad  V^{\pi,\ror}(s) &\defn \inf_{P\in \unb^{\ror}(P^{\no})} V^{\pi,P} (s), \qquad Q^{\pi,\ror}(s,a) \defn \inf_{P\in \unb^{\ror}(P^{\no})} Q^{\pi,P}(s,a),
\end{align*}
where the value function $V^{\pi,P}$ and Q-function $Q^{\pi,P}$ w.r.t. policy $\pi$ and transition kernel $P$ are defined respectively by
\begin{align}
	\forall s\in\cS:\qquad V^{\pi,P}(s ) &\defn  \mathbb{E}_{\pi, P} 
	\left[  \sum_{t=0}^{\infty}  \gamma^t r\big(s_{t}, a_t \big) \,\Big|\, s_0=s\right], \label{eq:def_V_infinite} \\
	\forall (s,a)\in \cS \times \cA:\qquad Q^{\pi, P} (s,a ) & \defn \mathbb{E}_{\pi,P} \left[  \sum_{t=0}^{\infty}  \gamma^t r (s_t, a_t ) \,\Big|\, s_0=s, a_0  = a\right] , \label{eq:def_Q_infinite}
\end{align}
where the expectation is taken over the randomness of the trajectory.
In words, the robust value/Q functions characterize the worst case over all the instances in the uncertainty set.

\paragraph{Optimal policy and robust Bellman equation.} Similar to the finite-horizon RMDPs, 
it is well-known that there exists at least one
deterministic policy that maximizes the robust value function and Q-function simultaneously in the infinite-horizon setting as well \citep{iyengar2005robust,nilim2005robust}. With this in mind, we denote the optimal policy as $\pi^\star$ and the corresponding {\em optimal robust value function} (resp.~{\em optimal robust Q-function}) as $V^{\star,\ror}$ (resp.~$Q^{\star,\ror}$), namely
\begin{subequations}
\begin{align}
	\forall s \in \cS: \quad &V^{\star,\ror}(s) \defn V^{\pi^\star,\ror}(s) = \max_\pi V^{\pi,\ror}(s), \\
	\forall (s,a) \in \cS \times \cA: \quad &Q^{\star,\ror}(s,a) \defn Q^{\pi^\star,\ror}(s,a) = \max_\pi Q^{\pi,\ror}(s,a).
\end{align}
\end{subequations}
In addition, we continue to admit the Bellman's optimality principle, resulting in the following {\em robust Bellman consistency equation} (resp.~{\em robust Bellman optimality equation}):
\begin{subequations}
\begin{align}
	\forall (s,a)\in \cS\times \cA: \quad &Q^{\pi,\ror}(s,a) = r(s,a) + \gamma\inf_{P\in \unb^{\ror}(P^{\no}_{s,a})} P V^{\pi,\ror}, \label{eq:bellman-equ-pi-infinite}\\
	\forall (s,a)\in \cS\times \cA: \quad &Q^{\star,\ror}(s,a) = r(s,a) + \gamma\inf_{P\in \unb^{\ror}(P^{\no}_{s,a})} P V^{\star,\ror}. \label{eq:bellman-equ-star-infinite}
\end{align}
\end{subequations}

\paragraph{Occupancy distributions.} To begin,
let $\rho$ be some initial state distribution. We denote $d^{\pi,P}(s \mymid \rho)$ and $d^{\pi,P}(s,a \mymid \rho)$ respectively as the state occupancy distribution and the state-action occupancy distribution induced by policy $\pi$, namely
\begin{subequations}\label{eq:visitation_dist_infty}
\begin{align} 
\forall s \in \cS :\quad\quad  	d^{\pi,P}(s )& \defn (1-\gamma) \sum_{t=0}^\infty \gamma^t \mathbb{P}(s_t = s \mymid s_0 \sim \rho, \pi,P),   \\
\forall (s,a)\in \cS\times \cA :\quad  d^{\pi,P}(s,a) & \defn  (1-\gamma) \sum_{t=0}^\infty \gamma^t \mathbb{P}(s_t = s \mymid s_0 \sim \rho, \pi,P) \, \pi(a \mymid s).
\end{align}
\end{subequations}
Here, the occupancy distributions are conditioned on $s_0\sim \rho$ and the sequence of actions and states are generated based on policy $\pi$ and transition kernel $P$.
Next, applying \eqref{eq:visitation_dist_infty} with $\pi=\pi^\star$, we adopt the the following short-hand notation for the occupancy distributions associated with the optimal policy:
\begin{subequations}
	\begin{align}\label{eq:visitation_dist-optimal-infty}
	\forall s\in \cS :\qquad d^{\star, P}(s)  & \defn d^{\piopt, P}(s),   \\
\forall (s,a)\in \cS\times\cA:\qquad 	d^{\star,P}(s,a )  & \defn d^{\piopt, P}(s,a)= d^{\star, P}(s ) \ind\{a = \piopt(s)\}.
\end{align}
\end{subequations}

\subsection{Data collection and constructing the empirical MDP}
	\label{sec:offline-concentrability-infty}
Suppose that we observe a batch/history dataset $\cD = \{(s_i, a_i, s_i')\}_{1\leq i \leq N}$ consisting of $N$ sample transitions. These transitions are independently generated, where the state-action pair is drawn from some behavior distribution $d^{\mathsf{b}} \in \Delta(\cS\times \cA)$, followed by a next state drawn over the nominal transition kernel $P^\no$, i.e.,
\begin{align}\label{eq:infinite-batch-set-generation}
	(s_i, a_i) \overset{\text{i.i.d.}}{\sim} d^{\mathsf{b}} \quad \text{and} \quad s_i' \overset{\text{i.i.d.}}{\sim} P^\no(\cdot \mymid s_i, a_i), \qquad 1\leq i \leq N.
\end{align}
Armed with these, we are ready to introduce the goal in the infinite-horizon setting. Given the history dataset $\cD$ obeying Assumption~\ref{assumption:dro-infinite}, for some target accuracy $\varepsilon>0$, we aim to find a near-optimal robust policy $\widehat{\pi}$, which satisfies
\begin{align}\label{eq:goal-infty}
	V^{\widehat{\pi},\ror}(\rho) \geq V^{\star,\ror}(\rho) -  \varepsilon
\end{align}
in a sample-efficient manner for some initial state distribution $\rho$.

\begin{remark}
For simplicity, we limit ourselves to the case when the history dataset consists of independent sample transitions. It is not difficult to generalize to the Markovian data case, when we only have access to a single trajectory of data generated by following some behavior policy, by combining the two-fold subsampling trick in \citet[Appendix D]{li2022settling} with our analysis. We leave this extension to interested readers.
\end{remark}

Similar to Assumption~\ref{assumption:dro-finite}, we design the following {\em robust single-policy clipped concentrability} assumption tailored for infinite-horizon RMDPs to characterize the quality of the history dataset.
\begin{assumption}[Robust single-policy clipped concentrability for infinite-horizon MDPs] 
\label{assumption:dro-infinite}
The behavior policy of the history dataset $\mathcal{D}$ satisfies 
\begin{align}
	\max_{(s, a, P) \in \mathcal{S} \times \cA \times \unb^{\ror}(P^{\no})} \frac{\min\big\{d^{\star,P}(s, a ), \frac{1}{S}\big\}}{\myrho(s, a )} \le \Cstar
	\label{eq:concentrate-infinite}
\end{align}
for some finite quantity $\Cstar \in\big[\frac{1}{S},\infty\big)$.
Following the convention $0/0=0$, we denote $\Cstar$ to be the smallest quantity satisfying \eqref{eq:concentrate-infinite}, and refer to it as the robust single-policy clipped concentrability coefficient.
\end{assumption}

\begin{remark}
Similar to Remark~\ref{remark:Cstar_full_finite}, we can bound $\Cstar\leq A$ when the batch dataset is generated using a simulator \citep{yang2021towards,panaganti2021sample}. By combining this bound of $\Cstar$ with the theoretical guarantees developed momentarily in Theorem~\ref{thm:dro-upper-infinite}, we obtain the comparison in Table~\ref{tab:prior-work}.
%Under the setting of generative model \citep{kearns1999finite,yang2021towards,panaganti2021sample}, we have $\Cstar \leq A$. Specifically, it is observed  that $\myrho(s, a) = \frac{1}{SA}$ for all $(s,a)\in \cS\times \cA$, which directly leads to $\Cstar = \max_{(s, a, P) \in \mathcal{S} \times \cA  \times \unb^{\ror}(P^{\no})} \frac{\min\big\{d^{\star,P}(s, a ), \frac{1}{S}\big\}}{\myrho(s, a )}  \leq \frac{1/S}{\frac{1}{SA}} = A$. 
\end{remark}

\paragraph{Building an empirical nominal MDP}
Recalling that we have $N$ independent samples in the dataset $\cD = \{(s_i, a_i, s_i')\}_{1\leq i \leq N}$. First, we denote $N(s,a)$ as the total number of sample transitions from any state-action pair $(s,a)$ as
\begin{align}\label{eq:defn-Nh-sa-infinite}
	N(s,a) &\coloneqq \sum_{i=1}^N \mathds{1} \big\{ (s_i, a_i) = (s,a) \big\}. 
\end{align}
Armed with $N(s,a)$, we construct the empirical estimate $\widehat{P}^\no$ of the nominal kernel $P^\no$ by the visiting frequencies of state-action pairs as follows:
\begin{align}
	\widehat{P}^0(s'\mymid s,a) \defn 
	\begin{cases} \frac{1}{N(s,a)} \sum\limits_{i=1}^N \mathds{1} \big\{ (s_i, a_i, s_i') = (s,a,s') \big\}, & \text{if } N(s,a) > 0 \\
		0, & \text{else}
	\end{cases}  
	\label{eq:empirical-P-infinite}
\end{align}
for any $(s,a,s')\in \cS\times \cA\times  \cS$.

\subsection{\DRLCB for discounted infinite-horizon RMDPs}

With the estimate $\widehat{P}^{\no}$ of the nominal transition kernel $P^{\no}$ in hand, we are positioned to introduce our algorithm \DRLCB for infinite-horizon RMDPs, which bears some similarity with the finite-horizon version (cf. Algorithm~\ref{alg:vi-lcb-dro-finite}), by
taking the uncertainties of the value estimates into consideration throughout the value iterations. The procedure is summarized in Algorithm~\ref{alg:vi-lcb-dro-infinite}.

\begin{algorithm}[t]
\DontPrintSemicolon
	\textbf{input:} a dataset $\mathcal{D}$; reward function $r$; uncertainty level $\ror$; number of iterations $M$. \\ 
	\textbf{initialization:} $\widehat{Q}_0(s,a)= 0$, $\widehat{V}_0(s)=0$ for all $(s,a) \in \cS\times \cA$. \\

	Compute the empirical nominal transition kernel $\widehat{P}^{\no}$ according to \eqref{eq:empirical-P-infinite}; \\
	Compute the penalty term $b(s,a)$ according to \eqref{def:bonus-dro-infinite}; \\
   \For{$m = 1,2,\cdots, M$}
	{
		
		\For{$s\in \cS, a\in \cA$}{
		
			Set $\widehat{Q}_m(s, a)$ according to \eqref{eq:pessimism-operator-equal};  \label{line:infinite-update-Q}\\ 
		}
		\For{$s\in \cS$}{
			Set $\widehat{V}_m(s) = \max_a \widehat{Q}_m(s, a)$;
		}
	}

	\textbf{output:} $\widehat{\pi}$ s.t. $\widehat{\pi}(s) = \arg\max_a \widehat{Q}_M(s,a)$ for all $s \in \cS$.
	\caption{Robust value iteration with LCB (\DRLCB) for infinite-horizon RMDPs.}
 \label{alg:vi-lcb-dro-infinite}
\end{algorithm}

\paragraph{The pessimistic robust Bellman operator.} At the core of \DRLCB is a pessimistic variant of the classical robust Bellman operator  in the infinite-horizon setting \citep{zhou2021finite,iyengar2005robust,nilim2005robust}, denoted as $\cT^\ror(\cdot): \mathbb{R}^{SA} \rightarrow \mathbb{R}^{SA}$, which we recall as follows: 
\begin{align}\label{eq:robust_bellman}
	\forall (s,a)\in &\cS\times \cA :\quad \cT^\ror(Q)(s,a) \defn r(s,a) + \gamma \inf_{ \cP \in \unb^{\sigma}(P^{\no}_{s,a})} \cP V,  \quad \text{with}  \quad  V(s) \defn \max_a Q(s,a).
\end{align}
Encouragingly, the robust Bellman operator shares the nice $\gamma$-contraction property of the standard Bellman operator, ensuring fast convergence of robust value iteration by applying the robust Bellman operator \eqref{eq:robust_bellman} recursively. In the robust offline setting, instead of recursing using the population robust Bellman operator, we need to construct a pessimistic variant of the robust Bellman operator $\tpe(\cdot)$ w.r.t. the empirical nominal kernel $\widehat{P}^\no$ as follows:
\begin{align}\label{eq:pessimism-operator}
	\forall (s,a)\in \cS\times \cA :\quad  \tpe(Q)(s,a) = \max\left\{ r(s, a)  + \gamma \inf_{ \cP \in \unb^{\sigma}(\widehat{P}^{\no}_{s,a})} \cP V - b\big(s,a\big), \,  0 \right\},
\end{align}
where $b(s,a)$ denotes the penalty term that measures the data-dependent uncertainty of the value estimates. 

To specify the tailored penalty term $b(s,a)$ in \eqref{eq:pessimism-operator}, we first introduce an additional term
\begin{align}\label{eq:P-min-hat-def-infinite}
	\forall (s,a)\in \cS\times \cA: \quad \widehat{P}_{\mathsf{min}}(s,a) \defn  \min_{s'} \Big\{\widehat{P}^{\no}(s' \mymid s,a): \; \widehat{P}^{\no}(s' \mymid s,a)>0 \Big\},
\end{align}
which in words represents the smallest positive transition probability of the estimated nominal kernel $\widehat{P}^{\no}(s' \mymid s,a)$. Then for some $\delta\in (0,1)$, some universal constant $\cb>0$, $b(s,a)$ is defined as
\begin{align}
	&b(s,a)= \begin{cases}
	\min \left\{\frac{\cb}{\ror(1-\gamma)} \sqrt{\frac{\log\big(\frac{2(1+\ror)N^3S}{(1-\gamma)\delta}\big)}{ \widehat{P}_{\mathsf{min}}(s,a) N(s,a)}} + \frac{4}{ \ror N(1-\gamma)}\;, \; \frac{1}{1-\gamma}\right\} + \frac{2}{\ror N}& \text{if } N(s,a) > 0, \\
	\frac{1}{1-\gamma}  + \frac{2}{\ror N} & \text{ otherwise}.
	\end{cases}
	\label{def:bonus-dro-infinite}
\end{align}

As shall be illuminated, our proposed {\rm pessimistic robust Bellman operator} $\tpe(\cdot)$ (cf.~\eqref{eq:pessimism-operator}) plays an important role in \DRLCB. Encouragingly, despite the additional data-driven penalty term $b(s,a)$, it still enjoys the celebrated $\gamma$-contractive property, which greatly facilitates the analysis. Before continuing, we summarize the $\gamma$-contraction property below, whose proof is postponed to Appendix~\ref{proof:lem:contration-of-T}. 

\begin{lemma}[$\gamma$-Contraction]\label{lem:contration-of-T}
	For any $\gamma\in  [ 0, 1 )$, the operator $\tpe(\cdot)$ (cf.~\eqref{eq:pessimism-operator}) is a $\gamma$-contraction w.r.t. $\|\cdot \|_\infty$. Namely, for any $Q_1, Q_2 \in \mathbb{R}^{SA}$ s.t. $Q_1(s,a), Q_2(s,a) \in \big[0, \frac{1}{1-\gamma}\big]$ for all $(s,a) \in \cS \times \cA$, one has
	\begin{align}\label{eq:lemma-contraction}
		\left\| \tpe(Q_1) - \tpe(Q_2) \right\|_\infty \leq \gamma \left\| Q_1 -Q_2\right\|_\infty.
	\end{align}
Additionally, there exists a unique fixed point $\widehat{Q}^{\star,\ror}_{\mathsf{pe}}$ of the operator $\tpe(\cdot)$ obeying $0\leq \widehat{Q}^{\star,\ror}_{\mathsf{pe}}(s,a) \leq \frac{1}{1-\gamma}$ for all $(s,a)\in \cS\times \cA$.
\end{lemma}

\paragraph{Our algorithm \DRLCB for infinite-horizon robust offline RL.}
Armed with the $\gamma$-contraction property of the pessimistic robust Bellman operator $\tpe(\cdot)$, we are positioned to introduce \DRLCB for infinite-horizon RMDPs, summarized in Algorithm~\ref{alg:vi-lcb-dro-infinite}. Specifically, \DRLCB can be seen as a value iteration algorithm w.r.t. $\tpe(\cdot)$ (cf.~\eqref{eq:pessimism-operator}), whose update rule at the $m$-th iteration can be formulated as 
\begin{align}\label{eq:pessimism-operator-primal}
	\widehat{Q}_{m}(s,a) &= \tpe(\widehat{Q}_{m-1})(s,a) = \max\left\{ r(s, a)  + \gamma \inf_{ \cP \in \unb^{\sigma}(\widehat{P}^{\no}_{s,a})} \cP \widehat{V}_{m-1} - b\big(s,a\big), \,  0 \right\},
\end{align}
and $\widehat{V}_m(s) = \max_a \widehat{Q}_m(s, a)$ for all $m=1,2,\cdots, M$.
In view of strong duality \citep{hu2013kullback}, the above convex problem can be translated into a dual formulation, leading to the following equivalent update rule:
\begin{align}\label{eq:pessimism-operator-equal}
	\widehat{Q}_{m}(s,a) = \max\left\{ r(s, a)  +\sup_{\lambda\geq 0}  \left\{ -\lambda \log\left(\widehat{P}^0_{s, a} \cdot \exp \left(\frac{-\widehat{V}_{m-1}}{\lambda}\right) \right) - \lambda \ror \right\} - b\big(s,a\big), \,  0 \right\}, 
\end{align}
which can be solved efficiently \citep{iyengar2005robust,yang2021towards,panaganti2021sample} as a one-dimensional optimization problem. 

To finish the description, we initialize the estimates of Q-function ($\widehat{Q}_0$) and value function ($\widehat{V}_0$) to be zero and output the greedy policy of the final Q-estimates ($\widehat{Q}_M$) as the final policy $\widehat{\pi}$, namely,
\begin{align}
	\widehat{\pi}(s) = \arg\max_a \widehat{Q}_M(s,a) \quad \text{for all } s \in \cS.
\end{align}
It turns out that the iterates $\big\{\widehat{Q}_m \big\}_{m\geq0}$ of \DRLCB converge linearly to the fixed point $\widehat{Q}^{\star,\ror}_{\mathsf{pe}}$ owing to the nice $\gamma$-contraction property outlined in Lemma~\ref{lem:contration-of-T}. This fact is summarized in the following lemma, whose proof is postponed to Appendix~\ref{proof:lem:infinite-converge}.

\begin{lemma}\label{lem:infinite-converge}
	Let $\widehat{Q}_0 =0$. The iterates of Algorithm~\ref{alg:vi-lcb-dro-infinite} obey
	\begin{align}
		\forall  m\geq 0: \quad \widehat{Q}_m \leq \widehat{Q}^{\star,\ror}_{\mathsf{pe}} \quad \text{and} \quad \big\| \widehat{Q}_m -\widehat{Q}^{\star,\ror}_{\mathsf{pe}} \big\|_\infty \leq \frac{\gamma^m}{1-\gamma}.
	\end{align}
\end{lemma}

\subsection{Performance guarantees}
 
Before introducing the main theorems, we first define several essential metrics. 
\begin{itemize}
	\item $d_{\mathsf{min}}^{\mathsf{b}}$: the smallest positive entry of the distribution $\myrho$, i.e., 
\begin{align}\label{eq:d-min-b-infinite}
d_{\mathsf{min}}^{\mathsf{b}} \defn  \min_{s,a} \left\{\myrho(s, a ): \; \myrho(s, a ) >0 \right\}.
\end{align}

\item $P_{\mathsf{min}}^{\mathsf{b}}$: the smallest positive state transition probability under the nominal kernel $P^{\no}$ in the region covered by dataset $\cD$, i.e.,
\begin{align}
	P_{\mathsf{min}}^{\mathsf{b}} \defn \min_{s,a,s'} \Big\{P^{\no}\left(s'  \mymid s, a\right):\; \myrho(s, a )>0 ,\, P^{\no}\left(s' \mymid s, a\right)>0 \Big\}. \label{eq:def-P-min-b-infinite}
\end{align}
Note that $P_{\mathsf{min}}^{\mathsf{b}}$ is determined only by the state-action pairs covered by the batch dataset $\cD$.

\item $\minpall$: the smallest positive state transition probability of the optimal robust policy $\piopt$ under the nominal kernel $P^{\no}$, namely
\begin{align}\label{eq:P-min-star-def-infinite}
	\minpall \defn \min_{s,s'} \Big\{P^{\no}\big(s'  \mymid s, \pi^\star(s)\big):\; P^{\no}\big(s'  \mymid s, \pi^\star(s)\big)>0 \Big\}.
\end{align}
We also note that $\minpall$ is determined only by the state-action pairs covered by the optimal robust policy $\pi^\star$ under the nominal model $P^{\no}$.

\end{itemize}

We are now positioned to introduce the sample complexity upper bound of \DRLCB, together with the minimax lower bound, for solving infinite-horizon RMDPs. 
First, we present the performance guarantees of \DRLCB for robust offline RL in the infinite-horizon case, with the proof deferred to Appendix~\ref{proof:thm:dro-upper-infinite}.

\begin{theorem}\label{thm:dro-upper-infinite}
Let $c_0$ and $c_1$ be some sufficiently large universal constants. Given an uncertainty level $\ror>0$, suppose that the penalty terms in Algorithm~\ref{alg:vi-lcb-dro-infinite} are chosen as \eqref{def:bonus-dro-infinite} for sufficiently large $\cb$. With probability at least $1-\delta$, the output $\widehat{\pi}$ of Algorithm~\ref{alg:vi-lcb-dro-infinite} obeys
\begin{align}
	V^{\star, \ror}(\rho) - V^{\widehat{\pi}, \ror}(\rho) \leq \frac{c_0}{\ror(1-\gamma)^2} \sqrt{  \frac{ S\Cstar \log^2 \left(\frac{(1+\ror)N^3S}{(1-\gamma)\delta} \right)}{   P_{\mathsf{min}}^\star N }},
\end{align}
as long as the number of samples $N$ satisfies
\begin{align} \label{eq:dro-b-bound-N-condition-infinite}
	N \geq \frac{c_1 \log( NS/ \delta )}{  d_{\mathsf{min}}^{\mathsf{b}}  P_{\mathsf{min}}^{\mathsf{b}} }.
\end{align}
\end{theorem}

The result directly indicates that \DRLCB can finds an $\varepsilon$-optimal policy as long as the sample size in dataset $\cD$ exceeds the order of (ignoring logarithmic factors)
\begin{align}\label{eq:infinite-upper-bound}
	 \underbrace{\frac{S\Cstar}{P_{\mathsf{min}}^\star (1-\gamma)^4 \ror^2 \varepsilon^2 }}_{\varepsilon\textsf{-dependent}}  + \underbrace{\frac{ 1}{  d_{\mathsf{min}}^{\mathsf{b}}  P_{\mathsf{min}}^{\mathsf{b}}  }}_{\textsf{burn-in cost}}.
\end{align}
Note that the burn-in cost is independent with the accuracy level $\varepsilon$, which tells us that the sample complexity is no more than
\begin{align}
 \widetilde{O} \left( \frac{S \Cstar }{P_{\mathsf{min}}^\star (1-\gamma)^4 \sigma^2 \varepsilon^2 } \right)
\end{align}
as long as $\varepsilon$ is small enough. 
The sample complexity of \DRLCB still dramatically outperforms prior works under full coverage, which has been compared in detail in  Section~\ref{sec:main-contri} (cf.~Table~\ref{tab:prior-work}). In particular, our sample complexity produces an exponential improvement over \citet{zhou2021finite,panaganti2021sample} in terms of the dependency with the effective horizon $\frac{1}{1-\gamma}$, which is especially significant for long-horizon problems. Compared with \citet{yang2021towards}, our sample complexity is better by at least a factor of $S/\minpall$. To achieve the claimed bound, we resort to a delicate technique called the leave-one-out analysis \citep{agarwal2019optimality,li2020breaking,li2022settling}, by carefully  designing an auxiliary set of RMDPs to decouple the statistical dependency introduced across the iterates of pessimistic robust value iteration. This is the first time that the leave-one-out analysis is applied to understanding the sample efficiency of model-based robust RL algorithms, which is of potential independent interest to tighten the sample complexity of other robust RL problems.

To complement the upper bound, we develop an information-theoretic lower bound for robust offline RL as provided in the following theorem whose proof can be found in Appendix~\ref{proof:thm:dro-lower-infinite}.

\begin{theorem}\label{thm:dro-lower-infinite}
For any $(S, \minpall, \Cstar,\gamma, \ror, \varepsilon)$ obeying $\frac{1}{1-\gamma} \geq 2e^8$, $\minpall \in \big(0, 1-\gamma\big]$, $S \geq \log \big( 1/\minpall \big)$, $\Cstar\geq 8/S$, and $\varepsilon \leq \frac{1}{384   e^6 (1-\gamma)\log \big( 1/\minpall \big) }$, we can construct a collection of infinite-horizon RMDPs $\{ \cM_\theta \mymid \theta \in \Theta\}$, an initial state distribution $\rho$, and a batch dataset with $N$ independent samples, such that
	\[
	\inf_{\widehat{\pi}}\max_{\theta\in\Theta} P_\theta \left\{  V^{\star,\ror}(\rho)-V^{\widehat{\pi}, \ror}(\rho)>\varepsilon\big)\right\} \geq\frac{1}{8},
\]
provided that
\begin{align}
  N \leq 
 \begin{cases}
 \frac{c_1 S\Cstar }{(1-\gamma)^3 \varepsilon^2} \quad &\text{ if } \quad 0 < \ror \leq \frac{1-\gamma}{20},
  \\
 \frac{c_1 S\Cstar}{\minpall (1-\gamma)^2\ror^2  \varepsilon^2}  \quad &\text{ if } \quad \log \big( 1/\minpall \big) - 6\leq \ror \leq \log \big( 1/\minpall \big) - 5
		\end{cases}
		\end{align}
Here, $c_1>0$ is some universal constant, the infimum is taken over all estimators $\widehat{\pi}$, and $\mathbb{P}_{\theta}$ denotes the probability
when the RMDP is $\mathcal{M}_{\theta}$.
\end{theorem} 

Similar to the finite-horizon setting, the messages of Theorem~\ref{thm:dro-lower-infinite} are two-fold. 
\begin{itemize}
\item When the uncertainty level $\ror\lesssim 1-\gamma$ is relatively small, Theorem~\ref{thm:dro-lower-infinite} shows that no algorithm can succeed in finding an $\varepsilon$-optimal robust policy when the sample complexity falls below the order of 
$$ \Omega\left( \frac{S\Cstar }{  (1-\gamma)^3 \varepsilon^2}  \right), $$
which is at least as large as the sample complexity requirement of non-robust offline RL \citep{li2022settling}. Consequently, this again suggests that learning robust RMDPs with the KL uncertainty set can be at least as hard as the standard MDPs (which corresponds to $\ror=0$), for sufficiently small uncertainty levels. 
\item When the uncertainty level $\ror\asymp  \log (1/\minpall) $, Theorem~\ref{thm:dro-lower-infinite} shows that no algorithm can succeed in finding an $\varepsilon$-optimal robust policy when the sample complexity falls below the order of 
$$\Omega\left(\frac{S\Cstar}{\minpall (1-\gamma)^2\ror^2  \varepsilon^2}\right),$$ 
which directly confirms that \DRLCB is near-optimal up to a polynomial factor of the effective horizon length $\frac{1}{1-\gamma}$ (cf.~\eqref{eq:infinite-upper-bound}). To the best of our knowledge, \DRLCB is the first provable algorithm with near-optimal sample complexity for infinite-horizon robust offline RL. Moreover, the requirement imposed on the history dataset is also much weaker than prior literature on robust offline RL \citep{yang2021towards,zhou2021finite}, without the need of full coverage of the state-action space.
\end{itemize}

%% file: experiment.tex
%!TEX root = ./DRO_OfflineRL.tex
\section{Numerical experiments}\label{sec:experiments}
We conduct experiments on the gambler's problem \citep{sutton2018reinforcement,zhou2021finite} to evaluate the performance of the proposed algorithm \DRLCB, with comparisons to the robust value iteration algorithm \DRVI without pessimism \citep{panaganti2021sample}. Our code can be accessed at:
\begin{center}
\url{https://github.com/Laixishi/Robust-RL-with-KL-divergence}. 
\end{center}

\paragraph{Gambler's problem.}
In the gambler's game \citep{sutton2018reinforcement,zhou2021finite},  a gambler bets on a sequence of  coin flips, winning the stake with heads and losing with tails. Starting from some initial balance, the game ends when the gambler's balance either reaches $50$ or $0$, or the total number of bets $H$ is hit. This problem can be formulated as an episodic finite-horizon MDP, with a state space $\mathcal{S} =\{0,1,\cdots, 50\}$ and the associated possible actions $a\in \big\{0,1,\cdots, \min\{s, 50-s\} \big\}$ at state $s$. Here, we set the horizon length $H = 100$. Moreover, the parameter of the transition kernel, which is the probability of heads for the coin flip, is fixed as $p_{\mathsf{head}}$ and remains the same in all time steps $h\in[H]$. The reward is set as $1$ when the state reaches $s=50$ and $0$ for all other cases. In addition, suppose the initial state (i.e., the gambler's initial balance) distribution $\rho$ is taken uniformly at random within $\mathcal{S}$. 
 
\begin{figure}[ht]
	\centering
	\begin{tabular}{cc}
		\includegraphics[width=0.4\linewidth]{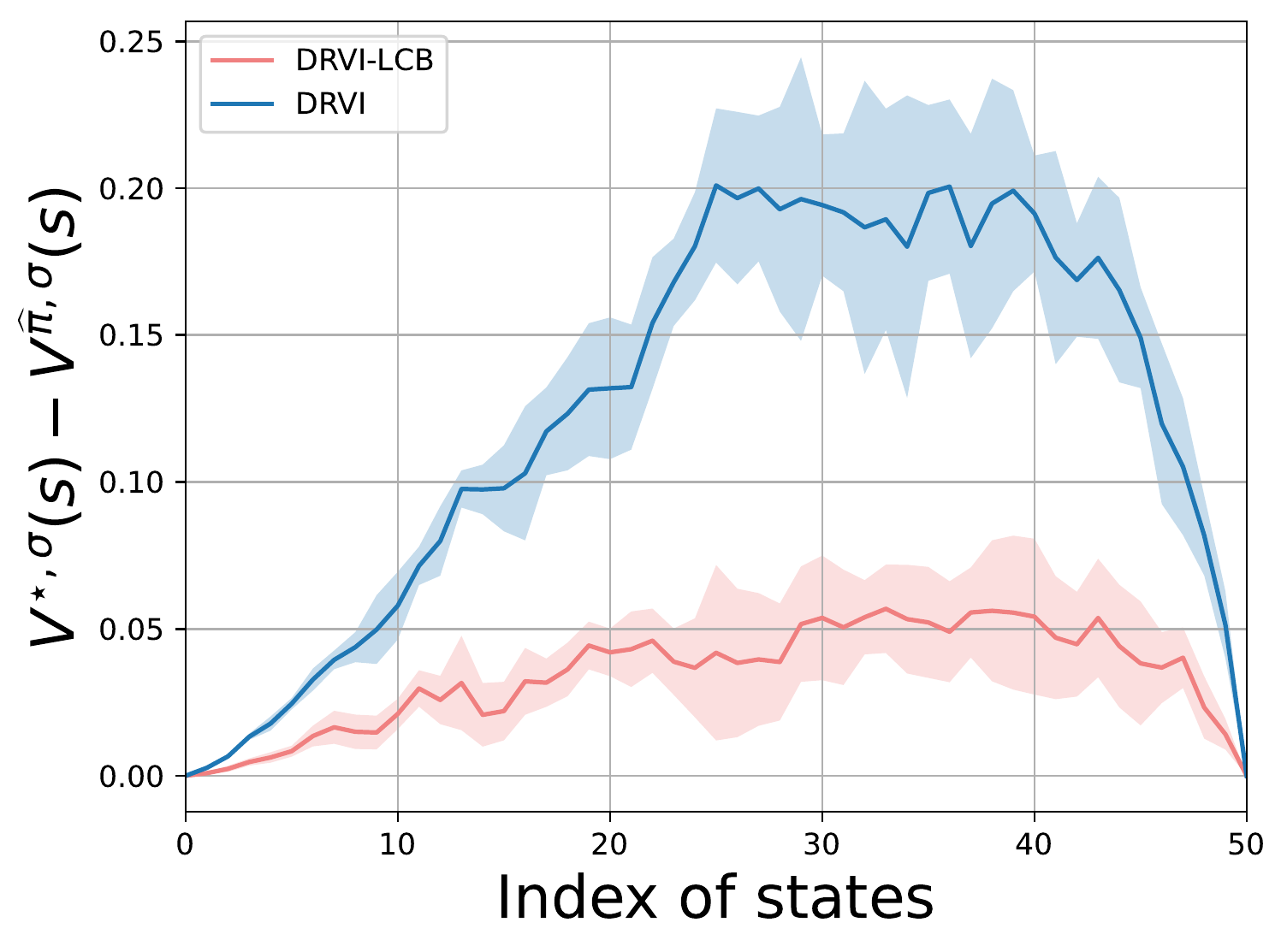} & \includegraphics[width=0.395\linewidth]{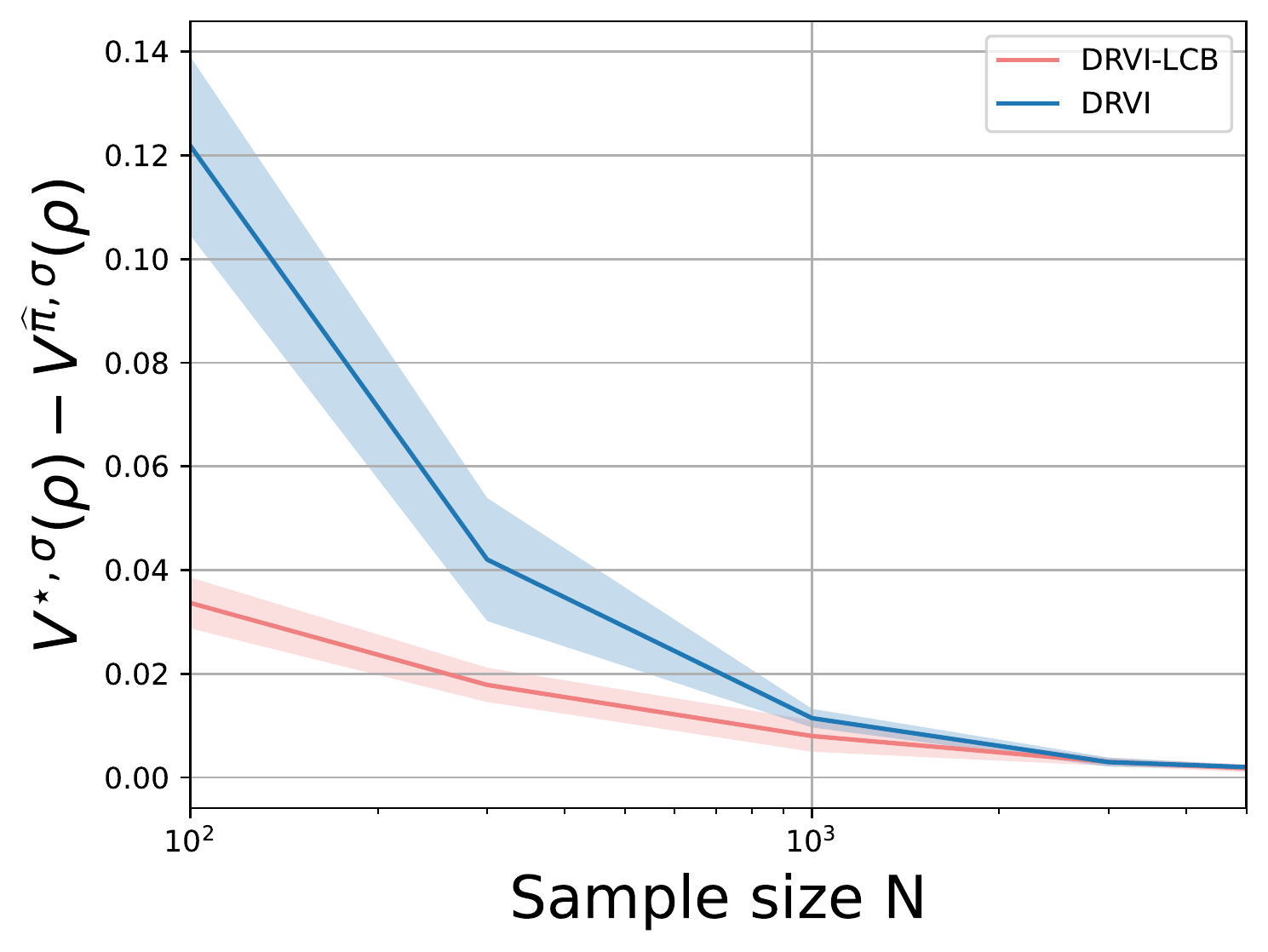} \\
				(a) value gap versus states & (b) value gap versus sample size \\
		 \includegraphics[width=0.4\linewidth]{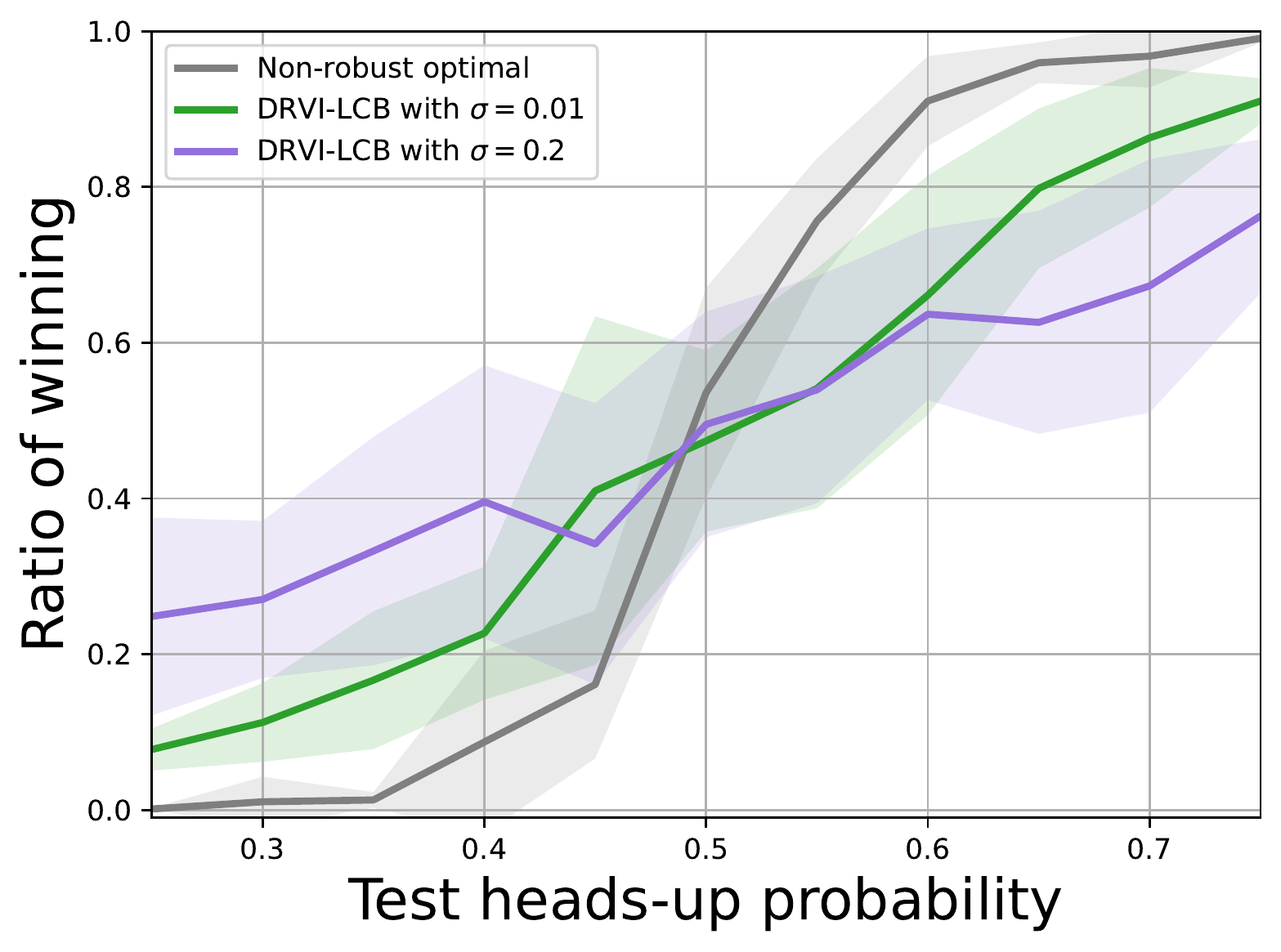} & \includegraphics[width=0.4\linewidth]{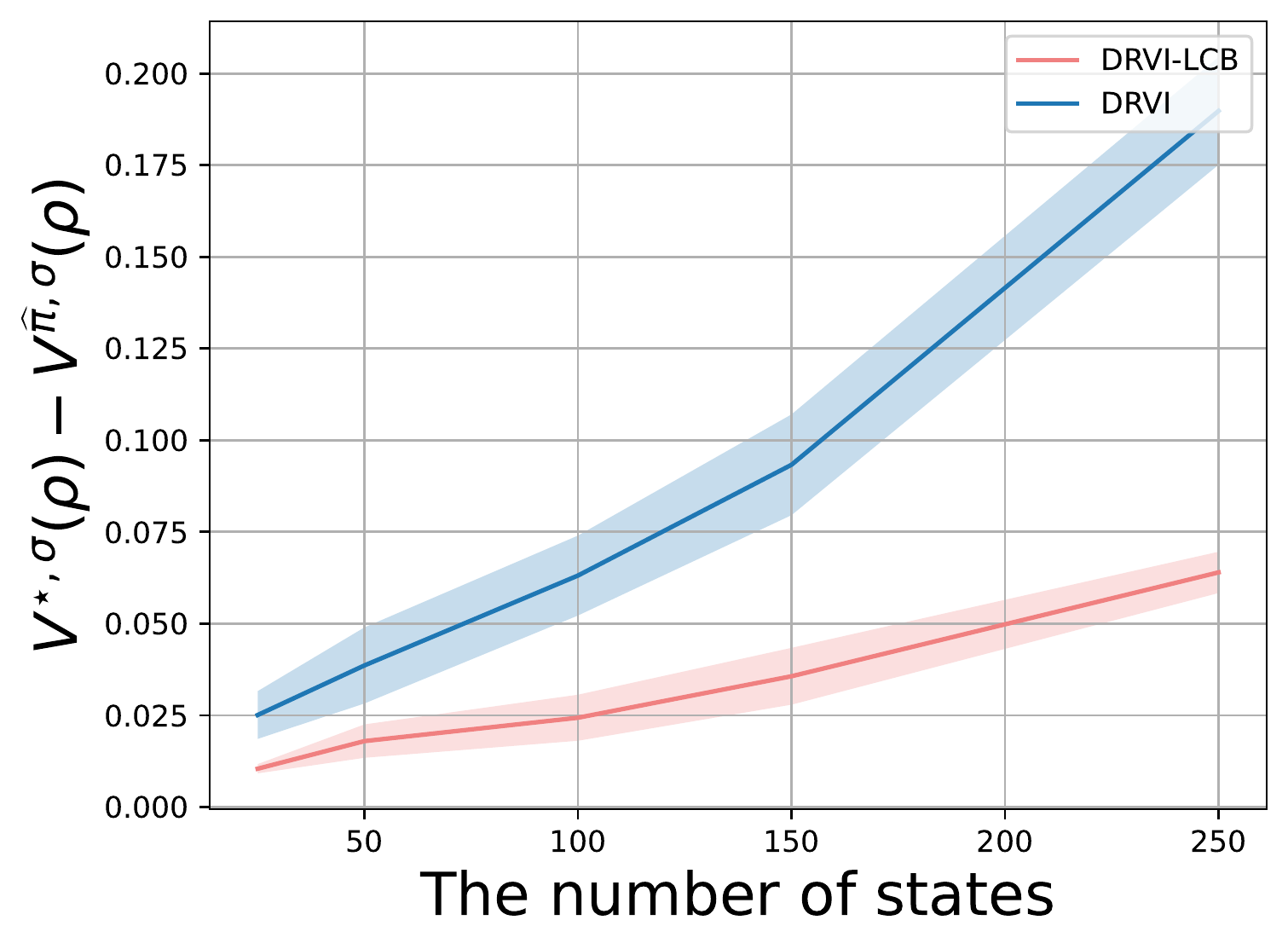}  \\
		 (c) winning rate versus $p_{\mathsf{head}}$ & (d) value gap versus the number of states
	\end{tabular}
	\caption{Performance of the proposed algorithm \DRLCB using independent samples per state-action pair and time step, where it shows better sample efficiency than the baseline algorithm \DRVI without pessimism, as well as better robustness in the learned policy compare to its non-robust counterpart.  }
		\label{fig:test_on_s}
	\end{figure}

\paragraph{The benefit of pessimism.} We first utilize a history dataset with $N$ independent samples per state-action pair and time step, generated from the nominal MDP with $p_{\mathsf{head}}^{\no} = 0.6$. 
We evaluate the performance of the learned policy $\widehat{\pi}$ using our proposed method \DRLCB with comparison to \DRVI without pessimism, where we fix the uncertainty level $\sigma = 0.1$ for learning the robust optimal policy. The experiments are repeated $10$ times with the average and standard deviations reported.
To begin with, Figure \ref{fig:test_on_s}(a) plots the sub-optimality value gap $V_1^{\star,\ror}(s) - V_1^{\widehat{\pi}, \ror}(s)$ for every $s\in\cS$, when a sample size $N =100$ is used to learn the robust policies. It is shown that \DRLCB outperform the baseline \DRVI uniformly over the state space when the sample size is small, corroborating the benefit of pessimism in the sample-starved regime. Furthermore, Figure~\ref{fig:test_on_s}(b) shows the sub-optimality gap $V_1^{\star,\ror}(\rho) - V_1^{\widehat{\pi}, \ror}(\rho)$ with varying sample sizes $N = 100, 300, 1000, 3000, 5000$, where the initial test distribution $\rho$ is generated randomly.\footnote{The probability distribution vector $\rho \in \Delta(\cS)$ is generated as $\rho(s) =u_s/\sum_{s\in \cS} u_s$, where $u_s$ is drawn independently from a uniform distribution.} While the performance of \DRLCB and \DRVI both improves with the increase of the sample size,  the proposed algorithm \DRLCB achieves much better performance with fewer samples.

\paragraph{The benefit of distributional robustness.} To corroborate the benefit of distributional robustness, we evaluate the performance of the policy learned from $N=1000$ samples using \DRLCB on perturbed environments with varying model parameters $p_{\mathsf{head}} \in[0.25,0.75]$. We measure the practical performance based on the ratio of winning (i.e., reaching the state $s=50$) calculated from $3000$ episodes. Figure~\ref{fig:test_on_s}(c) illustrates the ratio of winning  against the test probability of heads for the policies learned from \DRLCB with $\ror=0.01$ and $\ror=0.2$, which are benchmarked against the non-robust optimal policy of the nominal MDP using the exact model. It can be seen that the policies learned from \DRLCB deviate from the non-robust optimal policy as $\ror$ increases, which achieves better worst-case rates of winning across a wide range of perturbed environments. On the other end, while the non-robust policy maximizes the performance when the test environment is close to the history one used for training, its performance degenerates to be much worse than the robust policies when the probability of heads is mismatched significantly, especially when $p_{\mathsf{head}}$ drops below, say around, $0.5$. 

\paragraph{Impact of the number of states.} We evaluate the performance of \DRLCB and \DRVI when the number of states varies within $[25,50, 100, 150, 200, 250]$, using a fixed sample size $N=300$.
Figure~\ref{fig:test_on_s}(d) show that \DRLCB performs consistently better than \DRVI as the number of states $S$ increases, where the value gap exhibits a linear scaling with respect to the number of states.

\paragraph{Performance using trajectory data.} Instead of using independent samples, we now evaluate the proposed algorithm using a dataset consisting of $K$ sample trajectories generated from a uniform random policy, for the same setting of Figure \ref{fig:test_on_s}(b). Figure~\ref{fig:ablation} shows the sub-optimality value gap with respect to the number of trajectories $K$, where the performance of both \DRLCB and \DRVI improves as $K$ increases, and \DRLCB achieves better performance especially when $K$ is small, consistent with the observation under independent data.

\begin{figure}[t]
	\centering
\includegraphics[width=0.4\linewidth]{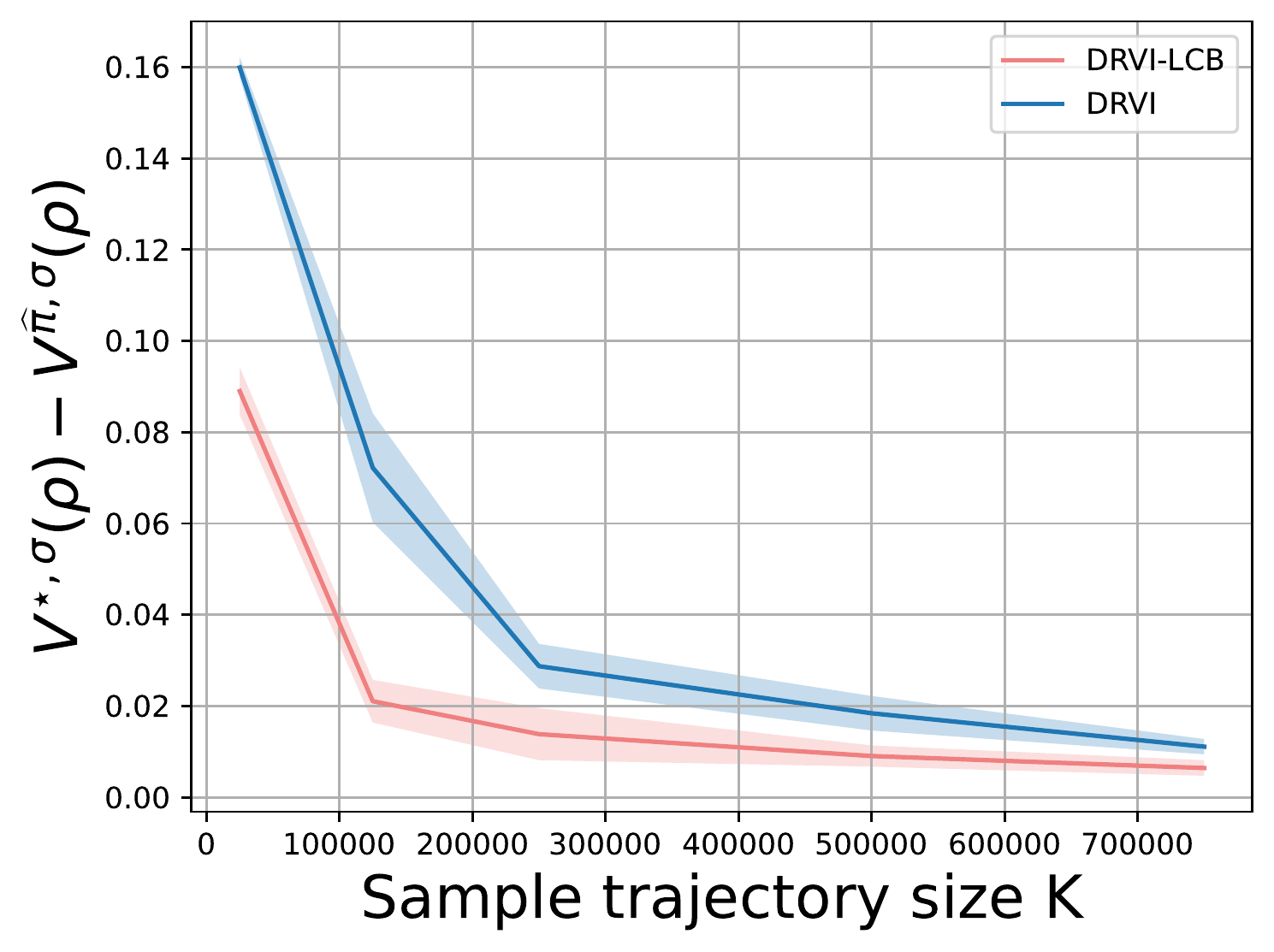}  
	\caption{Performance of the proposed algorithm \DRLCB compared against \DRVI using trajectory data.} 
		\label{fig:ablation}
	\end{figure}

%
%\begin{figure}[ht]
%	\centering
%	\begin{tabular}{ccc}
%		\includegraphics[width=0.3\linewidth]{./figs/test_lcb_vi_state_frozen.pdf} & \includegraphics[width=0.3\linewidth]{./figs/test_lcb_vi_N_frozen.pdf} & \includegraphics[width=0.3\linewidth]{./figs/test_winning_frozen.pdf} \\
%		(a) Value gap versus states & (b) Value gap versus sample size & (c) winning rate versus $p_{\mathsf{direction}}$
%	\end{tabular}
%	\caption{The performance evaluation of the proposed algorithm \DRLCB, where it shows better sample efficiency than the baseline algorithm \DRVI without pessimism, as well as better robustness in the learned policy compare to its non-robust counterpart.  }
%		\label{fig:test_on_frozen_lake}
%	\end{figure}
 

%% file: conclusion.tex
\section{Conclusion}\label{sec:discussions}

To accommodate both model robustness and sample efficiency, this paper proposes a distributionally robust model-based algorithm for offline RL with the principle of pessimism. We study the finite-sample complexity of the proposed algorithm \DRLCB, and 
establish an information-theoretic lower bound to benchmark its near-optimality for a range of uncertainty levels. Numerical experiments are provided to demonstrate the efficacy of the proposed algorithm. To the best our knowledge, this provides the first provably near-optimal robust offline RL algorithm that learns under model perturbation and partial coverage. This work opens up several interesting directions. 
\begin{itemize}

\item {\em Tightening the gap between upper and lower bounds.} Our upper and lower bounds still leave room for future improvements. For example, it is yet to establish the information-theoretic lower bound over the full range of the uncertainty set, and close the gap between the upper and lower bounds with respect to the horizon length.

\item {\em Model-free algorithms for robust offline RL.} Can we design provably efficient model-free algorithms for robust offline RL with partial coverage? Recent works \citep{wang2023finite,wang2023sample} in understanding robust variants of Q-learning in the generative model might shed light on how to approach this question.
%\item {\em Multi-agent robust offline RL.} It is appealing to extend the algorithm design to the multi-agent setting, where in a scalable manner? 

\item {\em Choice of uncertainty sets.} Moreover, it is possible to extend our framework to handle uncertainty sets defined using other distances such as the chi-square distance and the total variation distance in a similar fashion. \citet{shi2023curious} recently established near minimax-optimal sample complexities for the total variation distance and the chi-square distance in the generative model setting, paving ways to study these uncertainty sets in the offline setting.

\item {\em Adaptive tuning of the uncertainty set.} In this work, we treat the radius of the uncertainty set as a fixed, a priori specified parameter, and study the sample complexity of learning a robust optimal policy with respect to the given uncertainty set modeling the sim-to-real gap. It is of great interest to incorporate the tuning of the uncertainty set (both its size and metric) to complete the pipeline of the algorithm design, which will require a different framework than the one adopted in the current paper.

\end{itemize}

We leave these questions to future investigations.

%% file: auxiliary-lemmas.tex
%!TEX root = ./DRO-offline.tex
\section{Preliminaries} 
Before starting, let us introduce some additional notation  useful throughout the theoretical analysis. Let $\mathrm{ess}\inf X$ denote the essential infimum of a function/variable $X$.

\subsection{Properties of the robust Bellman operator}
To begin with, we introduce the following strong duality lemma which is widely used in distributionally robust optimization when the uncertainty set is defined with respect to the KL divergence.

\begin{lemma}[\citep{hu2013kullback}, Theorem~1]\label{lem:strong-duality}
Suppose $f(x)$ has a finite moment generating function in some neighborhood around $x = 0$,
then for any $\ror >0$ and a nominal distribution $P^{\no}$, we have
\begin{equation}
	\sup_{\cP\in \unb^\ror(P^{\no})} \mathbb{E}_{X\sim \cP} [f(X)] = \inf_{\lambda \geq 0} \left\{\lambda \log \mathbb{E}_{X\sim P^{\no}}\left[\exp \left( \frac{f(X)}{\lambda} \right) \right] + \lambda  \ror  \right\}.
\end{equation}
\end{lemma}

Armed with the above lemma, it is easily verified that for any positive constant $M$ and a nominal distribution vector $P^{\no} \in \mathbb{R}^{1\times S}$ supported over the state space $\cS$, if $X(s)\in [0,M]$ for all $s\in\cS$, then 
\begin{align}\label{eq:strong-duality-mdp}
	\inf_{\cP\in \unb^\ror(P^{\no})}  \cP X = \sup_{\lambda \geq 0}  \left\{ -\lambda \log\left(P^0   \exp \left(-\frac{X}{\lambda}\right) \right) - \lambda \ror \right\}.
\end{align}

For convenience, we introduce the following lemma, paraphrased from \citet[Lemma~4]{zhou2021finite} and its proof, to further characterize several essential properties of the optimal dual value. 
\begin{lemma}[\citep{zhou2021finite}] \label{lem:lambda-n-bound}
	 Let $ X \sim P$ be a bounded random variable with $X \in [0,M]$.  Let $\ror > 0$ be any uncertainty level and the corresponding optimal dual variable be 
	 \begin{align} 
	 \lambda^{\star} \in \arg\max_{\lambda\geq 0} \; f(\lambda, P), \qquad \mbox{where~} f(\lambda, P) \defn   \left\{-\lambda \log \mathbb{E}_{X\sim P}\left[ \exp\left(\frac{-X}{\lambda}\right)\right] - \lambda  \ror  \right\}.
	 \end{align}
	 Then the optimal value $\lambda^\star$ obeys 
    \begin{align}\label{eq:lambda-n-range}
    	 \lambda^\star \in \left[0, \frac{M}{\ror}\right],
    \end{align}
    where $\lambda^\star = 0$ if and only if 
    \begin{align}\label{eq:lambda-star-0-condition}
  \log \big( \mathbb{P}(X = \mathrm{essinf} X)\big) + \ror
   \geq 0.
\end{align}
Moreover, when $\lambda^\star=0$, we have 
\begin{align}\label{eq:lambda-0-func-value}
    \lim_{\lambda \rightarrow 0} f(\lambda, P) = \lim_{\lambda \rightarrow 0}  \left\{-\lambda \log \mathbb{E}_{X\sim P}\left[ \exp\left( \frac{-X}{\lambda} \right) \right] - \lambda  \ror  \right\} = \mathrm{essinf}  X.
\end{align}
\end{lemma}

\subsection{Concentration inequalities}

In light of Lemma~\ref{lem:lambda-n-bound} (cf.~\ref{eq:lambda-0-func-value}), we are interested in comparing the values of $\mathrm{essinf} X$ when $X$ is drawn from the population nominal distribution or its empirical estimate. This is supplied by the following lemma from \citet{zhou2021finite}.

\begin{lemma}[\citep{zhou2021finite}] \label{lem:empirical-essinf}
Let $ X \sim P$ be a discrete bounded random variable with $X \in [0,M]$.  Let $P_n$ denote the empirical distribution constructed from $n$ independent samples $X_1, X_2, \cdots, X_n$, and let $\widehat{X} \sim P_n$. Denote $P_{\mathsf{min}, X}$ as the smallest positive probability $ P_{\mathsf{min}, X} \defn \min\{\mathbb{P}(X = x): x\in \mathsf{supp}(X)\}$, where $\mathsf{supp}(X)$ is the support of $X$. Then for any $\delta \in (0,1)$, with probability at least $1 -  \delta $, we have
   \begin{align}\label{eq:essinf-of-P-n}
      \min_{i\in[n]} X_i = \mathrm{essinf} \widehat{X} =  \mathrm{essinf}  X,
   \end{align}
   as long as 
   \begin{align}
    n \geq  - \frac{\log(2 /\delta)}{\log (1- P_{\mathsf{min},X})}.
   \end{align}
 \end{lemma}

We next gather a few elementary facts about the Binomial distribution, which will be useful throughout the proof.
\begin{lemma}[Chernoff's inequality]\label{lem:binomial-small}
Suppose $N\sim \mathsf{Binomial}(n,p)$, where $n\geq 1$ and $p\in [0,1)$. For some universal constant $c_{\mathsf{f}}>0$, we have  
  \begin{align}
    \mathbb{P}\left( \left| N/n-p\right| \geq pt\right) & \leq \exp\left(-c_{\mathsf{f}} np t^2\right), \qquad \forall t\in[0,1].  \label{equ:binomial-1} 
  \end{align}
\end{lemma} 

\begin{lemma}[{\citep[Lemma~8]{shi2022pessimistic}}]\label{lem:binomial} 
  Suppose $N\sim \mathsf{Binomial}(n,p)$, where $n\geq 1$ and $p\in [0,1]$. For any $\delta\in (0,1)$, we have
\begin{subequations}
  \label{equ:binomial-all}
  \begin{align}
    N&\geq \frac{np}{8\log\left(\frac{1}{\delta}\right)} \qquad \text{ if } np \geq 8\log\left(\frac{1}{\delta}\right), \label{equ:binomial-3}\\
    N &\leq \begin{cases}
 e^2np & \text{ if } np \geq \log\left(\frac{1}{\delta}\right), \\ 
2e^2 \log\left(\frac{1}{\delta}\right) & \text{ if } np \leq 2\log\left(\frac{1}{\delta}\right)     
\end{cases}  \label{equ:binomial-2}
  \end{align}
\end{subequations}
 hold with probability at least $1 - 4\delta$. 
 
\end{lemma}

\subsection{Kullback-Leibler (KL) divergence}
We next introduce some useful facts about the Kullback-Leibler (KL) divergence for two distributions $P$ and $Q$, denoted as $\mathsf{KL}(P \parallel Q)$. 
%Towards this, let's first introduce the notations of several useful distributions. 
Denoting $\mathsf{Ber}(p)$(resp.~$\mathsf{Ber}(q)$) as the Bernoulli distribution with mean $p$ (resp.~$q$), we introduce
\begin{align}    \label{eq:defn-KL-bernoulli}
    \mathsf{KL}\big(\mathsf{Ber}(p) \parallel \mathsf{Ber}(q)\big) &\defn p\log\frac{p}{q} +(1-p)\log\frac{1-p}{1-q},   
\end{align}
which represents the KL divergence from $\mathsf{Ber}(p)$ to $\mathsf{Ber}(q)$. We now introduce the following lemma.

\begin{lemma}
    \label{lem:KL-key-result} 
    For any $p, q \in \left[\frac{1}{2},1\right)$ and $p >q$, it holds that 
    \begin{align} 
    \mathsf{KL}\big(\mathsf{Ber}(p) \parallel \mathsf{Ber}(q)\big)  \leq \mathsf{KL}\big(\mathsf{Ber}(q) \parallel \mathsf{Ber}(p)\big) 
      \leq   \frac{(p-q)^2}{p(1-p)}. \label{eq:KL-dis-key-result} 
    \end{align}
    Moreover, for any $0 \leq x < y < q$, it holds
\begin{align}\label{eq:show-subset-of-sigma-sigma'}
 \mathsf{KL}\left(\mathsf{Ber}\left(x\right) \parallel \mathsf{Ber}(q)\right) > \mathsf{KL}\left(\mathsf{Ber}\left(y\right) \parallel \mathsf{Ber}(q)\right).
\end{align}
\end{lemma}
 
 \begin{proof}
 The first half of this lemma is proven in \citet[Lemma~10]{li2022settling}. For the latter half, it follows from that the function
\begin{align*}
    f(x, q) \defn \mathsf{KL}\left(\mathsf{Ber}\left(x\right) \parallel \mathsf{Ber}(q)\right) 
\end{align*}
is monotonically decreasing for all $x\in(0,q]$, since its derivative with respect to $x$ satisfies 
$\frac{\partial f(x, q)}{\partial x} = \log\frac{x}{q} + \log \frac{1-q}{1-x} < 0$.
\end{proof}

%% file: upper-bound-analysis.tex
%!TEX root = ./DRO-offline.tex
\section{Analysis: episodic finite-horizon RMDPs}\label{sec:analysis}

\subsection{Proof of Theorem~\ref{thm:dro-upper-finite}}

Before starting, we introduce several additional notation that will be useful in the analysis.
First, we denote the state-action space covered by the behavior policy $\pib$ in the nominal model $P^{\no}$ as  
\begin{align}\label{eq:cover-space-pib}
	\cC^{\mathsf{b}} = \left\{(h,s,a): \myrho_h(s, a ) >0 \right\}. 
\end{align}
Moreover, we recall the definition in \eqref{eq:P-min-hat-def} and define a similar one based on the exact nominal model $P^{\no}$ as
\begin{align}\label{eq:P-min-pib} 
	P_{\mathsf{min},h}(s,a) \defn  \min_{s'} \Big\{P_h^{\no}(s' \mymid s,a): \; P_h^{\no}(s' \mymid s,a)>0 \Big\}.
\end{align} 
Clearly, by comparing with the definitions \eqref{eq:P-min-star-def} and \eqref{eq:def-P-min-b}, it holds that 
\begin{equation}\label{eq:link_minpall_pmin}
\minpall = \min_{h,s}\; P_{\mathsf{min},h}(s,\pi_h^{\star}(s)) , \qquad P_{\mathsf{min}}^{\mathsf{b}} = \min_{(h,s,a)\in \cC^{\mathsf{b}} }\;  P_{\mathsf{min},h}(s,a). 
\end{equation} 
For any time step $h \in [H]$, we denote the set of possible state occupancy distributions  associated with the optimal policy $\pi^\star$ in a model  within the uncertainty set $P\in \unb^{\ror} \left(P^{\no} \right)$ as
\begin{align} \label{eq:def-D-star-h}
	\cD^\star_h \defn \left\{ \left[d_h^{\star,P}(s)\right]_{s\in\cS} : P \in \unb^{\ror} \left(P^{\no} \right) \right\} = \left\{ \left[d_h^{\star,P}\big(s,\pi_h^\star(s) \big)\right]_{s\in\cS} : P \in \unb^{\ror} \left(P^{\no} \right)\right\},
\end{align}
where the second equality is due to the fact that $\pi^\star$ is chosen to be deterministic.

With these in place, the proof of Theorem~\ref{thm:dro-upper-finite} is separated into several key steps, as outlined below.
\paragraph{Step 1: establishing the pessimism property.} 
To achieve this claim, we heavily count on the following lemma whose proof can be found in Appendix~\ref{proof:lemma:dro-b-bound}.  
\begin{lemma}\label{lemma:dro-b-bound}
Instate the assumptions in Theorem~\ref{thm:dro-upper-finite}.
Then for all $(h,s,a)\in  [H]\times \cS\times \cA$, consider any vector $V\in \mathbb{R}^S$ independent of $\widehat{P}^0_{h,s,a}$ obeying $\|V\|_{\infty} \le H$. With probability at least $1- \delta$, one has 
\begin{align}\label{eq:dro-b-bound}
	&\left| \inf_{ \cP \in \unb^{\sigma}(\widehat{P}_{h,s,a}^0)} \cP V - \inf_{ \cP \in \unb^{\sigma}(P^{\no}_{h,s,a})}  \cP V \right| \leq b_h (s,a)
\end{align}
with $b_h (s,a)$ given in \eqref{def:bonus-dro}. Moreover, for all $(h,s,a)\in \cC^{\mathsf{b}} $, with probability at least $1-\delta$, one has
\begin{align}
  	   \frac{P_{\mathsf{min},h}(s,a)}{8  \log(KHS / \delta)} \leq   \widehat{P}_{\mathsf{min},h}(s,a)  \leq e^2 P_{\mathsf{min},h}(s,a) . \label{eq:convert-pmin-to-estimation}
\end{align}
 
\end{lemma}

Armed with the above lemma,  with probability at least $1-\delta$, we shall show the following relation holds 
\begin{align}\label{eq:finite-pessimism-assertion}
	&\forall (s,a,h) \in \cS\times \cA\times [H+1]:
	\qquad \widehat{Q}_h(s,a) \leq Q^{\widehat{\pi},\ror}_h(s,a),\qquad \widehat{V}_h(s) \leq V^{\widehat{\pi},\ror}_h(s),
\end{align}
which means that $\widehat{Q}_h$ (resp. $\widehat{V}_h$) is a pessimistic estimate of $Q^{\widehat{\pi},\ror}_h$ (resp. $V^{\widehat{\pi},\ror}_h$).
Towards this, it is easily verified that the latter assertion concerning $V^{\widehat{\pi},\ror}_h$ is implied by the former, since
\begin{align}
	\widehat{V}_h(s) = \max_a \widehat{Q}_h(s, a) \leq \max_a Q^{\widehat{\pi},\ror}_h(s,a) = V^{\widehat{\pi},\ror}_h(s).
\end{align}
Therefore, the remainder of this step focuses on verifying the former assertion in \eqref{eq:finite-pessimism-assertion} by induction. 
\begin{itemize}
\item To begin, the claim \eqref{eq:finite-pessimism-assertion} holds at the base case when $h = H+1$, by invoking the trivial fact $\widehat{Q}_{H+1}(s,a) = Q^{\widehat{\pi},\sigma}_{H+1}(s,a) = 0$. 
\item Then, suppose that $\widehat{Q}_{h+1}(s,a) \leq Q^{\widehat{\pi},\sigma}_{h+1}(s,a)$ holds for all $(s,a)\in\cS\times \cA$ at some time step $h\in [H]$, it boils down to show $\widehat{Q}_h(s,a) \leq Q^{\widehat{\pi},\ror}_h(s,a)$.

By the update rule of $\widehat{Q}_h(s,a)$ in Algorithm~\ref{alg:vi-lcb-dro-finite} (cf.~line~\ref{line:finite-update-Q}), the above relation holds immediately if $\widehat{Q}_h(s,a)=0$ since $\widehat{Q}_h(s,a) =0 \leq Q^{\widehat{\pi},\ror}_h(s,a)$. Otherwise, $\widehat{Q}_h(s,a)$ is updated via
\begin{align}
 \widehat{Q}_h(s,a) 
	&= r_h(s, a) + \sup_{\lambda\geq 0}  \left\{ -\lambda \log\left(\widehat{P}^0_{h, s, a} \cdot \exp \left(\frac{-\widehat{V}_{h+1}}{\lambda}\right) \right) - \lambda \ror \right\} - b_h(s, a) \notag\\
	& \overset{\mathrm{(i)}}{=} r_h(s, a) + \inf_{ \cP \in \unb^{\ror}(\widehat{P}_{h,s,a}^0)}   \cP \widehat{V}_{h+1} - b_h(s, a) \notag\\
	& \leq r_h(s, a) + \inf_{ \cP \in \unb^{\ror}(P_{h,s,a}^0)}  \cP \widehat{V}_{h+1}   + \left|\inf_{ \cP \in \unb^{\ror}(\widehat{P}_{h,s,a}^0)}  \cP \widehat{V}_{h+1} - \inf_{ \cP \in \unb^{\ror}(P^0_{h,s,a})}  \cP \widehat{V}_{h+1} \right| - b_h(s,a) \notag\\
	&\overset{\mathrm{(ii)}}{\leq} r_h(s, a) + \inf_{ \cP \in \unb^{\ror}(P_{h,s,a}^0)}  \cP V^{\widehat{\pi}, \ror}_{h+1} + 0 \overset{\mathrm{(iii)}}{=}  Q^{\widehat{\pi}, \ror}_h(s,a),
\end{align}
where (i) rewrites the update rule back to its primal form (cf.~\eqref{eq:VI-primal}), (ii) holds by applying \eqref{eq:dro-b-bound} with the condition \eqref{eq:dro-b-bound-N-condition} satisfied and the induction hypothesis $\widehat{V}_{h+1} \leq V^{\widehat{\pi},\ror}_{h+1}$, and lastly, (iii) follows by the robust Bellman consistency equation \eqref{eq:robust_bellman_consistency}. 
\end{itemize}
Putting them together, we have verified the claim \eqref{eq:finite-pessimism-assertion} by induction.

\paragraph{Step 2: bounding $V_h^{\star,\ror}(s) - V^{\widehat{\pi},\ror}_{h}(s)$.} With the pessimism property \eqref{eq:finite-pessimism-assertion} in place, we observe that the following relation holds
\begin{align}
	0\leq V_h^{\star,\ror}(s) - V^{\widehat{\pi},\ror}_{h}(s) \leq V_h^{\star,\ror}(s) - \widehat{V}_{h}(s) \leq Q_h^{\star,\ror}\big(s, \pi^\star_h(s)\big) - \widehat{Q}_{h} \big(s, \pi^\star_h(s) \big),
\end{align}
where the last inequality follows from $\widehat{Q}_{h} \big(s, \pi^\star_h(s) \big) \leq \max_a \widehat{Q}_{h}  (s, a ) = \widehat{V}_{h}(s)$.
Then, by the robust Bellman optimality equation in \eqref{eq:robust_bellman_optimality} and the primal version of the update rule (cf.~\eqref{eq:VI-primal})
\begin{align*}
	Q_h^{\star,\ror}\big(s, \pi^\star_h(s)\big) &= r_h\big(s, \pi^\star_h(s)\big) + \inf_{ \cP \in \unb^{\ror}\big(P^0_{h,s,\pi^\star_h(s)} \big)}  \cP V_{h+1}^{\star,\ror},\\
	\widehat{Q}_h\big(s, \pi^\star_h(s)\big) &= r_h\big(s, \pi^\star_h(s)\big)  + \inf_{ \cP \in \unb^{\ror} \big(\widehat{P}_{h,s,\pi^\star_h(s)}^0 \big)} \cP \widehat{V}_{h+1} - b_h\left(s, \pi^\star_h(s)  \right),
\end{align*} 
we arrive at
\begin{align}
	V_h^{\star,\ror}(s) - \widehat{V}_{h}(s) & \leq Q_h^{\star,\ror}\big(s, \pi^\star_h(s)\big) - \widehat{Q}_{h} \big(s, \pi^\star_h(s) \big) \notag\\
	&= \inf_{ \cP \in \unb^{\ror} \big(P^0_{h,s,\pi^\star_h(s)} \big)}  \cP V_{h+1}^{\star,\ror} -\inf_{ \cP \in \unb^{\ror}\big(\widehat{P}_{h,s,\pi^\star_h(s)}^0 \big)}  \cP \widehat{V}_{h+1} + b_h\big(s, \pi^\star_h(s)\big) \notag\\
	& \leq \inf_{ \cP \in \unb^{\ror}\big(P^0_{h,s,\pi^\star_h(s)}\big)}  \cP V_{h+1}^{\star,\ror} - \inf_{ \cP \in \unb^{\ror} \big(P^0_{h,s,\pi^\star_h(s)} \big)}  \cP \widehat{V}_{h+1} \notag\\
	&\quad +\left|\inf_{\cP \in \unb^{\ror}\big(\widehat{P}_{h,s,\pi^\star_h(s)}^0 \big)}  \cP \widehat{V}_{h+1} - \inf_{ \cP \in \unb^{\ror} \big(P^0_{h,s,\pi^\star_h(s)}\big)}  \cP \widehat{V}_{h+1} \right| + b_h\big(s, \pi^\star_h(s)\big) \notag \\
	&\overset{\mathrm{(i)}}{\leq} \inf_{ \cP \in \unb^{\ror} \big(P^0_{h,s,\pi^\star_h(s)} \big)}  \cP V_{h+1}^{\star,\ror} - \inf_{ \cP \in \unb^{\ror}\big(P^0_{h,s,\pi^\star_h(s)} \big)}   \cP \widehat{V}_{h+1} + 2b_h\big(s, \pi^\star_h(s)\big) \notag \\
	& \overset{\mathrm{(ii)}}{\leq} \widehat{P}^{\inf}_{h,s,\pi^\star_h(s)} \big( V_{h+1}^{\star,\ror} - \widehat{V}_{h+1} \big) + 2b_h\big(s, \pi^\star_h(s) \big),\label{eq:finite-recursion-basic}
\end{align}
where (i) holds by applying Lemma 2~(cf. \eqref{eq:dro-b-bound}) since $\widehat{V}_{h+1} $ is independent of $P^0_{h,s,\pi^\star_h(s)}$ by construction, and (ii) arises from introducing the notation 
\begin{align}
	\widehat{P}^{\inf}_{h,s,\pi^\star_h(s)} \defn \mathrm{argmin}_{\cP \in \unb^{\ror} \big(P^0_{h,s,\pi^\star_h(s)} \big)} \;  \cP \widehat{V}_{h+1}
\end{align}
and consequently,
\begin{align*}
\inf_{ \cP \in \unb^{\ror}\big(P^0_{h,s,\pi^\star_h(s)} \big)}  \cP V_{h+1}^{\star,\ror}   \leq  \widehat{P}^{\inf}_{h,s,\pi^\star_h(s)} V_{h+1}^{\star,\ror} , \qquad \mbox{and} \qquad \inf_{ \cP \in \unb^{\ror} \big(P^0_{h,s,\pi^\star_h(s)} \big)}   \cP \widehat{V}_{h+1} = \widehat{P}^{\inf}_{h,s,\pi^\star_h(s)}  \widehat{V}_{h+1}. 
\end{align*}

To continue, let us introduce some additional notation for convenience. Define a sequence of matrices $\widehat{P}^{\inf}_h \in \mathbb{R}^{S\times S}$ and vectors $b_h^\star\in\mathbb{R}^S$ for $h\in[H]$, where their $s$-th rows (resp. entries) are given by
\begin{align}
	\left[\widehat{P}^{\inf}_h \right]_{s,\cdot} = \widehat{P}^{\inf}_{h,s,\pi^\star_h(s)}, \qquad \mbox{and} \qquad  b_h^\star(s) = b_h\big(s, \pi^\star_h(s) \big).
\end{align}

Applying \eqref{eq:finite-recursion-basic} recursively over the time steps $h, h+1,\cdots, H$ using the above notation gives
\begin{align}
0 \leq V_h^{\star,\ror} - \widehat{V}_{h} \notag
& \leq \widehat{P}^{\inf}_{h} \big( V_{h+1}^{\star,\ror} - \widehat{V}_{h+1} \big) + 2 b_h^\star \notag\\
&\leq \widehat{P}^{\inf}_{h}\widehat{P}^{\inf}_{h+1} \big( V_{h+2}^{\star,\ror} - \widehat{V}_{h+2} \big) + 2 \widehat{P}^{\inf}_{h} b_{h+1}^\star + 2 b_h^\star \leq \cdots \leq 2 \sum_{i=h}^{H}\left(\prod_{j=h}^{i-1} \widehat{P}^{\inf}_{j}\right) b_i^\star, \label{eq:recursion-result}
\end{align}
where we let $\left(\prod_{j=i}^{i-1} \widehat{P}^{\inf}_{j}\right) = I$ for convenience.

For any $d_h^{\star} \in \cD^\star_h$ (cf.~\eqref{eq:def-D-star-h}), taking inner product with \eqref{eq:recursion-result} leads to
\begin{align}\label{eq:performance-gap-h}
	\left< d_h^\star, V_h^{\star,\ror} - \widehat{V}_{h}\right> \leq \left< d_h^\star, 2 \sum_{i=h}^{H}\left(\prod_{j=h}^{i-1} \widehat{P}^{\inf}_{j}\right) b_i^\star  \right> = 2 \sum_{i=h}^H \left\langle d_i^\star, b_i^\star \right\rangle,
\end{align} 
where
\begin{align}\label{eq:defn-of-di-star}
d_i^{\star} \defn \left[ \big(d_h^{\star}\big)^\top \left(\prod_{j=h}^{i-1} \widehat{P}^{\inf}_{j}\right)\right]^\top \in \cD^\star_i
\end{align}
by the definition of $ \cD^\star_i$ (cf.~\eqref{eq:def-D-star-h}) for all $i=h+1, \cdots, H$.
 
\paragraph{Step 3: controlling $\langle d_i^\star, b_i^\star \rangle$ using concentrability.} Since $\langle d_i^\star, b_i^\star \rangle = \sum_{s\in\cS}d_i^\star(s) b_i^\star(s) $, we shall divide the discussion in two different cases.

\begin{itemize}
	\item For $s\in S$ where $\max_{P \in  \unb^{\ror}(P^{\no})} d_i^{\star,P}\big( s, \pi_i^\star(s)\big) = \max_{P \in  \unb^{\ror}(P^{\no})} d_i^{\star,P} ( s )  =  0$, it follows from the definition (cf.~\eqref{eq:def-D-star-h}) that  for any $d_i^{\star}\in \cD^\star_i$, it satisfies that 
	\begin{align}
		d_i^{\star}(s) =    0.
	\end{align}

	\item For $s\in S$ where $\max_{P \in  \unb^{\ror}(P^{\no})} d_i^{\star,P}\big( s, \pi_i^\star(s)\big) = \max_{P \in  \unb^{\ror}(P^{\no})} d_i^{\star,P} ( s )  >  0$, by the assumption in \eqref{eq:concentrate-finite} 
\begin{align*}
 \max_{P \in  \unb^{\ror}(P^{\no})}  \frac{\min\big\{d_i^{\star,P} \big( s, \pi_i^\star(s)\big), \frac{1}{S}\big\}}{\myrho_i \big( s, \pi_i^\star(s)\big)} =  \max_{P \in  \unb^{\ror}(P^{\no})}  \frac{\min\big\{d_i^{\star,P}(s), \frac{1}{S}\big\}}{\myrho_i \big( s, \pi_i^\star(s)\big)}  \le \Cstar <\infty,
\end{align*}
it implies that
\begin{align}\label{eq:relation-dstar-db}
\myrho_i \big( s, \pi_i^\star(s)\big) >0 \quad \text{and} \quad \big(i,s, \pi_i^\star(s)\big) \in \cC^{\mathsf{b}}.
\end{align}
Lemma~\ref{lemma:D0-property} tells that with probability at least $1-8\delta$,
	\begin{align}
		N_i\big(s, \pi_i^\star(s)\big) &\geq \frac{K \myrho_i\big(s, \pi_i^\star(s)\big)}{8} - 5 \sqrt{ K \myrho_i\big(s, \pi_i^\star(s)\big) \log \frac{KH}{\delta} } \overset{\mathrm{(i)}}{\geq}  \frac{K \myrho_i\big(s, \pi_i^\star(s)\big)}{16} \notag\\
		& \overset{\mathrm{(ii)}}{\geq} \frac{K\max_{P\in \unb^\ror(P^{\no})}\min\left\{d_i^{\star, P}\big( s, \pi_i^\star(s)\big), \frac{1}{S}\right\}}{16 \Cstar} \geq \frac{K \min\left\{d_i^{\star}(s), \frac{1}{S}\right\}}{16 \Cstar}, \label{eq:new-lower-bound-Ni}
	\end{align}
	where  (i) holds 	due to  
	\begin{align}\label{eq:summary-K-implication}
K  \myrho_i\big( s, \pi_i^\star(s)\big) &\geq c_1 \frac{  \myrho_i\big( s, \pi_i^\star(s)\big) \log( KHS/ \delta )}{  d_{\mathsf{min}}^{\mathsf{b}}  P_{\mathsf{min}}^{\mathsf{b}} }  \geq  \frac{c_1 \log\frac{KH}{\delta}}{P_{\mathsf{min}}^{\mathsf{b}} }
 \geq  c_1 \log\frac{KH}{\delta}
\end{align}
for some sufficiently large $c_1$, where the first inequality follows from Condition \eqref{eq:dro-b-bound-N-condition}, the second inequality follows from 
\begin{equation}\label{eq:dmin_dbpi}
d_{\mathsf{min}}^{\mathsf{b}} =  \min_{h,s,a} \left\{\myrho_h(s, a ): \myrho_h(s, a ) >0 \right\} \leq \myrho_i\big( s, \pi_i^\star(s)\big)
\end{equation}
and the last inequality follows from $P_{\mathsf{min}}^{\mathsf{b}}\leq 1$. In addition, (ii) follows from Assumption~\ref{assumption:dro-finite}.

With this in place, we observe that the pessimistic penalty (see \eqref{def:bonus-dro}) obeys 
	\begin{align}
		b_i^\star(s) &\leq \cb \frac{H}{\ror} \sqrt{\frac{\log(\frac{KHS}{\delta})}{ \widehat{P}_{\mathsf{min},i}\big(s, \pi_i^\star(s)\big) N_i\big(s, \pi_i^\star(s)\big)}} \overset{\mathrm{(i)}}{\leq} 4\cb \frac{H}{\ror} \sqrt{\frac{\log^2(\frac{KHS}{\delta})}{ P_{\mathsf{min},i}\big(s, \pi_i^\star(s)\big) N_i\big(s, \pi_i^\star(s)\big)}} \notag \\
		&\leq 16\cb\frac{H}{\ror} \sqrt{\frac{\Cstar\log^2 \frac{KHS}{\delta} }{ P_{\mathsf{min},i}\big(s, \pi_i^\star(s)\big) K \min\left\{d_i^{\star}(s), \frac{1}{S}\right\} } },
	\end{align}
where (i) holds by applying \eqref{eq:convert-pmin-to-estimation} in view of the fact that $\big(i,s, \pi_i^\star(s)\big) \in \cC^{\mathsf{b}}$ by \eqref{eq:relation-dstar-db}, and the last inequality holds by \eqref{eq:new-lower-bound-Ni}.

\end{itemize}

Combining the results in the above two cases leads to
\begin{align}
	\sum_{s\in\cS}d_i^\star(s) b_i^\star(s) &\leq \sum_{s\in\cS} 16d_i^\star(s) \cb\frac{H}{\ror} \sqrt{\frac{\Cstar\log^2 \frac{KHS}{\delta} }{ P_{\mathsf{min},i}\big(s, \pi_i^\star(s)\big) K \min\left\{d_i^{\star}(s), \frac{1}{S}\right\} } } \notag \\
	& \overset{\mathrm{(i)}}{\leq} 16 \cb\frac{H}{\ror} \sqrt{\sum_{s\in\cS} d_i^\star(s) \frac{ \Cstar\log^2 \frac{KHS}{\delta} }{P_{\mathsf{min},i}\big(s, \pi_i^\star(s)\big) K \min\left\{d_i^{\star}(s), \frac{1}{S}\right\} } }\sqrt{\sum_{s\in\cS} d_i^\star(s)} \notag \\
	& \leq 32 \cb\frac{H}{\ror}\sqrt{\frac{S \Cstar\log^2  \frac{KHS}{\delta} }{P_{\mathsf{min},i}\big(s, \pi_i^\star(s)\big) K} }, \label{eq:b-d-produce-bound}
\end{align}
where (i) follows from the Cauchy-Schwarz inequality and the last inequality hold by the trivial fact
\begin{align}\label{eq:d-star-1-S-bound}
	\sum_{s\in\cS} \frac{ d_i^\star(s)}{\min\left\{d_i^{\star}(s), \frac{1}{S}\right\}} \leq \sum_{s\in\cS} d_i^\star(s) \left(\frac{ 1}{d_i^{\star}(s)} + \frac{ 1}{1/S}\right) = \sum_{s\in\cS}1 + \frac{1}{S}\sum_{s\in\cS}d_i^\star(s) \leq 2S.
\end{align}

\paragraph{Step 4: finishing up the proof.}
Then, inserting \eqref{eq:b-d-produce-bound} back into \eqref{eq:performance-gap-h}  with $h=1$ shows
\begin{align}
	\left< d_1^\star, V_1^{\star,\ror} - \widehat{V}_{1}\right> \leq  2 \sum_{i=1}^H \left< d_i^\star, b_i^\star \right> &\leq \sum_{i=1}^H 64 \cb\frac{H}{\ror}\sqrt{\frac{S \Cstar\log^2  \frac{KH}{\delta} }{P_{\mathsf{min},i}\big(s, \pi_i^\star(s)\big) K} }  
	 \leq c_2\frac{H^2}{\ror} \sqrt{\frac{S \Cstar\log^2  \frac{KH}{\delta} }{\minpall K} },
\end{align}
where the last inequality holds by plugging in the relation $\minpall \leq P_{\mathsf{min},i}\big(s, \pi_i^\star(s)\big) $ for $i=1,\ldots, H$ by the definition in \eqref{eq:P-min-star-def} (see also \eqref{eq:link_minpall_pmin}), and choosing $c_2$ to be large enough. The proof is completed.

%% file: auxiliary-upper-bound.tex
%!TEX root = ./DRO-offline.tex

\subsection{Proof of Lemma~\ref{lemma:dro-b-bound}}\label{proof:lemma:dro-b-bound} 
To begin, we shall introduce the following fact that
\begin{align}\label{eq:fact-of-N-b-assumption}
\forall (h,s,a) \in \cC^{\mathsf{b}} : \qquad	N_h(s,a) \geq \frac{ c_1 \log\frac{KHS}{\delta}}{16 P_{\mathsf{min}, h}(s,a)} \geq  - \frac{\log\frac{2KHS}{\delta}}{\log (1- P_{\mathsf{min},h}(s,a) )},   
\end{align}
as long as Condition \eqref{eq:dro-b-bound-N-condition} holds. The proof is postponed to Appendix~\ref{sec:proof-eq:fact-of-N-b-assumption}. 
With this in mind, we shall first establish the simpler bound \eqref{eq:convert-pmin-to-estimation} and then move on to show \eqref{eq:dro-b-bound}.

\subsubsection{Proof of \eqref{eq:convert-pmin-to-estimation}}\label{sec:proof-sec:proof-eq:convert-pmin-to-estimation}

To begin, recall that  \eqref{eq:fact-of-N-b-assumption} is satisfied for all $(h,s,a)\in\mathcal{C}^{\mathsf{b}}$. By Lemma~\ref{lem:binomial} and the union bound, it holds that with probability at least $1-\delta$ that for all $(h,s,a)\in\mathcal{C}^{\mathsf{b}}$:
\begin{equation}\label{eq:con_sand} 
\forall s'\in\cS: \qquad P^{\no}_h(s' \mymid s,a) \geq \frac{\widehat{P}^{\no}_h(s' \mymid s,a) }{ e^2} \geq \frac{P^{\no}_h(s'  \mymid s,a)}{8 e^2 \log(\frac{KHS}{\delta})} .
\end{equation}
To characterize the relation between $P_{\mathsf{min},h}(s,a)$ and $\widehat{P}_{\mathsf{min},h}(s,a)$ for any $(h,s,a)\in\mathcal{C}^{\mathsf{b}}$, we suppose---without loss of generality---that $P_{\mathsf{min},h}(s,a) =P^{\no}_h(s_1 \mymid s,a)$ and $\widehat{P}_{\mathsf{min},h}(s,a) = \widehat{P}^{\no}_h(s_2 \mymid s,a)$ for some $s_1, s_2 \in \cS$. 
Then, it follows that
\begin{align*}
	P_{\mathsf{min},h}(s,a) &= P^{\no}_h(s_1 \mymid s,a) \overset{\mathrm{(i)}}{\geq} \frac{\widehat{P}^{\no}_h(s_1 \mymid s,a)}{e^2} \geq \frac{\widehat{P}_{\mathsf{min},h}(s,a) }{e^2}  = \frac{\widehat{P}^{\no}_h(s_2 \mymid s,a)}{e^2} \notag\\
	&  \overset{\mathrm{(ii)}}{\geq} \frac{P^{\no}_h(s_2 \mymid s,a)}{8 e^2 \log(\frac{KHS}{\delta})} \geq \frac{P_{\mathsf{min},h}(s,a)}{8 e^2 \log(\frac{KHS}{\delta})},
\end{align*}
where (i) and (ii) follow from \eqref{eq:con_sand}.

\subsubsection{Proof of \eqref{eq:dro-b-bound}}\label{proof:finite-control-uncertainty-gap}
The main goal of \eqref{eq:dro-b-bound} is to control the gap between robust Bellman operations based on the nominal transition kernel $P_{h,s,a}^{\no}$ and the estimated kernel $\widehat{P}_{h,s,a}^{\no}$ by the constructed penalty term. 
Towards this, first consider $(h,s,a) \notin \cC^{\mathsf{b}}$, which corresponds to the state-action pairs $(s,a)$ that haven't been visited at step $h$ by the behavior policy. In other words, $N_h(s,a) =0$. In this case, \eqref{eq:dro-b-bound} can be easily verified that
\begin{align}
	\left|\inf_{ \cP \in \unb^{\sigma}(\widehat{P}_{h,s,a}^0)} \cP V - \inf_{ \cP \in \unb^{\sigma}(P^{\no}_{h,s,a})}  \cP V \right| \overset{\mathrm{(i)}}{=}  \inf_{ \cP \in \unb^{\sigma}(P^{\no}_{h,s,a})}  \cP V  \leq  \|V\|_\infty \overset{\mathrm{(ii)}}{\leq} H \overset{\mathrm{(iii)}}{=}  b_h(s,a),
\end{align}
where (i) follows from the fact $\widehat{P}_{h,s,a}^0 =0$ when $N_h(s,a) =0$ (see \eqref{eq:empirical-P-finite}), (ii) arises from the assumption $\|V\|_\infty \leq H$, and (iii) holds by the definition of $b_h(s,a)$ in \eqref{def:bonus-dro}.
Therefore, the remainder of the proof will focus on verifying \eqref{eq:dro-b-bound} for $(h,s,a) \in \cC^{\mathsf{b}}$. Rewriting the term of interest via duality (cf.~Lemma~\ref{lem:strong-duality}) yields
\begin{align}
	&\left|\inf_{ \cP \in \unb^{\sigma}(\widehat{P}_{h,s,a}^0)} \cP V - \inf_{ \cP \in \unb^{\sigma}(P^{\no}_{h,s,a})}  \cP V \right| \notag\\
	& = \left|\sup_{\lambda \geq 0}  \left\{ -\lambda \log\left(\widehat{P}^0_{h, s, a}   \exp \left(\frac{-V}{\lambda}\right) \right) - \lambda \ror \right\} - \sup_{\lambda \geq 0}  \left\{ -\lambda \log\left(P^0_{h, s, a}   \exp \left(\frac{-V}{\lambda}\right) \right) - \lambda \ror \right\}\right| .\label{finite-bonus-b-sup-version}
\end{align}
Denoting 
\begin{subequations}
\label{eq:upper-lambda-lambda-hat-def}
\begin{align}
	\widehat{\lambda}_{h,s,a}^\star &\defn \arg \max_{\lambda \geq 0}  \left\{ -\lambda \log\left(\widehat{P}^0_{h, s, a}   \exp \left(\frac{-V}{\lambda}\right) \right) - \lambda \ror \right\}, \\
	\lambda_{h,s,a}^\star &\defn \arg \max_{\lambda \geq 0}  \left\{ -\lambda \log\left(P^0_{h, s, a}   \exp \left(\frac{-V}{\lambda}\right) \right) - \lambda \ror \right\},
\end{align}
\end{subequations}
 Lemma~\ref{lem:lambda-n-bound} (cf. \eqref{eq:lambda-n-range}) then gives that
\begin{align}\label{eq:bound_lambda}
\lambda_{h,s,a}^\star \in \left[0, \frac{H}{\ror} \right], \qquad \widehat{\lambda}_{h,s,a}^\star  \in \left[0, \frac{H}{\ror} \right],
\end{align}
due to $\|V\|_{\infty} \leq H$. We shall control \eqref{finite-bonus-b-sup-version} in three different cases separately: (a) $\lambda_{h,s,a}^\star =0$ and $\widehat{\lambda}_{h,s,a}^\star =0$; (b) $\lambda_{h,s,a}^\star >0$ and $\widehat{\lambda}_{h,s,a}^\star =0$ or $\lambda_{h,s,a}^\star =0$ and $\widehat{\lambda}_{h,s,a}^\star >0$; and (c)  $\lambda_{h,s,a}^\star \neq 0$ or $\widehat{\lambda}_{h,s,a}^\star \neq 0$.

\paragraph{Case (a): $\lambda_{h,s,a}^\star =0$ and $\widehat{\lambda}_{h,s,a}^\star =0$.} Applying Lemma~\ref{lem:lambda-n-bound} and Lemma~\ref{lem:empirical-essinf} to \eqref{finite-bonus-b-sup-version} gives that, with probability at least $1-\frac{\delta}{KH}$,
\begin{align}
\left|\inf_{ \cP \in \unb^{\sigma}(\widehat{P}_{h,s,a}^0)} \cP V - \inf_{ \cP \in \unb^{\sigma}(P^{\no}_{h,s,a})}  \cP V \right| 
	&\overset{\mathrm{(i)}}{=} \left| \mathrm{essinf}_{s\sim \widehat{P}^{\no}_{h,s,a}} V(s) - \mathrm{essinf}_{s \sim P^{\no}_{h,s,a}} V(s) \right| \notag \\
	& \overset{\mathrm{(ii)}}{=}  \left|  \mathrm{essinf}_{s \sim P^{\no}_{h,s,a}} V(s)  - \mathrm{essinf}_{s \sim P^{\no}_{h,s,a}} V(s) \right| \notag\\
	& =  0 \leq b_h(s,a). \label{eq:b-control-0-case}
\end{align}
where (i) holds by Lemma~\ref{lem:lambda-n-bound} (cf.~\eqref{eq:lambda-0-func-value}) and (ii) arises from Lemma~\ref{lem:empirical-essinf} (cf.~\eqref{eq:essinf-of-P-n}) given \eqref{eq:fact-of-N-b-assumption}.

\paragraph{Case (b): $\lambda_{h,s,a}^\star >0$ and $\widehat{\lambda}_{h,s,a}^\star =0$ or $\lambda_{h,s,a}^\star =0$ and $\widehat{\lambda}_{h,s,a}^\star >0$.}
Towards this, note that two trivial facts are implied by the definition \eqref{eq:upper-lambda-lambda-hat-def}:
\begin{subequations}\label{eq:property-lambda-sup}
\begin{align}
	&\quad \sup_{\lambda \geq 0}  \left\{ -\lambda \log\left(P^0_{h, s, a}   \exp \left(\frac{-V}{\lambda}\right) \right) - \lambda \ror \right\}  \geq   -\widehat{\lambda}^\star_{h,s,a} \log\left(P^0_{h, s, a} \cdot \exp \left(\frac{-V}{\widehat{\lambda}^\star_{h,s,a}}\right) \right) - \widehat{\lambda}^\star_{h,s,a} \ror, \label{eq:property-lambda-sup-1} \\
	&\quad \sup_{\lambda \geq 0}  \left\{ -\lambda \log\left(\widehat{P}^{\no}_{h, s, a}   \exp \left(\frac{-V}{\lambda}\right) \right) - \lambda \ror \right\} \geq   -\lambda^\star_{h,s,a} \log\left(\widehat{P}^0_{h, s, a} \cdot \exp \left(\frac{-V}{\lambda^\star_{h,s,a}}\right) \right) - \lambda^\star_{h,s,a} \ror. \label{eq:property-lambda-sup-2}
\end{align}
\end{subequations}
To continue, first, we consider a subcase when $\lambda_{h,s,a}^\star =0$ and $\widehat{\lambda}_{h,s,a}^\star >0$. With probability at least $1-\frac{\delta}{KH}$, it follows from
Lemma~\ref{lem:lambda-n-bound} (cf.~\eqref{eq:lambda-0-func-value}) and Lemma~\ref{lem:empirical-essinf}  (cf.~\eqref{eq:essinf-of-P-n}) that 
\begin{align}
	\sup_{\lambda \geq 0}  \left\{ -\lambda \log\left(\widehat{P}^0_{h, s, a}   \exp \left(\frac{-V}{\lambda}\right) \right) - \lambda \ror \right\} &\geq \lim_{\lambda \rightarrow 0}  \left\{ -\lambda \log\left(\widehat{P}^0_{h, s, a}   \exp \left(\frac{-V}{\lambda}\right) \right) - \lambda \ror \right\} \notag \\
	& = \mathrm{essinf}_{s\sim \widehat{P}^{\no}_{h,s,a}} V(s) = \mathrm{essinf}_{s\sim P^{\no}_{h,s,a}} V(s) \notag \\
	& =  \sup_{\lambda \geq 0}  \left\{ -\lambda \log\left(P^0_{h, s, a}   \exp \left(\frac{-V}{\lambda}\right) \right) - \lambda \ror \right\},
\end{align}
leading to
\begin{align}
	&\left|\sup_{\lambda \geq 0}  \left\{ -\lambda \log\left(\widehat{P}^0_{h, s, a}   \exp \left(\frac{-V}{\lambda}\right) \right) - \lambda \ror \right\} - \sup_{\lambda \geq 0}  \left\{ -\lambda \log\left(P^0_{h, s, a}   \exp \left(\frac{-V}{\lambda}\right) \right) - \lambda \ror \right\}\right| \notag \\
	&  \overset{\mathrm{(i)}}{\leq} \left(- \widehat{\lambda}^\star_{h,s,a} \log\left(\widehat{P}^0_{h, s, a} \cdot \exp \left(\frac{-V}{\widehat{\lambda}^\star_{h,s,a} }\right) \right) - \widehat{\lambda}^\star_{h,s,a}   \ror  \right) -  \left( -\widehat{\lambda}^\star_{h,s,a} \log\left(P^0_{h, s, a} \cdot \exp \left(\frac{-V}{\widehat{\lambda}^\star_{h,s,a}}\right) \right) - \widehat{\lambda}^\star_{h,s,a} \ror \right) \notag \\
	& \leq  \widehat{\lambda}^\star_{h,s,a} \left| \log\left(\widehat{P}^0_{h, s, a} \cdot \exp \left(\frac{-V}{\widehat{\lambda}^\star_{h,s,a}}\right) \right) - \log\left(P^0_{h, s, a} \cdot \exp \left(\frac{-V}{\widehat{\lambda}^\star_{h,s,a}}\right) \right) \right|,  \label{eq:finite-upper-b-log} 
\end{align}
where (i)  follows from the definition of $\widehat{\lambda}^\star_{h,s,a}$ in \eqref{eq:upper-lambda-lambda-hat-def} and the fact in \eqref{eq:property-lambda-sup-1}.

We pause to claim that with probability at least $1-\delta$, the following bound holds 
\begin{equation}\label{eq:control-I-final}
\forall (h,s,a) \in \mathcal{C}^{\mathsf{b}}, \, V\in\mathbb{R}^S: \quad \frac{\left|\left(\widehat{P}^0_{h, s, a} - P^0_{h, s, a}\right) \cdot \exp \left(\frac{-V}{\lambda}\right)\right| }{P^0_{h, s, a} \cdot \exp \left(\frac{-V}{\lambda}\right)} \leq  \sqrt{\frac{\log(\frac{KHS}{\delta})}{ c_{\mathsf{f}} N_h(s,a) P_{\mathsf{min},h}(s,a) }} \leq \frac{1}{2}.
\end{equation}
The proof is postponed to Appendix~\ref{sec:proof-eq:control-I-final}. With \eqref{eq:control-I-final}
in place, we can further bound \eqref{eq:finite-upper-b-log} (which is plugged into \eqref{finite-bonus-b-sup-version}) as
\begin{align}
\left|\inf_{ \cP \in \unb^{\sigma}(\widehat{P}_{h,s,a}^0)} \cP V - \inf_{ \cP \in \unb^{\sigma}(P^{\no}_{h,s,a})}  \cP V \right| 
& \leq   \widehat{\lambda}^\star_{h,s,a} \left| \log\left(1 + \frac{\left(\widehat{P}^0_{h, s, a} - P^0_{h, s, a}\right) \cdot \exp \left(\frac{-V}{\lambda}\right) }{P^0_{h, s, a} \cdot \exp \left(\frac{-V}{\lambda}\right)} \right) \right|\notag  \\
&\overset{\mathrm{(i)}}{\leq}   2 \widehat{\lambda}_{h,s,a}^\star  \frac{\left|\left(\widehat{P}^0_{h, s, a} - P^0_{h, s, a}\right) \cdot \exp \left(\frac{-V}{\lambda}\right)\right| }{P^0_{h, s, a} \cdot \exp \left(\frac{-V}{\lambda}\right)} \notag \\
&\overset{\mathrm{(ii)}}{\leq}  \frac{2H}{\ror} \sqrt{\frac{\log(\frac{KHS}{\delta})}{ c_{\mathsf{f}} N_h(s,a) P_{\mathsf{min},h}(s,a) }}  \notag \\
& \leq \frac{2 e H}{\ror}  \sqrt{\frac{\log(\frac{KHS}{\delta})}{c_{\mathsf{f}}  N_h(s,a) \widehat{P}_{\mathsf{min},h}(s,a) }}  \leq \cb \frac{H}{\ror} \sqrt{\frac{\log(\frac{KHS}{\delta})}{ \widehat{P}_{\mathsf{min},h}(s,a) N_h(s,a)}},	\label{eq:lemma-1-lambda-large-upper} 
\end{align}
where (i) follows from $\log(1+x) \leq 2|x|$ for any $|x|\leq \frac{1}{2}$ in view of \eqref{eq:control-I-final},
 (ii) follows from \eqref{eq:bound_lambda}  as well as \eqref{eq:control-I-final}, and the last line follows from \eqref{eq:convert-pmin-to-estimation} and choosing $\cb$ to be sufficiently large.  

Moreover, note that it can be easily verified that $$\left|\inf_{ \cP \in \unb^{\sigma}(\widehat{P}_{h,s,a}^0)} \cP V - \inf_{ \cP \in \unb^{\sigma}(P^{\no}_{h,s,a})}  \cP V \right| \leq H$$ due to the assumption $\|V\|_\infty \leq H$. Plugging in the definition of $b_h(s,a)$ in \eqref{def:bonus-dro}, combined with the above bounds, we have that with probability at least $1- \delta$, 
\begin{align}
	&\left|\inf_{ \cP \in \unb^{\sigma}(\widehat{P}_{h,s,a}^0)} \cP V - \inf_{ \cP \in \unb^{\sigma}(P^{\no}_{h,s,a})}  \cP V \right| \leq \min\left\{ \cb \frac{H}{\ror}  \sqrt{\frac{\log(\frac{KHS}{\delta})}{ N_h(s,a) \widehat{P}_{\mathsf{min},h}(s,a) }}~,~H\right\} \eqqcolon b_h(s,a).
\end{align} 

The other subcase when $\lambda_{h,s,a}^\star >0$ and $\widehat{\lambda}_{h,s,a}^\star =0$ follows similarly from the bound
\begin{align}
	&\left|\sup_{\lambda \geq 0}  \left\{ -\lambda \log\left(\widehat{P}^0_{h, s, a}   \exp \left(\frac{-V}{\lambda}\right) \right) - \lambda \ror \right\} - \sup_{\lambda \geq 0}  \left\{ -\lambda \log\left(P^0_{h, s, a}   \exp \left(\frac{-V}{\lambda}\right) \right) - \lambda \ror \right\}\right| \notag \\
	& \leq  \lambda^\star_{h,s,a} \left| \log\left(\widehat{P}^0_{h, s, a} \cdot \exp \left(\frac{-V}{\lambda^\star_{h,s,a}}\right) \right) - \log\left(P^0_{h, s, a} \cdot \exp \left(\frac{-V}{\lambda^\star_{h,s,a}}\right) \right) \right|,
	\label{eq:b-control-lam-not0-lam_hat-0}
\end{align}
and therefore, will be omitted for simplicity.

\paragraph{Case (c): $\lambda_{h,s,a}^\star >0$ and $\widehat{\lambda}_{h,s,a}^\star >0$.}
It follows that
\begin{align}
& \left|\sup_{\lambda \geq 0}  \left\{ -\lambda \log\left(\widehat{P}^0_{h, s, a}   \exp \left(\frac{-V}{\lambda}\right) \right) - \lambda \ror \right\} - \sup_{\lambda \geq 0}  \left\{ -\lambda \log\left(P^0_{h, s, a}   \exp \left(\frac{-V}{\lambda}\right) \right) - \lambda \ror \right\}\right| \notag \\
& \overset{\mathrm{(i)}}{\leq} \max\Bigg\{ \left(- \widehat{\lambda}^\star_{h,s,a} \log\left(\widehat{P}^0_{h, s, a} \cdot e^{\frac{-V}{\widehat{\lambda}^\star_{h,s,a} }} \right) - \widehat{\lambda}^\star_{h,s,a}   \ror  \right) -  \left( -\widehat{\lambda}^\star_{h,s,a} \log\left(P^0_{h, s, a} \cdot e^{\frac{-V}{\widehat{\lambda}^\star_{h,s,a}}} \right) - \widehat{\lambda}^\star_{h,s,a} \ror \right) , \notag \\
&\qquad \left(- \lambda^\star_{h,s,a} \log\left(P^0_{h, s, a} \cdot e^{\frac{-V}{\lambda^\star_{h,s,a} }} \right) - \lambda^\star_{h,s,a}   \ror  \right)  -  \left( -\lambda^\star_{h,s,a} \log\left(\widehat{P}^0_{h, s, a} \cdot e^{\frac{-V}{\lambda^\star_{h,s,a}}} \right) - \lambda^\star_{h,s,a} \ror \right) \Bigg\} \notag\\
& \leq \max_{\lambda\in\{\lambda_{h,s,a}^\star, \widehat{\lambda}_{h,s,a}^\star \}} \lambda  \left| \log\left(\widehat{P}^0_{h, s, a} \cdot \exp \left(\frac{-V}{\lambda}\right) \right) - \log\left(P^0_{h, s, a} \cdot \exp \left(\frac{-V}{\lambda}\right) \right) \right| , \label{eq:b-control-lam-not0-lam_hat-not0}
	\end{align}
where (i) can be verified by applying the facts in \eqref{eq:property-lambda-sup}. Hence, the above term \eqref{eq:b-control-lam-not0-lam_hat-not0} can be controlled again in a similar manner as \eqref{eq:finite-upper-b-log}; we omit the details for simplicity.

\paragraph{Summing up.}
Combining the previous results in different cases by the union bound, with probability at least $1-10\delta$, it is satisfied that for all $(h,s,a)\in \cC^{\mathsf{b}}$:
\begin{align*}
	\left|\inf_{ \cP \in \unb^{\sigma}(\widehat{P}_{h,s,a}^0)} \cP V - \inf_{ \cP \in \unb^{\sigma}(P^{\no}_{h,s,a})}  \cP V \right| &\leq b_h(s,a) ,
\end{align*}
which concludes the proof.

\subsubsection{ Proof of \eqref{eq:fact-of-N-b-assumption}}
\label{sec:proof-eq:fact-of-N-b-assumption}
Observe that for all  $(h,s,a)\in \mathcal{C}^{\mathsf{b}}$:
\begin{align}
K \myrho_h\big(s, a\big)  \overset{\mathrm{(i)}}{\geq}  \frac{ c_1 \myrho_h\big(s, a\big) \log( KHS/ \delta )}{  d_{\mathsf{min}}^{\mathsf{b}}  P_{\mathsf{min}}^{\mathsf{b}} } \overset{\mathrm{(ii)}}{\geq}  \frac{ c_1  \log( KHS/ \delta ) }{P_{\mathsf{min}}^{\mathsf{b}}  }  \overset{\mathrm{(iii)}}{\geq}  \frac{ c_1\log( KHS/ \delta )  }{P_{\mathsf{min}, h}(s,a) }  ,
\end{align}
where (i) follows from Condition \eqref{eq:dro-b-bound-N-condition}, (ii) follows from the definition 
that $d_{\mathsf{min}}^{\mathsf{b}}  \leq \myrho_h(s,a)$ for $(h,s,a)\in \mathcal{C}^{\mathsf{b}}$, and (iii) comes from \eqref{eq:link_minpall_pmin}.

 Lemma~\ref{lemma:D0-property} then tells that with probability at least $1-8\delta$,
\begin{align}
	N_h(s,a) &\geq \frac{K \myrho_h\big(s, a\big)}{8} - 5 \sqrt{ K \myrho_h\big(s,a\big) \log \frac{KH}{\delta} } \notag \\
		&\geq \frac{K \myrho_i\big(s, a\big)}{16}  \geq \frac{ c_1 \log\frac{KH}{\delta}}{16 P_{\mathsf{min}, h}(s,a)},
\end{align}
where the second line follows from the above relation as long as $c_1$ is sufficiently large. The last inequality of \eqref{eq:fact-of-N-b-assumption} then follows from
\begin{align}
	 \frac{ c_1 \log\frac{KHS}{\delta}}{16 P_{\mathsf{min}, h}(s,a)} \geq -  \frac{\log\frac{2KHS}{\delta}}{\log (1- P_{\mathsf{min},h}(s,a) )},
\end{align}
since $x\leq - \log(1-x)$ for all $x\in [0,1]$.

\subsubsection{Proof of \eqref{eq:control-I-final}}\label{sec:proof-eq:control-I-final}
Denoting 
$$\mathsf{supp} \big( P^0_{h, s, a}  \big) \defn \big\{s' \in \cS: P_h^{\no}(s' \mymid s,a) >0 \big\}$$ as the support of $P^0_{h, s, a}$,  we observe that
\begin{align}\label{eq:finite-upper-b-2}
	\frac{\left|\left(\widehat{P}^0_{h, s, a} - P^0_{h, s, a}\right) \cdot \exp \left(\frac{-V}{\lambda}\right)\right| }{P^0_{h, s, a} \cdot \exp \left(\frac{-V}{\lambda}\right)} &\leq \frac{\sum_{s'\in \mathsf{supp}\big( P^0_{h, s, a}  \big) } \left| \widehat{P}^{\no}_h(s'\mymid s,a) - P^{\no}_h(s'\mymid s,a) \right|\exp\left(\frac{-V(s')}{\lambda}\right)}{\sum_{s'\in\mathsf{supp}\big( P^0_{h, s, a}  \big) } P^{\no}_h(s'\mymid s,a) \exp\left(\frac{-V(s')}{\lambda}\right)} \notag \\
	& \leq \max_{s'\in \mathsf{supp} \big( P^0_{h, s, a}  \big) }\frac{\left|\widehat{P}^0_h(s' \mymid s, a) - P^0_h(s' \mymid s, a)\right|}{P^0_h(s' \mymid s, a) },
\end{align}
where the second line follows from $\sum_i a_i  = \sum_i b_i \frac{a_i}{b_i}  \leq ( \max_i \frac{a_i}{b_i}) \sum_i b_i$ for any positive sequences $\{a_i, b_i\}_{i} $ obeying $a_i, b_i >0$. 
 
To continue, note that for any $(h,s,a) \in\mathcal{C}^{\mathsf{b}}$ and $s'\in\mathsf{supp}\big(P_{h,s,a}^{\no}  \big)$, $N_h(s,a)  \widehat{P}^0_h(s' \mymid s, a)$  follows the binomial distribution $\mathsf{Binomial}\left(N_h(s,a), P^0_h(s' \mymid s, a)\right)$. Thus, applying Lemma~\ref{lem:binomial-small} with $t = \sqrt{\frac{\log\left(\frac{KHS}{\delta}\right)}{c_{\mathsf{f}} N_h(s,a)P^0_h(s' \mymid s, a)}}$ yields  
	\begin{align}\label{eq:pointwise-I-upper}
		\mathbb{P}\left(\left|\widehat{P}^0_h(s' \mymid s, a) - P^0_h(s' \mymid s, a)\right| \geq P^0_h(s' \mymid s, a)t \right) \leq \exp\left(-c_{\mathsf{f}}  N_h(s,a) P^0_h(s' \mymid s, a) t^2\right) \leq \frac{\delta}{KHS},
	\end{align}
as soon as $t\leq \frac{1}{2}$, which can be verified by the fact \eqref{eq:fact-of-N-b-assumption} and $P_{\mathsf{min}, h}(s,a) \leq P^0_h(s' \mymid s, a)$ (cf. \eqref{eq:P-min-pib}), namely,
\begin{align}\label{eq:Nh_cf}
	N_h(s,a) \geq \frac{ c_1 \log\frac{KHS}{\delta}}{16 P_{\mathsf{min}, h}(s,a)} \geq \frac{\log\left(\frac{KHS}{\delta}\right)}{4 c_{\mathsf{f}} P_{\mathsf{min}, h}(s,a)} \geq \frac{\log\left(\frac{KHS}{\delta}\right)}{4 c_{\mathsf{f}} P^0_h(s' \mymid s, a)}
\end{align}
as long as $c_1$ is sufficiently large.

Applying \eqref{eq:pointwise-I-upper} and taking the union bound over $s\in \mathsf{supp}\big(P_{h,s,a}^{\no} \big)$ lead to that with probability at least $1-\frac{\delta}{KH}$,
	\begin{align}
	\max_{s'\in \mathsf{supp}\big(P_{h,s,a}^{\no}  \big)} \frac{\left|\widehat{P}^0_h(s' \mymid s, a) - P^0_h(s' \mymid s, a)\right|}{P^0_h(s' \mymid s, a)} &\leq \max_{s'\in \mathsf{supp}\big( P_{h,s,a}^{\no} \big)} \frac{P^0_h(s' \mymid s, a)\sqrt{\frac{\log\left(\frac{KHS}{\delta}\right)}{c_{\mathsf{f}} N_h(s,a)P^0_h(s' \mymid s, a)}}}{P^0_h(s' \mymid s, a)} \notag \\
	& = \max_{s'\in \mathsf{supp}\big( P_{h,s,a}^{\no} \big)} \sqrt{\frac{\log(\frac{KHS}{\delta})}{ c_{\mathsf{f}} N_h(s,a) P^0_h(s' \mymid s, a) }} \notag\\
	& \leq  \sqrt{\frac{\log(\frac{KHS}{\delta})}{ c_{\mathsf{f}} N_h(s,a) P_{\mathsf{min},h}(s,a) }} \leq \frac{1}{2}, \notag
\end{align}
where the last line uses again \eqref{eq:Nh_cf}. Plugging this back into \eqref{eq:finite-upper-b-2} and applying the union bound over $(h,s,a) \in\mathcal{C}^{\mathsf{b}}$ then completes the proof.

%% file: lower-bound-analysis.tex
%!TEX root = ./DRO_OfflineRL.tex
\subsection{Proof of Theorem~\ref{thm:dro-lower-finite}} 
\label{proof:thm:dro-lower-finite}
 
The proof of Theorem~\ref{thm:dro-lower-finite} is inspired by the construction in \citet{li2022settling} for standard MDPs, but is considerably more involved to handle the uncertainty set unique in robust MDPs. In particular, we construct two different classes of hard instances for different range of the uncertainty level $\ror$ to achieve a tighter $\sigma$-dependent  lower bound. 
In what follows, we start with the lower bound for the case when the uncertainty level is relatively small, by first constructing some hard instances and then characterizing the sample complexity requirements over these instances. We then move onto the case when the uncertainty level is relatively large, and carry out a similar argument.

\subsubsection{Construction of hard problem instances: small uncertainty level}\label{sec:small-uncertain-instance}

\paragraph{Construction of a collection of hard MDPs}
To begin, let's consider a collection $\Theta \subseteq \{0,1\}^{H}$, consisting of vectors with $H$ dimensions. The Gilbert-Varshamov lemma \citep{gilbert1952comparison} tells that there exists a set $\Theta \subseteq \{0, 1\}^{H}$ such that:%
\begin{equation}
    |\Theta| \ge e^{H/8}
    \qquad \text{and} \qquad 
    \|\theta - \widetilde{\theta}\|_1 \ge \frac{H}{8}
    \quad 
    \text{for any }\theta,\widetilde{\theta}\in \Theta \text{ obeying }\theta \ne \widetilde{\theta}. 
    \label{eq:property-Theta} 
\end{equation}
%
%This implies that the designed set $\Theta$ encompasses an exponentially vast quantity of vectors, each adequately distant from the others. This characteristic is crucial in the subsequent analysis.
Armed with $\Theta$, we then generate a collection of MDPs 
\begin{equation}
    \mathsf{MDP}(\Theta) = \left\{\mathcal{M}^{\theta} = 
    \big(\mathcal{S}, \mathcal{A}, P^{\theta} = \{P^{\theta_h}_h\}_{h=1}^H, \{r_h\}_{h=1}^H, H \big) 
    \mid \theta = [\theta_h]_{1\leq h\leq H} \in \Theta
    \right\}, \label{eq:theta-class}
\end{equation}
where
\begin{align*}
    \cS = \{0, 1, \ldots, S-1\}, 
    \qquad \text{and} \qquad  \mathcal{A} = \{0, 1\}.
%   \qquad \text{and} \qquad 
%   \theta = \{\theta_h\}_{h=1}^H \in \Theta.
\end{align*}
The transition kernel $P^{\theta} = \{P^{\theta_h}_h\}_{h=1}^H$ of the MDP $\mathcal{M}^\theta$ is specified as follows:
\begin{align}
P^{\theta_h}_h(s^{\prime} \mymid s, a) = \left\{ \begin{array}{lll}
         p\mathds{1}(s^{\prime} = 0) + (1-p)\mathds{1}(s^{\prime} = 1)  & \text{if} & (s, a) = (0, \theta_h) \\
         q\mathds{1}(s^{\prime} = 0) + (1-q)\mathds{1}(s^{\prime} = 1) & \text{if} & (s, a) = (0, 1-\theta_h) \\
         \mathds{1}(s^{\prime} = 1) & \text{if}   & s=1 \\ 
         q\mathds{1}(s^{\prime} = s) + (1-q)\mathds{1}(s^{\prime} = 1) & \text{if}   & s > 1 \
                \end{array}\right.
        \label{eq:Ph-construction-lower-bound-finite-theta}
\end{align}
for any $(s,a,s',h)\in \cS\times \cA\times \cS\times [H]$. Here,
 $p$ and $q$ are set according to
\begin{equation}
    p = 1 - \frac{c_1}{H} 
    \qquad \text{and} \qquad
    q = p- \frac{c_2 \varepsilon}{H^2} \label{eq:finite-p-q-def-theta}
\end{equation}
% for some sufficiently large constant $c_1 > 0$.
%for some sufficiently small (resp.~large) constant $c_1 > 0$ (resp.~$c_2 > 0$) obeying 
for $c_1 = 1/8$ and some $c_2$ that satisfies
\begin{equation}
     \frac{c_2\varepsilon}{H^2} \leq \frac{c_1}{2H} \leq \frac{1}{8}. 
\end{equation}
Clearly, it follows that
\begin{align}
    p> q \geq \frac{1}{2} 
    \label{eq:p-q-order-LB-finite-theta}
\end{align}
by construction. Furthermore, the MDP will stay in the state subset $\{0,1\}$ if its initial state falls in  $\{0,1\}$.
The reward function of these MDPs is set as
\begin{align}
r_h(s, a) = \left\{ \begin{array}{lll}
         1 & \text{if} & s = 0 \\
         0 & \text{if}   & s >0 \
                \end{array}\right.
        \label{eq:rh-construction-lower-bound-finite-theta}
\end{align}
for any $(s,a,h)\in \cS\times \cA\times [H]$.

\paragraph{Uncertainty set of the transition kernels.}
Denote the transition kernel vector as
\begin{align}
    P_{h,s,a}^{\theta} \defn P^{\theta}_h(\cdot \mymid s,a) \in [0,1]^{1\times S}.
\end{align}
For any $(s,a,h)\in \cS\times\cA\times [H]$, the perturbation of the transition kernels in $\cM^\theta$ is restricted to the following uncertainty set 
\begin{align}
    \unb^\ror(P^{\theta})\defn \otimes \; \unb^\ror \left(P^{\theta}_{h,s,a}\right),\qquad \unb^\ror(P_{h,s,a}^{\theta}) \defn \left\{ P_{h, s,a} \in \Delta (\cS): \mathsf{KL}\left(P_{h,s,a} \parallel P^{\theta}_{h,s,a}\right) \leq \ror \right\}
\end{align}
with the uncertainty level $\ror$ satisfying
\begin{align}\label{eq:finite-lower-ror-bounded-theta}
0 < \ror \leq \frac{1}{20H}.
\end{align}
% Denoting the smallest positive state transition probability of any nominal transition kernel of $\cM_\theta$ as $P_{\min}$, we observe that \begin{align}\label{eq:P-min-phi-nominal-theta}
% P_{\mathsf{min}} \defn \min_{h,s, a, s'} \Big\{P_h^{\theta_h}\left(s' |s,a\right):\; P_h^{\theta_h}\left(s' |s,a\right)>0 \Big\} = P_1^{\theta}\left( 1 |0, 1- \theta \right) = 1- p = \alpha.
% \end{align}

% Note that the reasonable range of the uncertainty level $\sigma$ is $0<\sigma \leq \log(\frac{1}{P_{\mathsf{min}}}) $.

Before continuing, we shall introduce some notation for convenience. For any $P^{\theta_h}_h(\cdot  \mymid s, a)$ in \eqref{eq:Ph-construction-lower-bound-finite-theta}, we define the limit of the perturbed kernel transiting to the next state $s'$ from the current state-action pair $(s,a)$ by
\begin{align}\label{eq:finite-lw-def-p-q}
\underline{P}^{\theta_h}_{h}(s' \mymid s,a) &\defn \inf_{P_{h,s,a} \in \unb^{\ror}(P^{\theta_h}_{h,s,a})} P_h(s'  \mymid s,a),
\end{align}
and in particular, denote
\begin{align}\label{eq:finite-lw-p-q-perturb-inf-theta}
\underline{p}_h &\defn \underline{P}^{\theta_h}_{h}(0 \mymid 0,\theta_h)  ,\qquad \underline{q}_h \defn \underline{P}^{\theta_h}_{h}(0  \mymid 0, 1-\theta_h).
\end{align}
Armed with the above definitions, we introduce the following lemma which implies some useful properties of the uncertainty set.
\begin{lemma}\label{lem:finite-lb-uncertainty-set-KL}
The perturbed transition kernels obey
\begin{align}
   \underline{p}_1 = \underline{p}_2 = \cdots \underline{p}_H, \quad \underline{q}_1 = \underline{q}_2 = \cdots \underline{q}_H. \label{eq:finite-lw-upper-p-q-theta-0}
\end{align}
Denoting $\underline{p}_1 \defn p^\star$ and $\underline{q}_1 \defn q^\star$, we have that when the uncertainty level $\ror$ satisfies \eqref{eq:finite-lower-ror-bounded-theta},
\begin{align}\label{eq:finite-lw-upper-p-q-theta}
   \underline{p}_\star \geq \underline{q}_\star \geq 1 - \frac{c_3}{H}  \quad \text{and} \quad \underline{p}_\star-\underline{q}_\star \geq p-q \geq 0 
\end{align}
for  constant $c_3 =2 c_1 = \frac{1}{4}$.
\end{lemma}
\begin{proof}
The proof is postponed to Appendix~\ref{proof:lem:finite-lb-uncertainty-set-KL}.
\end{proof}

\paragraph{Value functions and optimal policies.}

We take a moment to derive the corresponding value functions and  identify the optimal policies. With some abuse of  notation, for any MDP $\cM_\theta$, we denote $\pi^{\star,\theta} = \{\pi^{\star,\theta}_h\}_{h=1}^H$ as the optimal policy, and let $V_{h}^{\pi,\ror,\theta}$ (resp.~$V_{h}^{\star, \ror,\theta}$) represent the robust value function of policy $\pi$ (resp.~$\pi^{\star,\theta}$) at step $h$ with uncertainty radius $\ror$. 
Armed with these notation, we introduce the following lemma which collects the properties concerning the value functions and optimal policies.
\begin{lemma}\label{lem:finite-lb-value-theta}
Consider any $\theta\in\Theta$ and any policy $\pi$. Then it holds that
\begin{align}
    V_h^{\pi, \ror, \theta}(0) =  1 + x_h^{\pi,\theta} V_{h+1}^{\pi,\sigma, \theta}(0) 
    \label{eq:finite-lemma-value-0-pi-theta}
\end{align}
for any $h\in [H]$, where 
\begin{align}
x_h^{\pi,\theta} = \underline{p}_\star\pi_h(\theta_h\mymid 0) + \underline{q}_\star\pi_h(1-\theta_h\mymid 0).\label{eq:finite-x-h-theta}
\end{align}
In addition,  for any $h\in[H]$ and $s\in \cS\setminus \{0\}$,  the optimal policies and the optimal value functions obey
\begin{subequations}
    \label{eq:finite-lb-value-lemma-theta}
\begin{align}
    \pi_h^{\star,\theta}(\theta_h \mymid 0) &= 1,  &V_h^{\star, \ror,\theta}(0) \ge \tfrac{2}{3}(H+1-h), \\
    \pi_h^{\star,\theta}(\theta_h \mymid s) &= 1,      &V_h^{\star,\ror, \theta}(s)  = 0,
\end{align}
\end{subequations}
provided that $0<c_1\leq 1/2$. 
\end{lemma}
\begin{proof}
See Appendix~\ref{proof:lem:finite-lb-value-theta}.
\end{proof}

\paragraph{Construction of the history/batch dataset.}
In the nominal environment $\cM_\theta$, a batch dataset is generated consisting of $K$ {\em independent} sample trajectories each of length $H$, where each trajectory is generated according to \eqref{eq:finite-batch-size-def}, based on the following  initial state distribution $\rhob$ and behavior policy 
$\pib = \{\pib_h\}_{h=1}^H$: 
\begin{align}\label{eq:lower-dataset-assumption-theta}
    \rhob (s) = \mu(s)\quad \text{and } \quad\pi_h^{\mathsf{b}}(a \mymid s) = \frac{1}{2}, \qquad \forall (s,a,h)\in \cS\times \cA\times [H].
\end{align}
Here, $\mu(s)$ is  defined as the following state distribution supported on the state subset $\{0,1\}$:
\begin{align}\label{finite-mu-assumption-theta}
       \mu(s) = \frac{1}{CS}\mathds{1}(s = 0) + \Big(1 - \frac{1}{CS}\Big)\mathds{1}(s = 1),
\end{align}
where $\mathds{1}(\cdot)$ is the indicator function, and $C>0$ is some constant that will determine the concentrability coefficient $\Cstar$ (as we shall detail momentarily) and obeys
\begin{align}\label{eq:lower-C-assumption-theta}
    \frac{1}{CS} \leq \frac{1}{4}.
\end{align}

As it turns out, for any MDP $\mathcal{M}_\theta$, the occupancy distributions of the above batch dataset are the same (due to symmetry) and admit the following simple characterization: 
\begin{subequations}    \label{eq:finite-lb-behavior-distribution-theta}
\begin{align}
d^{\mathsf{b}, P^\theta}_1(0, a) = \frac{1}{2} \mu(0), \qquad &\forall a\in \cA ,   \\
     \frac{\mu(s)}{2} \leq  d^{\mathsf{b}, P^\theta}_h(s) \leq 2\mu(s), \qquad  \frac{\mu(s)}{4} \leq  d^{\mathsf{b}, P^\theta}_h(s,a) \leq \mu(s), \qquad &\forall (s,a,h)\in\cS\times \cA\times [H]. 
\end{align}
\end{subequations}
In addition, we choose the following initial state distribution  
\begin{align}
    \rho(s) = 
    \begin{cases} 1, \quad &\text{if }s=0 \\
        0, &\text{if }s>0
    \end{cases}.
    \label{eq:rho-defn-finite-LB-theta}
\end{align}

With this choice of $\rho$, the single-policy clipped concentrability coefficient $\Cstar$ and the quantity $C$ are intimately connected as follows:
\begin{align}
    2C \leq \Cstar \leq 4C. \label{eq:expression-Cstar-LB-finite-theta}
\end{align}

The proof of the claim \eqref{eq:finite-lb-behavior-distribution-theta} and \eqref{eq:expression-Cstar-LB-finite-theta} are postponed to Appendix~\ref{proof:eq:finite-lb-behavior-distribution-theta}.

\subsubsection{Establishing the minimax lower bound: small uncertainty level}

Towards this, we first make the following claim: for an arbitrary policy $\pi$ obeying
\begin{align}
    \sum_{h=1}^H \big\|\pi_h(\cdot\mymid 0) - \pi_h^{\star, \theta}(\cdot\mymid 0) \big\|_1 \ge \frac{H}{8},
\end{align}
one has
\begin{align}
    \big\langle \rho, V_1^{\star, \ror, \theta} - V_1^{\pi, \ror, \theta} \big\rangle > \varepsilon. \label{eq:finite-Value-0-recursive-theta}
\end{align}
We shall postpone the proof of this claim to Appendix~\ref{proof:eq:finite-Value-0-recursive-theta}. 

Armed with the above claim and following the same arguments in \cite[Section C.3.2]{li2022settling}, we complete the proof by observing: for some small enough constant $c_4$, as long as the sample size is beneath 
\begin{align}
    N  = KH \leq \frac{c_4 \Cstar S H^4}{ \varepsilon^2} , \label{eq:finite-LB-sample-condition}
\end{align}
then we necessarily have
\begin{align}
    \inf_{\widehat{\pi}} \max_{\theta \in \Theta}\mathbb{P}_\theta \left\{ V^{\star,\ror}_\theta(\rho) - V^{\widehat{\pi},\ror}_\theta (\rho)  \geq \varepsilon  \right\} \ge \frac{1}{4},
\end{align}
where  $\mathbb{P}_\theta$ denote the probability conditioned on that the MDP is $\cM_\theta$. We omit the details for brevity and complete the proof; interested readers can referred to \cite[Section C.3.2]{li2022settling}.
%

%%%%% ------%%%%% ------%%%%% ------%%%%% ------%%%%% ------

\subsubsection{Construction of hard problem instances: large uncertainty level}

We now move onto the case when the uncertainty level is relatively large, we construct another class of hard instances, which is almost the same as the previous one except for the transition kernel.

\paragraph{Construction of a collection of hard MDPs.} Let us introduce two MDPs
\begin{align}
   \left\{ \cM_\phi=
    \left(\mathcal{S}, \mathcal{A}, P^{\phi} = \{P^{\phi}_h\}_{h=1}^H, \{r_h\}_{h=1}^H, H \right) 
    \mymid \phi = \{0,1\}
    \right\},
\end{align}
where the state space is $\cS = \{0, 1, \ldots, S-1\}$, and the action space is $\mathcal{A} = \{0, 1\}$. 
The transition kernel $P^\phi$ of the constructed MDP $\cM_\phi$ is defined as
\begin{subequations}  \label{eq:Ph-construction-lower-finite}
\begin{align}
P^{\phi}_1(s^{\prime} \mymid s, a) = \left\{ \begin{array}{lll}
         p\mathds{1}(s^{\prime} = 0) + (1-p)\mathds{1}(s^{\prime} = 1)  & \text{if} & (s, a) = (0, \phi) \\
         q\mathds{1}(s^{\prime} = 0) + (1-q)\mathds{1}(s^{\prime} = 1) & \text{if} & (s, a) = (0, 1-\phi) \\
         \mathds{1}(s^{\prime} = 1) & \text{if}   & s=1 \\ 
         q\mathds{1}(s^{\prime} = s) + (1-q)\mathds{1}(s^{\prime} = 1) & \text{if}   & s > 1 \
                \end{array}\right.
        \label{eq:Ph-construction-lower-finite-1}
\end{align}
and 
\begin{align}
P^{\phi}_h(s^{\prime} \mymid s, a) = \mathds{1}(s^{\prime} = s),\qquad \forall (h,s,a) \in \{2,\ldots, H\} \times \cS \times \cA.
        \label{eq:Ph-construction-lower-finite-h}
\end{align}
\end{subequations}
In words, except at step $h=1$, the MDP always stays in the same state. Additionally, the MDP will always stay in the state subset $\{0,1\}$ if the initial distribution is supported only on $\{0,1\}$, in view of \eqref{eq:Ph-construction-lower-finite}. Here, $p$ and $q$ are set to be 
\begin{equation}
    p = 1 - \alpha 
    \qquad \text{and} \qquad
    q = 1 - \alpha - \Delta \label{eq:finite-p-q-def}
\end{equation}
for some $H \geq 2e^8$, $\alpha$ and $\Delta$ obeying  
\begin{equation}\label{eq:lower-p-q-beta-c1}
    0 < \alpha \leq \frac{1}{H} \leq \frac{1}{2e^8} \qquad \qquad \text{and} \qquad \Delta \leq \frac{\alpha}{2} \leq \frac{1}{2H} \leq \frac{1}{4e^8},
\end{equation}
where $\beta$ is set as 
 \begin{align}\label{eq:lower-bound-H-assumption}
\beta  \defn \frac{\log \frac{1}{\alpha+\Delta}}{2} \geq \frac{\log (2H/3)}{2} \geq 4.
\end{align} 
The assumption \eqref{eq:lower-p-q-beta-c1} immediately indicates the facts
\begin{align}\label{eq:lower-p-q-assumption}
    1> p>q \geq \frac{1}{2}.
\end{align}

Moreover, for any $(h,s,a)\in  [H]\times\cS\times \cA$, the reward function is defined as
\begin{align}
r_h(s, a) = \left\{ \begin{array}{lll}
         1 & \text{if } s = 0 \\
         0 & \text{otherwise}  \ 
                \end{array}\right. .
        \label{eq:rh-construction-lower-bound-finite}
\end{align}

\paragraph{Construction of the history/batch dataset.}
We utilize the same batch dataset described in Appendix~\ref{sec:small-uncertain-instance} and choose the same initial state distribution $\rho$ in \eqref{eq:rho-defn-finite-LB-theta}
% In the nominal environment $\cM_\phi$, a batch dataset is generated consisting of $K$ {\em independent} sample trajectories each of length $H$, where each trajectory is generated according to \eqref{eq:finite-batch-size-def}, based on the following  initial state distribution $\rhob$ and behavior policy 
% $\pib = \{\pib_h\}_{h=1}^H$: 
% %
% \begin{align}\label{eq:lower-dataset-assumption-theta}
%     \rhob (s) = \mu(s)\quad \text{and } \quad\pi_h^{\mathsf{b}}(a \mymid s) = \frac{1}{2}, \qquad \forall (s,a,h)\in \cS\times \cA\times [H].
% \end{align}
% %
% Here, $\mu(s)$ is  defined as the following state distribution supported on the state subset $\{0,1\}$:
% %
% \begin{align}\label{finite-mu-assumption-theta}
%        \mu(s) = \frac{1}{CS}\mathds{1}(s = 0) + \Big(1 - \frac{1}{CS}\Big)\mathds{1}(s = 1),
% \end{align}
% %
% where $\mathds{1}(\cdot)$ is the indicator function, and $C>0$ is some constant that will determine the concentrability coefficient $\Cstar$ (as we shall detail momentarily) and obeys
% \begin{align}\label{eq:lower-C-assumption}
%     \frac{1}{CS} \leq \frac{1}{4}.
% \end{align}
As a result, for any MDP $\mathcal{M}_\phi$, the occupancy distributions of the above batch dataset are the same (due to symmetry) and admit the following simple characterization: 
\begin{subequations}    \label{eq:finite-lb-behavior-distribution}
\begin{align}
d^{\mathsf{b}, P^\phi}_1(0, a) = \frac{1}{2} \mu(0), \qquad &\forall a\in \cA ,   \\
     \frac{\mu(s)}{2} \leq  d^{\mathsf{b}, P^\phi}_h(s) \leq 2\mu(s), \qquad  \frac{\mu(s)}{4} \leq  d^{\mathsf{b}, P^\phi}_h(s,a) \leq \mu(s), \qquad &\forall (s,a,h)\in\cS\times \cA\times [H]. 
\end{align}
\end{subequations}
The proof of the claim \eqref{eq:finite-lb-behavior-distribution} is postponed to Appendix~\ref{proof:eq:finite-lb-behavior-distribution}.

\paragraph{Uncertainty set of the transition kernels.}
Denote the transition kernel vector as
\begin{align}
    P_{h,s,a}^{\phi} \defn P^{\phi}_h(\cdot \mymid s,a) \in [0,1]^{1\times S}.
\end{align}
For any $(s,a,h)\in \cS\times\cA\times [H]$, the perturbation of the transition kernels in $\cM_\phi$ is restricted to the following uncertainty set
\begin{align}
    \unb^\ror(P^{\phi})\defn \otimes \; \unb^\ror \left(P^{\phi}_{h,s,a}\right),\qquad \unb^\ror(P_{h,s,a}^{\phi}) \defn \left\{ P_{h, s,a} \in \Delta (\cS): \mathsf{KL}\left(P_{h,s,a} \parallel P^{\phi}_{h,s,a}\right) \leq \ror \right\},
\end{align}
where the radius of the uncertainty set $\sigma$ obeys:
\begin{align}\label{eq:finite-lower-ror-bounded}
 \left(1- \frac{3}{\beta}\right) \log\left(\frac{1}{\alpha+ \Delta}\right) \leq \ror \leq \left(1- \frac{2}{\beta}\right)\log\left(\frac{1}{\alpha+ \Delta}\right).
\end{align}
% Note that the reasonable range of the uncertainty level $\sigma$ is $0<\sigma \leq \log(\frac{1}{P_{\mathsf{min}}}) = \log(\frac{1}{\alpha})$.

% the radius of the uncertainty set $\sigma$ obeys
% \begin{align}\label{eq:finite-lower-ror-bounded}
% \left(1- \frac{3}{\beta}\right) \log\left(\frac{1}{\alpha+ \Delta}\right) \leq \ror \leq \left(1- \frac{2}{\beta}\right)\log\left(\frac{1}{\alpha+ \Delta}\right).
% \end{align}
% Here $c_0$ is some constant satisfying $c_0 \in [2, $

Before continuing, we shall introduce some notation for convenience. For any $P^{\phi}_h(\cdot  \mymid s, a)$ in \eqref{eq:Ph-construction-lower-finite}, we define the limit of the perturbed kernel transiting to the next state $s'$ from the current state-action pair $(s,a)$ by
\begin{align}\label{eq:finite-lw-def-p-q}
\underline{P}^{\phi}_{h}(s' \mymid s,a) &\defn \inf_{P_{h,s,a} \in \unb^{\ror}(P^{\phi}_{h,s,a})} P_h(s'  \mymid s,a),
\end{align}
and in particular, denote
\begin{align}\label{eq:finite-lw-p-q-perturb-inf}
\underline{p} &\defn \underline{P}^{\phi}_{1}(0 \mymid 0,\phi)  ,\qquad \underline{q} = \underline{P}^{\phi}_{1}(0  \mymid 0, 1-\phi).
\end{align}
Armed with the above definitions, we introduce the following lemma which implies some useful properties of the uncertainty set.

\begin{lemma}\label{lem:finite-lb-uncertainty-set}
When $\beta$ satisfies \eqref{eq:lower-bound-H-assumption} and the uncertainty level $\ror$ satisfies \eqref{eq:finite-lower-ror-bounded}, the perturbed transition kernels obey
\begin{align}\label{eq:finite-lw-upper-p-q}
   \underline{p}\geq \underline{q} \geq \frac{1}{\beta}.
\end{align}
\end{lemma}

% \begin{lemma}\label{lem:finite-lb-uncertainty-set}
% When $\beta$ satisfies \eqref{eq:lower-bound-H-assumption} and the uncertainty level $\ror$ satisfies \eqref{eq:finite-lower-ror-bounded}, the perturbed transition kernels obey
% \begin{align}\label{eq:finite-lw-upper-p-q}
%    \underline{p}\geq \underline{q} \geq \frac{1}{\beta}.
% \end{align}
% \end{lemma}
\begin{proof}
See Appendix~\ref{proof:lem:finite-lb-uncertainty-set}.  
\end{proof} 

\paragraph{Value functions and optimal policies.}
Similar to Appendix~\ref{sec:small-uncertain-instance}, for any MDP $\cM_\phi$, we denote $\pi^{\star,\phi} = \{\pi^{\star,\phi}_h\}_{h=1}^H$ as the optimal policy, and let $V_{h}^{\pi,\ror,\phi}$ (resp.~$V_{h}^{\star, \ror,\phi}$) represent the robust value function of policy $\pi$ (resp.~$\pi^{\star,\phi}$) at step $h$ with uncertainty radius $\ror$. Then we introduce the following lemma which collects the properties concerning the value functions and optimal policies.
\begin{lemma}\label{lem:finite-lb-value}
For any $\phi = \{0,1\}$ and any policy $\pi$, defining 
\begin{align}
z_{\phi}^{\pi} \defn \underline{p}\pi_1(\phi\mymid 0) + \underline{q} \pi_1(1-\phi \mymid 0),\label{eq:finite-x-h}
\end{align}
it holds that
\begin{align}
    V_1^{\pi, \sigma, \phi}(0) =  1 + z_{\phi}^{\pi} (H-1).
    \label{eq:finite-lemma-value-0-pi}
\end{align}
In addition, the optimal policies and the optimal value functions obey 
\begin{subequations}
    \label{eq:finite-lb-value-lemma}
\begin{align}
    V_1^{\star,\sigma, \phi}(0) &= 1 + \underline{p}(H-1),\\
    \forall h\in[H] \setminus \{1\}:\quad V_h^{\star,\sigma, \phi}(0) &= H - h + 1,\\
    \forall h\in[H]: \quad \pi_h^{\star,\phi}(\phi \mymid 0) &= 1, \quad \pi_h^{\star,\phi}(\phi \mymid 1) = 1, \qquad V_h^{\star,\sigma, \phi}(1) = 0.
\end{align}
\end{subequations}
The robust single-policy clipped concentrability coefficient $\Cstar$ obeys
\begin{align}
  2C \leq  \Cstar \leq 4C. \label{eq:expression-Cstar-LB-finite}
\end{align}

\end{lemma}
\begin{proof}
See Appendix~\ref{proof:lem:finite-lb-value}.
\end{proof}

In view of Lemma~\ref{lem:finite-lb-value}, we note that the smallest positive state transition probability of the optimal policy $\pi^\star$ under any MDP $\cM_\phi$ with $\phi \in \{0,1\}$ thus can be given by
\begin{align}\label{eq:P-min-phi}
\minpall \defn \min_{h,s,s'} \Big\{P_h^{\phi}\left(s' |s, \pi^{\star,\phi}_h(s)\right):\; P_h^{\phi}\left(s' |s, \pi^{\star,\phi}_h(s)\right)>0 \Big\} = P_1^{\phi}\left( 1 |0, 1- \phi \right) = 1- p ,
\end{align}
which obeys 
$$ \alpha = \minpall \in (0, 1/H]$$ 
according to \eqref{eq:finite-p-q-def} and \eqref{eq:lower-p-q-beta-c1}.

\subsubsection{Establishing the minimax lower bound: large uncertainty level}\label{eq:proof-finite-lower-bound}
We are now ready to establish the sample complexity lower bound. With the choice of the initial distribution $\rho$ in \eqref{eq:rho-defn-finite-LB-theta}, for any policy estimator $\widehat{\pi}$ computed based on the batch dataset, we plan to control the quantity
$$\big\langle \rho, V_1^{\star, \sigma, \phi} - V_1^{\widehat{\pi},\sigma, \phi} \big\rangle = V_1^{\star, \sigma, \phi}(0) - V_1^{\widehat{\pi}, \sigma, \phi}(0).$$

\paragraph{Step 1: converting the goal to estimate $\phi$.} 
We make the following claim which shall be verified in Appendix~\ref{proof:finite-lower-diff-control}: given $    \varepsilon \leq \frac{H}{384  e^6 \log\left(\frac{1}{\alpha}\right)}\leq  \frac{H}{384  e^6 \log\left(\frac{1}{\alpha+ \Delta}\right)}$,
choosing 
\begin{align}\label{eq:Delta-chosen}
    \Delta = \frac{128  e^6  \sigma (1-q) \varepsilon}{H} =\frac{128  e^6  \sigma (\alpha+ \Delta) \varepsilon}{H}  \leq \frac{128  e^6   (\alpha+ \Delta) \varepsilon \log\left(\frac{1}{\alpha+ \Delta}\right)}{  H} \leq \frac{\alpha}{2},
\end{align}
which satisfies \eqref{eq:lower-p-q-beta-c1} with the aid of \eqref{eq:finite-lower-ror-bounded} and \eqref{eq:finite-p-q-def},
it holds that for any policy $\widehat{\pi}$, 
\begin{align}
    \big\langle \rho, V_1^{\star, \sigma, \phi} - V_1^{\widehat{\pi}, \sigma, \phi} \big\rangle  \geq 2\varepsilon \big(1-\widehat{\pi}_1(\phi\mymid 0)\big). \label{eq:finite-Value-0-recursive}
\end{align}
Armed with this relation between the policy $\widehat{\pi}$ and its sub-optimality gap, we are positioned to construct an estimate of $\phi$. We denote $\mathbb{P}_\phi$ as the probability distribution when the MDP is $\mathcal{M}_\phi$, for any $\phi \in \{0,1\}$.

Suppose for the moment that a policy estimate $\widehat{\pi}$ achieves
\begin{align}
    \mathbb{P}_\phi \left\{\big\langle \rho, V_1^{\star, \sigma, \phi} - V_1^{\widehat{\pi}, \sigma, \phi} \big\rangle \leq \varepsilon\right\} \geq \frac{7}{8},
    \label{eq:assumption-theta-small-LB-finite}
\end{align}
then in view of \eqref{eq:finite-Value-0-recursive}, 
we necessarily have $\widehat{\pi}_1(\phi\mymid 0) \geq \frac{1}{2}$ with probability at least $\frac{7}{8}$.
With this in mind, we are motivated to construct the following estimate $\widehat{\phi}$ for $\phi\in \{0,1\}$: 
\begin{align}
    \widehat{\phi}=\arg\max_{a\in \{0,1\}} \, \widehat{\pi}_1(a\mymid 0),
    \label{eq:defn-theta-hat-inf-LB}
\end{align}
which obeys
\begin{align}
    \mathbb{P}_{\phi}\big\{ \widehat{\phi} = \phi \big\} 
    \geq \mathbb{P}_{\phi}\big\{ \widehat{\pi}_1(\phi \mymid 0) > 1/2 \big\} \geq \frac{7}{8}. 
    \label{eq:P-theta-accuracy-inf}
\end{align}
In what follows, we would like to show \eqref{eq:P-theta-accuracy-inf} cannot happen without enough samples, which would in turn contradict \eqref{eq:finite-Value-0-recursive}.

\paragraph{Step 2: probability of error in testing two hypotheses.}
Armed with the above preparation, we shall focus on differentiating the two hypotheses $\phi \in\{ 0, 1\}$. 
Towards this, consider the minimax probability of error defined as follows:
\begin{equation}
    p_{\mathrm{e}} \coloneqq \inf_{\psi}\max \big\{ \mathbb{P}_{0}(\psi \neq 0), \, \mathbb{P}_{1}(\psi \neq 1) \big\} , \label{eq:error-prob-two-hypotheses-finite-LB}
\end{equation}
where the infimum is taken over all possible tests $\psi$ constructed from the batch dataset.

Let $\mu^{\mathsf{b},\phi}$ (resp.~$\mu^{\mathsf{b},\phi}_h(s_h)$) be the distribution of a sample trajectory $\{s_h, a_h\}_{h=1}^H$ (resp.~a sample $(a_h,s_{h+1})$ conditional on $s_h$) for the MDP $\mathcal{M}_\phi$. Following standard results from \citet[Theorem~2.2]{tsybakov2009introduction} and  the additivity of the KL divergence (cf.~\citet[Page~85]{tsybakov2009introduction}), we obtain
\begin{align}
p_{\mathrm{e}} &  \geq \frac{1}{4}\exp\Big(- K\mathsf{KL} \big(\mu^{\mathsf{b},0}\parallel \mu^{\mathsf{b},1} \big)\Big)\nonumber\\
    & \geq \frac{1}{4}\exp\bigg\{-\frac{1}{2} K\mu(0)\Big(\mathsf{KL}\big(P_{1}^0(\cdot\mymid0,0)\parallel P_{1}^1(\cdot\mymid0,0)\big)+\mathsf{KL}\big(P_{1}^0(\cdot\mymid0,1)\parallel P_{1}^1(\cdot\mymid0,1)\big)\Big)\bigg\},
    \label{eq:finite-remainder-KL}
\end{align}
where we also use the independence of the $K$ trajectories in the batch dataset in the first line. Here, the second line arises from the chain rule of the KL divergence \citep[Lemma 5.2.8]{duchi2018introductory} and the Markov property of the sample trajectories (recall that $d^{\mathsf{b}, P^0}_h = d^{\mathsf{b}, P^1}_h$) according to
\begin{align}
\mathsf{KL}\big(\mu^{\mathsf{b},0}\parallel\mu^{\mathsf{b},1}\big) & =\sum_{h=1}^{H}\mathop{\mathbb{E}}\limits _{s_{h}\sim d^{\mathsf{b}, P^0}_{h}}\left[\mathsf{KL}\big(\mu_{h}^{\mathsf{b},0}(s_{h})\parallel\mu_{h}^{\mathsf{b},1}(s_{h})\big)\right] \notag \\
&=\sum_{a\in\{0,1\}}  d^{\mathsf{b}, P^0}_{1}(0,a)\mathsf{KL}\left(P_{1}^{0}(\cdot\mymid0,a)\parallel P_{1}^{1}(\cdot\mymid0,a)\right)  \notag\\
    & =  \frac{1}{2}\mu(0)\sum_{a\in\{0,1\}}\mathsf{KL}\big(P_{1}^{0}(\cdot\mymid0,a)\parallel P_{1}^{1}(\cdot\mymid0,a)\big), \notag 
\end{align}
where the penultimate equality holds by the fact that $P_{h}^{0}(\cdot\mymid s,a)$ and $P_{h}^{1}(\cdot\mymid s,a)$ only differ when $h=1$ and $s=0$, and the last equality follows from \eqref{eq:finite-lb-behavior-distribution}.

It remains to control the KL divergence terms in \eqref{eq:finite-remainder-KL}.
Given $p\geq q \geq 1/2$ (cf.~\eqref{eq:lower-p-q-assumption}), applying Lemma~\ref{lem:KL-key-result} (cf.~\eqref{eq:KL-dis-key-result}) yields
\begin{align}
\mathsf{KL}\big(P_1^{0}(\cdot \mymid 0, 0)\parallel P_1^{1}(\cdot \mymid 0, 0)\big) & =\mathsf{KL}\left(p\parallel q \right)  \leq \frac{(p-q)^2}{(1-p)p}  \overset{\mathrm{(i)}}{=} \frac{\Delta^2}{p(1-p)} \notag\\
    & \overset{\mathrm{(ii)}}{=} \frac{ 128^2 e^{12} \ror^2 (1-q)^2 \varepsilon^2}{H^{2} p(1-p)} \notag \\
    & \overset{\mathrm{(iii)}}{\leq} \frac{ c_1 \ror^2 \minpall  \varepsilon^2}{H^{2}},
    \label{eq:finite-KL-bounded}
\end{align}
where (i) follows from the definition \eqref{eq:finite-p-q-def}, (ii) holds by plugging in the expression of $\Delta$ in \eqref{eq:Delta-chosen}, (iii) arises from $1-q \leq 2(1-p) = 2\minpall$ (see \eqref{eq:lower-p-q-beta-c1} and \eqref{eq:P-min-phi}), $p>\frac{1}{2}$, as long as $c_1$ is a large enough constant.
It can be shown that $\mathsf{KL}\big(P_1^{0}(\cdot \mymid 0, 1)\parallel P_1^{1}(\cdot \mymid 0, 1)\big)$ 
can be upper bounded in the same way. 
Substituting \eqref{eq:finite-KL-bounded} back into \eqref{eq:finite-remainder-KL} 
demonstrates that: if the sample size is chosen as
\begin{align}\label{eq:finite-sample-N-condition}
    KH \leq \frac{H^{3 } S\Cstar \log 2}{ 4 c_1  \minpall \ror^2 \varepsilon^2},
\end{align}
then one necessarily has
\begin{align}
    p_{\mathrm{e}} &\geq \frac{1}{4}\exp\bigg\{-\frac{1}{2} K\mu(0) \cdot 2\frac{ c_1 \ror^2 \minpall  \varepsilon^2}{H^{2}} \bigg\}  \overset{\mathrm{(i)}}{=} \frac{1}{4}\exp\bigg\{ -K \frac{ c_1 \ror^2 \minpall  \varepsilon^2}{SC H^{2}} \bigg\} \notag \\
    & \overset{\mathrm{(ii)}}{\geq} \frac{1}{4}\exp\bigg\{ -K\frac{4 c_1 \ror^2 \minpall  \varepsilon^2 }{S\Cstar H^{2 }} \bigg\} \geq \frac{1}{8}, \label{eq:pe-LB-13579-inf}
\end{align}
where (i) follows from \eqref{finite-mu-assumption-theta} and (ii) holds by \eqref{eq:expression-Cstar-LB-finite}. 

\paragraph{Step 3: putting things together.} 
Finally, suppose that there exists an estimator $\widehat{\pi}$ such that
\[
    \mathbb{P}_0 \big\{  \big\langle \rho, V_1^{\star, \sigma, 0} - V_1^{\widehat{\pi}, \sigma, 0} \big\rangle > \varepsilon  \big\} < \frac{1}{8}
    \qquad \text{and} \qquad
    \mathbb{P}_1 \big\{  \big\langle \rho, V_1^{\star,\sigma, 1} - V_1^{\widehat{\pi},\sigma, 1} \big\rangle > \varepsilon  \big\} < \frac{1}{8}.
\]
Then Step 1 tells us that the estimator $\widehat{\phi}$ defined in \eqref{eq:defn-theta-hat-inf-LB} must satisfy
\[
    \mathbb{P}_0\big(\widehat{\phi} \neq 0\big) < \frac{1}{8} 
    \qquad \text{and} \qquad
    \mathbb{P}_1\big(\widehat{\phi} \neq 1\big) < \frac{1}{8},
\]
which cannot happen under the sample size condition in \eqref{eq:finite-sample-N-condition} to avoid contradition with \eqref{eq:pe-LB-13579-inf}.  The proof is thus finished.

%%%%% ------%%%%% ------%%%%% ------%%%%% ------%%%%% ------

\subsubsection{Proof of Lemma~\ref{lem:finite-lb-uncertainty-set-KL}}\label{proof:lem:finite-lb-uncertainty-set-KL}

First, \eqref{eq:finite-lw-upper-p-q-theta-0} can be easily verified by the definition of $\underline{p}_h$ and $\underline{q}_h$ in \eqref{eq:finite-lw-p-q-perturb-inf-theta} and the transition $P^{\theta_h}_h$ in \eqref{eq:Ph-construction-lower-bound-finite-theta}.

\paragraph{Proof of the first inequality in \eqref{eq:finite-lw-upper-p-q-theta}.}

It is observed that
\begin{align}
\mathsf{KL}\left(\mathsf{Ber}\left(1-\frac{2c_1}{H} \right) \parallel \mathsf{Ber}(q)\right)  
& = \left(1-\frac{2c_1}{H} \right)\log\left(\frac{1-\frac{2c_1}{H}}{q} \right)  + \frac{2c_1}{H}\log\left( \frac{\frac{2c_1}{H}}{1-q}\right)  \notag \\
& = \left(1-\frac{2c_1}{H} \right)\log\left( 1 + \frac{1- q -\frac{2c_1}{H}}{q} \right)  + \frac{2c_1}{H}\log\left(  1 + \frac{q-1 + \frac{2c_1}{H}}{1-q}\right)  \notag \\
&\overset{\mathrm{(i)}}{\geq} \left(1-\frac{2c_1}{H} \right) \frac{\frac{1- q -\frac{2c_1}{H}}{q} }{2q}\frac{1+q - \frac{2c_1}{H}}{1-\frac{2c_1}{H}} + \frac{2c_1}{H} \cdot 2 \frac{q-1 +\frac{2c_1}{H} }{1-q+\frac{2c_1}{H} } \notag \\
& = (q-1 +\frac{2c_1}{H}) \left[ -\frac{1+q - \frac{2c_1}{H}}{2q^2} + \frac{4c_1}{H}\frac{1}{1-q+\frac{2c_1}{H} }\right] \notag \\  
&\overset{\mathrm{(ii)}}{\geq} \frac{c_1}{2H} \left[ \frac{-1}{(1-\frac{3c_1}{2H})^2}+ \frac{8}{7} \right] \geq \frac{1}{20H} \geq \ror,
\end{align}
where (i) holds by $\log(1+x) \geq \frac{x}{2(1+x)}$ when $0 \leq x < \infty$ and $\log(1+x) \geq \frac{x}{2}\frac{2+x}{1+x}$ when $-1< x \leq 0$ \citep{topsoe12007some}, and the penultimate inequality holds by $1-\frac{3c_1}{2H} \geq 1 - \frac{3}{256} \geq \frac{63}{64}$. Here, (ii) can be verified by 
\begin{align*}
    -\frac{1+q - \frac{2c_1}{H}}{2q^2} &\geq -\frac{1+q }{2q^2} \geq \frac{-2}{2q^2} \overset{\mathrm{(iii)}}{\geq} \frac{-1}{(1-\frac{3c_1}{2H})^2},   \\
    \frac{4c_1}{H}\frac{1}{1-q+\frac{2c_1}{H} }  &\overset{\mathrm{(iv)}}{\geq}  \frac{4c_1}{H}\frac{1}{ \frac{3c_1}{2H} +\frac{2c_1}{H} } \geq \frac{8}{7},
\end{align*} 
where (iii) holds by $q\geq 1-\frac{3c_1}{2H}$, and (iv) arises from $  0\leq 1- q \leq \frac{3c_1}{2H}$.

With above fact and $\mathsf{KL}\left(\mathsf{Ber}\left(\underline{q}_\star \right) \parallel \mathsf{Ber}(q)\right) = \ror $ in mind, applying  Lemma~\ref{lem:KL-key-result} leads to
$\underline{q}_\star \geq 1 - \frac{2c_1}{H}$.

\paragraph{Proof of the second inequality in \eqref{eq:finite-lw-upper-p-q-theta}.}
First, observing the first claim in \eqref{eq:finite-lw-upper-p-q-theta}, combined with Lemma~\ref{lem:KL-key-result}, we know that for any uncertainty level $\ror \leq \frac{1}{20H}$, there exists a unique $\underline{q}_\star$ obeying $\underline{q}_\star \geq 1- \frac{2c_1}{H} > \frac{1}{2}$ such that
\begin{align}
\ror = \mathsf{KL}\left(\mathsf{Ber}\left(\underline{q}_\star \right) \parallel \mathsf{Ber}(q)\right). \label{eq:some-fact-of-underline-q}
\end{align}
Then let us define the following function for $0<x\leq \frac{1-q}{3}$ (i.e., $p=q+x$)
\begin{align}
g(x, q) = \mathsf{KL}\left(\mathsf{Ber}\left(\underline{q}_\star +x \right) \parallel \mathsf{Ber}(q + x)\right) = (\underline{q}_\star +x)\log\frac{\underline{q}_\star +x}{q+x} + (1-\underline{q}_\star -x) \log\frac{1-\underline{q}_\star -x}{1-q-x}
\end{align}
The first derivative $\nabla_x g(x,q)$ is
\begin{align}
&\nabla_x g(x,q) \nonumber\\
&= \log\left(\frac{\underline{q}_\star +x}{q+x}\right) + (q+x)\frac{q+x - (\underline{q}_\star+x)}{(q+x)^2} - \log\left(\frac{1-\underline{q}_\star -x}{1-q-x}\right)  + (1-q - x)\frac{-(1-q-x)+(1-\underline{q}_\star-x)}{(1-q-x)^2} \notag \\
& =\log\left(\frac{\underline{q}_\star +x}{q+x}\right) - \log\left(\frac{1-\underline{q}_\star -x}{1-q-x}\right) + \frac{q -\underline{q}_\star  }{ q+x} + \frac{q -\underline{q}_\star}{1-q-x} \notag \\
& = \log\left(1 + \frac{\underline{q}_\star -q}{q+x}\right) - \log\left( 1 + \frac{q-\underline{q}_\star}{1-q-x}\right) + \frac{q -\underline{q}_\star  }{ q+x} + \frac{q -\underline{q}_\star}{1-q-x} \notag \\
&\overset{\mathrm{(i)}}{\geq}\frac{\underline{q}_\star -q}{2(q+x)} \frac{2 + \frac{\underline{q}_\star -q}{q+x}}{1 + \frac{\underline{q}_\star -q}{q+x}} - \frac{q-\underline{q}_\star}{2(1-q-x)}\frac{2+\frac{q-\underline{q}_\star}{1-q-x}}{1+\frac{q-\underline{q}_\star}{1-q-x}}+ \frac{q -\underline{q}_\star  }{ q+x} + \frac{q -\underline{q}_\star}{1-q-x} \notag \\
&= \frac{q - \underline{q}_\star}{2(q+x)} \left(2 - \frac{2q+2x + \underline{q}_\star -q}{q+x + \underline{q}_\star -q} \right) +  \frac{q-\underline{q}_\star}{2(1-q-x)} \left(2 - \frac{2 - 2q -2x +q-\underline{q}_\star}{1-q-x +q-\underline{q}_\star} \right) \notag \\
& = \frac{q - \underline{q}_\star}{2(q+x)} \frac{\underline{q}_\star - q}{x + \underline{q}_\star} + \frac{q-\underline{q}_\star}{2(1-q-x)} \frac{q - \underline{q}_\star}{1-\underline{q}_\star - x} \notag \\
& = \frac{(q-\underline{q}_\star)^2 \left[(q+x)(\underline{q}_\star+x) - (1-q-x)(1-\underline{q}_\star-x)\right]}{2(q+x)(\underline{q}_\star+x)(1-q-x)(1-\underline{q}_\star-x)} \geq 0
\end{align}
where (i) holds by $\log(1+x) \leq \frac{x}{2}\frac{2+x}{1+x}$ when $0 \leq x < \infty$ and $\log(1+x) \geq \frac{x}{2}\frac{2+x}{1+x}$ when $-1< x \leq 0$ \citep{topsoe12007some}, and the last inequality always holds for any $0 <x \leq \frac{1-q}{3}$ and $\underline{q}_\star \geq \frac{1}{2}$.

The above fact shows that $g(p-q,q) \geq g(0,q) = \sigma$, and thus
\begin{align}
\mathsf{KL}\left(\mathsf{Ber}\left(p - q + \underline{q}_\star\right) \parallel \mathsf{Ber}(p)\right) \geq \sigma,
\end{align}
which complete the proof by observing that
$\underline{p}_\star \geq p - q + \underline{q}_\star$ via applying Lemma~\ref{lem:KL-key-result} with \eqref{eq:some-fact-of-underline-q}.

 \subsubsection{Proof of Lemma~\ref{lem:finite-lb-value-theta}}\label{proof:lem:finite-lb-value-theta}

\paragraph{Ordering the value function for different states.}
First, note that for any policy $\pi$ at the final step $H+1$, we have
\begin{align}
\forall s\in\cS: \quad V_{H+1}^{\pi,\ror,\theta}(s) = 0.
\end{align}
Then for any $\theta \in \Theta$ and any policy $\pi$, it is easily verified that
\begin{align}
V_{H}^{\pi,\ror,\theta}(0) &= \mathbb{E}_{a\sim\pi_h(\cdot \mymid 0)}\left[r_h(0,a) + \inf_{ \cP \in \unb^{\sigma}(P^{\theta_h}_{h,0,a})}  \cP  V_{H+1}^{\pi,\ror,\theta}\right] = 1, \notag \\
% V_{H}^{\pi,\ror,\theta}(1) &= \mathbb{E}_{a\sim\pi_h(\cdot \mymid 1)}\left[r_h(1,a) + \inf_{ \cP \in \unb^{\sigma}(P^{\theta_h}_{h,1,a})}  \cP  V_{H+1}^{\pi,\ror,\theta}\right] \leq \frac{1}{2}, \notag \\
V_{H}^{\pi,\ror,\theta}(s) &= \mathbb{E}_{a\sim\pi_h(\cdot \mymid s)}\left[r_h(s,a) + \inf_{ \cP \in \unb^{\sigma}(P^{\theta_h}_{h,s,a})}  \cP  V_{H+1}^{\pi,\ror,\theta}\right] =0, \quad \forall s\in \cS \setminus \{0\},
\end{align}
which directly indicates that
\begin{align}
    \forall s\in\cS\setminus\{0,1\}:\quad  V_{H}^{\pi,\ror,\theta}(0)  > V_{H}^{\pi,\ror,\theta}(s) = 0.\label{eq:ordering-states-base-case}
\end{align}

Then we provide the following claim which will be proved momentarily using induction: For any $\theta \in \Theta$ and any policy $\pi$, the following equation holds
\begin{align}
    \forall (h,s)\in [H] \times \cS\setminus\{0\} : \quad  V_{h}^{\pi,\ror,\theta}(0) > V_{h}^{\pi,\ror,\theta}(s) =0. \label{eq:ordering-states}
\end{align}

The above result leads to the following immediate fact for state $s\in \cS\setminus\{0\}$:
\begin{align}
    \forall (h,s)\in [H] \times\cS\setminus\{0\}:  \quad V_h^{\star,\sigma,\theta }(s) = \max_{\pi}  V_h^{\pi,\sigma,\theta }(s) = 0,
\end{align}
since \eqref{eq:ordering-states} holds for any $\pi$ and $h\in[H]$.
Therefore, for any state $s\in\cS\setminus\{0\}$, without loss of generality, we choose the optimal policy obeying
\begin{align}
\forall (h,s)\in [H] \times\cS\setminus\{0\}: \qquad  \pi_h^{\star,\theta}(\theta_h  \mymid s) = 1.
\end{align}

Then the rest of the proof will focus on deriving the value function and optimal policy over state $s=0$. To begin with, recalling the value function in \eqref{eq:value-0-theta}
\begin{align}
V_h^{\pi,\ror, \theta}(0) &=1 + x_h^{\pi,\theta} V_{h+1}^{\pi,\sigma, \theta}(0) \label{eq:interemediate_opt_10-theta}
\end{align}
and observing that the function $V_h^{\pi,\ror, \theta}(0)$ is increasing in $x_h^{\pi,\theta}$ and that $x_h^{\pi,\theta}$ is increasing in $\pi_h(\theta_h \mymid 0)$ (due to the fact $\underline{p}_\star\geq \underline{q}_\star$ in \eqref{eq:finite-lw-upper-p-q-theta}).
As a result,  the optimal policy obeys
\begin{equation}
    \pi_h^{\star,\theta}(\theta_h \mymid 0) = 1 \label{eq:finite-lb-optimal-policy-theta}
\end{equation}
 at state $0$. Plugging it back to \eqref{eq:interemediate_opt_10-theta} gives
 \begin{align}
    V_h^{\star,\ror, \theta}(0) &=1 + x_h^{\pi^{\star,\theta}} V_{h+1}^{\star,\sigma, \theta}(0) \notag \\
    &\overset{\mathrm{(i)}}{=} 1 + \underline{p}_\star V_{h+1}^{\star, \ror,\theta}(0)  \ge \sum_{j=0}^{H-h} \underline{p}_\star^j \ge \sum_{j = 0}^{H-h} \Big(1 - \frac{c_3}{H} \Big)^j = \frac{1-\big(1-\frac{c_{3}}{H}\big)^{H-h+1}}{c_{3}/H}\notag\\
    &\ge \frac{2}{3}(H+1-h).
 \end{align}
 where (i) holds by $ x_h^{\pi^\star,\theta} = \underline{p}_\star\pi^{\star, \theta}_h(\theta_h \mymid 0) + \underline{q}_\star \pi^{\star, \theta}_h(1-\theta_h \mymid 0) = \underline{p}_\star$. Here, the last inequality holds since we observe that
\[
    \Big(1-\frac{c_{3}}{H}\Big)^{H-h+1}\leq\exp\left(-\frac{c_{3}}{H}(H-h+1)\right)\leq1-\frac{2c_{3}(H-h+1)}{3H},
\]
as long as $c_3\leq 0.5$, which follows due to the elementary inequalities $1-x\leq \exp(-x)$ for any $x\geq 0$ and $\exp(-x)\leq 1-2x/3$ for any $0\leq x\leq 1/2$.

\paragraph{Proof of claim in \eqref{eq:ordering-states}.}
We shall \eqref{eq:ordering-states} through induction. Towards this, assuming that at time step $h+1$, the following holds
\begin{align}
 \forall s\in\cS\setminus\{0,1\}: \quad  V_{h+1}^{\pi,\ror,\theta}(0) > V_{h+1}^{\pi,\ror,\theta}(s) = 0. \label{eq:ordering-states-assumption}
\end{align}
Observing that the base case when $h=H$ has already been confirmed in \eqref{eq:ordering-states-base-case}, now we move on to prove the same property for time step $h$.

To start with, the robust value function of state $0$ at step $h$ satisfies
\begin{align}
    V_h^{\pi,\ror, \theta}(0) &= \mathbb{E}_{a\sim\pi_h(\cdot \mymid 0)}\left[r_h(0,a) + \inf_{ \cP \in \unb^{\sigma}(P^{\theta_h}_{h,0,a})}  \cP V_{h+1}^{\pi,\ror, \theta}\right] \notag \\
    & \overset{\mathrm{(i)}}{=} 1 + \pi_h(\theta_h \mymid 0)\Big(  \inf_{ \cP \in \unb^{\sigma}(P^{\theta_h}_{h,0,\theta})}  \cP   V_{h+1}^{\pi,\ror,\theta} \Big) 
     + \pi_h(1-\theta_h \mymid 0)\Big(   \inf_{ \cP \in \unb^{\sigma}(P^{\theta_h}_{h,0,1-\theta_h})}  \cP V_{h+1}^{\pi,\ror, \theta}   \Big) \notag \\
    & \overset{\mathrm{(ii)}}{=} 1 + \pi_h(\theta_h \mymid 0)\Big[ \underline{p} V_{h+1}^{\pi,\sigma,\theta}(0) + \left(1- \underline{p}\right) V_{h+1}^{\pi,\sigma,\theta}(1)\Big] + \pi_h(1-\theta_h \mymid 0)\Big[ \underline{q} V_{h+1}^{\pi,\sigma,\theta}(0) + \left(1-\underline{q} \right) V_{h+1}^{\pi,\sigma,\theta}(1)\Big] \notag\\
    & \overset{\mathrm{(iii)}}{=} 1 + V_{h+1}^{\pi,\sigma,\theta}(1) + x_h^{\pi,\theta}  \left[V_{h+1}^{\pi,\sigma,\theta}(0) - V_{h+1}^{\pi,\sigma,\theta}(1) \right] \notag\\
    & = 1 + x_h^{\pi,\theta} V_{h+1}^{\pi,\sigma, \theta}(0) \label{eq:value-0-theta}
\end{align}
where (i) uses the definition of the reward function in \eqref{eq:rh-construction-lower-bound-finite-theta}, (ii) uses the induction assumption in \eqref{eq:ordering-states-assumption} so that the infimum is attained by picking the choice specified in \eqref{eq:finite-lw-p-q-perturb-inf-theta} with a smallest probability mass imposed on the transition to state $0$. Finally, we plug in the definition \eqref{eq:finite-x-h-theta} of $x_h^{\pi,\theta}$ in (iii), and the last line follows from \eqref{eq:ordering-states-assumption}.

\begin{align*}
    V_h^{\pi,\sigma,\theta }(s) &= \mathbb{E}_{a\sim\pi_h(\cdot \mymid s)}\left[r_h(s,a) + \inf_{ \cP \in \unb^{\sigma}(P^{\theta_h}_{h,s,a})}  \cP  V_{h+1}^{\pi,\sigma,\theta}\right] = 0, \label{eq:value-otherstate-theta}
\end{align*}
where the last inequality holds by the reward and transition function in \eqref{eq:rh-construction-lower-bound-finite-theta} and \eqref{eq:Ph-construction-lower-bound-finite-theta} with the induction assumption \eqref{eq:ordering-states-assumption}.

Combining \eqref{eq:value-0-theta} and \eqref{eq:value-otherstate-theta}, we complete the proof:
\begin{align}
    \forall s\in\cS\setminus\{0\}: \quad  V_h^{\pi,\ror, \theta_h}(0) \geq 1 > V_h^{\pi,\sigma,\theta }(s) =0.
\end{align}

\subsubsection{Proof of claim \eqref{eq:finite-lb-behavior-distribution-theta} and \eqref{eq:expression-Cstar-LB-finite-theta}}\label{proof:eq:finite-lb-behavior-distribution-theta}

 \paragraph{Proof of the claim~\eqref{eq:finite-lb-behavior-distribution-theta}. }
With the initial state distribution and behavior policy defined in \eqref{eq:lower-dataset-assumption-theta}, we have for any MDP $\mathcal{M}_\theta$,
\begin{align*}
     d^{\mathsf{b}, P^\theta}_1(s) = \rhob (s) = \mu(s),
\end{align*}
which leads to
\begin{align}
   \forall a\in \cA:\quad  d^{\mathsf{b}, P^\theta}_1(0, a) = \frac{1}{2} \mu(0).
\end{align}
In view of \eqref{eq:Ph-construction-lower-finite-1}, the state occupancy distribution at any step $h=2, 3, \cdots, H$ obeys
\begin{align}
    d^{\mathsf{b}, P^\theta}_h(0) &\geq \mathbb{P}\left\{s_h = 0 \mymid s_{h-1} =0 ; \pib \right\}  \geq d^{\mathsf{b}, P^\theta}_{h-1}(0)\left[\pib_{h-1}(\theta_{h-1} \mymid 0)\underline{p}_\star + \pib_{h-1}(1-\theta_{h-1} \mymid 0)\underline{q}_\star \right] \notag \\
    & \geq d^{\mathsf{b}, P^\theta}_{h-1}(0) \underline{q}_\star \geq \cdots \geq  d_{1}^{\mathsf{b}, P^\theta}(0) \prod_{j = 0}^{h-1}\underline{q}_\star   \ge d_{1}^{\mathsf{b}, P^\theta}(0) \Big(1-\frac{c_3}{H}\Big)^{H} > \frac{\mu(0)}{2}, \label{eq:finite-optimal-d-large-theta1}
\end{align}
where the last line makes use of the properties $ \underline{q}_\star \geq 1- c_3/H$ in Lemma~\ref{lem:finite-lb-uncertainty-set-KL} and 
\[
\left(1-\frac{c_{3}}{H}\right)^{H}\geq\Big(1-\frac{1}{2H}\Big)^{H} > \frac{1}{2}
\]
provided that $0<c_3 = 1/4 < 1/2$. 
In addition, as state $1$ is an absorbing state and state $0$ will only transfer to itself or state $1$ at each time step, we directly achieve that
\begin{align}
    d^{\mathsf{b}, P^\theta}_h(0) \leq d^{\mathsf{b}, P^\theta}_{h-1}(0) \leq \cdots \leq d^{\mathsf{b}, P^\theta}_1(0)  \leq  \mu(0). \label{eq:finite-optimal-d-large-theta2}
\end{align}

For state $1$, as it is absorbing, we directly have
\begin{align}
     d^{\mathsf{b}, P^\theta}_h(1) &= \mathbb{P} \left\{s_h = 1 \mymid s_{h-1}=1 ; \pib \right\}  \geq d^{\mathsf{b}, P^\theta}_{h-1}(1) \geq \cdots \geq d^{\mathsf{b}, P^\theta}_{1}(1) = \mu(1). \label{eq:finite-optimal-d-large-theta3}
\end{align}
According to the assumption in \eqref{eq:lower-C-assumption-theta}, it is easily verified that
\begin{align}
d^{\mathsf{b}, P^\theta}_h(1) \leq 1 \leq  2\mu(1). \label{eq:finite-optimal-d-large-theta4}
\end{align}

Finally, combining \eqref{eq:finite-optimal-d-large-theta1}, \eqref{eq:finite-optimal-d-large-theta2}, \eqref{eq:finite-optimal-d-large-theta3}, \eqref{eq:finite-optimal-d-large-theta4}, the definitions of $P_h^{\theta_h}(\cdot \mymid s,a)$ in \eqref{eq:Ph-construction-lower-bound-finite-theta} and the Markov property, we arrive at for any $(h,s)\in [H] \times \cS$,
\begin{align}
   \frac{\mu(s)}{2} \leq  d^{\mathsf{b}, P^\theta}_h(s) \leq 2\mu(s),
\end{align}
which directly leads to
\begin{align}
    \frac{\mu(s)}{4} \leq  d^{\mathsf{b}, P^\theta}_h(s,a) = \pib_1(a\mymid s) d^{\mathsf{b}, P^\theta}_h(s) \leq \mu(s).
\end{align}

\paragraph{Proof of the claim~\eqref{eq:expression-Cstar-LB-finite-theta}.}
Examining the definition of $\Cstar$ in \eqref{eq:concentrate-finite}, we make the following observations.
\begin{itemize} 
\item For $h=1$, we have
\begin{align}
    \max_{(s, a,  P) \in \mathcal{S} \times \cA \times \unb^{\ror}(P^{\theta})} \frac{\min\big\{d_1^{\star,P}(s, a ), \frac{1}{S}\big\}}{ d^{\mathsf{b}, P^\theta}_1(s, a )}  &\overset{\mathrm{(i)}}{=} \max_{P \in \unb^{\ror}(P^{\theta})} \frac{\min\big\{d_1^{\star,P}(0, \theta_h), \frac{1}{S}\big\}}{d^{\mathsf{b}, P^\theta}_1(0, \theta_h )}  \overset{\mathrm{(ii)}}{=}  \max_{P \in \unb^{\ror}(P^{\theta})} \frac{1 }{S d^{\mathsf{b}, P^\theta}_1(0, \theta_h )}  \notag \\
    & \overset{\mathrm{(iii)}}{=} \frac{2}{S\mu(0)} = 2C,
\end{align}
where (i) holds by $d_1^{\star,P}(s) = \rho(s) = 0$ for all $s\in \cS\setminus \{0\}$ (see \eqref{eq:rho-defn-finite-LB-theta}) and $\pi_h^{\star,\theta}(\theta_h \mymid 0) = 1$ for all $h\in[H]$, (ii) follows from the fact $d_1^{\star,P}(0, \theta) = 1$, (iii) is verified in \eqref{eq:finite-lb-behavior-distribution-theta}, and the last equality arises from the definition in \eqref{finite-mu-assumption-theta}. 
\item Similarly, for $h=2, 3,\cdots, H$, we arrive at
\begin{align}
\max_{(s, a, P) \in \mathcal{S} \times \cA \times \unb^{\ror}(P^{\theta})} \frac{\min\big\{d_h^{\star,P}(s, a ), \frac{1}{S}\big\}}{ d^{\mathsf{b}, P^\theta}_h(s, a )}  &\overset{\mathrm{(i)}}{=}  \max_{s\in \{0,1\}, P \in \unb^{\ror}(P^{\theta})} \frac{\min\big\{d_h^{\star,P}(s, \theta_h), \frac{1}{S}\big\}}{d^{\mathsf{b}, P^\theta}_h(s, \theta_h )} \notag \\
&\leq \max_{s\in \{0,1\}, P \in \unb^{\ror}(P^{\theta})} \frac{1}{S d^{\mathsf{b}, P^\theta}_h(s, \theta_h )} \overset{\mathrm{(ii)}}{\leq} \frac{4}{S\mu(0)} = 4C,
\end{align}
where (i) holds by the optimal policy in \eqref{eq:finite-lb-value-lemma-theta} and the trivial fact that $d_h^{\star,P}(s) =  0$ for all $s\in \cS\setminus \{0,1 \}$ (see \eqref{eq:rho-defn-finite-LB-theta} and \eqref{eq:Ph-construction-lower-bound-finite-theta}), (ii) arises from \eqref{eq:finite-lb-behavior-distribution-theta}, and the last equality comes from \eqref{finite-mu-assumption-theta}. 
\end{itemize}
Combining the above cases, we complete the proof by
\begin{align*}
  2C \leq  \Cstar = \max_{(h,s, a, P) \in [H] \times \mathcal{S} \times \cA \times \unb^{\ror}(P^{\theta})} \frac{\min\big\{d_h^{\star,P}(s, a ), \frac{1}{S}\big\}}{ d^{\mathsf{b}, P^\theta}_h(s, a )}  \leq 4C.
\end{align*}

\subsubsection{Proof of the claim~\eqref{eq:finite-Value-0-recursive-theta}}\label{proof:eq:finite-Value-0-recursive-theta}

By virtue of \eqref{eq:finite-x-h-theta} and \eqref{eq:finite-lb-value-lemma-theta}, we see that $x_h^{\pi^{\star,\theta},\theta} = \underline{p}_\star$ for all $h\in[H]$, which combined with \eqref{eq:finite-lemma-value-0-pi-theta} gives
\begin{align}
\big\langle \rho, V_h^{\star, \ror, \theta} - V_h^{\pi, \ror, \theta} \big\rangle &= V_h^{\star, \theta}(0) - V_h^{\pi, \theta}(0) \notag\\
&= \underline{p}_\star V_{h+1}^{\star, \ror, \theta}(0)   - x_h^{\pi,\theta} V_{h+1}^{\pi,\theta}(0) \notag\\
&  = x_h^{\pi,\theta} \left(V_{h+1}^{\star,\ror, \theta}(0) - V_{h+1}^{\pi,\ror,\theta}(0)\right)  + (\underline{p}_\star - x_h^{\pi,\theta})V_{h+1}^{\star, \ror, \theta}(0)   \notag \\
& \overset{\mathrm{(i)}}{\geq} \underline{q}_\star \left(V_{h+1}^{\star,\ror, \theta}(0) - V_{h+1}^{\pi,\ror,\theta}(0)\right) + (\underline{p}_\star - x_h^{\pi,\theta}) V_{h+1}^{\star, \ror, \theta}(0) \notag \\
&\overset{\mathrm{(ii)}}{\geq} \underline{q}_\star \left(V_{h+1}^{\star,\ror, \theta}(0) - V_{h+1}^{\pi,\ror,\theta}(0)\right) + \frac{1}{2}(p-q)\big\|\pi_{h}^{\star,\theta}(\cdot\mymid 0)-\pi_{h}(\cdot\mymid 0)\big\|_{1} V_{h+1}^{\star, \ror, \theta}(0) \notag \\
&\overset{\mathrm{(iii)}}{\geq} \underline{q}_\star \left(V_{h+1}^{\star,\ror, \theta}(0) - V_{h+1}^{\pi,\ror,\theta}(0)\right) + \frac{c_2\varepsilon}{3H^2}(H+1-h) \big\|\pi_h^{\star, \theta}(\cdot\mymid 0) - \pi_h(\cdot\mymid 0)\big\|_1 \label{eq:finite-lb-recursion-theta}
% & \overset{\mathrm{(ii)}}{\geq} \underline{q}_\star \left(V_{h+1}^{\star,\ror, \theta}(0) - V_{h+1}^{\pi,\ror,\theta}(0)\right) + (p-q) V_{h+1}^{\star, \ror, \theta}(0) \notag \\
\end{align}
where (i) follows from the fact that $x_h^\pi \geq \underline{q}_\star$ for any $\pi$ and $h\in[H]$, and (iii) holds by the facts \eqref{eq:finite-lb-value-lemma-theta} and the choice \eqref{eq:finite-p-q-def-theta} of $(p,q)$. Here, (ii) arises from
\begin{align}
 \underline{p}_\star-x_{h}^{\pi,\theta}&=( \underline{p}_\star-  \underline{q}_\star)\big(1-\pi_{h}(\theta_{h}\mymid0)\big) \notag \\
& \geq ( p-q)\big(1-\pi_{h}(\theta_{h}\mymid0)\big) \notag \\
 &=\frac{1}{2}(p-q)\big(1-\pi_{h}(\theta_{h}\mymid0)+\pi_{h}(1-\theta_{h}\mymid0)\big)=\frac{1}{2}(p-q)\big\|\pi_{h}^{\star,\theta}(\cdot\mymid 0)-\pi_{h}(\cdot\mymid 0)\big\|_{1}, 
\end{align}
where the first inequality holds by applying Lemma~\ref{lem:finite-lb-uncertainty-set-KL}.
With the fact of \eqref{eq:finite-lb-recursion-theta} in mind, combined with the fact $\underline{q}_\star \geq 1- \frac{c_3}{H}$, following the same proof pipeline of \cite[(276) to (278)]{li2022settling} leads to 
\begin{align}
\big\langle \rho, V_h^{\star, \ror, \theta} - V_h^{\pi, \ror, \theta} \big\rangle  > \varepsilon. 
\end{align}
We omit the proof here for conciseness.

%%%%%------------------%%%%%------------------%%%%%------------------%%%%%------------------

\subsubsection{Proof of \eqref{eq:finite-lb-behavior-distribution}}\label{proof:eq:finite-lb-behavior-distribution}

With the initial state distribution and behavior policy defined in \eqref{eq:lower-dataset-assumption-theta}, we have for any MDP $\mathcal{M}_\phi$ with $\phi \in\{0,1\}$,
\begin{align*}
     d^{\mathsf{b}, P^\phi}_1(s) = \rhob (s) = \mu(s),
\end{align*}
which leads to
\begin{align}
   \forall a\in \cA:\quad  d^{\mathsf{b}, P^\phi}_1(0, a) = \frac{1}{2} \mu(0).
\end{align}
In view of \eqref{eq:Ph-construction-lower-finite-1}, the state occupancy distribution at step $h=2$ obeys
\begin{align*}
    d^{\mathsf{b}, P^\phi}_2(0) &= \mathbb{P}\left\{s_2 = 0 \mymid s_1\sim d^{\mathsf{b}, P^\phi}_1 ; \pib \right\}  = \mu(0)\left[\pib_1(\phi \mymid 0)p + \pib_1(1-\phi \mymid 0)q \right] = \frac{(p+q)\mu(0)}{2},
    \end{align*}
    and
    \begin{align*}
     d^{\mathsf{b}, P^\phi}_2(1) &= \mathbb{P} \left\{s_2 = 1 \mymid s_1\sim  d^{\mathsf{b}, P^\phi}_1 ; \pib \right\} \notag \\
    &= \mu(0)\left[\pib_1(\phi \mymid 0)(1-p) + \pib_1(1-\phi \mymid 0)(1-q) \right] + \mu(1) = \mu(1) + \frac{(2-p-q)\mu(0)}{2}.
\end{align*}
With the above result in mind and recalling the assumption in \eqref{eq:lower-p-q-assumption}, we arrive at
\begin{align}
   \frac{\mu(0)}{2} \leq  d^{\mathsf{b}, P^\phi}_2(0) \leq \mu(0), \qquad \mu(1)\leq  d^{\mathsf{b}, P^\phi}_2(1) \overset{\mathrm{(i)}}{\leq} 2\mu(1),
\end{align}
where (i) holds by applying \eqref{eq:lower-p-q-assumption} and \eqref{eq:lower-C-assumption-theta} (which implies $\mu(0)\leq \mu(1)$ by the assumption in \eqref{eq:lower-C-assumption-theta})
\begin{align*}
    d^{\mathsf{b}, P^\phi}_2(1) = \mu(1) + \frac{(2-p-q)\mu(0)}{2} \leq \mu(1) + \mu(0) \leq 2 \mu(1).
\end{align*} 
Finally, from the definitions of $P_h^\phi(\cdot \mymid s,a)$ in \eqref{eq:Ph-construction-lower-finite-h} and the Markov property, we arrive at for any $(h,s)\in [H] \times \cS$,
\begin{align}
   \frac{\mu(s)}{2} \leq  d^{\mathsf{b}, P^\phi}_h(s) \leq 2\mu(s),
\end{align}
which directly leads to
\begin{align}
    \frac{\mu(s)}{4} \leq  d^{\mathsf{b}, P^\phi}_h(s,a) = \pib_1(a\mymid s) d^{\mathsf{b}, P^\phi}_h(s) \leq \mu(s).
\end{align}

\subsubsection{Proof of Lemma~\ref{lem:finite-lb-uncertainty-set}}\label{proof:lem:finite-lb-uncertainty-set}

Note that $\underline{p} \geq \underline{q}$ can be easily verified since $p>q$, which indicates that the first assertion is true. So we will focus on the second assertion in \eqref{eq:finite-lw-upper-p-q}. Towards this, invoking the definition in \eqref{eq:defn-KL-bernoulli}, let $\ror'$ be the KL divergence from $\mathsf{Ber}\big(\frac{1}{\beta}\big)$ to $\mathsf{Ber}(q)$, defined as follows
\begin{align}\label{eq:def-of-sigma'}
    \ror' & \defn \mathsf{KL}\left(\mathsf{Ber}\left(\frac{1}{\beta}\right) \parallel \mathsf{Ber}(q)\right) = \frac{1}{\beta}\log\frac{\frac{1}{\beta}}{q} + \left(1- \frac{1}{\beta}\right)\log\frac{\left(1- \frac{1}{\beta}\right)}{1-q} \notag \\
     &  = \left(\frac{1}{\beta}\right)\log\left(\frac{1}{\beta}\right) - \left(\frac{1}{\beta}\right) \log(q) + \left(1- \frac{1}{\beta}\right)\log\left(\frac{1}{\alpha+\Delta}\right)   + \left(1- \frac{1}{\beta}\right)\log\left(1- \frac{1}{\beta}\right),
    \end{align}
where the second line uses the definition of $q$ in \eqref{eq:finite-p-q-def}. We claim that $\ror'$ satisfies
the following relation with $\ror$, which will be proven at the end of this proof:
\begin{align}\label{eq:bound_ror_ror'}
0 <\ror \leq \left(1- \frac{2}{\beta}\right)\log\left(\frac{1}{\alpha+\Delta}\right) \leq \ror' \leq \left(1- \frac{1}{\beta}\right)\log\left(\frac{1}{\alpha+\Delta}\right). 
 \end{align}

Recalling the definition of the transition kernel in \eqref{eq:Ph-construction-lower-finite-1}
\begin{align*}
    P_1^\phi(0 \mymid 0, 1-\phi) &= q, \quad  P_1^\phi(1 \mymid 0, 1-\phi) = 1-q, \quad 
     P_1^\phi(s \mymid 0, 1-\phi) = 0, \quad \forall s\in \cS \setminus \{0,1\},
\end{align*}
the uncertainty set of the transition kernel with radius $\sigma$ is thus given as 
\begin{align}
    \unb^\ror(P^\phi_{1,0, 1-\phi}) = \left\{P_{1, 0,1-\phi} \in\Delta(\cS): P(0 \mymid 0, 1-\phi) = q', P(1 \mymid 0, 1-\phi) = 1- q', \mathsf{KL}\left(\mathsf{Ber}\left(q'\right) \parallel \mathsf{Ber}(q)\right) \leq \sigma \right\}.
\end{align}  
Recalling the definition of $\underline{q}$ in \eqref{eq:finite-lw-p-q-perturb-inf}, we  can bound  
\begin{align*}
   \underline{q} &= \inf_{P_{1,0,1-\phi} \in  \unb^\ror(P^\phi_{1,0, 1-\phi})}  P(0 \mymid 0,1-\phi)  =  \inf_{q': \mathsf{KL}\left(\mathsf{Ber}\left(q'\right) \parallel \mathsf{Ber}(q)\right) \leq \sigma} q' \notag \\
   &\overset{\mathrm{(i)}}{\geq} \inf_{q': \mathsf{KL}\left(\mathsf{Ber}\left(q'\right) \parallel \mathsf{Ber}(q)\right) \leq \sigma'} q'  = \frac{1}{\beta},
\end{align*}
where (i) holds by $\sigma \leq \sigma'$ (cf.~\eqref{eq:bound_ror_ror'}) and the last equality follows from applying Lemma~\ref{lem:KL-key-result} (cf.~\eqref{eq:show-subset-of-sigma-sigma'}) and \eqref{eq:def-of-sigma'} to arrive at
\begin{align*}
    \forall 0\leq q'< \frac{1}{\beta}: \qquad \mathsf{KL}\left(\mathsf{Ber}\left(q'\right) \parallel \mathsf{Ber}(q)\right) > \mathsf{KL}\left(\mathsf{Ber}\left(\frac{1}{\beta}\right) \parallel \mathsf{Ber}(q)\right) = \sigma'.
\end{align*}

\paragraph{Proof of \eqref{eq:bound_ror_ror'}.} To control $\sigma'$, we plug in the assumptions in \eqref{eq:lower-p-q-assumption} and $\beta \geq 4$ and arrive at the trivial facts 
\begin{align*}
    \left(\frac{1}{\beta}\right)\log\left(\frac{1}{\beta}\right) - \left(\frac{1}{\beta}\right) \log(q) <0, \quad \left(1- \frac{1}{\beta}\right)\log\left(1- \frac{1}{\beta}\right) <0.
\end{align*}
The above facts directly lead to
\begin{align}\label{eq:ror'-upper}
    \ror' \leq \left(1- \frac{1}{\beta}\right)\log\left(\frac{1}{\alpha+\Delta}\right).
\end{align}
Similarly, observing 
\begin{align*}
    -1 \leq \left(\frac{1}{\beta}\right)\log\left(\frac{1}{\beta}\right) + \left(1- \frac{1}{\beta}\right)\log\left(1- \frac{1}{\beta}\right) & \leq 0, \quad - \left(\frac{1}{\beta}\right) \log(q) \geq 0,
\end{align*}
we arrive at
\begin{align}\label{eq:ror'-lower}
    \ror' \geq -1 + \left(1- \frac{1}{\beta}\right)\log\left(\frac{1}{\alpha+\Delta}\right) \geq \left(1- \frac{2}{\beta}\right)\log\left(\frac{1}{\alpha+\Delta}\right)
\end{align}
as long as 
$ \log\left(\frac{1}{\alpha+ \Delta}\right)\geq \beta$ (cf. \eqref{eq:lower-bound-H-assumption}).
With \eqref{eq:ror'-upper} and \eqref{eq:ror'-lower} in hand, it is straightforward to see that the choice of the uncertainty radius $\ror$ in \eqref{eq:finite-lower-ror-bounded} obeys the advertised bound \eqref{eq:bound_ror_ror'}.

\subsubsection{Proof of Lemma~\ref{lem:finite-lb-value}}\label{proof:lem:finite-lb-value}

For notational conciseness, we shall drop the superscript $\phi$ and use the shorthand $V_{h}^{\pi, \sigma} = V_{h}^{\pi,  \ror,\phi}$ and $V_{h}^{\star,\ror} = V_{h}^{\star, \ror,\phi}$ whenever it is  clear from the context. We begin by deriving the robust value function for any policy $\pi$. 
Starting with state $1$, at any step $h\in[H]$, it obeys
\begin{align*}
    V_h^{\pi,\sigma }(1) &= \mathbb{E}_{a\sim\pi_h(\cdot \mymid 1)}\left[r_h(1,a) + \inf_{ \cP \in \unb^{\sigma}(P^{\phi}_{h,1,a})}  \cP  V_{h+1}^{\pi,\sigma}\right] = 0 + V_{h+1}^{\pi,\sigma}(1),
\end{align*}
where the first equality follows from the robust Bellman consistency equation (cf.~\eqref{eq:robust_bellman_consistency}), and the second equality follows from the observation that the distribution $P^{\phi}_{h,1,a}$ is supported solely on state $1$ in view of \eqref{eq:Ph-construction-lower-finite-1}, therefore $\unb^{\sigma}(P^{\phi}_{h,1,a}) = P^{\phi}_{h,1,a}$. Leveraging the terminal condition $ V_{H+1}^{\pi,\sigma}(1) = 0$, and recursively applying the previous relation, we have
\begin{align}
    V_h^{\star,\sigma}(1) = V_h^{\pi,\sigma}(1) = 0,\qquad \forall h\in[H]. \label{eq:finite-s-1-upper-value}
\end{align}
Similarly, turning to state $0$, at any step $h>1$, the robust value function satisfies 
\begin{align*}
    V_h^{\pi,\sigma}(0) &= \mathbb{E}_{a\sim\pi_h(\cdot \mymid 0)}\left[r_h(0,a) + \inf_{ \cP \in \unb^{\sigma}(P^{\phi}_{h,0,a})}  \cP  V_{h+1}^{\pi,\sigma}\right] = 1 + V_{h+1}^{\pi,\sigma}(0),
\end{align*}
which again uses the fact that the distribution $P^{\phi}_{h,0,a}$ is supported solely on state $0$ in view of \eqref{eq:Ph-construction-lower-finite-h}, therefore $\unb^{\sigma}(P^{\phi}_{h,0,a}) = P^{\phi}_{h,0,a}$.
Leveraging the terminal condition $ V_{H+1}^{\pi,\sigma}(0) = 0$, and recursively applying the previous relation, we have
\begin{align}
    V_h^{\star,\sigma}(0) = V_h^{\pi,\sigma}(0) = H-h+1, \qquad   2 \leq h \leq H. \label{eq:finite-s-0-h-upper-value}
\end{align}

Taking \eqref{eq:finite-s-1-upper-value} and \eqref{eq:finite-s-0-h-upper-value} together, it follows that
\begin{align}\label{eq:o-larger-than-1-for-h-2-H}
   \forall\; 2 \leq h\leq H: \qquad  V_h^{\pi,\sigma}(0) > V_h^{\pi,\sigma}(1).
\end{align}
Consequently, the robust value function of state $0$ at step $h=1$ satisfies
\begin{align}
    V_1^{\pi,\sigma}(0) &= \mathbb{E}_{a\sim\pi_1(\cdot \mymid 0)}\left[r_1(0,a) + \inf_{ \cP \in \unb^{\sigma}(P^{\phi}_{1,0,a})}  \cP V_{2}^{\pi,\sigma}\right] \notag \\
    & \overset{\mathrm{(i)}}{=} 1 + \pi_1(\phi \mymid 0)\Big(  \inf_{ \cP \in \unb^{\sigma}(P^{\phi}_{1,0,\phi})}  \cP   V_{2}^{\pi,\sigma} \Big) 
     + \pi_1(1-\phi \mymid 0)\Big(   \inf_{ \cP \in \unb^{\sigma}(P^{\phi}_{1,0,1-\phi})}  \cP V_{2}^{\pi,\sigma}   \Big) \notag \\
    & \overset{\mathrm{(ii)}}{=} 1 + \pi_1(\phi \mymid 0)\Big[ \underline{p} V_{2}^{\pi,\sigma}(0) + \left(1- \underline{p}\right) V_{2}^{\pi,\sigma}(1)\Big] + \pi_1(1-\phi \mymid 0)\Big[ \underline{q} V_{2}^{\pi,\sigma}(0) + \left(1-\underline{q} \right) V_{2}^{\pi,\sigma}(1)\Big] \notag\\
    & \overset{\mathrm{(iii)}}{=} 1 + V_{2}^{\pi,\sigma}(1) + z_{\phi}^{\pi}  \left[V_{2}^{\pi,\sigma}(0) - V_{2}^{\pi,\sigma}(1) \right] \notag\\
    & = 1 + z_{\phi}^{\pi} V_{2}^{\pi,\sigma}(0)
\end{align}
where (i) uses the definition of the reward function in \eqref{eq:rh-construction-lower-bound-finite}, (ii) uses \eqref{eq:o-larger-than-1-for-h-2-H} so that the infimum is attained by picking the choice specified in \eqref{eq:finite-lw-p-q-perturb-inf} with a smallest probability mass imposed on the transition to state $0$. Finally, we plug in the definition \eqref{eq:finite-x-h} of $z_{\phi}^{\pi}$ in (iii), and the last line follows from \eqref{eq:finite-s-1-upper-value}.

Therefore, taking $\pi = \pi^{\star,\phi}$ in the previous relation directly leads to
\begin{align}\label{eq:interemediate_opt_10}
    V_1^{\star,\sigma}(0) = 1 + z_{\phi}^{\pi^{\star,\phi}} V_{2}^{\star,\sigma}(0) = 1 + z_{\phi}^{\pi^{\star,\phi}} (H-1),
\end{align} 
where the second equality follows from \eqref{eq:finite-s-0-h-upper-value}.
Observing that the function $(H-1)z$ is increasing in $z$ and that $z_{\phi}^{\pi}$ is increasing in $\pi_1(\phi \mymid 0)$ (due to the fact $\underline{p}\geq \underline{q}$ in \eqref{eq:finite-lw-upper-p-q}).
As a result,  the optimal policy obeys
\begin{equation}
    \pi_1^{\star,\phi}(\phi \mymid 0) = 1 \label{eq:finite-lb-optimal-policy}
\end{equation}
 at state $0$, and plugging back to \eqref{eq:interemediate_opt_10} gives
  $$   V_1^{\star,\sigma}(0)   =  1 + z_{\phi}^{\pi^{\star,\phi}} (H-1) = 1+ \underline{p}(H-1), $$
 where $ z_{\phi}^{\pi^{\star,\phi}}= \underline{p}\pi^{\star, \phi}_1(\phi\mymid 0) + \underline{q} \pi^{\star, \phi}_1(1-\phi \mymid 0) = \underline{p}$.
 For the rest of the states, without loss of generality, we choose the optimal policy obeying
\begin{align}
\forall h\in [H]: \qquad  \pi_h^{\star,\phi}(\phi \mymid 0)=1,\quad \pi_h^{\star,\phi}(\phi  \mymid 1) = 1.
\end{align}

\paragraph{Proof of claim \eqref{eq:expression-Cstar-LB-finite}.}
Given that $\pi_h^{\star,\phi}(\phi \mymid 0) =1$ for all $ h\in[H]$ and $\rho(0) =1$, for any $P \in \unb^\ror(P^\phi)$, we have
\begin{align}
    d_2^{\star, P}(0, \phi) &= d_2^{\star, P}(0)\pi_2^{\star,\phi}(\phi \mymid 0) = d_2^{\star, P}(0) = \mathbb{P}_{s_2 \sim P (\cdot \mymid s_1, \pi_1^{\star,\phi}(s_1)  )}\big\{s_2 =0 \mymid s_1 \sim \rho ; \pi^{\star,\phi}\big\} \notag \\
    & = P_1(0 \mymid 0, \phi)  \overset{\mathrm{(i)}}{\geq} \underline{P}^{\phi}_{1}(0 \mymid 0,\phi) \overset{\mathrm{(ii)}}{=} \underline{p} \geq \frac{1}{\beta},
\end{align}
which (i) holds by plugging in the definition \eqref{eq:finite-lw-def-p-q}, (ii) follows from the definition \eqref{eq:finite-lw-p-q-perturb-inf}, and the final inequality arises from Lemma~\ref{lem:finite-lb-uncertainty-set}. Hence, for all $2\leq h\leq H$, by the Markov property and $P^\phi_h(0 \mymid 0, \phi) = 1$,   we have
\begin{align}
    d_h^{\star, P}(0, \phi) = d_2^{\star, P}(0, \phi) \geq \frac{1}{\beta}.
\end{align}
Examining the definition of $\Cstar$ in \eqref{eq:concentrate-finite}, we make the following observations.
\begin{itemize} 
\item For $h=1$, we have
\begin{align}
    \max_{(s, a,  P) \in \mathcal{S} \times \cA \times \unb^{\ror}(P^{\phi})} \frac{\min\big\{d_1^{\star,P}(s, a ), \frac{1}{S}\big\}}{ d^{\mathsf{b}, P^\phi}_1(s, a )}  &\overset{\mathrm{(i)}}{=} \max_{P \in \unb^{\ror}(P^{\phi})} \frac{\min\big\{d_1^{\star,P}(0, \phi), \frac{1}{S}\big\}}{d^{\mathsf{b}, P^\phi}_1(0, \phi )}  \overset{\mathrm{(ii)}}{=}  \max_{P \in \unb^{\ror}(P^{\phi})} \frac{1 }{S d^{\mathsf{b}, P^\phi}_1(0, \phi )}  \notag \\
    & \overset{\mathrm{(iii)}}{=} \frac{2}{S\mu(0)} = 2C,
\end{align}
where (i) holds by $d_1^{\star,P}(s) = \rho(s) = 0$ for all $s\in \cS\setminus \{0\}$ (see \eqref{eq:rho-defn-finite-LB-theta}) and $\pi_h^{\star,\phi}(\phi \mymid 0) = 1$ for all $h\in[H]$, (ii) follows from the fact $d_1^{\star,P}(0, \phi) = 1$, (iii) is verified in \eqref{eq:finite-lb-behavior-distribution}, and the last equality arises from the definition in \eqref{finite-mu-assumption-theta}. 
\item Similarly, for $h=2$, we arrive at
\begin{align}
\max_{(s, a, P) \in \mathcal{S} \times \cA \times \unb^{\ror}(P^{\phi})} \frac{\min\big\{d_2^{\star,P}(s, a ), \frac{1}{S}\big\}}{ d^{\mathsf{b}, P^\phi}_2(s, a )}  &\overset{\mathrm{(i)}}{=}  \max_{s\in \{0,1\}, P \in \unb^{\ror}(P^{\phi})} \frac{\min\big\{d_2^{\star,P}(s, \phi), \frac{1}{S}\big\}}{d^{\mathsf{b}, P^\phi}_2(s, \phi )} \notag \\
&\leq \max_{s\in \{0,1\}, P \in \unb^{\ror}(P^{\phi})} \frac{1}{S d^{\mathsf{b}, P^\phi}_2(s, \phi )} \overset{\mathrm{(ii)}}{\leq} \frac{4}{S\mu(0)} = 4C,
\end{align}
where (i) holds by the optimal policy in \eqref{eq:finite-lb-value-lemma} and the trivial fact that $d_2^{\star,P}(s) =  0$ for all $s\in \cS\setminus \{0,1 \}$ (see \eqref{eq:rho-defn-finite-LB-theta} and \eqref{eq:Ph-construction-lower-finite-1}), (ii) arises from \eqref{eq:finite-lb-behavior-distribution}, and the last equality comes from \eqref{finite-mu-assumption-theta}. 
\item For all other steps $h=3,\ldots, H$, observing from the deterministic transition kernels in \eqref{eq:Ph-construction-lower-finite-h}, it can be easily verified that
\begin{align}
\max_{(s, a, P) \in   \mathcal{S} \times \cA \times \unb^{\ror}(P^{\phi})} \frac{\min\big\{d_h^{\star,P}(s, a ), \frac{1}{S}\big\}}{ d^{\mathsf{b}, P^\phi}_h(s, a )}  &= \max_{(s, a, P) \in \mathcal{S} \times \cA \times \unb^{\ror}(P^{\phi})} \frac{\min\big\{d_2^{\star,P}(s, a ), \frac{1}{S}\big\}}{ d^{\mathsf{b}, P^\phi}_2(s, a )} \leq  4C.
\end{align}
\end{itemize}
Combining the above cases, we complete the proof by
\begin{align*}
  2C \leq  \Cstar = \max_{(h,s, a, P) \in [H] \times \mathcal{S} \times \cA \times \unb^{\ror}(P^{\phi})} \frac{\min\big\{d_h^{\star,P}(s, a ), \frac{1}{S}\big\}}{ d^{\mathsf{b}, P^\phi}_h(s, a )}  \leq 4C.
\end{align*}

\subsubsection{Proof of the claim \eqref{eq:finite-Value-0-recursive}}\label{proof:finite-lower-diff-control}

Recall that by virtue of \eqref{eq:finite-x-h} and \eqref{eq:finite-lb-value-lemma}, we arrive at
\begin{align*}
z_{\phi}^{\star} \defn  z_{\phi}^{\pi^{\star,\phi}}= \underline{p}\pi^{\star, \phi}_1(\phi\mymid 0) + \underline{q} \pi^{\star, \phi}_1(1-\phi \mymid 0) = \underline{p}.
\end{align*}
Applying \eqref{eq:finite-lemma-value-0-pi} yields
\begin{equation} \label{eq:finite-lower-diff}
    \big\langle \rho, V_1^{\star,\sigma, \phi} - V_1^{\pi,\sigma, \phi} \big\rangle = V_h^{\star, \sigma, \phi}(0) - V_h^{\pi,\sigma, \phi}(0) = \left(\underline{p} - z_{\phi}^{\pi}\right) (H-1)    = \left(\underline{p} -\underline{q}\right)(H-1)\left(1-\pi_1(\phi\mymid 0)\right), 
\end{equation}
where the last equality uses the definition \eqref{eq:finite-x-h}. Therefore, it boils down to control $\underline{p} -\underline{q}$.

To continue, we define an auxiliary value function vector $\overline{V}\in\mathbb{R}^{S\times 1}$ obeying
\begin{align}
    \overline{V}(0) = H-1 \quad\text{ and } \quad \overline{V}(s) = 0, \quad \forall s\in \cS \setminus  \{0\}. \label{eq:lower-bound-Vbar-defn}
\end{align}
With this in hand, applying Lemma~\ref{lem:strong-duality} gives
\begin{align}
     (H-1) \left(\underline{p} - \underline{q}\right) &  \overset{\mathrm{(i)}}{=} \inf_{ \cP \in \unb^{\sigma}(P_{1,0,\phi}^{\phi})}  \cP  \overline{V} -  \inf_{ \cP \in \unb^{\sigma}(P_{1,0,1-\phi}^{\phi})}  \cP \overline{V} \notag\\
    & = \sup_{\lambda \geq 0}  \left\{ -\lambda \log\left(P_{1,0,\phi}^{\phi} \cdot \exp \left(\frac{-\overline{V}}{\lambda}\right) \right) - \lambda \ror \right\} - \sup_{\lambda \geq 0}  \left\{ -\lambda \log\left(P_{1,0,1-\phi}^{\phi} \cdot \exp \left(\frac{- \overline{V}}{\lambda}\right) \right) - \lambda \ror \right\}\notag\\
    & \overset{\mathrm{(ii)}}{\geq}  \left\{ -\lambda^\star  \log\left(P_{1,0,\phi}^{\phi} \cdot \exp \left(\frac{-\overline{V}}{\lambda^\star}\right) \right) - \lambda^\star\ror \right\} -  \left\{ -\lambda^\star \log\left(P_{1,0,1-\phi}^{\phi} \cdot \exp \left(\frac{- \overline{V}}{\lambda^\star}\right) \right) - \lambda^\star \ror\right\} \notag \\
    & = -\lambda^\star \left[\log\left(P_{1,0,\phi}^{\phi} \cdot \exp \left(\frac{-\overline{V}}{\lambda^\star}\right) \right) - \log\left(P_{1,0,1-\phi}^{\phi} \cdot \exp \left(\frac{- \overline{V}}{\lambda^\star}\right) \right)\right], \label{eq:staccato}
\end{align}
where (i) follows from (see the definition of $\underline{p}$ in \eqref{eq:finite-lw-p-q-perturb-inf})
\begin{align*}
    \inf_{ \cP \in \unb^{\sigma}(P_{1,0,\phi}^{\phi})}  \cP  \overline{V} &= \underline{P}^{\phi}_{1}(0 \mymid 0,\phi) \overline{V}(0) = (H-1)\underline{p}  , \\
    \inf_{ \cP \in \unb^{\sigma}(P_{1,0, 1- \phi}^{\phi})}  \cP  \overline{V} &= \underline{P}^{\phi}_{1}(0 \mymid 0, 1- \phi)\overline{V}(0) = (H-1)\underline{q}.
\end{align*}
Here, (ii) holds by letting
\begin{align}\label{eq:lower-bound-defn-f-lambda}
    \lambda^\star \defn \arg \max_{\lambda \geq 0} f(\lambda) \defn \arg \max_{\lambda \geq 0}  \left\{ -\lambda \log\left(P_{1,0,1-\phi}^{\phi} \cdot \exp \left(\frac{- \overline{V}}{\lambda}\right) \right) - \lambda \ror \right\}.
\end{align}
The rest of the proof is then to control \eqref{eq:staccato}. We start with the observation that $\lambda^\star>0$; this is because in view of Lemma~\ref{lem:lambda-n-bound} (cf.~\eqref{eq:lambda-star-0-condition}), it suffices to verify that 
\begin{align}
    \log(1-q) + \ror \overset{\mathrm{(i)}}{\leq}  \log (\alpha+ \Delta) + \left(1- \frac{2}{\beta}\right) \log\left(\frac{1}{\alpha+ \Delta}\right) =  - \frac{2}{\beta} \log\left(\frac{1}{\alpha+ \Delta}\right)  <0,
\end{align}
where (i) holds by \eqref{eq:finite-lower-ror-bounded}. We now claim the following bound for $\lambda^{\star}$ holds, whose proof is postponed to the end:
\begin{align}
 \frac{H}{16\ror} \leq  \frac{H-1}{\log\left(\frac{\beta}{\alpha+ \Delta}\right) }  \leq \lambda^\star  &\leq  \frac{H-1}{\left(1-\frac{3}{\beta}\right) \log\left(\frac{1}{\alpha+ \Delta}\right) }, \label{eq:finite-lower-check-lambda-star-1} 
\end{align}
which immediately implies the following by taking exponential maps given $\lambda^{\star}>0$:
\begin{align}
      \frac{\alpha+ \Delta}{\beta} &\leq e^{-(H-1) /\lambda^\star} \leq  (\alpha+ \Delta)^{1-3/\beta}. \label{eq:finite-lower-check-lambda-star-2}
      \end{align}
Moving to the second term of  \eqref{eq:staccato}, it follows that
\begin{align}
\log\left(P_{1,0,\phi}^{\phi} \cdot \exp \left(\frac{-\overline{V}}{\lambda^\star}\right) \right) - \log\left(P_{1,0,1-\phi}^{\phi} \cdot \exp \left(\frac{- \overline{V}}{\lambda^\star}\right) \right)   
    & \overset{\mathrm{(i)}}{=} \log \frac{p e^{- (H-1)/\lambda^\star} + (1-p)}{q e^{-(H-1)/\lambda^\star} + (1-q)} \notag \\
    & =   \log \left(1 + \frac{(p-q) \left(e^{-(H-1)/\lambda^\star} -1 \right)}{qe^{-(H-1)/\lambda^\star} + (1-q)} \right) \notag \\
    & \overset{\mathrm{(ii)}}{<}  - \frac{\Delta  \left(1- e^{-(H-1)/\lambda^\star} \right)}{q e^{-(H-1)/\lambda^\star} + (1-q)} \notag \\ 
    & \overset{\mathrm{(iii)}}{\leq} - \frac{1}{2} \frac{\Delta}{ \left(\frac{1}{\alpha+ \Delta}\right)^{\frac{3}{\beta}}(1-q) + (1-q)} \notag \\
    & \leq - \frac{ \Delta}{4 e^6 (1-q)},
    \label{eq:finite-lower-V-diff-bounded}
\end{align}
where (i) follows from the definitions in \eqref{eq:Ph-construction-lower-finite} and \eqref{eq:lower-bound-Vbar-defn}, 
(ii) holds by $\log(1+x)<x$ for $x\in(-1,\infty)$, (iii) can be verified by \eqref{eq:finite-lower-check-lambda-star-2}, $\beta\geq 4$, and \eqref{eq:lower-p-q-beta-c1}:
\begin{align*}
   1- e^{-(H-1)/\lambda^\star} \geq 1-(\alpha+ \Delta)^{1-3/\beta} \geq  1 - (\alpha+ \Delta)^{1/4} \geq 1- \left(\frac{3}{2H}\right)^{1/4} \geq \frac{1}{2},
\end{align*}
and the last line uses $\left(\frac{1}{\alpha+ \Delta}\right)^{\frac{3}{\beta}} = \left(\frac{1}{\alpha+ \Delta}\right)^{6/\log\left(\frac{1}{\alpha+ \Delta}\right) } = e^6$ by the definition of $\beta$ in \eqref{eq:lower-bound-H-assumption}.
Plugging \eqref{eq:finite-lower-check-lambda-star-1} and \eqref{eq:finite-lower-V-diff-bounded} back into \eqref{eq:staccato} and \eqref{eq:finite-lower-diff}, we arrive at
\begin{align*}
\big\langle \rho, V_1^{\star, \ror, \phi} - V_1^{\pi, \ror, \phi} \big\rangle & =  (H-1) \left(\underline{p} - \underline{q}\right) \left(1-\pi_1(\phi\mymid 0)\right)   \\
& \overset{\mathrm{(i)}}{\geq}    \frac{H\Delta}{64   e^6   \ror (1-q) }   \left(1-\pi_1(\phi\mymid 0)\right) = 2\varepsilon \left(1-\pi_1(\phi\mymid 0)\right),
\end{align*}
where (i) holds by \eqref{eq:finite-lower-check-lambda-star-1} and the last equality follows directly from the choice of $\Delta$ in \eqref{eq:Delta-chosen}.

\paragraph{Proof of inequality \eqref{eq:finite-lower-check-lambda-star-1}.}

Applying \eqref{eq:lambda-n-range} in Lemma~\ref
{lem:lambda-n-bound} to $\lambda^\star$ in \eqref{eq:lower-bound-defn-f-lambda} leads to the upper bound in \eqref{eq:finite-lower-check-lambda-star-1}:
\begin{align}
    \lambda^\star \leq \frac{H-1}{\ror} \leq \frac{H-1}{\left(1-\frac{3}{\beta}\right) \log\left(\frac{1}{\alpha+ \Delta}\right) }, 
\end{align}
where the last inequality holds by \eqref{eq:finite-lower-ror-bounded}. 
As a result, we shall focus on showing the lower bounds in \eqref{eq:finite-lower-check-lambda-star-1} in the remainder of the proof.

Recalling the definition of $q$ in \eqref{eq:finite-p-q-def}, we can reparameterize $1-q$ using two positive variables $c_q$ and $\lambda_q$ (whose choices will be made clearer soon) as follows:
\begin{align}
1- q = \alpha = c_q e^{-(H-1)/\lambda_q}. \label{eq:assumption-of-q-0}
\end{align}
Deriving the first derivative of the function of interest $f(\lambda)$ in \eqref{eq:lower-bound-defn-f-lambda} as follows:
\begin{align}
    \nabla_\lambda f(\lambda) &= \nabla_\lambda\left( -\lambda \log\left(P_{1,0,1-\phi}^{\phi} \cdot \exp \left(\frac{- \overline{V}}{\lambda}\right) \right) - \lambda \ror \right) \notag\\
    & \overset{\mathrm{(i)}}{=} \nabla_\lambda \left( -\lambda \log\left(q e^{-(H-1)/\lambda} + 1-q\right) - \lambda \ror \right) \notag\\
    & = -\ror -\log \left(q e^{-(H-1)/\lambda} + 1-q\right) - \frac{1}{\lambda}\cdot \frac{q (H-1)e^{-(H-1)/\lambda} }{q e^{-(H-1)/\lambda} + 1-q},
    \end{align}
where (i) holds by the chosen transition kernels in \eqref{eq:Ph-construction-lower-finite} and the last line arises from basic calculus.
To continue, when $\lambda = \lambda_q$, the derivative of the function $f(\lambda)$ can be expressed as 
\begin{align}
     \nabla_\lambda f(\lambda) \mymid_{\lambda =\lambda_q} &= -\ror - \log\left( (1-q)\frac{q}{c_q}  + 1 - q\right) + \frac{(1-q)\frac{q}{c_q} \log\frac{1-q}{c_q} }{(1-q)\frac{q}{c_q}   + 1-q} \notag \\
    & = -\ror - \log(1-q) - \log\left( 1 + \frac{q}{c_q} \right) + \frac{\frac{q}{c_q} \log\frac{1-q}{c_q} }{\frac{q}{c_q}   + 1} \notag \\
    & = -\ror - \log(1-q)\left( 1 -\frac{q/c_q}{q/c_q + 1}\right) - \log\left( 1 + \frac{q}{c_q} \right) - \frac{ \frac{q}{c_q}\log(c_q)}{1 + q/c_q }  \notag \\
    & \overset{\mathrm{(i)}}{=} -\ror + \log\left(\frac{1}{\alpha+ \Delta}\right) \left( 1 -\frac{q/c_q}{q/c_q + 1}\right) - \log\left( 1 + \frac{q}{c_q} \right) - \frac{ \frac{q}{c_q}\log(c_q)}{1 + q/c_q }  \\
    & \overset{\mathrm{(ii)}}{\geq} \log\left(\frac{1}{\alpha+ \Delta}\right) \left( \frac{2}{\beta} -\frac{q/c_q}{q/c_q + 1}\right)  - \log\left( 1 + \frac{q}{c_q} \right) - \frac{ \frac{q}{c_q}\log(c_q)}{1 + q/c_q } \notag \\
    & \overset{\mathrm{(iii)}}{\geq} \frac{1}{\beta} \log\left(\frac{1}{\alpha+ \Delta}\right)   - \log(1+\frac{1}{\beta}) - 1 \notag \\
    &  \geq \frac{1}{\beta} \log\left(\frac{1}{\alpha+ \Delta}\right)   - 2 = 0,
\end{align}
where (i) holds by \eqref{eq:assumption-of-q-0}, (ii) follows from the bound of $\ror$ in \eqref{eq:finite-lower-ror-bounded}, (iii) arises from letting $c_q = \beta \geq 4$ and noting the fact $1/2\leq q < 1$ (see \eqref{eq:lower-p-q-assumption}), leading to
\begin{align}
    \frac{}{}\frac{1}{2\beta} \leq \frac{q}{c_q}< \frac{1}{\beta}, \qquad  \frac{q/c_q}{q/c_q + 1} \leq \frac{1}{\beta}, \qquad \frac{ \frac{q}{c_q}\log(c_q)}{1 + q/c_q }<1.
\end{align} 
Finally,
the last line holds by $1/\beta \leq \frac{1}{4}$ and $\log\left(\frac{1}{\alpha+\Delta}\right)   = 2\beta$ (see \eqref{eq:lower-bound-H-assumption}).

To proceed, note that the function $f(\lambda)$ is concave with respect to $\lambda$. Therefore, observing $\nabla_\lambda f(\lambda) \mymid_{\lambda =\lambda_q} \geq 0$ with $c_q =\beta$, we have $\lambda_q \leq \lambda^\star$, which implies (see \eqref{eq:assumption-of-q-0})
\begin{align}
    1- q= \alpha+ \Delta =  \beta e^{-(H-1)/\lambda_q} \leq \beta e^{-(H-1) /\lambda^\star}.
\end{align}
The above assertion directly gives
\begin{align*}
    \lambda^\star \geq \frac{H-1}{\log\left(\frac{\beta}{\alpha+ \Delta}\right) }.
\end{align*}
The proof is completed by noticing
\begin{align*}
 \frac{H-1}{\log\left(\frac{\beta}{\alpha+ \Delta}\right) } =\frac{H-1}{\log\left(\frac{1}{\alpha+ \Delta}\right) + \log \beta } \overset{\mathrm{(i)}}{\geq} \frac{H-1}{2\log\left(\frac{1}{\alpha+ \Delta}\right)} \geq \frac{H}{16\ror},
\end{align*}
where (i) follows from \eqref{eq:lower-bound-H-assumption}, and the last inequality follows from \eqref{eq:finite-lower-ror-bounded} and the fact $\beta \in [4,\infty)$.

%% file: upper-bound-analysis_infty.tex
%!TEX root = ./DRO-offline.tex
\section{Analysis: discounted infinite-horizon RMDPs}\label{sec:analysis-infinite}

\input{auxiliary-upper-bound-infty.tex}

\subsection{Proof of Theorem~\ref{thm:dro-upper-infinite}}
\label{proof:thm:dro-upper-infinite}

 	 To begin, we introduce some additional notation that will be useful throughout the analysis. We denote the state-action space covered by the batch dataset $\cD$ as  
\begin{align}\label{eq:cover-space-pib-infinite}
	\cC^{\mathsf{b}} = \left\{(s,a): \myrho(s, a ) >0 \right\}. 
\end{align}
In addition, recalling the definition in \eqref{eq:P-min-hat-def-infinite}, we define a similar one based on the true nominal model $P^{\no}$ as
\begin{align}\label{eq:P-min-pib-infinite} 
	P_{\mathsf{min}}(s,a) \defn  \min_{s'} \Big\{P^{\no}(s' \mymid s,a): \; P^{\no}(s' \mymid s,a)>0 \Big\},
\end{align}
which directly indicates that
\begin{equation}\label{eq:link_minpall_pmin-infinite}
\minpall = \min_{s}\; P_{\mathsf{min}}(s,\pi^{\star}(s)) , \qquad P_{\mathsf{min}}^{\mathsf{b}} = \min_{(s,a)\in \cC^{\mathsf{b}} }\;  P_{\mathsf{min}}(s,a). 
\end{equation}
Next, we denote the set of possible state occupancy distributions associated with the optimal policy $\pi^\star$ in a model within the uncertainty set $P\in \unb^{\ror} \left(P^{\no} \right)$ as
\begin{align} \label{eq:def-D-star-infinite}
	\cD^\star \defn \left\{ \left[d^{\star,P}(s)\right]_{s\in\cS} : P \in \unb^{\ror} \left(P^{\no} \right) \right\} = \left\{ \left[d^{\star,P}\big(s,\pi^\star(s) \big)\right]_{s\in\cS} : P \in \unb^{\ror} \left(P^{\no} \right)\right\},
\end{align}
where the second equality is due to the fact that $\pi^\star$ is chosen to be deterministic.

We are now ready to embark on the proof of Theorem~\ref{thm:dro-upper-infinite}. We first introduce a fact that is used throughout the proof; the proof is postponed to Appendix~\ref{sec:proof-eq:fact-of-N-b-assumption-infinite}:
\begin{align}\label{eq:fact-of-N-b-assumption-infinite}
\forall (s,a) \in \cC^{\mathsf{b}} : \qquad	N(s,a) \geq \frac{N \myrho(s,a)}{12} \geq \frac{ c_1\log( NS/ \delta )  }{12 P_{\mathsf{min}}(s,a) } \geq  - \frac{\log\frac{2NS}{\delta}}{\log (1- P_{\mathsf{min}}(s,a) )}
\end{align}
as long as \eqref{eq:dro-b-bound-N-condition-infinite} holds. 

For notation simplicity, denote the output Q-function and value function from Algorithm~\ref{alg:vi-lcb-dro-infinite} as $\widehat{Q} = \widehat{Q}_M$ and $\widehat{V} = \widehat{V}_M$.
Invoking Lemma~\ref{lem:infinite-converge} with $M \geq \frac{\log \frac{ \ror N}{1-\gamma}}{\log \frac{1}{\gamma}}$ directly leads to 
\begin{align}\label{eq:infinite-Q-N}
	\big \| \widehat{Q} - \widehat{Q}^{\star,\ror}_{\mathsf{pe}} \big\|_\infty \leq \frac{1}{ \ror N}
\end{align}
and therefore 
\begin{align}\label{eq:infinite-V-N}
	\big\| \widehat{V} - \widehat{V}^{\star,\ror}_{\mathsf{pe}} \big\|_\infty \leq \max_s \left| \max_a \widehat{Q}(s,a) - \max_a \widehat{Q}^{\star,\ror}_{\mathsf{pe}}(s,a)\right| \leq \big\| \widehat{Q} - \widehat{Q}^{\star,\ror}_{\mathsf{pe}}\big\|_\infty \leq \frac{1}{\ror N}.
\end{align}

The proof of Theorem~\ref{thm:dro-upper-infinite} is separated into several key steps as follows.

\paragraph{Step 1: controlling the uncertainty via leave-one-out analysis.} Given access to only a finite number of samples for estimating the nominal transition kernel $P^\no$, we need to efficiently control 
$$\left| \inf_{ \cP \in \unb^{\sigma}(\widehat{P}_{s,a}^0)} \cP \widehat{V} - \inf_{ \cP \in \unb^{\sigma}(P^{\no}_{s,a})}  \cP \widehat{V} \right|$$ 
across the robust value iterations, where $\widehat{V}$ is statistically dependent on $\widehat{P}_{s,a}^0$ (since $\widehat{P}_{s,a}^0$ will be reused in the update rule (cf.~\eqref{eq:pessimism-operator-equal}) for all the iterations). A naive treatment via the standard covering arguments will unfortunately lead to rather loose bounds \citep{zhou2021finite,panaganti2021sample,yang2021towards}. To overcome this challenge, we resort to the leave-one-out analysis---pioneered by \citet{agarwal2019optimality,li2020breaking,li2022settling} in the context of model-based RL---to decouple the statistical dependency. The results are summarized in the following lemma, with the proof provided in Appendix~\ref{proof:lemma:dro-b-bound-infinite}.

\begin{lemma}\label{lemma:dro-b-bound-infinite}
Instate the assumptions in Theorem~\ref{thm:dro-upper-infinite}. 
Then for all vector $\widetilde{V}$ obeying $\big\|\widetilde{V} - \widehat{V}^{\star, \ror}_{\mathsf{pe}} \big\|_\infty \le \frac{1}{\sigma N}$ and $\|\widetilde{V}\|_{\infty} \le \frac{1}{ 1-\gamma}$, with probability at least $1- \delta$, one has 
\begin{align}\label{eq:dro-b-bound-infinite}
	&\left| \inf_{ \cP \in \unb^{\sigma}(\widehat{P}_{s,a}^0)} \cP \widetilde{V} - \inf_{ \cP \in \unb^{\sigma}(P^{\no}_{s,a})}  \cP \widetilde{V} \right| \leq \min \left\{\frac{\cb}{\ror(1-\gamma)} \sqrt{\frac{\log(\frac{2(1+\ror)N^3S}{(1-\gamma)\delta})}{ \widehat{P}_{\mathsf{min}}(s,a) N(s,a)}} + \frac{4}{N \ror (1-\gamma)}\;, \; \frac{1}{1-\gamma}\right\}
\end{align}
for all  $(s,a)\in \cS\times \cA$. In addition, for all $(s,a)\in \cC^{\mathsf{b}} $, with probability at least $1-\delta$, one has
\begin{align}
  	   \frac{P_{\mathsf{min}}(s,a)}{8  \log(NS / \delta)} \leq   \widehat{P}_{\mathsf{min}}(s,a)  \leq e^2 P_{\mathsf{min}}(s,a) . \label{eq:convert-pmin-to-estimation-infinite}
\end{align}
 
\end{lemma}

\paragraph{Step 2: establishing the pessimism property.}
Armed with Lemma~\ref{lemma:dro-b-bound-infinite}, we aim to show the key property that 
\begin{align}\label{eq:infinite-pessimism-assertion}
	&\forall (s,a) \in \cS\times \cA :
	\qquad \widehat{Q}(s,a) \leq Q^{\widehat{\pi},\ror}(s,a),\qquad \widehat{V}(s) \leq V^{\widehat{\pi},\ror}(s).
\end{align}
Similar to the finite-horizon setting, it suffices to focus on verifying the former assertion in \eqref{eq:infinite-pessimism-assertion}. Towards this, we first recall that the fixed point $\widehat{Q}^{\star,\ror}_{\mathsf{pe}}$ of the pessimistic robust Bellman operator $\tpe(\cdot)$ (cf.~\eqref{eq:pessimism-operator}) obeys
\begin{align}\label{eq:infinite-fixed-Q-recall}
	\widehat{Q}^{\star,\ror}_{\mathsf{pe}} = \tpe(\widehat{Q}^{\star,\ror}_{\mathsf{pe}}) = \max\left\{ r(s, a)  + \gamma \inf_{ \cP \in \unb^{\sigma}(\widehat{P}^{\no}_{s,a})} \cP \widehat{V}^{\star,\ror}_{\mathsf{pe}} - b\big(s,a\big), \,  0 \right\}.
\end{align}
If $\widehat{Q}^{\star,\ror}_{\mathsf{pe}}(s,a) = 0$. Given the initialization $\widehat{Q}_0 = 0$, invoking Lemma~\ref{lem:infinite-converge} gives
		\begin{align*}
			\widehat{Q}(s,a) =\widehat{Q}_M(s,a) \leq \widehat{Q}^{\star,\ror}_{\mathsf{pe}}(s,a) = 0.
		\end{align*}
		As a result, $Q^{\widehat{\pi},\ror}(s,a) \geq 0 = \widehat{Q}(s,a)$ as desired. Therefore, it boils down to examine the case when $\widehat{Q}^{\star,\ror}_{\mathsf{pe}}(s,a) > 0$. One has
		\begin{align}
			\widehat{Q}(s,a) &\overset{\mathrm{(i)}}{\leq} \widehat{Q}^{\star,\ror}_{\mathsf{pe}}(s,a) + \frac{1}{\ror N} = r(s, a)  + \gamma \inf_{ \cP \in \unb^{\sigma}(\widehat{P}^{\no}_{s,a})} \cP \widehat{V}^{\star,\ror}_{\mathsf{pe}} - b\big(s,a\big) + \frac{1}{\ror N} \nonumber\\
			& \leq r(s, a) + \gamma \inf_{ \cP \in \unb^{\sigma}(\widehat{P}^{\no}_{s,a})} \cP \widehat{V} - b(s,a) + \frac{1}{\ror N} + \gamma \left| \inf_{ \cP \in \unb^{\sigma}(\widehat{P}^{\no}_{s,a})} \cP \widehat{V}^{\star,\ror}_{\mathsf{pe}} - \inf_{ \cP \in \unb^{\sigma}(\widehat{P}^{\no}_{s,a})} \cP \widehat{V} \right| \nonumber\\
			& \overset{\mathrm{(ii)}}{\leq} r(s, a) + \gamma \inf_{ \cP \in \unb^{\sigma}(\widehat{P}^{\no}_{s,a})} \cP \widehat{V} - b(s,a) + \frac{2}{ \ror N} \nonumber \\
			&\leq r(s, a) + \gamma \inf_{ \cP \in \unb^{\sigma}(P^{\no}_{s,a})} \cP \widehat{V} - b(s,a) + \frac{2}{ \ror N} + \gamma \left| \inf_{ \cP \in \unb^{\sigma}(\widehat{P}^{\no}_{s,a})} \cP  \widehat{V} - \inf_{ \cP \in \unb^{\sigma}(P^{\no}_{s,a})} \cP \widehat{V} \right| \nonumber \\
			& \leq r(s, a) + \gamma \inf_{ \cP \in \unb^{\sigma}(P^{\no}_{s,a})} \cP \widehat{V},
		\end{align}
		where (i) follows from \eqref{eq:infinite-Q-N}, (ii) arises from \eqref{eq:infinite-V-N} and the basic fact that infimum operator is $1$-contraction w.r.t $\|\cdot \|_\infty$, and the last inequality holds by the definition of $b(s,a)$ (cf.~\eqref{def:bonus-dro-infinite}) and Lemma~\ref{lemma:dro-b-bound-infinite}. Putting the above inequality together with the robust Bellman equation (cf.~\eqref{eq:bellman-equ-pi-infinite}) pertaining to $Q^{\widehat{\pi}, \ror}(s,a)$, we arrive at
		\begin{align}
			Q^{\widehat{\pi},\ror}(s,a) - \widehat{Q}(s,a) & \geq r(s,a) + \gamma \inf_{ \cP \in \unb^{\sigma}(P^{\no}_{s,a})} \cP V^{\widehat{\pi},\ror} - \left(  r(s, a) + \gamma \inf_{ \cP \in \unb^{\sigma}(P^{\no}_{s,a})} \cP \widehat{V} \right) \nonumber \\
			& =  \gamma \left( \inf_{ \cP \in \unb^{\sigma}(P^{\no}_{s,a})} \cP V^{\widehat{\pi},\ror} - \inf_{ \cP \in \unb^{\sigma}(P^{\no}_{s,a})} \cP \widehat{V}\right) \nonumber \\
			& \overset{\mathrm{(i)}}{=} \gamma \left( \widetilde{P}_{s,a} V^{\widehat{\pi},\ror} - \inf_{ \cP \in \unb^{\sigma}(P^{\no}_{s,a})} \cP \widehat{V}\right) \geq \gamma \widetilde{P}_{s,a} \left( V^{\widehat{\pi},\ror} - \widehat{V}\right), \nonumber
		\end{align}
		where (i) holds by setting $\widetilde{P}_{s,a} = \mathrm{argmin}_{\cP \in \unb^{\sigma}(P^{\no}_{s,a})} \cP V^{\widehat{\pi},\ror}$.
				Consequently, one has
		\begin{align}
			\min_{s,a} \left[ Q^{\widehat{\pi},\ror}(s,a) - \widehat{Q}(s,a)\right] &\geq \min_{s,a} \left[\gamma \widetilde{P}_{s,a} \left( V^{\widehat{\pi},\ror} - \widehat{V}\right)\right] \overset{\mathrm{(i)}}{\geq} \gamma \min_s \left[V^{\widehat{\pi},\ror}(s) - \widehat{V}(s) \right] \nonumber \\
			& = \gamma \min_s \left[Q^{\widehat{\pi},\ror}\big(s, \widehat{\pi}(s)\big) - \widehat{Q}\big(s, \widehat{\pi}(s)\big) \right] \nonumber \\
			& \geq \gamma \min_{s,a} \left[Q^{\widehat{\pi},\ror}\big(s, a\big) - \widehat{Q}\big(s, a\big) \right] ,\label{eq:gamma-pessimism-relation}
		\end{align}
		where (i) follows from $\widetilde{P}_{s,a} \in \Delta(\cS)$ for all $(s,a)\in\cS\times\cA$. Noting that $0 \leq \gamma <1$, we conclude $Q^{\widehat{\pi},\ror}(s,a) - \widehat{Q}(s,a) \geq 0$ for all $(s,a) \in \cS\times \cA$. This establishes the claim \eqref{eq:infinite-pessimism-assertion}.

\paragraph{Step 3: bounding $V^{\star, \ror}(\rho) - V^{\widehat{\pi}, \ror}(\rho)$.} In view of the pessimistic property (cf.~\eqref{eq:infinite-pessimism-assertion}), it follows that
\begin{align}\label{eq:infinite-pessimism-gap-transfer}
	V^{\star, \ror}(s) - V^{\widehat{\pi}, \ror}(s) \leq V^{\star, \ror}(s) - \widehat{V}(s).
\end{align}
Towards this, note that
\begin{align}
	\widehat{V}(s) &= \max_a \widehat{Q}(s,a) \geq \widehat{Q} \big(s, \pi^\star(s)\big) \overset{\mathrm{(i)}}{\geq} \widehat{Q}^{\star,\ror}_{\mathsf{pe}}\big(s, \pi^\star(s)\big) - \frac{1}{\ror N} \nonumber \\
	& \overset{\mathrm{(ii)}}{\geq} r \big(s, \pi^\star(s)\big)  + \gamma \inf_{ \cP \in \unb^{\sigma}\left(\widehat{P}^{\no}_{s, \pi^\star(s)}\right)} \cP \widehat{V}^{\star,\ror}_{\mathsf{pe}} - b\big(s, \pi^\star(s)\big) - \frac{1}{\ror N}  \nonumber \\
	& \geq r \big(s, \pi^\star(s)\big)  + \gamma \inf_{ \cP \in \unb^{\sigma}\left(\widehat{P}^{\no}_{s, \pi^\star(s)}\right)} \cP \widehat{V}  - b\big(s, \pi^\star(s)\big) - \frac{1}{\ror N}   - \gamma \left| \inf_{ \cP \in \unb^{\sigma}\left(\widehat{P}^{\no}_{s, \pi^\star(s)}\right)} \cP \widehat{V}^{\star,\ror}_{\mathsf{pe}} - \inf_{ \cP \in \unb^{\sigma}\left(\widehat{P}^{\no}_{s, \pi^\star(s)}\right)} \cP \widehat{V} \right|  \nonumber \\
	& \overset{\mathrm{(iii)}}{\geq} r \big(s, \pi^\star(s)\big)  + \gamma \inf_{ \cP \in \unb^{\sigma}\left(\widehat{P}^{\no}_{s, \pi^\star(s)}\right)} \cP \widehat{V}  - b\big(s, \pi^\star(s)\big) - \frac{2}{\ror N}  \nonumber \\
	& \geq r \big(s, \pi^\star(s)\big)  + \gamma \inf_{ \cP \in \unb^{\sigma}\left(P^{\no}_{s, \pi^\star(s)}\right)} \cP \widehat{V}  - b\big(s, \pi^\star(s)\big) - \frac{2}{\ror N}   - \gamma \left| \inf_{ \cP \in \unb^{\sigma}\left(\widehat{P}^{\no}_{s, \pi^\star(s)}\right)} \cP  \widehat{V} - \inf_{ \cP \in \unb^{\sigma}\left(P^{\no}_{s, \pi^\star(s)}\right)} \cP \widehat{V} \right|  \nonumber \\
	& \geq r \big(s, \pi^\star(s)\big)  + \gamma \inf_{ \cP \in \unb^{\sigma}\left(P^{\no}_{s, \pi^\star(s)}\right)} \cP \widehat{V}  - 2b\big(s, \pi^\star(s)\big), \label{eq:infinite-V-hat-upper-bound}
\end{align}
where (i) follows from \eqref{eq:infinite-Q-N}, (ii) holds by applying \eqref{eq:infinite-fixed-Q-recall}, (iii) arises from \eqref{eq:infinite-V-N}, and the basic fact that the infimum operator is a $1$-contraction w.r.t. $\|\cdot \|_\infty$, and the final inequality holds by the definition of $b(s,a)$ (see \eqref{def:bonus-dro-infinite}) and Lemma~\ref{lemma:dro-b-bound-infinite}.

To continue, invoking the robust Bellman optimality equation in \eqref{eq:bellman-equ-star-infinite} gives
\begin{align*}
	V^{\star,\ror}(s) =  Q^{\star, \ror}\big(s, \pi^\star(s)\big)= r\big(s, \pi^\star(s)\big) + \gamma \inf_{ \cP \in \unb^{\sigma}\left(P^{\no}_{s,\pi^\star(s)}\right)} \cP V^{\star,\ror}.
\end{align*}
Combining the above relation with \eqref{eq:infinite-V-hat-upper-bound}, we arrive at
\begin{align}
V^{\star,\ror}(s) - \widehat{V}(s)   & \leq \gamma \inf_{ \cP \in \unb^{\sigma}\left(P^{\no}_{s,\pi^\star(s)}\right)} \cP V^{\star,\ror} -  \gamma \inf_{ \cP \in \unb^{\sigma}\left(P^{\no}_{s, \pi^\star(s)}\right)} \cP \widehat{V}  + 2b\big(s, \pi^\star(s)\big)  \nonumber \\
	 & \leq \gamma\widehat{P}^{\inf}_{s,\pi^\star(s)} \left( V^{\star,\ror} -  \widehat{V}\right) + 2b\big(s, \pi^\star(s)\big) ,\label{eq:infinite-recursion-basic}
\end{align}
where the final inequality holds evidently, by introducing 
\begin{align}\label{eq:infinite-inf-hat}
	\widehat{P}^{\inf}_{s,\pi^\star(s)} \defn \mathrm{argmin}_{\cP \in \unb^{\ror} \big(P^0_{s,\pi^\star(s)} \big)} \;  \cP \widehat{V}.
\end{align}

Before continuing, for convenience, let us introduce a matrix $\widehat{P}^{\inf} \in\mathbb{R}^{S\times\cS}$ and a vector $b^\star \in \mathbb{R}^\cS$, where their $s$-th rows (resp.~entries) are defined as
\begin{align}\label{eq:infinite-matrix-notation}
	\left[\widehat{P}^{\inf} \right]_{s,\cdot} = \widehat{P}^{\inf}_{s,\pi^\star(s)}, \qquad \mbox{and} \qquad  b^\star(s) = b\big(s, \pi^\star(s) \big).
\end{align}
With these notation in hand, averaging \eqref{eq:infinite-recursion-basic} over the initial state distribution $\rho$ leads to 
\begin{align}
	V^{\star, \ror}(\rho) - \widehat{V}(\rho) &= \sum_{s\in\cS}\rho(s) \left( V^{\star,\ror}(s)  - \widehat{V}(s)\right) \nonumber \\
	&\leq  \gamma\sum_{s\in\cS}\rho(s) \widehat{P}^{\inf}_{s,\pi^\star(s)} \left( V^{\star, \ror} -  \widehat{V}\right) + 2 \sum_{s\in\cS}\rho(s) b\big(s, \pi^\star(s)\big) \nonumber \\
	& = \gamma \rho^\top \widehat{P}^{\inf} \left( V^{\star, \ror} -  \widehat{V}\right) + 2 \rho^\top b^\star.
\end{align}
Applying the above result recursively gives
\begin{align}
	V^{\star, \ror}(\rho) - \widehat{V}(\rho)  &\leq \gamma \rho^\top \widehat{P}^{\inf} \left( V^{\star, \ror} -  \widehat{V}\right) + 2 \rho^\top b^\star \nonumber \\
	&\leq \gamma \left(\gamma \rho^\top \widehat{P}^{\inf}\right)  \widehat{P}^{\inf}\left( V^{\star, \ror} -  \widehat{V}\right) + 2\left(\gamma \rho^\top \widehat{P}^{\inf}\right)  b^\star + 2 \rho^\top b^\star  \nonumber \\
	& \leq \cdots \leq \left\{\lim_{i\rightarrow \infty} \gamma^i \rho^\top \left(\widehat{P}^{\inf}\right)^i \left( V^{\star, \ror} -  \widehat{V}\right) \right\} + 2 \rho^\top\sum_{i=0}^\infty  \gamma^i \left(\widehat{P}^{\inf}\right)^i b^\star \nonumber \\
	& \overset{\mathrm{(i)}}{\leq} 2 \rho^\top\sum_{i=0}^\infty \gamma^i \left(\widehat{P}^{\inf}\right)^i b^\star = 2 \rho^\top  \left(I - \gamma \widehat{P}^{\inf} \right)^{-1} b^\star, \label{eq:infinite-perform-gap-1}
\end{align}
where (i) holds by $\big|\rho^\top \left(\widehat{P}^{\inf}\right)^i \left( V^{\star, \ror} -  \widehat{V}\right)\big| \leq \frac{1}{1-\gamma}$ for all $i\geq 0$, and that $\lim_{i\rightarrow \infty} \gamma^i \rho^\top \left(\widehat{P}^{\inf}\right)^i \left( V^{\star, \ror} -  \widehat{V}\right) =0$ since $\lim_{i\rightarrow \infty} \gamma^i = 0$ for all $0\leq \gamma<1$.

To further characterize the above performance gap, invoking the definition of $d^{\star, P}$ (cf.~\eqref{eq:visitation_dist_infty} and \eqref{eq:visitation_dist-optimal-infty}), we arrive at
\begin{align}
\left(d^{\star, \widehat{P}^{\inf}}\right)^\top= (1-\gamma) \rho^\top \sum_{t=0}^\infty \gamma^t \left(\widehat{P}^{\inf}\right)^t = (1-\gamma) \rho^\top \left(I - \gamma \widehat{P}^{\inf} \right)^{-1}. 
\end{align}
Plugging the above expression back into \eqref{eq:infinite-perform-gap-1}, and combining with\eqref{eq:infinite-pessimism-gap-transfer},  yields
\begin{align}\label{eq:summary-gap-format}
 V^{\star,\ror}(\rho)  - V^{\widehat{\pi},\ror}(\rho)  \leq V^{\star, \ror}(\rho) - \widehat{V}(\rho)  \leq \frac{2}{1-\gamma}\left\langle d^{\star, \widehat{P}^{\inf}}, b^\star \right\rangle.
\end{align}

\paragraph{Step 4: controlling $\left\langle d^{\star, \widehat{P}^{\inf}}, b^\star \right\rangle$ using concentrability.}
Note that $\widehat{P}^{\inf} \in \cU^\ror(P^\no)$ (see \eqref{eq:infinite-inf-hat} and \eqref{eq:infinite-matrix-notation}), which in words means $\widehat{P}^{\inf}$ is some transition kernel inside $\cU^\ror(P^\no)$ --- the uncertainty set around the nominal kernel $P^\no$.
Similar to the finite-horizon case, observing that we can express $\left<d^{\star, \widehat{P}^{\inf}}, b^\star \right> = \sum_{s\in\cS} d^{\star, \widehat{P}^{\inf}}(s) b^\star(s)$, we divide the states into two cases and control them separately.
\begin{itemize}
	\item {\bf Case 1: $s\in\cS$ where $\max_{ P \in \unb^{\ror}\big(P^0\big)} d^{\star, P} \big(s, \pi^\star(s)\big) = 0$.} Since $\widehat{P}^{\inf} \in \cU^\ror(P^\no)$, one has
	\begin{align*}
		0 \leq d^{\star, \widehat{P}^{\inf} }(s) = d^{\star, \widehat{P}^{\inf} }\big(s, \pi^\star(s)\big) \leq  \max_{ P \in \unb^{\ror}\big(P^0\big)} d^{\star, P} \big(s, \pi^\star(s)\big) =0, 
	\end{align*}
	which consequently indicates
	\begin{align}
			d^{\star, \widehat{P}^{\inf}}(s) = 0.
	\end{align}
	\item {\bf Case 2: $s\in\cS$ where $\max_{ P \in \unb^{\ror}\big(P^0\big)} d^{\star, P} \big(s, \pi^\star(s)\big) > 0$.} For any such state $s$, we claim that 
	\begin{align}
	\myrho\big(s, \pi^\star(s)\big) >0 \quad \text{ and } \quad  \big(s, \pi^\star(s)\big)\in\cC^{\mathsf{b}}.
\end{align}
This is due to Assumption~\ref{assumption:dro-infinite}, which requires $\Cstar$ to be finite given the numerator is positive:
\begin{align}
	\max_{P \in \unb^{\ror}(P^{\no})} \frac{\min\big\{d^{\star,P}\big(s, \pi^\star(s)\big), \frac{1}{S}\big\}}{\myrho\big(s, \pi^\star(s)\big)} = \max_{P \in \unb^{\ror}(P^{\no})} \frac{\min\big\{d^{\star,P}(s), \frac{1}{S}\big\}}{\myrho(s, a)}\le \Cstar < \infty.
\end{align}

To continue, invoking the fact in \eqref{eq:fact-of-N-b-assumption-infinite} with $\big(s, \pi^\star(s)\big)\in\cC^{\mathsf{b}}$ gives
\begin{align}\label{eq:infinite-pipeline-Nsa-bound}
	N\big(s, \pi^\star(s)\big) &\geq \frac{N \myrho\big(s, \pi^\star(s)\big)}{12}  \nonumber \\
	&\overset{\mathrm{(i)}}{\geq} \frac{ N \max_{P \in \unb^{\ror}(P^{\no})}\min\big\{d^{\star,P}\big(s, \pi^\star(s)\big), \frac{1}{S}\big\}}{12 \Cstar } \geq \frac{ N \min\big\{d^{\star, \widehat{P}^{\inf}}(s), \frac{1}{S}\big\}}{12 \Cstar },  
\end{align}
where (i) holds by Assumption~\ref{assumption:dro-infinite}, and the last inequality holds by $\widehat{P}^{\inf} \in \cU^\ror(P^\no)$.
With this in mind, we can control the pessimistic penalty $b^\star(s)$ (cf.~\eqref{def:bonus-dro-infinite}) by
\begin{align}
	b^\star(s) &\leq \frac{\cb}{\ror(1-\gamma)} \sqrt{\frac{\log \left(\frac{2(1+\ror)N^3S}{(1-\gamma)\delta} \right)}{ \widehat{P}_{\mathsf{min}}\big(s, \pi^\star(s)\big) N\big(s, \pi^\star(s)\big)}} + \frac{4}{\ror N(1-\gamma)}+ \frac{2}{\ror N} \nonumber \\
	& \overset{\mathrm{(i)}}{\leq}\frac{4\cb}{\ror(1-\gamma)} \sqrt{\frac{\log^2 \left(\frac{2(1+\ror)N^3S}{(1-\gamma)\delta} \right)}{ P_{\mathsf{min}}\big(s, \pi^\star(s)\big) N\big(s, \pi^\star(s)\big)}} + \frac{4}{\ror N(1-\gamma)}+ \frac{2}{\ror N}  \nonumber\\
	&\leq \frac{16\cb}{\ror(1-\gamma)} \sqrt{ \frac{ \Cstar \log^2 \left(\frac{2(1+\ror)N^3S}{(1-\gamma)\delta} \right)}{   P_{\mathsf{min}}\big(s, \pi^\star(s)\big) N \min\big\{d^{\star, \widehat{P}^{\inf}}(s), \frac{1}{S}\big\}}} + \frac{6}{\ror N(1-\gamma)}  \nonumber\\
	&\leq \frac{20\cb}{\ror(1-\gamma)} \sqrt{ \frac{ \Cstar \log^2 \left(\frac{2(1+\ror)N^3S}{(1-\gamma)\delta} \right)}{   P_{\mathsf{min}}\big(s, \pi^\star(s)\big) N \min\big\{d^{\star, \widehat{P}^{\inf}}(s), \frac{1}{S}\big\}}}, \nonumber
\end{align}
where (i) arises from \eqref{eq:convert-pmin-to-estimation-infinite}, the penultimate inequality follows from \eqref{eq:infinite-pipeline-Nsa-bound}, and the last inequality holds as long as $\cb$ is large enough.

\end{itemize} 

Summing up the above two cases, we arrive at
\begin{align}
	\left<d^{\star, \widehat{P}^{\inf}}, b^\star \right>& =  \sum_{s\in\cS} d^{\star, \widehat{P}^{\inf}}(s) b^\star(s) \nonumber \\
	&\leq \sum_{s\in\cS} d^{\star, \widehat{P}^{\inf}}(s) \frac{20\cb}{\ror(1-\gamma)} \sqrt{ \frac{ \Cstar \log^2 \left(\frac{2(1+\ror)N^3S}{(1-\gamma)\delta} \right)}{   P_{\mathsf{min}}\big(s, \pi^\star(s)\big) N \min\big\{d^{\star, \widehat{P}^{\inf}}(s), \frac{1}{S}\big\}}} \nonumber \\
	& \overset{\mathrm{(i)}}{\leq}  \frac{20\cb}{\ror(1-\gamma)} \sqrt{ \sum_{s\in\cS} d^{\star, \widehat{P}^{\inf}}(s) \frac{ \Cstar \log^2 \left(\frac{2(1+\ror)N^3S}{(1-\gamma)\delta} \right)}{   P_{\mathsf{min}}\big(s, \pi^\star(s)\big) N \min\big\{d^{\star, \widehat{P}^{\inf}}(s), \frac{1}{S}\big\}}}\sqrt{\sum_{s\in\cS} d^{\star, \widehat{P}^{\inf}}(s)}  \nonumber \\
	& \leq  \frac{40\cb}{\ror(1-\gamma)} \sqrt{  \frac{ S\Cstar \log^2 \left(\frac{2(1+\ror)N^3S}{(1-\gamma)\delta} \right)}{   P_{\mathsf{min}}^\star N }}, \label{eq:infinite-d-b-gap}
\end{align}
where (i) arises from Cauchy-Schwarz inequality, and the last inequality holds since $ P_{\mathsf{min}}\big(s, \pi^\star(s)\big) \geq P_{\mathsf{min}}^\star$ for all $s\in\cS$ (see \eqref{eq:link_minpall_pmin-infinite}) and the following fact (which has been established in \eqref{eq:d-star-1-S-bound}):
\begin{align*}
	\sum_{s\in\cS} \frac{d^{\star, \widehat{P}^{\inf}}(s)}{\min\big\{d^{\star, \widehat{P}^{\inf}}(s), \frac{1}{S}\big\}} \leq 2S.
\end{align*}

Finally, inserting \eqref{eq:infinite-d-b-gap} back into \eqref{eq:summary-gap-format}, with probability at least $1-2\delta$, one has
\begin{align*}
  V^{\star,\ror}(\rho)  - V^{\widehat{\pi},\ror}(\rho) \leq \frac{2}{1-\gamma}\left< d^{\star, \widehat{P}^{\inf}}, b^\star \right> \leq \frac{80\cb}{\ror(1-\gamma)^2} \sqrt{  \frac{ S\Cstar \log^2 \left(\frac{2(1+\ror)N^3S}{(1-\gamma)\delta} \right)}{   P_{\mathsf{min}}^\star N }},
\end{align*}
which concludes the proof.

\subsubsection{Proof of Lemma~\ref{lemma:dro-b-bound-infinite}}\label{proof:lemma:dro-b-bound-infinite}

We first note that the second assertion in \eqref{eq:convert-pmin-to-estimation-infinite} is the counterpart of \eqref{eq:convert-pmin-to-estimation}, which can be verified following the same argument in Appendix~\ref{sec:proof-sec:proof-eq:convert-pmin-to-estimation}. For brevity, we omit its proof, and shall focus on verifying \eqref{eq:dro-b-bound-infinite}.

To begin with, we consider the situation when $N(s,a) =0$. In this case, \eqref{eq:dro-b-bound-infinite} can be easily verified since
\begin{align}
	\left|\inf_{ \cP \in \unb^{\sigma}(\widehat{P}_{s,a}^0)} \cP V - \inf_{ \cP \in \unb^{\sigma}(P^{\no}_{s,a})}  \cP V \right| \overset{\mathrm{(i)}}{=}  \inf_{ \cP \in \unb^{\sigma}(P^{\no}_{s,a})}  \cP V  \leq  \|V\|_\infty \overset{\mathrm{(ii)}}{\leq} \frac{1}{1-\gamma},
\end{align}
where (i) follows from the fact $\widehat{P}_{s,a}^0 =0$ when $N(s,a) =0$ (see \eqref{eq:empirical-P-infinite}), and (ii) arises from the assumption $\|V\|_\infty \leq \frac{1}{1-\gamma}$. Consequently, in the remainder of the proof, we focus on verifying \eqref{eq:dro-b-bound-infinite} when $N(s,a) >0$.
Let us first introduce the counterpart of the claim \eqref{eq:dro-b-bound} in Lemma~\ref{lemma:dro-b-bound} as follows.
\begin{lemma}\label{lemma:infinite-pointwise-uncertainty-bound}
For all $(s,a)\in  \cS\times \cA$ with $N(s,a) >0$, consider any vector $V\in \mathbb{R}^S$ independent of $\widehat{P}^0_{s,a}$ obeying $\|V\|_{\infty} \le \frac{1}{1-\gamma}$. With probability at least $1- \delta$, one has 
\begin{align}
	&\left| \inf_{ \cP \in \unb^{\sigma}(\widehat{P}_{s,a}^0)} \cP V- \inf_{ \cP \in \unb^{\sigma}(P^{\no}_{s,a})}  \cP V \right| \leq  \frac{\cb}{\ror(1-\gamma)} \sqrt{\frac{\log(\frac{NS}{\delta})}{ \widehat{P}_{\mathsf{min}}(s,a) N(s,a)}}.
\end{align}
\end{lemma}
\begin{proof}
The proof follows from the same arguments in Appendix~\ref{proof:finite-control-uncertainty-gap}, with small modifications to adapt to the infinite-horizon setting; we omit the details for conciseness.% only difference is the upper bound on $\|V\|_\infty$ is $\frac{1}{1 -\gamma}$ (as opposed to $H$), the union bound is taken over $N$ (as opposed to $KH$), and some notations are exchanged to that of the infinite-horizon case. We omit the proof details for conciseness.
\end{proof}

Armed with the above point-wise concentration bound, we are now ready to derive the uniform concentration bound desired as in Lemma~\ref{lemma:dro-b-bound-infinite}, counting on a leave-one-out argument divided into the following steps. The crux of the analysis is to construct a set of auxiliary RMDPs, each different from the empirical RMDP only at a single state but possessing crucial statistical independence that facilitates the concentration arguments, which can then be transferred back to the empirical RMDP via a simple triangle inequality.

\paragraph{Step 1: construction of auxiliary RMDPs with state-absorbing empirical nominal transitions.}
Denote the empirical infinite-horizon robust MDP with the nominal transition kernel $\widehat{P}^\no$ as $\widehat{\cM}_{\mathsf{rob}}$.  Then, for each state $s$ and each scalar $u\geq 0$, we can construct an auxiliary robust MDP $\widehat{\cM}^{s,u}_{\mathsf{rob}}$ so that it is the same as $\widehat{\cM}_{\mathsf{rob}}$ except the properties in state $s$. To be precise, let the nominal transition kernel and reward function of $\widehat{\cM}^{s,u}_{\mathsf{rob}}$ be $P^{s,u}$ and
$r^{s,u}$, which are given respectively as
 \begin{align}\label{eq:auxiliary-P-infinite}
	\begin{cases}
		P^{s,u}(s'\mymid s,a ) = \ind(s' = s) & \qquad \qquad \text{for all } (s', a) \in \cS\times \cA,  \\
		P^{s,u}(\cdot \mymid \widetilde{s} ,a ) = \widehat{P}^\no(\cdot \mymid \widetilde{s} ,a) & \qquad \qquad \text{for all } (\widetilde{s} ,a) \in \cS\times \cA \text{ and } \widetilde{s}  \neq s,
	\end{cases}
\end{align}
and
\begin{align}
\begin{cases}\label{eq:auxiliary-r-infinite}
		r^{s,u}(s,a) = u & \qquad \qquad \qquad  \text{for all } a \in \cA,  \\
		r^{s,u}(\widetilde{s},a) = r(\widetilde{s},a) & \qquad \qquad \qquad \text{for all } (\widetilde{s} ,a) \in \cS\times \cA \text{ and } \widetilde{s}  \neq s.
	\end{cases}
\end{align}
Clearly, state $s$ of the auxiliary $\widehat{\cM}^{s,u}_{\mathsf{rob}}$ is absorbing, meaning that the state stays at $s$ once entering it. This removes the randomness of $\widehat{P}^\no_{s,a}$ for all $a\in\cA$ in state $s$, a key property we will leverage later. 
 
With the robust MDP $\widehat{\cM}^{s,u}_{\mathsf{rob}}$ in hand, we still need to complete the design by defining the corresponding penalty term for all $(\widetilde{s},a)\in\cS\times \cA$, which is given as follows
\begin{align}
	b^{s,u}(\widetilde{s},a) \defn \begin{cases}
	\min \left\{\frac{\cb}{\ror(1-\gamma)} \sqrt{\frac{\log\left(\frac{2(1+\ror)N^3S}{(1-\gamma)\delta}\right)}{ P^{s.u}_{\mathsf{min}}(s,a) N(\widetilde{s},a)}} + \frac{4}{N \ror (1-\gamma)}\;, \; \frac{1}{1-\gamma}\right\} + \frac{2}{ \ror N}& \text{if } N( \widetilde{s},a) > 0, \\
	\frac{1}{1-\gamma} + \frac{2}{ \ror N}   & \text{ otherwise},
	\end{cases}
	\label{def:bonus-dro-auxiliary}
\end{align}
where $P^{s,u}_{\mathsf{min}}(\widetilde{s},a)$ is defined as the smallest positive state transition probability over the nominal kernel $P^{s,u}(\cdot \mymid \widetilde{s},a)$:
\begin{align}\label{eq:P-min-hat-def-auxiliary}
	\forall (\widetilde{s},a) \in\cS\times \cA: \quad P^{s,u}_{\mathsf{min}}(\widetilde{s},a) \defn  \min_{s'} \Big\{P^{s,u}(s' \mymid \widetilde{s},a): \; P^{s,u}(s' \mymid \widetilde{s},a)>0 \Big\}.
\end{align}
In view of \eqref{eq:auxiliary-P-infinite} and \eqref{eq:P-min-hat-def-infinite}, it holds that $P^{s,u}_{\mathsf{min}}(\widetilde{s},a) = \widehat{P}_{\mathsf{min}}(\widetilde{s},a)  $, and therefore $b^{s,u}(\widetilde{s},a) = b(\widetilde{s},a)$, when $\widetilde{s}\neq s$ for any $u\geq 0$. 
Armed with the above definitions, the pessimistic robust Bellman operator $\widehat{\cT}^{\sigma}_{s,u}(Q)(\cdot)$ of the RMDP $\widehat{\cM}^{s,u}_{\mathsf{rob}}$ is defined as
\begin{align}\label{eq:pessimism-operator-auxiliary}
	\forall (s,a)\in \cS\times \cA :\quad  \widehat{\cT}^{\sigma}_{s,u}(Q)(s,a) = \max\left\{ r(s, a)  + \gamma \inf_{ \cP \in \unb^{\sigma}(P^{s,u}_{s,a})} \cP V - b^{s,u}(s,a), \,  0 \right\}.
\end{align}

\paragraph{Step 2: fixed-point equivalence between $\widehat{\cM}_{\mathsf{rob}}$ and the auxiliary RMDP $\widehat{\cM}^{s,u}_{\mathsf{rob}}$.}
Recall that $\widehat{Q}^{\star,\ror}_{\mathsf{pe}}$ is the unique fixed point of $\tpe(\cdot)$ with the corresponding value $\widehat{V}^{\star,\ror}_{\mathsf{pe}}$. We claim that there exists some choice of $u$ such that the fixed point of $\widehat{\cT}^{\sigma}_{s,u}(Q)(\cdot)$ coincides with that of $\tpe(\cdot)$. In particular, given a state $s$, we show the following choice of $u$ suffices:  
\begin{align}\label{eq:def-u-star}
	u^\star \defn (1-\gamma) \widehat{V}^{\star, \sigma}_{\mathsf{pe}}(s) + \min \left\{\frac{\cb}{\ror(1-\gamma)} \sqrt{\frac{\log\left(\frac{2(1+\ror)N^3S}{(1-\gamma)\delta}\right)}{ P^{s.u}_{\mathsf{min}}(s,a) N(s,a)}} + \frac{4}{N \ror (1-\gamma)}\;, \; \frac{1}{1-\gamma}\right\}  + \frac{2}{ \ror N}.
\end{align}
Towards this, we shall break our arguments in two different cases.
\begin{itemize}
	\item {\bf For state $s' \neq s$.} In this case, for any $a\in\cA$, it can be verified that
\begin{align}
	&\max\left\{ r^{s,u^\star}(s', a)  + \gamma \inf_{ \cP \in \unb^{\sigma}(P^{s,u^\star}_{s',a})} \cP \widehat{V}^{\star,\ror}_{\mathsf{pe}} - b^{s,u^\star}(s',a), \,  0 \right\} \nonumber \\
	&= \max\left\{ r(s', a)  + \gamma \inf_{ \cP \in \unb^{\sigma}(\widehat{P}^\no_{s',a})} \cP \widehat{V}^{\star,\ror}_{\mathsf{pe}} - b(s',a), \,  0 \right\} \nonumber \\
	&= \widehat{\cT}^{\sigma}_{\mathsf{pe}}(\widehat{Q}^{\star,\ror}_{\mathsf{pe}})(s',a) =\widehat{Q}^{\star,\ror}_{\mathsf{pe}}(s',a),
\end{align}
where the second line follows from the definitions in \eqref{eq:auxiliary-r-infinite} and \eqref{eq:auxiliary-P-infinite} as well as $b^{s,u^\star}( s',a) = b(s',a)$ when $s'\neq s$, the last line arises from the definition of the pessimistic Bellman operator \eqref{eq:pessimism-operator}, and that $\widehat{Q}^{\star,\ror}_{\mathsf{pe}}$ is the fixed point.

	\item {\bf For state $s$.} In this case, for any $u$ and $a\in\cA$, observing that $P^{s,u}(s' \mymid s,a)$ has only one positive entry equal to $1$ (cf.~\eqref{eq:auxiliary-P-infinite}), applying \eqref{eq:P-min-hat-def-auxiliary} yields
	\begin{align}
		P^{s,u}_{\mathsf{min}}(s,a) = 1.
	\end{align}
	Plugging the above fact into \eqref{def:bonus-dro-auxiliary} leads to
	\begin{align}\label{eq:infinite-loo-b-s-u-final}
	b^{s,u}(s,a)= \begin{cases}
	\min \left\{\frac{\cb}{\ror(1-\gamma)} \sqrt{\frac{\log\left(\frac{2(1+\ror)N^3S}{(1-\gamma)\delta}\right)}{  N(s,a)}} + \frac{4}{N \ror (1-\gamma)}\;, \; \frac{1}{1-\gamma}\right\}  + \frac{2}{ \ror N}& \text{if } N(s,a) > 0, \\
	\frac{1}{1-\gamma}  & \text{ otherwise}
	\end{cases}
\end{align}
for all $a \in  \cA$.
As a result, we have for any $a\in\cA$:
\begin{align}
	&\max\left\{ r^{s,u^\star}(s, a)  + \gamma \inf_{ \cP \in \unb^{\sigma}(P^{s,u^\star}_{s,a})} \cP \widehat{V}^{\star,\ror}_{\mathsf{pe}} - b^{s,u^\star}(s,a), \,  0 \right\} \notag\\
	&= \max\left\{ u^\star  + \gamma \widehat{V}^{\star,\ror}_{\mathsf{pe}}(s) - b^{s,u^\star}(s,a), \,  0 \right\} \notag\\
	&  = \max\left\{ (1-\gamma) \widehat{V}^{\star,\ror}_{\mathsf{pe}}(s)+ \gamma \widehat{V}^{\star,\ror}_{\mathsf{pe}}(s), \,  0 \right\} = \widehat{V}^{\star,\ror}_{\mathsf{pe}}(s),
\end{align}
where the second line follows from the fact that $P^{s,u^\star}_{s,a}$ is a singleton distribution at state $s$, and  hence $\unb^{\sigma}(P^{s,u^\star}_{s,a}) =P^{s,u^\star}_{s,a}$ by the definition of the KL uncertainty set, and the second line follows from plugging in the definition of $u^\star$ in \eqref{eq:def-u-star} and $b^{s,u^\star}(s,a)$ in \eqref{eq:infinite-loo-b-s-u-final}.

\end{itemize}
Summing up the above two cases, we establish that there exists a fixed point $\widehat{Q}^{\star,\ror}_{s,u^\star}$ of the operator $\widehat{\cT}^{\sigma}_{s,u^\star}(\cdot)$ if we let
\begin{align}
	\begin{cases}
		\widehat{Q}^{\star,\ror}_{s,u^\star}(s,a) = \widehat{V}^{\star,\ror}_{\mathsf{pe}}(s) & \qquad \qquad \qquad \text{for all } a \in \cA, \\
		\widehat{Q}^{\star,\ror}_{s,u^\star}(s',a) = \widehat{Q}^{\star,\ror}_{\mathsf{pe}}(s',a) & \qquad \qquad \qquad \text{for all } s' \neq s \text{ and } a \in \cA.
	\end{cases}
\end{align}
Consequently, we confirm the existence of a fixed point of the operator $\widehat{\cT}^{\sigma}_{s,u^\star}(\cdot)$. In addition, its corresponding value function $\widehat{V}^{\star,\ror}_{s,u^\star}$ also coincides with $\widehat{V}^{\star,\ror}_{\mathsf{pe}}$. 

\paragraph{Step 3: building an $\varepsilon$-net for all reward values $u$.}
It is easily verified that the reward $u^\star$ obeys 
\begin{align} \label{eq:def-u-star-bound}
u^\star \leq 1 + \min \left\{\frac{\cb}{\ror(1-\gamma)} \sqrt{\frac{\log\left(\frac{2(1+\ror)N^3S}{(1-\gamma)\delta}\right)}{ P^{s,u}_{\mathsf{min}}(s,a) N(s,a)}} + \frac{4}{ \ror N(1-\gamma)}\;, \; \frac{1}{1-\gamma}\right\} + \frac{2}{ \ror N} \leq \frac{2}{ \ror } + \frac{2}{1-\gamma}.
\end{align}
As a result, we construct an $\varepsilon$-net \citep{vershynin2018high} of the line segment within the range $\big[0, \frac{2}{ \ror } + \frac{2}{1-\gamma } \big]$ with $\varepsilon = \frac{1}{ \ror N}$ as follows:
\begin{align}
	\cU_\varepsilon \defn \left\{ \frac{i}{ \ror N} \mymid 1 \leq i \leq \left\lfloor \ror N \left(\frac{2}{ \ror } + \frac{2}{1-\gamma}\right) \right\rfloor \right\}.
\end{align}

Armed with this covering net $\cU_\varepsilon$, we can construct an auxiliary robust MDP $\widehat{\cM}_{\mathsf{rob}}^{s,u}$ and its corresponding pessimistic robust Bellman operator for each $u\in \cU_\varepsilon$ (see Step 1). Following the same arguments in the proof of Lemma~\ref{lem:contration-of-T} (cf.~Appendix~\ref{proof:lem:contration-of-T}),  for each $u\in \cU_\varepsilon$, it can be verified that there exists a unique fixed point $\widehat{Q}^{\star,\ror}_{s,u}$ of the operator $\widehat{\cT}^{\sigma}_{s,u}(\cdot)$, which satisfies $ 0\leq \widehat{Q}^{\star,\ror}_{s,u} \leq \frac{1}{1-\gamma} \cdot 1$. In turn, the corresponding value function also satisfies $\|\widehat{V}^{\star,\ror}_{s,u} \|_\infty \leq \frac{1}{1-\gamma}$.

In view of the definitions in \eqref{eq:auxiliary-P-infinite} and  \eqref{eq:auxiliary-r-infinite}, for all $u\in \cU_\varepsilon$,  $\widehat{\cM}_{\mathsf{rob}}^{s,u}$ is statistically independent from $\widehat{P}^\no_{s,a}$, which indicates the independence between $\widehat{V}^{\star,\ror}_{s,u}$ and $\widehat{P}^\no_{s,a}$. This makes it possible to invoke Lemma~\ref{lemma:infinite-pointwise-uncertainty-bound}, and taking the union bound over all samples $N$ and $u\in \cU_\varepsilon$ give that, with probability at least $1-\delta$,
\begin{align}\label{eq:auxiliary=-union-u-bound}
	&\left| \inf_{ \cP \in \unb^{\sigma}(\widehat{P}_{s,a}^0)} \cP \widehat{V}^{\star,\ror}_{s,u}- \inf_{ \cP \in \unb^{\sigma}(P^{\no}_{s,a})}  \cP \widehat{V}^{\star,\ror}_{s,u} \right| \leq \frac{\cb}{\ror(1-\gamma)} \sqrt{\frac{\log\left(\frac{2(1+\ror)N^3S}{(1-\gamma)\delta}\right)}{ \widehat{P}_{\mathsf{min}}(s,a) N(s,a)}}
\end{align}
hold simultaneously for all $(s,a,u) \in \cS \times \cA\times \cU_\varepsilon$ with $N(s,a)>0$.

\paragraph{Step 4: a covering argument.}

Recalling that $u^\star \in \big[0, \frac{2}{ \ror } + \frac{2}{1-\gamma} \big]$ (see \eqref{eq:def-u-star-bound}), we can always find some $\widetilde{u} \in \cU_\varepsilon$ such that $|\widetilde{u} - u^\star| \leq \frac{1}{ \ror N}$.  Consequently, plugging in the operator in \eqref{eq:pessimism-operator-auxiliary} yields
\begin{align}
	\forall Q\in \mathbb{R}^{SA} :\quad \left\| \widehat{\cT}^{\sigma}_{s,\widetilde{u}}(Q) - \widehat{\cT}^{\sigma}_{s,u^\star}(Q) \right\|_\infty \overset{\mathrm{(i)}}{\leq} |\widetilde{u} - u^\star| \leq \frac{1}{ \ror N},
\end{align}
where (i) holds by $b^{s,\widetilde{u}}(s,a) = b^{s,u^\star}(s,a)$ for $s$ (see \eqref{eq:infinite-loo-b-s-u-final}) and $b^{s,\widetilde{u}}(s',a) = b^{s,u^\star}(s',a) = b(s',a)$ for all $s'\neq s$.

With this in mind, we observe that the fixed points of $\widehat{\cT}^{\sigma}_{s,\widetilde{u}}(\cdot)$ and $\widehat{\cT}^{\sigma}_{s,u^{\star}}(\cdot)$ obey
\begin{align}
	\left\| \widehat{Q}^{\star,\ror}_{s,\widetilde{u}} -  \widehat{Q}^{\star,\ror}_{s,u^\star}\right\|_\infty &= \left\| \widehat{\cT}^{\sigma}_{s,\widetilde{u}}(\widehat{Q}^{\star,\ror}_{s,\widetilde{u}}) - \widehat{\cT}^{\sigma}_{s,u^\star}(\widehat{Q}^{\star,\ror}_{s, u^\star}) \right\|_\infty \notag \\
	&\leq \left\| \widehat{\cT}^{\sigma}_{s,\widetilde{u}}(\widehat{Q}^{\star,\ror}_{s,\widetilde{u}}) - \widehat{\cT}^{\sigma}_{s,\widetilde{u}}(\widehat{Q}^{\star,\ror}_{s, u^\star}) \right\|_\infty + \left\| \widehat{\cT}^{\sigma}_{s,\widetilde{u}}(\widehat{Q}^{\star,\ror}_{s, u^\star}) - \widehat{\cT}^{\sigma}_{s,u^\star}(\widehat{Q}^{\star,\ror}_{s, u^\star}) \right\|_\infty \notag \\
	& \leq \gamma \left\| \widehat{Q}^{\star,\ror}_{s,\widetilde{u}} -  \widehat{Q}^{\star,\ror}_{s,u^\star}\right\|_\infty + \frac{1}{\ror N},
\end{align}
which directly indicates that
\begin{align}
	\left\| \widehat{Q}^{\star,\ror}_{s,\widetilde{u}} -  \widehat{Q}^{\star,\ror}_{s,u^\star}\right\|_\infty  \leq \frac{1}{(1-\gamma) \ror N}
\end{align}
and  
\begin{align}\label{eq:auxiliary-value-gap}
	\left\| \widehat{V}^{\star,\ror}_{s,\widetilde{u}} -  \widehat{V}^{\star,\ror}_{s,u^\star}\right\|_\infty \leq \left\| \widehat{Q}^{\star,\ror}_{s,\widetilde{u}} -  \widehat{Q}^{\star,\ror}_{s,u^\star}\right\|_\infty \leq \frac{1}{(1-\gamma) \ror N}.
\end{align}
Armed with the above facts, invoking the identity $\widehat{V}^{\star,\ror}_{\mathsf{pe}} = \widehat{V}^{\star,\ror}_{s,u^\star}$ established in Step 2 gives
\begin{align}
	&\left| \inf_{ \cP \in \unb^{\sigma}(\widehat{P}_{s,a}^0)} \cP \widehat{V}^{\star,\ror}_{\mathsf{pe}}- \inf_{ \cP \in \unb^{\sigma}(P^{\no}_{s,a})}  \cP \widehat{V}^{\star,\ror}_{\mathsf{pe}} \right| = \left| \inf_{ \cP \in \unb^{\sigma}(\widehat{P}_{s,a}^0)} \cP \widehat{V}^{\star,\ror}_{s,u^\star}- \inf_{ \cP \in \unb^{\sigma}(P^{\no}_{s,a})}  \cP \widehat{V}^{\star,\ror}_{s,u^\star} \right|  \nonumber\\
	& \overset{\mathrm{(i)}}{\leq} \left| \inf_{ \cP \in \unb^{\sigma}(\widehat{P}_{s,a}^0)} \cP \widehat{V}^{\star,\ror}_{s,\widetilde{u}}- \inf_{ \cP \in \unb^{\sigma}(P^{\no}_{s,a})}  \cP \widehat{V}^{\star,\ror}_{s,\widetilde{u}} \right| \nonumber\\
	&\qquad +  \left| \inf_{ \cP \in \unb^{\sigma}(\widehat{P}_{s,a}^0)} \cP \widehat{V}^{\star,\ror}_{s,\widetilde{u}}-  \inf_{ \cP \in \unb^{\sigma}(\widehat{P}_{s,a}^0)} \cP \widehat{V}^{\star,\ror}_{s,u^\star}  \right| + \left| \inf_{ \cP \in \unb^{\sigma}(P^{\no}_{s,a})}  \cP \widehat{V}^{\star,\ror}_{s,\widetilde{u}} - \inf_{ \cP \in \unb^{\sigma}(P^{\no}_{s,a})}  \cP \widehat{V}^{\star,\ror}_{s,u^\star}\right|\nonumber\\
	&\overset{\mathrm{(ii)}}{\leq} \left| \inf_{ \cP \in \unb^{\sigma}(\widehat{P}_{s,a}^0)} \cP \widehat{V}^{\star,\ror}_{s,\widetilde{u}}- \inf_{ \cP \in \unb^{\sigma}(P^{\no}_{s,a})}  \cP \widehat{V}^{\star,\ror}_{s,\widetilde{u}} \right| + \frac{2}{N \ror (1-\gamma)}\nonumber\\
	&\leq \frac{\cb}{\ror(1-\gamma)} \sqrt{\frac{\log\left(\frac{2(1+\ror)N^3S}{(1-\gamma)\delta}\right)}{ \widehat{P}_{\mathsf{min}}(s,a) N(s,a)}} + \frac{2}{N \ror (1-\gamma)}, \label{eq:auxiliary-uncertainty-gap-V-star}
\end{align}
where (i) holds by applying the triangle inequality, (ii) arises from \eqref{eq:auxiliary-value-gap} and the basic fact that infimum operator is a $1$-contraction w.r.t. $\|\cdot \|_\infty$, and the final inequality follows from \eqref{eq:auxiliary=-union-u-bound}.

\paragraph{Step 5: finishing up.}
Now we are positioned to finish up the proof. For all vector $\widetilde{V}$ obeying $\big\|\widetilde{V} - \widehat{V}^{\star, \ror}_{\mathsf{pe}} \big\|_\infty \le \frac{1}{\ror  N}$ and $\|\widetilde{V}\|_{\infty} \le \frac{1}{ 1-\gamma}$, we apply the triangle inequality and invoke \eqref{eq:auxiliary-uncertainty-gap-V-star} to reach
\begin{align}
	&\left| \inf_{ \cP \in \unb^{\sigma}(\widehat{P}_{s,a}^0)} \cP \widetilde{V} - \inf_{ \cP \in \unb^{\sigma}(P^{\no}_{s,a})}  \cP \widetilde{V} \right| \leq \left| \inf_{ \cP \in \unb^{\sigma}(\widehat{P}_{s,a}^0)} \cP \widehat{V}^{\star,\ror}_{\mathsf{pe}}- \inf_{ \cP \in \unb^{\sigma}(P^{\no}_{s,a})}  \cP \widehat{V}^{\star,\ror}_{\mathsf{pe}} \right| \nonumber \\
	&\qquad + \left| \inf_{ \cP \in \unb^{\sigma}(\widehat{P}_{s,a}^0)} \cP \widetilde{V} -  \inf_{ \cP \in \unb^{\sigma}(\widehat{P}_{s,a}^0)} \cP \widehat{V}^{\star,\ror}_{\mathsf{pe}} \right| + \left| \inf_{ \cP \in \unb^{\sigma}(P^{\no}_{s,a})}  \cP  \widetilde{V} - \inf_{ \cP \in \unb^{\sigma}(P^{\no}_{s,a})}  \cP \widehat{V}^{\star,\ror}_{\mathsf{pe}} \right|\nonumber\\
	& \leq \frac{\cb}{\ror(1-\gamma)} \sqrt{\frac{\log\left(\frac{2(1+\ror)N^3S}{(1-\gamma)\delta}\right)}{ \widehat{P}_{\mathsf{min}}(s,a) N(s,a)}} + \frac{4}{N \ror (1-\gamma)}.
\end{align}

Finally, we complete the proof by verifying that
\begin{align}
	\left| \inf_{ \cP \in \unb^{\sigma}(\widehat{P}_{s,a}^0)} \cP \widetilde{V} - \inf_{ \cP \in \unb^{\sigma}(P^{\no}_{s,a})}  \cP \widetilde{V} \right| \leq \left \| \widetilde{V}\right\|_\infty \leq \frac{1}{1-\gamma}.
\end{align}

\subsubsection{Proof of \eqref{eq:fact-of-N-b-assumption-infinite}}\label{sec:proof-eq:fact-of-N-b-assumption-infinite}
For all  $(s,a)\in \mathcal{C}^{\mathsf{b}}$, one has
\begin{align}
N \myrho\big(s, a\big)  \overset{\mathrm{(i)}}{\geq}  \frac{ c_1 \myrho\big(s, a\big) \log( NS/ \delta )}{  d_{\mathsf{min}}^{\mathsf{b}}  P_{\mathsf{min}}^{\mathsf{b}} } \overset{\mathrm{(ii)}}{\geq}  \frac{ c_1  \log( N S/ \delta ) }{P_{\mathsf{min}}^{\mathsf{b}}  }  \overset{\mathrm{(iii)}}{\geq}  \frac{ c_1\log( NS/ \delta )  }{P_{\mathsf{min}}(s,a) }  ,
\end{align}
where (i) follows from the condition \eqref{eq:dro-b-bound-N-condition-infinite}, (ii) arises from the definition 
that $d_{\mathsf{min}}^{\mathsf{b}}  \leq \myrho(s,a)$ for all $(s,a)\in \mathcal{C}^{\mathsf{b}}$, and (iii) follows from the definition in \eqref{eq:link_minpall_pmin-infinite}. In particular, when $c_1$ is large enough, one has $\frac{2}{3}\log\frac{NS}{\delta} <\frac{N \myrho(s,a)}{12}$. To continue, we recall a key property of $N(s,a)$ (cf.~\eqref{eq:defn-Nh-sa-infinite}) in the following lemma.
\begin{lemma}[{\citep[Lemma~7]{li2022settling}}]\label{lemma:N-prop-infinite}
Fix $\delta\in (0,1)$. With probability at least $1-\delta$, the quantities $\{N(s,a)\}$ in \eqref{eq:defn-Nh-sa-infinite} obey
\begin{align}
	\max \left\{ N(s,a),\, \frac{2}{3}\log\frac{NS}{\delta} \right\} \geq \frac{N \myrho(s,a)}{12}
\end{align}
simultaneously for all $(s,a) \in\cS\times \cA$.
\end{lemma}
Consequently, Lemma~\ref{lemma:N-prop-infinite} tells us that with probability at least $1-\delta$, 
\begin{align}
	N(s,a) \geq \frac{N \myrho(s,a)}{12} \geq \frac{ c_1\log( NS/ \delta )  }{12 P_{\mathsf{min}}(s,a) }
\end{align}
as long as $c_1$ is large enough. Last but not least, taking the basic fact $x\leq - \log(1-x)$ for all $x\in [0,1]$, the last inequality of \eqref{eq:fact-of-N-b-assumption-infinite} can be verified by
\begin{align}
	  \frac{ c_1\log( NS/ \delta )  }{12 P_{\mathsf{min}}(s,a) } \geq -  \frac{\log\frac{2NS}{\delta}}{\log (1- P_{\mathsf{min}}(s,a) )}.
\end{align}

%% file: auxiliary-upper-bound-infty.tex
\subsection{Proof of Lemma~\ref{lem:contration-of-T}}\label{proof:lem:contration-of-T}
We shall first show that the operator $\tpe(\cdot)$ (cf.~\eqref{eq:pessimism-operator}) is a $\gamma$-contraction, which will in turn imply the existence of the unique fixed point of $\tpe(\cdot)$.
Before starting, suppose that the entries of $ Q_1, Q_2 \in \mathbb{R}^{SA}$ are all bounded in $\big[0, \frac{1}{1-\gamma}\big]$ for all $(s,a)\in \cS\times \cA$. Denote that 
\begin{align}\label{eq:contraction-proof-notation}
	\forall s\in \cS: \quad & V_1(s) \defn  \max_a Q_1(s,a), \quad V_2(s) \defn \max_a Q_2(s,a).
\end{align}

\paragraph{Proof of $\gamma$-contraction.} We first show that $\tpe(\cdot)$ is a $\gamma$-contraction. Towards this, instead of $\tpe(\cdot)$, we begin with a simpler operator $\widetilde{\cT}^\ror_{\mathsf{pe}}(\cdot)$, defined as follows:
\begin{align}\label{eq:tilde-T-pe}
	\forall (s,a)\in \cS\times \cA :\quad  \widetilde{\cT}^\ror_{\mathsf{pe}}(Q)(s,a) = r(s, a)  + \gamma \inf_{ \cP \in \unb^{\sigma}(\widehat{P}^{\no}_{s,a})} \cP V - b\big(s,a\big),
\end{align}
which consequently leads to
\begin{align}\label{eq:tilde-T-pe-relation}
	\forall (s,a)\in \cS\times \cA :\quad \tpe(Q)(s,a) = \max \left\{\widetilde{\cT}^\ror_{\mathsf{pe}}(Q)(s,a), 0 \right\}.
\end{align}
It follows straightforwardly that
\begin{align}
	\left \|\tpe(Q_1) - \tpe(Q_2)\right\|_\infty \leq \left \|\widetilde{\cT}^\ror_{\mathsf{pe}}(Q_1) - \widetilde{\cT}^\ror_{\mathsf{pe}}(Q_2)\right\|_\infty ,
\end{align}
and hence it suffices to establish the $\gamma$-contraction of $\widetilde{\cT}^\ror_{\mathsf{pe}}(\cdot)$. With this in mind, we observe that
\begin{align}
	\left \|\widetilde{\cT}^\ror_{\mathsf{pe}}(Q_1) - \widetilde{\cT}^\ror_{\mathsf{pe}}(Q_2)\right\|_\infty & = \gamma \left\| \inf_{ \cP \in \unb^{\sigma}(\widehat{P}^{\no}_{s,a})} \cP V_1 - \inf_{ \cP \in \unb^{\sigma}(\widehat{P}^{\no}_{s,a})} \cP V_2 \right\|_\infty \overset{\mathrm{(i)}}{\leq} \gamma \left\| V_1 - V_2 \right\|_\infty \nonumber \\
	& \overset{\mathrm{(ii)}}{=} \gamma \max_s \left|\max_a Q_1(s,a) - \max_a Q_2(s,a)\right| \notag\\
	&  \leq \gamma \max_{(s,a)} \left|Q_1(s,a) - Q_2(s,a)\right| = \gamma \left\|  Q_1 -  Q_2 \right\|_\infty , 
\end{align}
where the first equality holds by the definition of $\widetilde{\cT}^\ror_{\mathsf{pe}}(\cdot)$ (cf.~\eqref{eq:tilde-T-pe}), (i) follows from that the infimum operator is a $1$-contraction w.r.t. $\|\cdot \|_\infty$ and $\| \cP V_1 - \cP V_2 \|_\infty \leq \| V_1 -V_2 \|_\infty$ for all $\cP \in \Delta(\cS)$, (ii) arises from the definitions in \eqref{eq:contraction-proof-notation}, and the last inequality is due to the maximum operator is also a $1$-contraction w.r.t. $\|\cdot \|_\infty$. Combining the above two inequalities establish the desired statement.

\paragraph{Existence of the unique fixed point.} To continue, we shall first claim that there exists at least one fixed point of $\tpe(\cdot)$. This is a standard argument, which we omit for brevity; interested readers are encouraged to refer to, e.g. \citet{li2022settling}, for details.

To prove the uniqueness of the fixed points of $\tpe(\cdot)$, suppose that there exist two fixed points $Q'$ and $Q''$ obeying obeying $Q' = \tpe(Q')$ and  $Q'' = \tpe(Q'')$. Moreover, the definition of $\tpe(\cdot)$  directly implies $0 \leq Q', Q'' \leq \frac{1}{1-\gamma}$, since for any $0\leq Q\leq \frac{1}{1-\gamma}$, it follows that  $0\leq \tpe(Q) \leq \frac{1}{1-\gamma}$. By the $\gamma$-contraction property, it follows that
\begin{align}\label{eq:contraction-contradiction}
	\left\| Q' - Q'' \right\|_\infty  = \left\| \tpe(Q') - \tpe  (Q''  ) \right\|_\infty \leq \gamma \left\| Q' - Q'' \right\|_\infty. 
\end{align}
However, \eqref{eq:contraction-contradiction} can't happen given $\gamma \in \big[\frac{1}{2}, 1\big)$, indicating the uniqueness of the fixed points of $\tpe(\cdot)$.

\subsection{Proof of Lemma~\ref{lem:infinite-converge}}\label{proof:lem:infinite-converge}
To begin with, considering any $Q, Q'$ obeying $Q\leq Q'$, and $0 \leq Q, Q' \leq \frac{1}{1-\gamma} $. We observe that the operator $\tpe(\cdot)$ (cf.~\eqref{eq:pessimism-operator}) has the monotone non-decreasing property, namely,
\begin{align}\label{eq:monotone-of-T-pe}
	\tpe(Q)(s,a) &= \max\left\{ r(s, a)  + \gamma \inf_{ \cP \in \unb^{\sigma}(\widehat{P}^{\no}_{s,a})} \cP V - b\big(s,a\big), \,  0 \right\} \notag \\
	& = \max\left\{ r(s, a)  + \gamma \inf_{ \cP \in \unb^{\sigma}(\widehat{P}^{\no}_{s,a})} \cP \max_{a'} Q(\cdot, a') - b\big(s,a\big), \,  0 \right\} \notag \\
	& \leq \max\left\{ r(s, a)  + \gamma \inf_{ \cP \in \unb^{\sigma}(\widehat{P}^{\no}_{s,a})} \cP \max_{a'} Q'(\cdot, a') - b\big(s,a\big), \,  0 \right\} = \tpe(Q')(s,a),
\end{align}
where the last line uses $Q\leq Q'$. Recalling the fixed point $\widehat{Q}^{\star,\ror}_{\mathsf{pe}}$ of $\tpe(\cdot)$, armed with \eqref{eq:monotone-of-T-pe} and the initialization $\widehat{Q}_0 = 0$, we arrive at
\begin{align*}
	\widehat{Q}_1 = \tpe(\widehat{Q}_0) \leq \tpe(\widehat{Q}^{\star,\ror}_{\mathsf{pe}}) = \widehat{Q}^{\star,\ror}_{\mathsf{pe}},
\end{align*}
where the inequality follows from $\widehat{Q}_0 =0 \leq \widehat{Q}^{\star,\ror}_{\mathsf{pe}}$.
Implementing the above result recursively gives
\begin{align*}
	\forall \; m \geq 0: \quad \widehat{Q}_m \leq \widehat{Q}^{\star,\ror}_{\mathsf{pe}}.
\end{align*}
Applying the $\gamma$-contraction property in Lemma~\ref{lem:contration-of-T} thus yields that for any $m\geq0$,
\begin{align}
	\| \widehat{Q}_m -\widehat{Q}^{\star,\ror}_{\mathsf{pe}}\|_\infty = \left\| \tpe(\widehat{Q}_{m-1}) - \tpe(\widehat{Q}^{\star,\ror}_{\mathsf{pe}}) \right\|_\infty & \leq \gamma \| \widehat{Q}_{m-1} -\widehat{Q}^{\star,\ror}_{\mathsf{pe}}\|_\infty \nonumber \\
	& \leq \cdots \leq \gamma^m \| \widehat{Q}_0 -\widehat{Q}^{\star,\ror}_{\mathsf{pe}}\|_\infty = \gamma^m \| \widehat{Q}^{\star,\ror}_{\mathsf{pe}}\|_\infty \leq \frac{\gamma^m}{1-\gamma}, \nonumber
\end{align}
where the last inequality holds by the fact $\| \widehat{Q}^{\star,\ror}_{\mathsf{pe}}\|_\infty \leq \frac{1}{1-\gamma}$ (see Lemma~\ref{lem:contration-of-T}).

%% file: lower-bound-analysis-infty.tex
%!TEX root = ./DRO_OfflineRL.tex

\subsection{Proof of Theorem~\ref{thm:dro-lower-infinite}}\label{proof:thm:dro-lower-infinite}

Similar to the finite-horizon case, we shall develop the lower bounds for the two cases when the uncertainty levels $\sigma$ vary separately.

\subsubsection{Construction of hard problem instances: small uncertainty level}
We first construct some hard discounted infinite-horizon RMDP instances and then characterize the sample complexity requirements over these instances.

\paragraph{Construction of a collection of hard MDPs.} Suppose there are two MDPs
\begin{align*}
   \left\{ \cM_\theta=
    \left(\mathcal{S}, \mathcal{A}, P^{\theta}, r, \gamma \right) 
    \mymid \theta = \{0,1\}
    \right\}.
\end{align*}
Here, $\cS = \{0, 1, \ldots, S-1\}$, and $\mathcal{A} = \{0, 1\}$.
The transition kernel $P^{\theta}$ of the MDP $\mathcal{M}^\theta$ is specified as follows:
\begin{align} \label{eq:Ph-construction-lower-infinite-theta}
P^{\theta}(s^{\prime} \mymid s, a) = \left\{ \begin{array}{lll}
         p\mathds{1}(s^{\prime} = 0) + (1-p)\mathds{1}(s^{\prime} = 1)  & \text{if} & (s, a) = (0, \theta) \\
         q\mathds{1}(s^{\prime} = 0) + (1-q)\mathds{1}(s^{\prime} = 1) & \text{if} & (s, a) = (0, 1-\theta) \\
         % \mathds{1}(s^{\prime} = s) & \text{if}   & s=1 \text{ or } s = 2\\ 
         q\mathds{1}(s^{\prime} = s) + (1-q)\mathds{1}(s^{\prime} = 1) & \text{if}   & s > 0 \
                \end{array}\right.,
\end{align}
for any $(s,a,s')\in \cS\times \cA\times \cS\times$, 
where $p$ and $q$ are set to be 
\begin{align}
  \label{eq:p-q-defn-infinite-theta}
  p = 1- c_1(1-\gamma) ,\qquad q = 1- c_1(1-\gamma) - c_2 (1-\gamma)^2 \varepsilon,
\end{align}
which satisfies
\begin{equation}
  \frac{2}{3}\leq \gamma < 1 \qquad \text{and} \qquad c_2 (1-\gamma)^2 \varepsilon \leq  \frac{c_1(1-\gamma)}{2} \leq \frac{1}{8}.
  %\varepsilon \leq \frac{\gamma}{28(1-\gamma)}. 
  \label{eq:infinite-epsilon-assumption-theta}
\end{equation}
for some $c_1 \leq \frac{1}{4}$ and small enough $c_2$.
In view of the assumptions~\eqref{eq:infinite-epsilon-assumption-theta}, one has
\begin{equation}
  p > q  \geq \frac{1}{2}.
  \label{eq:p-q-relation-LB-inf-theta}
\end{equation}

Finally, we define the reward function as
\begin{align}
r(s, a) = \left\{ \begin{array}{lll}
         1 & \text{if } s = 0 \text{ or } s = 2\\
         0 & \text{otherwise}  \ 
                \end{array}\right. .
        \label{eq:rh-construction-lower-bound-infinite-theta}
\end{align}

\paragraph{Construction of the history/batch dataset.} 
Define a useful state distribution (only supported on the state subset $\{0,1,2\}$) as
\begin{align}\label{infinite-mu-assumption-theta}
    \mu(s) = \frac{1}{CS}\mathds{1}(s = 0)  + \Big(1 - \frac{1}{CS}\Big)\mathds{1}(s = 1),
\end{align}
where $C>0$ is some constant that determines the robust concentrability coefficient $\Cstar$  (which will be made clear soon) and obeys
\begin{align}\label{eq:lower-C-assumption-infinite-theta}
    \frac{1}{CS} \leq \frac{1}{4}.
\end{align}

A batch dataset---consists of $N$ i.i.d samples $\{(s_i, a_i, s_i')\}_{1\leq i\leq N}$---is generated over the nominal environment $\cM_\theta$ according to \eqref{eq:infinite-batch-set-generation}, with the behavior distribution chosen to be: 
\begin{align}\label{eq:infinite-lower-bound-d-b-theta}
    \forall (s,a) \in \cS\times \cA: \quad d^{\mathsf{b}}(s, a) = \frac{\mu(s)}{2}.
\end{align}

Additionally, we choose the following initial state distribution:
\begin{align}
    \rho(s) = 
    \begin{cases} 1, \quad &\text{if }s=0 \\
        0, &\text{otherwise }
    \end{cases}.
    \label{eq:rho-defn-infinite-LB-theta}
\end{align}

\paragraph{Uncertainty set of the transition kernels.} We next describe the radius $\ror$ of the uncertainty set in our construction of the robust MDPs, along with some useful properties, which are similar to the finite-horizon case.
The perturbed transition kernels in $\cM_\theta$ is limited to the following uncertainty set
\begin{align}
    \unb^\ror(P^{\theta})\defn \otimes \; \unb^\ror \left(P^{\theta}_{s,a}\right),\qquad \unb^\ror(P_{s,a}^{\theta}) \defn \left\{ P_{s,a} \in \Delta (\cS): \mathsf{KL}\left(P_{s,a} \parallel P^{\theta}_{s,a}\right) \leq \ror \right\},
\end{align}
where $P_{s,a}^{\theta} \defn P^{\theta}(\cdot \mymid s,a) \in [0,1]^{1\times S}$. Moreover, the radius of the uncertainty set $\ror$ obeys
\begin{align}\label{eq:infinite-lower-ror-bounded-theta}
0  \leq \ror \leq \frac{1-\gamma}{20}.
\end{align}

For any $(s,a,s')\in\cS\times \cA \times \cS$, we denote the infimum entry of the perturbed transition kernel $P_{s,a} \in \unb^{\ror}(P^{\theta}_{s,a})$ moving to the next state $s'$ as
\begin{align}\label{eq:infinite-lw-def-p-q-theta}
\underline{P}^{\theta}(s' \mymid s,a) &\defn \inf_{P_{s,a} \in \unb^{\ror}(P^{\theta}_{s,a})} P(s'  \mymid s,a).
\end{align}
As shall be seen, the transition from state $0$ to state $2$ plays an important role in the analysis, for convenience, we denote
\begin{align}\label{eq:infinite-lw-p-q-perturb-inf-theta}
\underline{p}_\star &\defn \underline{P}^{\theta}(0 \mymid 0,\theta)  ,\qquad \underline{q}_\star  \defn \underline{P}^{\theta}(0  \mymid 0, 1-\theta).
\end{align}
With these definitions in place, we summarize some useful properties of the uncertainty set in the following lemma, which parallels Lemma~\ref{lem:finite-lb-uncertainty-set-KL} in the finite-horizon case.
\begin{lemma}\label{lem:infinite-lb-uncertainty-set-KL}
Suppose the uncertainty level $\ror$ satisfies \eqref{eq:infinite-lower-ror-bounded-theta}. The perturbed transition kernels obey
\begin{align}\label{eq:infinite-lw-upper-p-q-theta}
   \underline{p}_\star \geq \underline{q}_\star \geq 1 - c_3 (1-\gamma)  \quad \text{and} \quad \underline{p}_\star-\underline{q}_\star \geq p-q \geq 0 
\end{align}
for  constant $c_3 =2 c_1 \leq \frac{1}{4}$.
\end{lemma}
\begin{proof}
The proof follows from the same arguments as the proof for Lemma~\ref{lem:finite-lb-uncertainty-set-KL} in Appendix~\ref{proof:lem:finite-lb-uncertainty-set-KL} by replacing $H$ with $\frac{1}{1-\gamma}$; we omit the details for brevity.
\end{proof}

\paragraph{Value functions and optimal policies.}
Now we are positioned to derive the corresponding robust value functions and identify the optimal policies. For any MDP $\cM_\theta$ with the above uncertainty set, denote $\pi^{\star}_\theta$ as the optimal policy. In addition, we denote the robust value function of any policy $\pi$ (resp.~the optimal policy $\pi^{\star}_\theta$) as $V^{\pi,\ror}_\theta$ (resp.~$V^{\star,\ror}_\theta$). Then, we introduce the following lemma which describes some important properties of the robust value functions and optimal policies.

\begin{lemma}\label{lem:infinite-lb-value-theta}
For any $\theta = \{0,1\}$ and any policy $\pi$, one has
\begin{align}
    V^{\pi, \ror}_\theta(0) =  \frac{1}{1-\gamma x_{\theta}^{\pi}} ,
    \label{eq:infinite-lemma-value-0-pi-theta}
\end{align}
where $x_{\theta}^{\pi}$ is defined as
\begin{align}
x_{\theta}^{\pi} \defn \underline{p}_\star\pi(\theta\mymid 0) + \underline{q}_\star \pi(1-\theta \mymid 0).\label{eq:infinite-x-h-theta}
\end{align}
In addition, the optimal value functions and the optimal policies obey 
\begin{subequations}
    \label{eq:infinite-lb-value-lemma-theta}
\begin{align}
    V_\theta^{\star,\sigma}(0) &=  \frac{1}{1-\gamma \underline{p}_\star}, \quad V_\theta^{\star,\sigma}(s)= 0 \quad \text{ for } s =1 \text{ or } s>2, \\
    \pi_\theta^{\star }(\theta \mymid s) &= 1, \qquad  \qquad \text{ for } s \in\cS.
\end{align}
\end{subequations}
Moreover, the robust single-policy clipped concentrability coefficient $\Cstar$ obeys
\begin{align}
 \Cstar = 2C. \label{eq:expression-Cstar-LB-infinite-theta}
\end{align}
\end{lemma}
\begin{proof}
See Appendix~\ref{proof:lem:infinite-lb-value-theta}.
\end{proof}

\subsubsection{Establishing the minimax lower bound: small uncertainty level}
Towards this, we first introduce the following lemma, which parallels the claim in \eqref{eq:Delta-chosen}-\eqref{eq:finite-Value-0-recursive} in the finite-horizon case.
\begin{lemma} 
For any policy $\widehat{\pi}$,
\begin{align*}
    V^{\star,\ror}_\theta(0) - V^{\widehat{\pi},\ror}_\theta(0) \geq 2 \varepsilon \big(1- \widehat{\pi}(\theta \mymid 0) \big).
\end{align*}
\end{lemma}
\begin{proof}
This lemma can be directly verified by controlling $V^{\star,\ror}_\theta(0) - V^{\widehat{\pi},\ror}_\theta(0)$ with the help of Lemma~\ref{lem:infinite-lb-uncertainty-set-KL} and Lemma~\ref{lem:infinite-lb-value-theta}; we omit the details for brevity.
\end{proof}

Armed with this lemma, following the same arguments in Appendix~\ref{eq:proof-finite-lower-bound}, we can complete the proof by observing that: let $c_1$ be some sufficient large constant, as long as the sample size is beneath 
\begin{align}
    N \leq \frac{c_4 S \Cstar }{(1-\gamma)^3 \varepsilon^2},
\end{align}
then we necessarily have
\begin{align}
    \inf_{\widehat{\pi}} \max_{\theta \in \{0,1\}}\mathbb{P}_\theta \left\{ V^{\star,\ror}_\theta(\rho) - V^{\widehat{\pi},\ror}_\theta (\rho)  \geq \varepsilon  \right\} \ge \frac{1}{8},
\end{align}
where  $\mathbb{P}_\theta$ denote the probability conditioned on that the MDP is $\cM_\theta$. We omit the details for brevity and complete the proof.

%%%%%--------%%%%%--------%%%%%--------%%%%%--------%%%%%--------

\subsubsection{Construction of hard problem instances: large uncertainty level}

\paragraph{Construction of a collection of hard MDPs.} Suppose there are two MDPs
\begin{align*}
   \left\{ \cM_\phi=
    \left(\mathcal{S}, \mathcal{A}, P^{\phi}, r, \gamma \right) 
    \mymid \phi = \{0,1\}
    \right\}.
\end{align*}
Here, $\gamma$ is the discount parameter, $\cS = \{0, 1, \ldots, S-1\}$ is the state space, and $\mathcal{A} = \{0, 1\}$ is the action space. The transition kernel $P^\phi$ of either constructed MDP $\cM_\phi$ is defined as
\begin{align} \label{eq:Ph-construction-lower-infinite}
P^{\phi}(s^{\prime} \mymid s, a) = \left\{ \begin{array}{lll}
         p\mathds{1}(s^{\prime} = 2) + (1-p)\mathds{1}(s^{\prime} = 1)  & \text{if} & (s, a) = (0, \phi) \\
         q\mathds{1}(s^{\prime} = 2) + (1-q)\mathds{1}(s^{\prime} = 1) & \text{if} & (s, a) = (0, 1-\phi) \\
         \mathds{1}(s^{\prime} = s) & \text{if}   & s=1 \text{ or } s = 2\\ 
         q\mathds{1}(s^{\prime} = s) + (1-q)\mathds{1}(s^{\prime} = 1) & \text{if}   & s > 2 \
                \end{array}\right.,
\end{align}
where $p$ and $q$ are set as
\begin{align}\label{eq:p-q-defn-infinite}
    p =  1- \alpha  \quad \text{ and } \quad q = 1 - \alpha - \Delta
\end{align}
for some $\gamma$, $\alpha$ and $\Delta$ obeying  
\begin{align}\label{eq:lower-p-q-beta-c1-infinite}
    0< \alpha \leq 1-\gamma \leq 1/(2e^8) \leq \frac{1}{2} \quad \text{ and }  \quad \Delta \leq \frac{\alpha}{2}.
\end{align}
Here, $\alpha$ and $\Delta$ are some values that will be introduced later. Consequently, applying \eqref{eq:p-q-defn-infinite} directly leads to
\begin{align}\label{eq:infinite-p-q-bound}
    1 \geq p \geq q \geq \gamma \geq \frac{1}{2}.
\end{align}
Note that state $1$ and $2$ are absorbing states. In addition, if the initial distribution is supported on states $\{0,1,2\}$, the MDP will always stay in the state $\{1,2\}$ after the first transition.

Finally, we define the reward function as
\begin{align}
r(s, a) = \left\{ \begin{array}{lll}
         1 & \text{if } s = 0 \text{ or } s = 2\\
         0 & \text{otherwise}  \ 
                \end{array}\right. .
        \label{eq:rh-construction-lower-bound-infinite}
\end{align}

\paragraph{Construction of the history/batch dataset.} 
Define a useful state distribution (only supported on the state subset $\{0,1,2\}$) as
\begin{align}\label{infinite-mu-assumption}
    \mu(s) = \frac{1}{CS}\mathds{1}(s = 0) + \frac{1}{CS}\mathds{1}(s = 2) + \Big(1 - \frac{2}{CS}\Big)\mathds{1}(s = 1),
\end{align}
where $C>0$ is some constant that determines the robust concentrability coefficient $\Cstar$  (which will be made clear soon) and obeys
\begin{align}\label{eq:lower-C-assumption-infinite}
    \frac{1}{CS} \leq \frac{1}{4}.
\end{align}

A batch dataset---consists of $N$ i.i.d samples $\{(s_i, a_i, s_i')\}_{1\leq i\leq N}$---is generated over the nominal environment $\cM_\phi$ according to \eqref{eq:infinite-batch-set-generation}, with the behavior distribution chosen to be: 
\begin{align}\label{eq:infinite-lower-bound-d-b}
    \forall (s,a) \in \cS\times \cA: \quad d^{\mathsf{b}}(s, a) = \frac{\mu(s)}{2}.
\end{align}

Additionally, we choose the following initial state distribution:
\begin{align}
    \rho(s) = 
    \begin{cases} 1, \quad &\text{if }s=0 \\
        0, &\text{otherwise }
    \end{cases}.
    \label{eq:rho-defn-infinite-LB}
\end{align}

\paragraph{Uncertainty set of the transition kernels.} We next describe the radius $\ror$ of the uncertainty set in our construction of the robust MDPs, along with some useful properties, which are similar to the finite-horizon case.
To begin with, with slight abuse of notation, we introduce an important constant $\beta$ defined as
\begin{align}\label{eq:lower-bound-H-assumption-infinite}
    \beta \defn \frac{1}{2} \log \frac{1}{\alpha + \Delta} \geq 4.
\end{align}
The perturbed transition kernels in $\cM_\phi$ is limited to the following uncertainty set
\begin{align}
    \unb^\ror(P^{\phi})\defn \otimes \; \unb^\ror \left(P^{\phi}_{s,a}\right),\qquad \unb^\ror(P_{s,a}^{\phi}) \defn \left\{ P_{s,a} \in \Delta (\cS): \mathsf{KL}\left(P_{s,a} \parallel P^{\phi}_{s,a}\right) \leq \ror \right\},
\end{align}
where $P_{s,a}^{\phi} \defn P^{\phi}(\cdot \mymid s,a) \in [0,1]^{1\times S}$. Moreover, the radius of the uncertainty set $\ror$ obeys
\begin{align}\label{eq:infinite-lower-ror-bounded}
\left(1- \frac{3}{\beta}\right)\log\frac{1}{\alpha + \Delta} \leq \ror \leq \left(1- \frac{2}{\beta}\right)\log\frac{1}{\alpha + \Delta} .
\end{align}

For any $(s,a,s')\in\cS\times \cA \times \cS$, we denote the infimum entry of the perturbed transition kernel $P_{s,a} \in \unb^{\ror}(P^{\phi}_{s,a})$ moving to the next state $s'$ as
\begin{align}\label{eq:infinite-lw-def-p-q}
\underline{P}^{\phi}(s' \mymid s,a) &\defn \inf_{P_{s,a} \in \unb^{\ror}(P^{\phi}_{s,a})} P(s'  \mymid s,a).
\end{align}
As shall be seen, the transition from state $0$ to state $2$ plays an important role in the analysis, for convenience, we denote
\begin{align}\label{eq:infinite-lw-p-q-perturb-inf}
\underline{p} &\defn \underline{P}^{\phi}(2 \mymid 0,\phi)  ,\qquad \underline{q}  \defn \underline{P}^{\phi}(2  \mymid 0, 1-\phi).
\end{align}
With these definitions in place, we summarize some useful properties of the uncertainty set in the following lemma, which parallels Lemma~\ref{lem:finite-lb-uncertainty-set} in the finite-horizon case.
\begin{lemma}\label{lem:infinite-lb-uncertainty-set}
Suppose $\beta$ satisfies \eqref{eq:lower-bound-H-assumption-infinite} and the uncertainty level $\ror$ satisfies \eqref{eq:infinite-lower-ror-bounded}. The perturbed transition kernels obey
\begin{align}\label{eq:infinite-lw-upper-p-q}
   \underline{p}\geq \underline{q} \geq \frac{1}{\beta}.
\end{align}
\end{lemma}
\begin{proof}
The proof follows from the same arguments as Appendix~\ref{proof:lem:finite-lb-uncertainty-set} by replacing $H$ with $\frac{1}{1-\gamma}$; we omit the details for brevity.
\end{proof}

\paragraph{Value functions and optimal policies.}
Now we are positioned to derive the corresponding robust value functions and identify the optimal policies. For any MDP $\cM_\phi$ with the above uncertainty set, denote $\pi^{\star}_\phi$ as the optimal policy. In addition, we denote the robust value function of any policy $\pi$ (resp.~the optimal policy $\pi^{\star}_\phi$) as $V^{\pi,\ror}_\phi$ (resp.~$V^{\star,\ror}_\phi$). Then, we introduce the following lemma which describes some important properties of the robust value functions and optimal policies.

\begin{lemma}\label{lem:infinite-lb-value}
For any $\phi = \{0,1\}$ and any policy $\pi$, one has
\begin{align}
    V^{\pi, \ror}_\phi(0) =  1 +  \frac{\gamma}{1-\gamma} z_{\phi}^{\pi},
    \label{eq:infinite-lemma-value-0-pi}
\end{align}
where $z_{\phi}^{\pi}$ is defined as
\begin{align}
z_{\phi}^{\pi} \defn \underline{p}\pi(\phi\mymid 0) + \underline{q} \pi(1-\phi \mymid 0).\label{eq:infinite-x-h}
\end{align}
In addition, the optimal value functions and the optimal policies obey 
\begin{subequations}
    \label{eq:infinite-lb-value-lemma}
\begin{align}
    V_\phi^{\star,\sigma}(0) &= 1 +   \frac{\gamma}{1-\gamma} \underline{p}, \quad V_\phi^{\star,\sigma}(2) = \frac{1}{1-\gamma}, \quad V_\phi^{\star,\sigma}(s)= 0 \quad \text{ for } s =1 \text{ or } s>2, \\
    \pi_\phi^{\star }(\phi \mymid s) &= 1, \qquad  \qquad \text{ for } s \in\cS.
\end{align}
\end{subequations}
Moreover, choosing $S\geq 2\beta$, the robust single-policy clipped concentrability coefficient $\Cstar$ obeys
\begin{align}
 \Cstar = 2C. \label{eq:expression-Cstar-LB-infinite}
\end{align}
\end{lemma}
\begin{proof}
See Appendix~\ref{proof:lem:infinite-lb-value}.
\end{proof}

% Armed with Lemma~\ref{lem:infinite-lb-value}, it is easily observed that $\minpall = \alpha \in \left(0,1-\gamma\right]$.
\subsubsection{Establishing the minimax lower bound: large uncertainty level}
Now we are positioned to provide the sample complexity lower bound. In view of Lemma~\ref{lem:infinite-lb-value}, the smallest positive state transition probability of the optimal policy $\pi^\star_\phi$ under any nominal transition kernel $P^{\phi}$ with $\phi \in \{0,1\}$ satisfies:
\begin{align}\label{eq:P-min-phi-infinite}
\minpall \defn \min_{s,s'} \Big\{P^{\phi}\left(s' \mymid s, \pi^{\star}_\phi(s)\right):\; P^{\phi}\left(s' \mymid s, \pi^{\star}_\phi(s)\right)>0 \Big\} = P^{\phi}\left( 1 |0, \phi \right) = 1- p.
\end{align}

Our goal is to control the quantity w.r.t. any policy estimator $\widehat{\pi}$ based on the batch dataset and the chosen initial distribution $\rho$ in \eqref{eq:rho-defn-infinite-LB}, which gives
\begin{align}
  V^{\star,\ror}_\phi(\rho) - V^{\widehat{\pi},\ror}_\phi(\rho)   = V^{\star,\ror}_\phi(0) - V^{\widehat{\pi},\ror}_\phi(0).
\end{align}
Towards this, we first introduce the following lemma, which parallels the claim in \eqref{eq:Delta-chosen}-\eqref{eq:finite-Value-0-recursive} in the finite-horizon case.
\begin{lemma} 
Given $\varepsilon \leq \frac{1}{384e^6(1-\gamma) \log\left(\frac{1}{\alpha}\right)} \leq \frac{1}{384e^6(1-\gamma) \log\left(\frac{1}{\alpha +\Delta}\right)} $, choosing
$\Delta = 128 e^6 \ror (1-q) \varepsilon(1-\gamma)  \leq 128  e^6   (\alpha +\Delta) \varepsilon \log\left(\frac{1}{\alpha +\Delta}\right) (1-\gamma)  \leq \frac{\alpha}{2} $, 
one has  for any policy $\widehat{\pi}$,
\begin{align*}
    V^{\star,\ror}_\phi(0) - V^{\widehat{\pi},\ror}_\phi(0) \geq 2 \varepsilon \big(1- \widehat{\pi}(\phi \mymid 0) \big).
\end{align*}
\end{lemma}
\begin{proof}
This lemma  follows from the same arguments as Appendix~\ref{proof:finite-lower-diff-control} except replacing $H$ with $\frac{1}{1-\gamma}$ under the additional condition $\gamma \geq\frac{1}{2}$; we omit the details for brevity.
\end{proof}

Armed with this lemma, following the same arguments in Appendix~\ref{eq:proof-finite-lower-bound}, we can complete the proof by observing that: let $c_1$ be some sufficient large constant, as long as the sample size is beneath 
\begin{align}
    N \leq \frac{S \Cstar \log 2}{4c_1 \minpall \ror^2 (1-\gamma)^2 \varepsilon^2},
\end{align}
then we necessarily have
\begin{align}
    \inf_{\widehat{\pi}} \max_{\phi \in \{0,1\}}\mathbb{P}_\phi \left\{ V^{\star,\ror}_\phi(\rho) - V^{\widehat{\pi},\ror}_\phi (\rho)  \geq \varepsilon  \right\} \ge \frac{1}{8},
\end{align}
where  $\mathbb{P}_\phi$ denote the probability conditioned on that the MDP is $\cM_\phi$. We omit the details for brevity and complete the proof.

\subsubsection{Proof of Lemma~\ref{lem:infinite-lb-value-theta}}\label{proof:lem:infinite-lb-value-theta}

First, it is easily verified that for any policy $\pi$,
\begin{align}
  \forall s\in\cS\setminus \{0\}:\quad   V_{\theta}^{\pi,\ror}(s) &= \sum_{t=0}^\infty \gamma^t \cdot 0 = 0
\end{align}
since the reward function $r(s,a) =0$ for $(s,a) \in\cS\setminus \{0\} \times \cA$ in \eqref{eq:rh-construction-lower-bound-infinite-theta}.

To continue, we observe that the robust value function of state $0$ satisfies
\begin{align}
    V_\theta^{\pi,\ror}(0) &= \mathbb{E}_{a \sim \pi(\cdot \mymid 0)} \left[ r(0,a) + \gamma \inf_{ \cP \in \unb^{\sigma}(P^{\theta}_{0,a})}  \cP V^{\pi,\sigma}_\theta\right] \notag \\
    & \overset{\mathrm{(i)}}{=} 1 +  \gamma\pi(\theta \mymid 0) \inf_{ \cP \in \unb^{\sigma}(P^{\theta}_{0,\theta})}  \cP V^{\pi,\sigma}_\theta +  \gamma\pi(1 - \theta \mymid 0)  \inf_{ \cP \in \unb^{\sigma}(P^{\theta}_{0, 1- \theta})}  \cP V^{\pi,\sigma}_{\theta} \label{eq:infinite-theta-order} \\
    & \overset{\mathrm{(ii)}}{=} 1 + \gamma \pi(\theta \mymid 0)\Big[ \underline{p}_\star V_\theta^{\pi,\sigma}(0) + \left(1- \underline{p}_\star\right) V_\theta^{\pi,\sigma}(1)\Big]  +  \gamma \pi(1-\theta \mymid 0)\Big[ \underline{q}_\star V_{\theta}^{\pi,\sigma}(0) + \left(1-\underline{q}_\star \right) V_{\theta}^{\pi,\sigma}(1)\Big] \notag\\
    & \overset{\mathrm{(iii)}}{=} 1 +  \gamma x_{\theta}^{\pi}  \left[V_{\theta}^{\pi,\sigma}(0) - V_{\theta}^{\pi,\sigma}(1) \right] \notag\\
    & = \frac{1}{1- \gamma x_{\theta}^{\pi}}, \label{eq:interemediate_opt_10-infinite-theta}
\end{align}
where (i) holds by the reward function defined in \eqref{eq:rh-construction-lower-bound-infinite-theta}. To see (ii), note that  \eqref{eq:infinite-theta-order} indicates $V_\theta^{\pi,\ror}(0) \geq 1 \geq V_\theta^{\pi,\ror}(1) = 0$, so that the infimum is obtained by picking the smallest possible mass on the transition to state $2$, provided by the definition in \eqref{eq:infinite-lw-p-q-perturb-inf-theta}, and (iii) follows by plugging in the definition of $x_{\theta}^{\pi}$ in \eqref{eq:infinite-x-h-theta}.

Consequently, observing that the function $\frac{1}{1-\gamma x}$ is increasing in $x$ and $x_{\theta}^{\pi}$ is also increasing in $\pi(\theta \mymid 0)$ (see the fact $\underline{p}_\star\geq \underline{q}_\star$ in \eqref{eq:infinite-lw-upper-p-q-theta}),  the optimal policy in state $0$ thus obeys
\begin{equation}
    \pi_\theta^{\star}(\theta \mymid 0) = 1 \label{eq:infinite-lb-optimal-policy-theta}.
\end{equation}
Therefore,
\begin{align}
    x_{\theta}^{\star} \defn x_{\theta}^{\pi^\star} =\underline{p}_\star\pi^\star_\theta(\theta\mymid 0) + \underline{q}_\star \pi^\star_\theta(1-\theta \mymid 0) = \underline{p}_\star,
\end{align}
which combined with \eqref{eq:interemediate_opt_10-infinite-theta} yields 
\begin{align}
     V_\theta^{\star,\ror}(0) = \frac{1}{1- \gamma \underline{p}_\star}  . 
\end{align}
Regarding the optimal policy for the remaining states $s>0$, since the action does not influence the state transition, without loss of generality, we choose the optimal policy to obey
\begin{align}\label{eq:infinite-lower-optimal-pi}
    \forall s>0: \quad \pi_\theta^\star(\theta\mymid s) = 1.
\end{align}

\paragraph{Proof of \eqref{eq:expression-Cstar-LB-infinite-theta}.}
To begin with, for any MDP $\cM_\theta$ with $\theta\in\{0,1\}$, recall the definition of $\Cstar$ as
\begin{align}
   \Cstar =  \max_{(s, a, P) \in \mathcal{S} \times \cA \times \unb^{\ror}(P^{\theta})} \frac{\min\big\{d^{\star,P}(s, a ), \frac{1}{S}\big\}}{d^{\mathsf{b}}(s, a )}.
\end{align}
Given $\pi_\theta^\star(\theta\mymid s) = 1$ for all $s\in \cS$ and the initial distribution $\rho(0) = 1$, for any $P\in \unb^{\ror}(P^{\theta})$,  we arrive at
\begin{align}
    d^{\star,P}(0,\theta) &= (1-\gamma) \rho(0) \sum_{t=0}^\infty \gamma^t \left(\underline{P}^{\theta}(0 \mymid 0, \theta)  \right)^t \notag \\
    & \overset{\mathrm{(i)}}{=} (1-\gamma)  \sum_{t=0}^\infty \gamma^t \underline{p}_\star^t = \frac{1-\gamma}{ 1 - \gamma \underline{p}_\star} \overset{\mathrm{(ii)}}{\geq} \frac{1-\gamma}{ 1 - \gamma (1- c_3 (1-\gamma))} \geq \frac{1}{2}. \label{eq:d-2-phi-upper-1-beta-theta}
\end{align}
where (i) holds by \eqref{eq:infinite-lw-p-q-perturb-inf-theta} and (ii) follows from \eqref{eq:infinite-lw-upper-p-q-theta}. In addition, we have
\begin{align}
  d^{\star,P}(0, 1-\theta) =0 \quad \text{and} \quad  \forall s>1: \quad d^{\star,P}(s, a) = 0,
\end{align}
since $\rho(0) =1$ and state $0$ and $1$ are absorbing states for all policy and and all $P\in \unb^{\ror}(P^{\theta})$.

Armed with the above facts, we observe that  
\begin{align}\label{eq:state-012-infinite-lower-theta}
   \max_{(s, a, P) \in \mathcal{S} \times \cA \times \unb^{\ror}(P^{\theta})} \frac{\min\big\{d^{\star,P}(s, a ), \frac{1}{S}\big\}}{d^{\mathsf{b}}(s, a )} =\max_{s\in \{0,1\}, P \in \unb^{\ror}(P^{\theta})} \frac{\min\big\{d^{\star,P}(s, \theta), \frac{1}{S}\big\}}{d^{\mathsf{b}}(s, \theta)} 
\end{align}
which follows from the properties of the optimal policy in \eqref{eq:infinite-lb-value-lemma-theta}.

Consequently, we control $\Cstar$ in states separately: 
\begin{subequations}
\begin{align}
    \max_{ P \in \unb^{\ror}(P^{\theta})} \frac{\min\big\{d^{\star,P}(0, \theta), \frac{1}{S}\big\}}{d^{\mathsf{b}}(0, \theta)} &\overset{\mathrm{(i)}}{=} \frac{1}{Sd^{\mathsf{b}}(0,\theta)} \overset{\mathrm{(ii)}}{=} \frac{2}{S \mu(0)} = 2C, \\
     \max_{ P \in \unb^{\ror}(P^{\theta})} \frac{\min\big\{d^{\star,P}(1, \theta), \frac{1}{S}\big\}}{d^{\mathsf{b}}(1, \theta)} &\leq \frac{1}{Sd^{\mathsf{b}} (1,\theta)} \overset{\mathrm{(iii)}}{=}  \frac{2}{S \left(1 -\frac{1}{CS}\right)}  \overset{\mathrm{(iv)}}{\leq}  \frac{4}{S}  \overset{\mathrm{(v)}}{\leq} C, \label{eq:infinite-cstar-bound-3}
\end{align}
\end{subequations}
where (i) holds by \eqref{eq:d-2-phi-upper-1-beta-theta} and $S \geq 2 $, (ii) and (iii) follow from the definitions in \eqref{eq:infinite-lower-bound-d-b-theta} and \eqref{infinite-mu-assumption-theta}, and (iv) and (v) and arise from the assumption in \eqref{eq:lower-C-assumption-infinite-theta}.
Plugging the above results back into \eqref{eq:state-012-infinite-lower-theta} directly completes the proof of
\begin{align*}
    \Cstar = \max_{(s, a, P) \in \mathcal{S} \times \cA \times \unb^{\ror}(P^{\theta})} \frac{\min\big\{d^{\star,P}(s, a ), \frac{1}{S}\big\}}{d^{\mathsf{b}}(s, a )} = 2C.
\end{align*}

\subsubsection{Proof of Lemma~\ref{lem:infinite-lb-value}}\label{proof:lem:infinite-lb-value}
For any $\cM_\phi$ with $\phi\in\{0,1 \}$, we first characterize the robust value function for any policy $\pi$ over different states.  due to state absorbing, the uncertainty set becomes a singleton containing the nominal distribution at state $s=1$ and $s=2$. It is easily observed that for any policy $\pi$,
the robust value functions at state $s=1$ and $s=2$ obey
\begin{subequations}\label{eq:o-larger-than-1-infinite}
\begin{align}
  V_\phi^{\pi,\ror}(1) &= \sum_{t=0}^\infty \gamma^t \cdot 0 = 0 ,\\
    V_\phi^{\pi,\ror}(2) &= \sum_{t=0}^\infty \gamma^t \cdot 1 = \frac{1}{1-\gamma} , 
\end{align}
since $r(1,a)=0$ and $r(2,a)=1$.
In addition, for state $s>2$, the perturbed transition kernel is supported on itself and state $1$, both of which receive a reward of $0$ by design \eqref{eq:rh-construction-lower-bound-infinite}, leading to
\begin{align} 
    V_\phi^{\pi,\ror}(s) &= \sum_{t=0}^\infty \gamma^t \cdot 0 = 0, \qquad \text{for }   s>2.
\end{align}
\end{subequations}
Moving onto the remaining states, the robust value function of state $0$ satisfies
\begin{align}
    V_\phi^{\pi,\ror}(0) &= \mathbb{E}_{a \sim \pi(\cdot \mymid 0)} \left[ r(0,a) + \gamma \inf_{ \cP \in \unb^{\sigma}(P^{\phi}_{0,a})}  \cP V^{\pi,\sigma}_\phi\right] \notag \\
    & \overset{\mathrm{(i)}}{=} 1 +  \gamma\pi(\phi \mymid 0) \inf_{ \cP \in \unb^{\sigma}(P^{\phi}_{0,\phi})}  \cP V^{\pi,\sigma}_\phi +  \gamma\pi(1 - \phi \mymid 0)  \inf_{ \cP \in \unb^{\sigma}(P^{\phi}_{0, 1- \phi})}  \cP V^{\pi,\sigma}_{\phi} \notag \\
    & \overset{\mathrm{(ii)}}{=} 1 + \gamma \pi(\phi \mymid 0)\Big[ \underline{p} V_\phi^{\pi,\sigma}(2) + \left(1- \underline{p}\right) V_\phi^{\pi,\sigma}(1)\Big]  +  \gamma \pi(1-\phi \mymid 0)\Big[ \underline{q} V_{\phi}^{\pi,\sigma}(2) + \left(1-\underline{q} \right) V_{\phi}^{\pi,\sigma}(1)\Big] \notag\\
    & \overset{\mathrm{(iii)}}{=} 1 +  \gamma V_{\phi}^{\pi,\sigma}(1) +  \gamma z_{\phi}^{\pi}  \left[V_{\phi}^{\pi,\sigma}(2) - V_{\phi}^{\pi,\sigma}(1) \right] \notag\\
    & = 1 +  \frac{\gamma}{1-\gamma}   z_{\phi}^{\pi}  ,
\end{align}
where (i) holds by the reward function defined in \eqref{eq:rh-construction-lower-bound-infinite}. To see (ii), note that  \eqref{eq:o-larger-than-1-infinite} indicates $V_\phi^{\pi,\ror}(2) \geq V_\phi^{\pi,\ror}(1)$, so that the infimum is obtained by picking the smallest possible mass on the transition to state $2$, provided by the definition in \eqref{eq:infinite-lw-p-q-perturb-inf}. Last but not least, (iii) follows by plugging in the definition of $z_{\phi}^{\pi}$ in \eqref{eq:infinite-x-h}, and the last identity is due to \eqref{eq:o-larger-than-1-infinite}.
Consequently, taking $\pi = \pi^\star_\phi$, we directly arrive at
\begin{align}\label{eq:interemediate_opt_10-infinite}
    V_\phi^{\star,\ror}(0)    = 1 +  \frac{\gamma}{1-\gamma}   z_{\phi}^{\pi^\star}  .
\end{align}
Observing that the function $z\frac{\gamma}{1-\gamma}$ is increasing in $z$ and $z_{\phi}^{\pi}$ is also increasing in $\pi(\phi \mymid 0)$ (see the fact $\underline{p}\geq \underline{q}$ in \eqref{eq:infinite-lw-upper-p-q}),  the optimal policy in state $0$ thus obeys
\begin{equation}
    \pi_\phi^{\star}(\phi \mymid 0) = 1 \label{eq:infinite-lb-optimal-policy}.
\end{equation}
Finally, plugging the above fact back into \eqref{eq:infinite-x-h} leads to
\begin{align}
    z_{\phi}^{\star} \defn z_{\phi}^{\pi^\star} =\underline{p}\pi^\star_\phi(\phi\mymid 0) + \underline{q} \pi^\star_\phi(1-\phi \mymid 0) = \underline{p},
\end{align}
which combined with \eqref{eq:interemediate_opt_10-infinite} yields 
\begin{align}
     V_\phi^{\star,\ror}(0) = 1 +\frac{\gamma}{1-\gamma}   \underline{p} . 
\end{align}
Regarding the optimal policy for the remaining states $s>0$, since the action does not influence the state transition, without loss of generality, we choose the optimal policy to obey
\begin{align}\label{eq:infinite-lower-optimal-pi}
    \forall s>0: \quad \pi_\phi^\star(\phi\mymid s) = 1.
\end{align}

\paragraph{Proof of \eqref{eq:expression-Cstar-LB-infinite}.}
To begin with, for any MDP $\cM_\phi$ with $\phi\in\{0,1\}$, recall the definition of $\Cstar$ as
\begin{align}
   \Cstar =  \max_{(s, a, P) \in \mathcal{S} \times \cA \times \unb^{\ror}(P^{\phi})} \frac{\min\big\{d^{\star,P}(s, a ), \frac{1}{S}\big\}}{d^{\mathsf{b}}(s, a )}.
\end{align}
Given $\pi_\phi^\star(\phi\mymid s) = 1$ for all $s\in \cS$ and the initial distribution $\rho(0) = 1$, for any $P\in \unb^{\ror}(P^{\phi})$,  we arrive at
\begin{align}
    d^{\star,P}(0,\phi) = (1-\gamma) \rho(0) \pi_\phi^\star(\phi\mymid 0) = (1-\gamma),
\end{align}
which holds due to that the agent transits from state $0$ to other states at the first step and then will never go back to state $0$.
In addition, one has for any $P\in \unb^{\ror}(P^{\phi})$,  
\begin{align}\label{eq:d-2-phi-upper-1-beta}
    d^{\star,P}(2,\phi) &= (1-\gamma)  P(2 \mymid 0, \phi)
    \sum_{t=1}^{\infty} \gamma^t \big(P(2 \mymid 2, \phi) \big)^t \notag \\
    & =(1-\gamma)P(2 \mymid 0, \phi) \sum_{t=1}^\infty \gamma^t \overset{\mathrm{(i)}}{\geq} \gamma \underline{p} \geq   \frac{1}{2\beta},
    \end{align}
where (i) holds by \eqref{eq:infinite-lw-p-q-perturb-inf} and the final inequality follows from \eqref{eq:infinite-lw-upper-p-q} and $\gamma \geq 1/2$. 
Armed with the above facts, we observe that  
\begin{align}\label{eq:state-012-infinite-lower}
   \max_{(s, a, P) \in \mathcal{S} \times \cA \times \unb^{\ror}(P^{\phi})} \frac{\min\big\{d^{\star,P}(s, a ), \frac{1}{S}\big\}}{d^{\mathsf{b}}(s, a )} =\max_{s\in \{0,1,2\}, P \in \unb^{\ror}(P^{\phi})} \frac{\min\big\{d^{\star,P}(s, \phi), \frac{1}{S}\big\}}{d^{\mathsf{b}}(s, \phi)} 
\end{align}
which follows from the properties of the optimal policy in \eqref{eq:infinite-lower-optimal-pi} and consequently $d^{\star,P}(s) = d^{\star,P}(s,\phi) =0$ for all $s>2$ and all $P\in \unb^{\ror}(P^{\phi})$.

To continue, we control the term in states $\{0,1,2\}$ separately: 
\begin{subequations}
\begin{align}
    \max_{ P \in \unb^{\ror}(P^{\phi})} \frac{\min\big\{d^{\star,P}(2, \phi), \frac{1}{S}\big\}}{d^{\mathsf{b}}(2, \phi)} &\overset{\mathrm{(i)}}{=} \frac{1}{Sd^{\mathsf{b}}(2,\phi)} \overset{\mathrm{(ii)}}{=} \frac{2}{S \mu(2)} = 2C, \\
    \max_{ P \in \unb^{\ror}(P^{\phi})} \frac{\min\big\{d^{\star,P}(0, \phi), \frac{1}{S}\big\}}{d^{\mathsf{b}}(0, \phi)} &\leq \frac{1}{S d^{\mathsf{b}}(0,\phi)} \overset{\mathrm{(iii)}}{=} \frac{2}{S \mu(0)} = 2C, \\
     \max_{ P \in \unb^{\ror}(P^{\phi})} \frac{\min\big\{d^{\star,P}(1, \phi), \frac{1}{S}\big\}}{d^{\mathsf{b}}(1, \phi)} &\leq \frac{1}{Sd^{\mathsf{b}} (1,\phi)} \overset{\mathrm{(iv)}}{=}  \frac{2}{S \left(1 -\frac{2}{CS}\right)}  \overset{\mathrm{(v)}}{\leq}  \frac{4}{S}  \overset{\mathrm{(vi)}}{\leq} C, \label{eq:infinite-cstar-bound-3}
\end{align}
\end{subequations}
where (i) holds by \eqref{eq:d-2-phi-upper-1-beta} and $S \geq 2 \beta$, (ii), (iii) and (iv) follow from the definitions in \eqref{eq:infinite-lower-bound-d-b} and \eqref{infinite-mu-assumption}, (v) and (vi) arise from the assumption in \eqref{eq:lower-C-assumption-infinite}.
Plugging the above results back into \eqref{eq:state-012-infinite-lower} directly completes the proof of
\begin{align*}
    \Cstar = \max_{(s, a, P) \in \mathcal{S} \times \cA \times \unb^{\ror}(P^{\phi})} \frac{\min\big\{d^{\star,P}(s, a ), \frac{1}{S}\big\}}{d^{\mathsf{b}}(s, a )} = 2C.
\end{align*}